\documentclass{article}

\usepackage[preprint]{neurips_2026}

\usepackage[utf8]{inputenc}
\usepackage[T1]{fontenc}
\usepackage{microtype}
\usepackage{graphicx}
\usepackage{subcaption}
\usepackage{booktabs}
\usepackage{multirow}
\usepackage{enumitem}
\usepackage{colortbl}
\usepackage{xcolor}
\usepackage{array}
\usepackage{hyperref}
\usepackage{url}

\usepackage{amsmath}
\usepackage{amssymb}
\usepackage{amsfonts}
\usepackage{mathtools}
\usepackage{amsthm}
\usepackage{bm}

\usepackage{tikz}
\usepackage{pgfplots}
\pgfplotsset{compat=1.17}
\usepgfplotslibrary{groupplots}
\usetikzlibrary{arrows.meta, positioning, shapes.geometric, calc, decorations.pathreplacing, patterns}

\usepackage[capitalize,noabbrev]{cleveref}

\theoremstyle{plain}
\newtheorem{theorem}{Theorem}[section]
\newtheorem{proposition}[theorem]{Proposition}
\newtheorem{lemma}[theorem]{Lemma}

\theoremstyle{definition}
\newtheorem{definition}[theorem]{Definition}

\theoremstyle{remark}

\newcommand{\miniIN}{\emph{mini}-ImageNet}

\title{When Does Embedding Magnitude Matter? \\
A Cross-Task Functional-Symmetry Framework}

\author{%
  Xincan Feng\\
  Nara Institute of Science and Technology, Japan\\
  \texttt{feng.xincan.fy2@is.naist.jp}
  \And
  Taro Watanabe\\
  Nara Institute of Science and Technology, Japan\\
  \texttt{taro@is.naist.jp}
}

\begin{document}

\maketitle

\begin{abstract}
Cosine similarity normalizes both sides; dot product normalizes neither. We propose a 2$\times$2 framework that independently controls query-side and document-side normalization, exposing two intermediate variants (\textsc{QNorm}, \textsc{DNorm}) that have not been previously studied.
On retrieval with four encoders, evaluated in-domain on MS MARCO and out-of-domain on BEIR, BRIGHT, and multi-hop QA, the unilateral variants outperform \emph{both} cosine and dot product, with relative gains of up to +72\% out-of-domain and +24\% on downstream RAG.
Cross-evaluation reveals the mechanism: document magnitude scales inference scores while query magnitude modulates training gradients, and the Fisher Information Matrix condition number predicts which side to normalize.
We then classify tasks by \emph{functional symmetry}, defined as whether the aggregate scoring procedure treats $Q$ and $C$ as interchangeable, and test whether the mechanism extends beyond retrieval.
On five additional task families (semantic textual similarity, CLIP, knowledge graph completion, few-shot classification, recommender systems), the coarse prediction (cosine for symmetric, magnitude-preserving for asymmetric) holds in every case examined; the unilateral variants beat Cosine on recommendation, and on few-shot classification \textsc{DNorm} beats both Cosine and the standard Euclidean default of Prototypical Networks.
\end{abstract}


\section{Introduction}

Many systems compare two inputs through a learned similarity function $s(\bm{q}, \bm{c})$, including text retrieval \cite{karpukhin2020dpr, izacard2022contriever, xiao2022retromae, wang2022e5}, multimodal alignment \cite{radford2021clip}, knowledge graph completion \cite{sun2019rotate, chao2021pairre}, few-shot classification \cite{snell2017prototypical}, recommender systems \cite{he2020lightgcn, mao2021simplex}, and semantic textual similarity \cite{cer2017semeval, gao2021simcse}. Across these tasks, embedding \emph{magnitude} is treated inconsistently: some methods normalize it away (cosine in CLIP, BGE \cite{xiao2024c}, SimpleX \cite{mao2021simplex}), others preserve it (dot product in DPR, BPR/LightGCN), and others use it inside a distance score (Euclidean in TransE, Prototypical Networks). These choices are inherited largely by convention, with no principled criterion for when each is appropriate.

This inconsistency hides a sharper question: \emph{when does magnitude matter, and how should each side of the similarity comparison treat it?} Cosine and dot product both apply a symmetric treatment, normalizing both magnitudes or neither and thus collapsing the per-side distinction. Direct comparisons of cosine versus dot product \cite{karpukhin2020dpr, gao2021simcse} do not isolate per-side effects, and prior work studying magnitude's role \cite{steck2024cosine, wang2020hypersphere, draganov2024pitfalls} stops short of saying when to keep magnitude and when to discard it.

We address the question in two parts. First, on retrieval, a 2$\times$2 normalization framework independently controls the two magnitudes, exposing two intermediate variants \textsc{QNorm} and \textsc{DNorm} that outperform \emph{both} cosine and dot, with the Fisher Information Matrix predicting which side to normalize. Second, we generalize via \emph{functional symmetry}, the property that the aggregate scoring procedure treats $Q$ and $C$ as interchangeable, and find that this single property predicts the optimal similarity across five additional task families.

\begin{enumerate}[leftmargin=2em, itemsep=0.2em]
  \item \textbf{Unilateral normalization beats existing standards.} \textsc{QNorm} and \textsc{DNorm} are not just taxonomic gaps: across the three asymmetric tasks we examine, they beat Cosine on retrieval (up to +72\% out-of-domain), beat both Cosine and the Euclidean default of Prototypical Networks on few-shot classification, and beat Cosine on recommendation. Cross-evaluation explains why: document magnitude alters inference ranking, query magnitude modulates training gradients via an effective temperature, and normalizing one side anchors the optimization while leaving the other free to encode task signal.

  \item \textbf{Functional symmetry as a cross-task predictor.} We classify any task with a learned similarity score by whether the aggregate scoring procedure treats $Q$ and $C$ as interchangeable. Functionally symmetric tasks (semantic textual similarity, CLIP, KGC) favor cosine; functionally asymmetric tasks (retrieval, few-shot classification, recommender systems) favor unilateral normalization. The prediction holds in all 6 task families examined.

  \item \textbf{Predictive guidance for when and which.} Functional symmetry tells us \emph{whether} to keep magnitude; for asymmetric tasks, the Fisher Information Matrix condition number, computed on the pretrained model \emph{before training}, further predicts \emph{which side} to normalize (\textsc{QNorm} vs.\ \textsc{DNorm}). Learnable normalization provides a safe default when symmetry is uncertain.
\end{enumerate}

\section{The Geometry of Similarity Functions}
\label{sec:geometry}

\subsection{Background: Contrastive Learning with InfoNCE}

Contrastive learning trains encoders by pulling together representations of similar pairs while pushing apart dissimilar ones. The InfoNCE loss \cite{oord2018infonce} formalizes this objective:
\begin{equation}
\mathcal{L} = -\log \frac{\exp(\alpha \cdot s(\bm{q}, \bm{d}^+))}{\exp(\alpha \cdot s(\bm{q}, \bm{d}^+)) + \sum_{i=1}^{k} \exp(\alpha \cdot s(\bm{q}, \bm{d}^-_i))},
\label{eq:infonce}
\end{equation}
where $\bm{q}$ is the query embedding, $\bm{d}$ is the document embedding, $s(\cdot, \cdot)$ is a similarity function, and $\alpha = 1/\tau$ is the inverse temperature (loss scale). The choice of $s$ determines which directions in embedding space carry gradient signal.

\subsection{Cosine, Dot Product, and the Magnitude Degree of Freedom}

The standard cosine similarity $s_{\text{cos}}(\bm{q}, \bm{d}) = \hat{\bm{q}}^\top \hat{\bm{d}}$ (\Cref{eq:cosine_dot}) constrains representations to the unit hypersphere, encoding the implicit assumption that \emph{magnitude carries no task-relevant information}. The unnormalized dot product $s_{\text{dot}}(\bm{q}, \bm{d}) = \bm{q}^\top \bm{d} = \|\bm{q}\|\|\bm{d}\|\cos\theta$ (\Cref{eq:cosine_dot}) is therefore a \emph{magnitude-weighted angular similarity}, restoring magnitude as a learnable degree of freedom \cite{oyama2023norm, guo2024kvcache, liao2025magnitude}. We use $\|\bm{q}\|$ and $\|\bm{d}\|$ for Euclidean norms throughout.

\begin{equation}
s_{\text{cos}}(\bm{q}, \bm{d}) = \frac{\bm{q}^\top \bm{d}}{\|\bm{q}\| \cdot \|\bm{d}\|} = \hat{\bm{q}}^\top \hat{\bm{d}}, \qquad
s_{\text{dot}}(\bm{q}, \bm{d}) = \bm{q}^\top \bm{d} = \|\bm{q}\| \cdot \|\bm{d}\| \cdot \cos\theta.
\label{eq:cosine_dot}
\end{equation}

\begin{figure}[t]
  \centering
  \begin{tikzpicture}[
    scale=0.62,
    every node/.style={font=\scriptsize},
    vec/.style={-{Stealth[length=1.8mm]}, thick},
    qvec/.style={vec, blue!70!black},
    dvec/.style={vec, red!70!black},
    unitcircle/.style={draw=gray!50, dashed, thin},
    label/.style={font=\scriptsize},
    title/.style={font=\scriptsize\bfseries, align=center},
  ]

    \begin{scope}[shift={(0,0)}]
      \draw[unitcircle] (0,0) circle (1.2);
      \draw[qvec] (0,0) -- (55:1.2) node[above, label] {$\hat{\bm{q}}$};
      \draw[dvec] (0,0) -- (15:1.2) node[right, label] {$\hat{\bm{d}}$};
      \node[title, below] at (0,-1.6) {(a) Cosine};
      \fill[gray!30, opacity=0.3] (0,0) -- (15:0.4) arc (15:55:0.4) -- cycle;
      \node[font=\tiny] at (35:0.6) {$\theta$};
    \end{scope}

    \begin{scope}[shift={(4.0,0)}]
      \draw[unitcircle] (0,0) circle (1.2);
      \draw[qvec] (0,0) -- (55:1.2) node[above, label] {$\hat{\bm{q}}$};
      \draw[dvec] (0,0) -- (15:1.9) node[right, label] {$\bm{d}$};
      \node[title, below] at (0,-1.6) {(b) QNorm};
      \draw[gray, thin, dotted] (15:1.2) -- (15:1.9);
    \end{scope}

    \begin{scope}[shift={(8.0,0)}]
      \draw[unitcircle] (0,0) circle (1.2);
      \draw[qvec] (0,0) -- (55:1.9) node[above, label] {$\bm{q}$};
      \draw[dvec] (0,0) -- (15:1.2) node[right, label] {$\hat{\bm{d}}$};
      \node[title, below] at (0,-1.6) {(c) DNorm};
      \draw[gray, thin, dotted] (55:1.2) -- (55:1.9);
    \end{scope}

    \begin{scope}[shift={(12.0,0)}]
      \draw[unitcircle] (0,0) circle (1.2);
      \draw[qvec] (0,0) -- (55:1.9) node[above, label] {$\bm{q}$};
      \draw[dvec] (0,0) -- (15:1.7) node[right, label] {$\bm{d}$};
      \node[title, below] at (0,-1.6) {(d) Dot};
      \draw[gray, thin, dotted] (55:1.2) -- (55:1.9);
      \draw[gray, thin, dotted] (15:1.2) -- (15:1.7);
    \end{scope}

  \end{tikzpicture}
  \caption{Query-document normalization framework. The dashed circle represents the unit sphere. Normalized vectors ($\hat{\bm{v}}$) lie on the sphere; unnormalized vectors extend beyond. (a) Cosine: both normalized; (b) QNorm: only query normalized, document magnitude preserved; (c) DNorm: only document normalized, query magnitude preserved; (d) Dot: both unnormalized.}
  \label{fig:loss_comparison}
\end{figure}

\subsection{A Query-Document Normalization Framework}
\label{sec:ablation_framework}

The denominator $\|\bm{q}\| \cdot \|\bm{d}\|$ in \Cref{eq:cosine_dot} has two independent factors. Normalizing each independently exposes two intermediate variants:
\begin{equation}
s_{\text{qnorm}}(\bm{q}, \bm{d}) = \hat{\bm{q}}^\top \bm{d} = \|\bm{d}\|\cos\theta, \qquad
s_{\text{dnorm}}(\bm{q}, \bm{d}) = \bm{q}^\top \hat{\bm{d}} = \|\bm{q}\|\cos\theta.
\label{eq:qnorm_dnorm}
\end{equation}
Together with Cosine and Dot, these four variants span all query-document normalization combinations (\Cref{fig:loss_comparison}). One might object that magnitude in $\mathbb{R}^n$ is just an extra coordinate on $\mathbb{S}^n$ via $\varphi(\bm{v}){=}(\bm{v}/M, \sqrt{1{-}\|\bm{v}\|^2/M^2})$; although $\varphi$ is a bijection, it does \emph{not} preserve similarity ranking (Appendix~\ref{appendix:ranking_reversal}), so cosine and dot induce different rankings and optimization dynamics. The 2$\times$2 framework asks which of these four normalizations best matches a given task, rather than imposing the unit-norm constraint by default.

\paragraph{Learnable Normalization.}
\label{par:learnable_norm}
The four variants above are discrete points in a continuous space. We unify them through a learnable normalization framework:
\begin{equation}
s_{\text{learn}}(\bm{q}, \bm{d}) = \frac{\bm{q}^\top}{\|\bm{q}\|^{\gamma_q}} \cdot \frac{\bm{d}}{\|\bm{d}\|^{\gamma_d}}, \quad \gamma_q = \sigma(\hat{\gamma}_q),\; \gamma_d = \sigma(\hat{\gamma}_d) \in [0, 1],
\label{eq:learnable_norm}
\end{equation}
where $\hat{\gamma}_q, \hat{\gamma}_d$ are learnable parameters and $\sigma$ is the sigmoid function, ensuring $\gamma \in [0, 1]$. This subsumes all four discrete variants as special cases: $(\gamma_q, \gamma_d) = (1,1)$ yields Cosine, $(0,0)$ yields Dot, $(1,0)$ yields QNorm, and $(0,1)$ yields DNorm. By initializing at $\gamma = 0.5$, which is the geometric midpoint, we let the model \emph{discover} the optimal normalization level through gradient descent, providing an empirical test of which strategy the data prefers.

\subsection{Functional Symmetry and Asymmetric Learning Dynamics}
\label{sec:task_symmetry}

When does each cell of the 2$\times$2 framework apply? We answer by appealing to a structural property of the task itself.

\begin{definition}[Functional Symmetry]
\label{def:functional_symmetry}
A task is \textbf{functionally symmetric} if its \emph{aggregate scoring procedure}, comprising both training loss and inference protocol, treats $Q$ and $C$ as interchangeable roles. This includes (i) tasks that enforce $s(Q, C) = s(C, Q)$ as a definitional constraint (e.g., semantic textual similarity, STS), (ii) tasks whose loss aggregates both role configurations (e.g., bidirectional InfoNCE in CLIP), and (iii) tasks whose evaluation metric averages scores over both role configurations (e.g., filtered MRR over head and tail prediction in KGC). Otherwise the task is \textbf{functionally asymmetric}.
\end{definition}

Unilateral normalization is incompatible with functional symmetry: under QNorm, $s_{\text{qnorm}}(\bm{a}, \bm{b}) = \|\bm{b}\|\cos\theta \neq \|\bm{a}\|\cos\theta = s_{\text{qnorm}}(\bm{b}, \bm{a})$ unless $\|\bm{a}\| = \|\bm{b}\|$, breaking the role exchangeability that defines symmetric tasks. We therefore predict QNorm and DNorm to outperform Cosine and Dot only when the task is functionally asymmetric (e.g., retrieval, recommender systems, few-shot classification), and to underperform on functionally symmetric tasks (STS, CLIP, KGC). Sections~\ref{sec:experiments}--\ref{sec:generalization} test this prediction.

\paragraph{Asymmetric learning dynamics.} Beyond task compatibility, the framework reveals asymmetric roles of query and document magnitude in functionally asymmetric tasks. At \emph{inference time}, only document magnitude affects ranking (Proposition~\ref{prop:ranking}). At \emph{training time}, query magnitude modulates gradient dynamics: under DNorm, the effective scale becomes $\alpha_{\text{eff}} = \alpha \cdot \|\bm{q}\|$, making high-magnitude queries sharpen the softmax distribution and receive larger gradients. See Appendix~\ref{appendix:gradient_details} for derivations and Appendix~\ref{appendix:asymmetry_principle} for detailed analysis.

\section{Experiments}
\label{sec:experiments}

\subsection{Experimental Setup}

\paragraph{Models.}
We use three BERT-based retrievers (Contriever \cite{izacard2022contriever}, RetroMAE \cite{xiao2022retromae}, E5 \cite{wang2022e5}) and one LLM-based retriever (Qwen3-Base-0.6B). E5 results are in Appendix~\ref{appendix:e5} due to architectural constraints.

\paragraph{Data and Evaluation.}
We finetune on MS MARCO v1.1 QA \cite{nguyen2016msmarco} with 82K samples (and 503K samples for verification). We evaluate using NDCG@10 on MS MARCO Dev, TREC-DL, BEIR \cite{thakur2021beir}, BRIGHT \cite{bright2025}, and multi-hop QA \cite{ho2020wikimultihop, trivedi2022musique, gupta2025novelhopqa, yang2018hotpotqa}. See Appendix~\ref{sec:appendix_datasets} for details.

\paragraph{Hyperparameters.}
BERT-based models train for 100 epochs (batch size 128); Qwen for 40 epochs (batch size 64). Learning rates are grid-searched and fixed across similarity variants. We report results at $\alpha = 20$ ($\tau = 0.05$), the loss scale at which Cosine peaks (so any improvement over Cosine is not at its expense); the method ranking is stable across $\alpha \in \{10, 20, 30, 40\}$. Contriever and RetroMAE report mean$\pm$std across 3 seeds (0, 42, 1337); Qwen uses seed 0. Training-time validation NDCG@10 (Figure~\ref{fig:training_curves}) is computed on a small in-domain corpus ($\sim$10K relevant passages) for efficiency, while final evaluation (Table~\ref{tab:main_simplified}) uses full benchmark corpora ($\sim$8.8M documents); the method rankings agree between the two. See Appendix~\ref{appendix:training_config}.

\begin{table}[t]
  \caption{NDCG@10 across three training paradigms. DL = average of TREC DL'19 and DL'20. \textbf{Bold} = best, \underline{underline} = second best per model/scale.}
  \label{tab:main_simplified}
  \centering
  \scriptsize

  \textbf{(a) Finetune}
  \vspace{0.1cm}

  \resizebox{0.7\textwidth}{!}{
    \begin{tabular}{lcccccc|cccccc}
      \toprule
      & \multicolumn{6}{c|}{\textbf{Contriever}} & \multicolumn{6}{c}{\textbf{RetroMAE}} \\
      \cmidrule(lr){2-7} \cmidrule(lr){8-13}
      & PT & Cos & Dot & QN & DN & Ln & PT & Cos & Dot & QN & DN & Ln \\
      \midrule
      DL & 41.7 & 57.0 & 56.5 & \underline{57.5} & 56.3 & \textbf{58.2} & 11.6 & 56.5 & 57.3 & 57.8 & \underline{59.6} & \textbf{60.1} \\
      BEIR & 28.7 & 41.0 & \underline{43.5} & \textbf{44.0} & 42.5 & 43.6 & 14.3 & 38.3 & 40.1 & 41.1 & \underline{41.2} & \textbf{41.2} \\
      BRIGHT & 3.8 & 7.4 & \underline{12.5} & \textbf{12.7} & 9.8 & 11.7 & 5.2 & 6.1 & 8.7 & 9.4 & \underline{9.5} & \textbf{9.8} \\
      MHop & 39.7 & 51.4 & \textbf{58.2} & 57.4 & \underline{57.5} & 57.2 & 18.5 & 50.7 & 52.9 & 54.8 & \textbf{55.5} & \underline{55.2} \\
      \bottomrule
    \end{tabular}
  }

  \vspace{0.25cm}

  \textbf{(b) Foundation Model}
  \vspace{0.1cm}

  \resizebox{0.7\textwidth}{!}{
    \begin{tabular}{lccccc|ccccc}
      \toprule
      & \multicolumn{5}{c|}{\textbf{MARCO QA (82K)}} & \multicolumn{5}{c}{\textbf{MARCO Passage (503K)}} \\
      \cmidrule(lr){2-6} \cmidrule(lr){7-11}
      & Cos & Dot & QN & DN & Ln & Cos & Dot & QN & DN & Ln \\
      \midrule
      DL & \underline{54.4} & 14.9 & 53.6 & 49.6 & \textbf{56.5} & \underline{58.3} & 24.9 & 50.1 & \textbf{60.6} & 38.7 \\
      BEIR & \underline{42.6} & 15.4 & 29.7 & \textbf{42.7} & 40.3 & \underline{44.8} & 13.4 & 21.1 & \textbf{45.3} & 19.3 \\
      BRIGHT & \textbf{10.4} & 0.1 & 0.1 & \underline{9.8} & 2.2 & \underline{11.1} & 0.1 & 0.2 & \textbf{13.5} & 1.7 \\
      MHop & \textbf{49.2} & 18.7 & 37.4 & 44.1 & \underline{48.7} & \underline{54.6} & 16.8 & 47.2 & \textbf{57.6} & 10.7 \\
      \bottomrule
    \end{tabular}
  }

  \vspace{0.25cm}

  \textbf{(c) Random Init}
  \vspace{0.1cm}

  \resizebox{0.7\textwidth}{!}{
    \begin{tabular}{lccccc|ccccc}
      \toprule
      & \multicolumn{5}{c|}{\textbf{Contriever}} & \multicolumn{5}{c}{\textbf{RetroMAE}} \\
      \cmidrule(lr){2-6} \cmidrule(lr){7-11}
      & Cos & Dot & QN & DN & Ln & Cos & Dot & QN & DN & Ln \\
      \midrule
      DL & \textbf{38.3} & 24.6 & 21.5 & \underline{31.7} & 26.0 & \textbf{41.2} & 31.7 & 30.3 & \underline{38.0} & 31.9 \\
      BEIR & \textbf{24.4} & 13.4 & 15.5 & \underline{21.1} & 17.9 & \textbf{25.7} & 18.4 & 19.5 & \underline{23.1} & 20.6 \\
      BRIGHT & \underline{3.1} & 0.3 & 0.3 & \textbf{3.8} & 1.2 & \underline{3.2} & 0.8 & 1.0 & \textbf{3.9} & 1.7 \\
      MHop & \textbf{27.9} & 18.3 & 20.3 & \underline{27.7} & 22.7 & \textbf{31.0} & 25.6 & 26.8 & \underline{30.8} & 28.3 \\
      \bottomrule
    \end{tabular}
  }

\end{table}

\input{floats/figure_training_curves_combined.tex}

\paragraph{Training Paradigms.}
We distinguish three training paradigms (\Cref{tab:combined_results}): (a) \textbf{Finetuning}: training from retrieval-specialized pretrained models such as Contriever and RetroMAE, which have undergone contrastive pretraining on top of BERT; (b) \textbf{Training from foundation model}: training directly from a general-purpose pretrained LLM, specifically Qwen3-0.6B-Base, without retrieval-specific pretraining, following standard practice for LLM-based retrievers~\cite{karpukhin2020dpr}; and (c) \textbf{Random initialization}: training from randomly initialized weights using only the model architecture, an extreme ablation setting to isolate the role of pretrained representations.

\subsection{Results}

\Cref{tab:main_simplified} reports NDCG@10 for the four similarity variants on Contriever and RetroMAE finetuned on MS MARCO v1.1 QA; complete results across all training paradigms are in Appendix~\ref{appendix:complete_results}.

Across both models, all magnitude-aware variants outperform Cosine, with substantial gains on reasoning-intensive tasks: BRIGHT (+72\% for Contriever QNorm) and multi-hop QA (+13\% for Contriever Dot). The optimal variant is model-dependent: Contriever favors QNorm (document magnitude already correlates with relevance from contrastive pretraining), while RetroMAE favors DNorm (benefits from query-side gradient modulation). Learnable normalization provides a robust default when this prior is unavailable.

\paragraph{OOD gains exceed in-domain gains.}
The benefit of magnitude learning is consistently larger out-of-domain than in-domain. Contriever QNorm gains +1--2\% in-domain versus +7--72\% out-of-domain (BEIR +7.4\%, BRIGHT +72\%, multi-hop +12\%); RetroMAE DNorm shows the same pattern (in-domain +4--7\% vs.\ OOD +8--55\%). This suggests that magnitude carries domain-invariant relevance signals (e.g., specificity, information density) that transfer across domains, while angular similarity may overfit to domain-specific patterns.

\paragraph{LLM-based retriever.}
\Cref{tab:foundation} shows Qwen3-Base-0.6B results. With 82K samples, Cosine and Learnable perform best; with 503K samples, DNorm achieves the best results on most benchmarks. Foundation LLMs lack retrieval-specific inductive bias and must learn magnitude-relevance associations from scratch, requiring more data.

\paragraph{Statistical significance.}
Per-dataset paired t-tests confirm significance ($p < 0.01$ on BEIR for all magnitude-aware methods). Results are consistent across three random seeds. See Appendix~\ref{appendix:seed_sensitivity}.

\begin{table}[!t]
\centering
\caption{End-to-end RAG on open-domain QA. Retriever: Contriever (Table~\ref{tab:main_simplified}, finetuned on MS MARCO v1.1 QA). Reader: Flan-T5-Large. QNorm gains the most ($\Delta$ up to +7.9 EM, +24\% on TriviaQA); Learn matches but does not exceed Dot.}
\label{tab:rag_results}
\resizebox{0.6\textwidth}{!}{
\begin{tabular}{l|cc|cc|cc}
\toprule
& \multicolumn{2}{c|}{\textbf{NQ} (3.5K)} & \multicolumn{2}{c|}{\textbf{HotpotQA} (7.4K)} & \multicolumn{2}{c}{\textbf{TriviaQA} (11.3K)} \\
\textbf{Similarity} & EM & F1 & EM & F1 & EM & F1 \\
\midrule
Cosine & 23.0 & 31.9 & 27.2 & 35.9 & 32.3 & 38.7 \\
Dot & 24.5 & 33.8 & \textbf{32.9} & \textbf{42.7} & 39.0 & 46.3 \\
QNorm & \textbf{26.1} & \textbf{35.4} & 32.7 & 42.4 & \textbf{40.2} & \textbf{47.5} \\
DNorm & 22.8 & 31.5 & 31.5 & 41.0 & 37.7 & 44.6 \\
Learn & 25.0 & 34.0 & 32.6 & 42.4 & 39.4 & 46.5 \\
\midrule
\rowcolor{gray!10}
$\Delta$ (QNorm$-$Cos) & \textbf{+3.1} & \textbf{+3.5} & \textbf{+5.5} & \textbf{+6.5} & \textbf{+7.9} & \textbf{+8.8} \\
\bottomrule
\end{tabular}
}
\end{table}

\paragraph{Downstream RAG.} Retrieval gains transfer downstream: with Contriever$+$Flan-T5-Large ($k{=}5$), QNorm beats Cosine by +13.5\%/+20.2\%/+24.5\% on NQ/HotpotQA/TriviaQA (\Cref{tab:rag_results}, Appendix~\ref{appendix:rag_setup}).

\section{Analysis: Why Magnitude Helps}
\label{sec:analysis}

\begin{table}[t]
\centering
\caption{FIM analysis showing gradient sensitivity (log$_{10}$ scale). Higher values indicate greater sensitivity; darker backgrounds encode higher sensitivity. \textbf{Bold} indicates best methods for each model. Cosine eliminates magnitude gradients entirely (FIM $\approx 10^{-20}$), while magnitude-preserving methods show substantially higher sensitivity.}
\label{tab:fim_analysis}
\footnotesize

\begin{minipage}[t]{0.48\textwidth}
\centering
\textbf{(a) BERT-based}
\vspace{0.1cm}

\setlength{\tabcolsep}{3pt}
\begin{tabular}{ll|rrrrr}
\toprule
\textbf{Model} & \textbf{FIM Type} & Cos & Dot & QN & DN & Ln \\
\midrule
\multirow{3}{*}{Contriever}
& Doc Mag & \cellcolor{blue!0}-20.0 & \cellcolor{blue!47}-3.5 & \cellcolor{blue!45}\textbf{-4.3} & \cellcolor{blue!0}-19.8 & \cellcolor{blue!49}\textbf{-3.0} \\
& Query Mag & \cellcolor{blue!0}-20.0 & \cellcolor{blue!44}-6.3 & \cellcolor{blue!0}\textbf{-20.0} & \cellcolor{blue!38}-8.3 & \cellcolor{blue!46}\textbf{-5.7} \\
& Angular & \cellcolor{blue!21}-5.0 & \cellcolor{blue!33}-2.8 & \cellcolor{blue!28}\textbf{-3.8} & \cellcolor{blue!20}-5.2 & \cellcolor{blue!36}\textbf{-2.4} \\
\midrule
\multirow{3}{*}{RetroMAE}
& Doc Mag & \cellcolor{blue!0}-20.0 & \cellcolor{blue!51}-2.0 & \cellcolor{blue!44}-4.4 & \cellcolor{blue!6}\textbf{-17.7} & \cellcolor{blue!32}\textbf{-8.8} \\
& Query Mag & \cellcolor{blue!0}-20.0 & \cellcolor{blue!42}-7.0 & \cellcolor{blue!0}-20.0 & \cellcolor{blue!39}\textbf{-7.8} & \cellcolor{blue!20}\textbf{-13.6} \\
& Angular & \cellcolor{blue!17}-5.7 & \cellcolor{blue!33}-2.8 & \cellcolor{blue!23}-4.7 & \cellcolor{blue!28}\textbf{-3.9} & \cellcolor{blue!0}\textbf{-9.0} \\
\bottomrule
\end{tabular}
\end{minipage}%
\hfill
\begin{minipage}[t]{0.48\textwidth}
\centering
\textbf{(b) LLM-based}
\vspace{0.1cm}

\setlength{\tabcolsep}{3pt}
\begin{tabular}{ll|rrrrr}
\toprule
\textbf{Model} & \textbf{FIM Type} & Cos & Dot & QN & DN & Ln \\
\midrule
\multirow{3}{*}{Qwen 82K}
& Doc Mag & \cellcolor{blue!0}\textbf{-19.7} & \cellcolor{blue!67}3.4 & \cellcolor{blue!55}-0.8 & \cellcolor{blue!15}-14.7 & \cellcolor{blue!67}\textbf{3.5} \\
& Query Mag & \cellcolor{blue!0}\textbf{-20.0} & \cellcolor{blue!66}0.4 & \cellcolor{blue!9}-17.2 & \cellcolor{blue!53}-3.7 & \cellcolor{blue!66}\textbf{0.4} \\
& Angular & \cellcolor{blue!20}\textbf{-5.3} & \cellcolor{blue!66}3.1 & \cellcolor{blue!44}-0.9 & \cellcolor{blue!44}-0.9 & \cellcolor{blue!66}\textbf{3.2} \\
\midrule
\multirow{3}{*}{Qwen 503K}
& Doc Mag & \cellcolor{blue!0}-20.0 & \cellcolor{blue!67}3.5 & \cellcolor{blue!54}-1.0 & \cellcolor{blue!15}\textbf{-14.8} & \cellcolor{blue!70}4.3 \\
& Query Mag & \cellcolor{blue!0}-20.0 & \cellcolor{blue!67}0.6 & \cellcolor{blue!9}-17.2 & \cellcolor{blue!56}\textbf{-2.9} & \cellcolor{blue!70}1.4 \\
& Angular & \cellcolor{blue!18}-5.6 & \cellcolor{blue!65}3.0 & \cellcolor{blue!43}-1.0 & \cellcolor{blue!46}\textbf{-0.6} & \cellcolor{blue!70}3.8 \\
\bottomrule
\end{tabular}
\end{minipage}

\end{table}

\subsection{Inference vs Training: Asymmetric Roles}
\label{sec:analysis_asymmetric}

The decomposition $s_{\text{dot}}(\bm{q}, \bm{d}) = \|\bm{q}\| \cdot \|\bm{d}\| \cdot \cos\theta$ reveals that query and document magnitudes affect inference and training differently.

\begin{proposition}[Ranking Equivalence]
\label{prop:ranking}
At inference time, only \emph{document} magnitude can alter rankings. For a fixed query $\bm{q}$: (1) $\pi_{\emph{cos}} = \pi_{\emph{dnorm}}$, which both rank by $\cos\theta$ alone; (2) $\pi_{\emph{qnorm}} = \pi_{\emph{dot}}$, which both rank by $\|\bm{d}\| \cos\theta$. Query magnitude $\|\bm{q}\|$ scales all scores uniformly, never affecting rankings. Proof in Appendix~\ref{appendix:proofs}.
\end{proposition}

This yields \emph{inference equivalence}: Dot Eval $\equiv$ QNorm Eval, and Cosine Eval $\equiv$ DNorm Eval. However, \emph{training} dynamics differ: query magnitude modulates effective temperature, enabling DNorm to learn better angular representations even though $\|\bm{q}\|$ does not affect inference rankings.

\paragraph{Gradient sensitivity analysis.} We compute the Fisher Information Matrix (FIM)~\citep{fisher1925theory}, measuring gradient sensitivity for each embedding component. \Cref{tab:fim_analysis} reveals that Cosine eliminates magnitude gradients entirely (FIM $\approx 10^{-20}$), while magnitude-preserving methods show substantially higher sensitivity. See Appendix~\ref{appendix:fim_analysis} for derivations.

\subsection{Cross-Evaluation Analysis}
\label{sec:cross_eval}

\input{floats/fig_cross_eval_combined}

Cross-evaluation (training with one similarity function and evaluating with another) isolates training-time from inference-time effects (\Cref{fig:cross_eval_combined}), revealing three findings. \textbf{(1) Query magnitude improves angular learning}: under Cosine evaluation, where magnitude is discarded, DNorm-trained models still outperform Cosine-trained by +1.6 to +3.3 points across all models. \textbf{(2) Document magnitude separates relevant from irrelevant documents at inference}: Contriever QNorm-trained scores 33.4 with Dot evaluation versus 27.3 with Cosine evaluation (+6.1), but only when document magnitude has been learned during training; Qwen with 503K samples fails to learn it, and using document magnitude at inference then \emph{degrades} performance. \textbf{(3) Anchoring one side helps}: DNorm and QNorm match or beat Dot across configurations, suggesting that fixing one side's magnitude stabilizes the optimization while leaving the other side free to encode task signal. The findings hold across both in-domain and out-of-domain evaluation, with the out-of-domain gains more pronounced. Random initialization (an extreme ablation) prevents magnitude learning entirely, and bilateral Dot fails catastrophically.

\paragraph{Cohen's $d$ analysis.} Cohen's $d$ between the magnitudes of relevant and irrelevant documents correlates significantly with the relative gain $\Delta\%$ ($r{=}0.57$ for Contriever, $r{=}0.68$ for RetroMAE, $p<0.001$). Because $d$ requires no QNorm training, it gives a cheap pre-deployment estimate of QNorm benefit. Realising this benefit further requires retrieval-specialized pre-training, sufficient data, and a compatible architecture (Appendix~\ref{appendix:cohens_d}).

\subsection{Choosing the Variant}
\label{sec:variant_selection}
\label{sec:fim_prediction}

\begin{table}[t]
\centering
\caption{FIM condition number ($\kappa$, log$_{10}$ scale) for predicting QNorm vs DNorm. The strategy with smaller $\kappa$ is predicted to perform better. Prediction accuracy: 3/3.}
\label{tab:fim_prediction}
\footnotesize
\resizebox{0.5\textwidth}{!}{%
\begin{tabular}{lcccc}
\toprule
\textbf{Model} & $\kappa$(QNorm) & $\kappa$(DNorm) & \textbf{Predicted} & \textbf{Actual} \\
\midrule
Contriever & \textbf{3.83} & 3.84 & QNorm & QNorm \\
RetroMAE & 6.96 & \textbf{6.92} & DNorm & DNorm \\
Qwen & 7.23 & \textbf{7.14} & DNorm & DNorm \\
\bottomrule
\end{tabular}%
}
\end{table}

\paragraph{FIM selects \textsc{QNorm} vs.\ \textsc{DNorm}.}
The FIM condition number $\kappa = \lambda_{\max}/\lambda_{\min}$ predicts \emph{which} unilateral variant suits each model. Lower $\kappa$ indicates a more balanced loss landscape and thus more stable optimization. \Cref{tab:fim_prediction} reports $\kappa$ computed on the pretrained model \emph{before any training}: Contriever has $\kappa_{\text{QNorm}} < \kappa_{\text{DNorm}}$, predicting QNorm, while RetroMAE and Qwen have $\kappa_{\text{DNorm}} < \kappa_{\text{QNorm}}$, predicting DNorm. The empirical winner agrees with the prediction for all three models, giving a practical recipe: compute $\kappa$ on the pretrained model to choose the unilateral variant prior to training. The ratio distinguishes \textsc{QNorm} vs.\ \textsc{DNorm} only; it does not directly predict \textsc{Dot} (Appendix~\ref{appendix:fim_scope}).

\paragraph{Learnable normalization as a safe default.} When functional symmetry is uncertain, $\gamma_q, \gamma_d$ in \Cref{eq:learnable_norm} can be learned from data. Initializing both at $0.5$, the model drifts toward Dot for Contriever ($\gamma^*{\approx}0.499$) and toward Cosine for RetroMAE/Qwen ($\gamma^*{\approx}0.502$--$0.503$, \Cref{fig:training_curves}), and performs near the top across all models (Appendix~\ref{appendix:learnable_gradient}).

\section{Generalization Beyond Retrieval}
\label{sec:generalization}

The retrieval experiments establish the mechanism: $\|q\|$ and $\|d\|$ carry separable signals, so unilateral normalization outperforms both cosine and dot. If this mechanism is general, the prediction for any task with a learned similarity score should follow from its functional symmetry (Definition~\ref{def:functional_symmetry}). We test this on five additional task families (\Cref{tab:crosstask_summary}). The coarse prediction holds in every case; the unilateral variants further beat Cosine on recommendation, and on few-shot classification \textsc{DNorm} beats both Cosine and the Euclidean default of Prototypical Networks, suggesting the framework supplies competitive losses where existing community defaults leave room.

\begin{table}[t]
\centering
\caption{\textbf{Cross-task validation of the framework's predictions.} The prediction is layered: the binary functional-symmetry rule predicts Cosine vs.\ $\neq$Cosine; the FIM ratio further predicts QNorm vs.\ DNorm within the unilateral variants. The coarse rule matches the empirical winner in every case; the fine FIM prediction matches in retrieval, few-shot classification, and ML-100K recommendation, while ML-1M recommendation falls into a regime where Dot is competitive (outside the FIM ratio's scope, Appendix~\ref{appendix:fim_scope}). The ``Variants'' column lists the similarities tested per task: \textbf{C}osine, \textbf{D}ot, \textbf{QN}orm, \textbf{DN}orm, and \textbf{E}uclidean where it is the community default.}
\label{tab:crosstask_summary}
\footnotesize
\setlength{\tabcolsep}{3pt}
\resizebox{\textwidth}{!}{%
\begin{tabular}{lllllll}
\toprule
Task --- Dataset (Model) & Variants & Source of (a)symmetry & Predicted & Observed best & Best alt.\ vs.\ Cosine \\
\midrule
\multicolumn{6}{l}{\emph{Functionally symmetric (predicted: Cosine)}} \\
\midrule
Sentence sim.\ --- STS-B (Contriever, RetroMAE) & C/D/QN/DN & $s(a,b){=}s(b,a)$ definitional & Cosine & Cosine & $-40$ pts Spearman avg.\ \\
Multimodal --- MS-COCO (CLIP ViT-B/32) & C/D/QN/DN & Bidirectional InfoNCE loss & Cosine & Cosine & $\approx 0$ (Cohen's $d{\approx}0$) \\
KGC --- FB15k-237 (RotatE) & C/D/QN/DN/E & MRR avg.\ over h/t prediction & Cosine & Cosine (4/4) & $-58\%$ MRR avg.\ \\
KGC --- WN18RR (RotatE) & C/D/QN/DN/E & MRR avg.\ over h/t prediction & Cosine & Cosine (4/4) & $-67\%$ MRR avg.\ \\
KGC --- FB15k-237 (PairRE) & C/D/QN/DN/E & MRR avg.\ over h/t prediction & Cosine & Cosine or Euclidean & $-10\%$ MRR avg.\ \\
KGC --- WN18RR (PairRE) & C/D/QN/DN/E & MRR avg.\ over h/t prediction & Cosine & Cosine (4/4) & $-2\%$ MRR avg.\ \\
\midrule
\multicolumn{6}{l}{\emph{Functionally asymmetric (predicted: unilateral variant chosen by FIM)}} \\
\midrule
Text retrieval --- BEIR/BRIGHT/MultiHop (Contriever, RetroMAE, Qwen) & C/D/QN/DN & Query $\to$ Document only & QNorm/DNorm$^\star$ & QNorm/DNorm & up to $+72\%$ NDCG (BRIGHT) \\
Few-shot classif.\ --- CIFAR-100/\miniIN{} (CLIP+adapter) & C/D/QN/DN/E & Image $\to$ K-mean prototype & DNorm$^\star$ & DNorm (6/7) & $+5.9\%$ acc \\
Few-shot classif.\ --- CIFAR-100/\miniIN{} (DINOv2+adapter) & C/D/QN/DN/E & Image $\to$ K-mean prototype & DNorm$^\star$ & DNorm (3/4) & $+4.8\%$ acc \\
Recommendation --- MovieLens-1M (LightGCN) & C/D/QN/DN & User $\to$ Item, separate vocab & QNorm$^\star$ & Dot/DNorm$^\dagger$ & $+15\%$ NDCG@20 \\
Recommendation --- MovieLens-100K (LightGCN) & C/D/QN/DN & User $\to$ Item, separate vocab & QNorm$^\star$ & QNorm (3/4) & $+38\%$ NDCG@20 \\
\bottomrule
\end{tabular}%
}
\par\smallskip
\footnotesize
$^\star$Fine-grained prediction from the FIM condition number (Section~\ref{sec:fim_prediction}, Appendix~\ref{appendix:fim_crosstask}).
\quad
$^\dagger$Dot competitive in this regime, which lies outside the FIM ratio's scope (Appendix~\ref{appendix:fim_scope}).
\end{table}

\subsection{Functionally Symmetric Tasks (Predicted: Cosine)}
\label{sec:gen_symmetric}

\begin{table}[!t]
\centering
\small
\caption{Semantic Textual Similarity Benchmark (STS-B) test set Spearman correlation. Unlike retrieval, magnitude provides no benefit; asymmetric normalization severely degrades performance.}
\label{tab:sts_results}
\resizebox{0.5\textwidth}{!}{
\begin{tabular}{lcccc}
\toprule
\textbf{Model} & \textbf{Cosine} & \textbf{Dot} & \textbf{S1Norm} & \textbf{S2Norm} \\
\midrule
Contriever & \textbf{0.788} & 0.784 & 0.413 & 0.371 \\
RetroMAE & 0.749 & \textbf{0.750} & 0.314 & 0.328 \\
\bottomrule
\end{tabular}
}
\end{table}

\paragraph{Semantic Textual Similarity (STS).} STS predicts a numeric similarity required to satisfy $s(a, b) = s(b, a)$ by definition; the two inputs are sentences rather than query/document, so we relabel the unilateral variants \textsc{S1Norm}/\textsc{S2Norm}. Table~\ref{tab:sts_results} confirms Cosine $\approx$ Dot, while \textsc{S1Norm}/\textsc{S2Norm} drop by 40--45 points, a direct violation of role exchangeability.

\begin{figure}[!t]
  \centering
  \resizebox{0.5\textwidth}{!}{%
  \begin{tikzpicture}[
    scale=0.48,
    every node/.style={font=\scriptsize},
  ]
    \node[font=\small\bfseries, rotate=90] at (-5.8, 0) {Normalization};
    \node[font=\small\bfseries] at (3.8, 3.8) {Loss Direction};
    \node[font=\scriptsize\bfseries] at (0.1, 3.0) {Symmetric};
    \node[font=\scriptsize\bfseries] at (3.8, 3.0) {I2T only};
    \node[font=\scriptsize\bfseries] at (7.5, 3.0) {T2I only};
    \node[font=\scriptsize\bfseries, align=right, anchor=east] at (-1.9, 1.8) {Image Norm};
    \node[font=\scriptsize\bfseries, align=right, anchor=east] at (-1.9, 0) {Text Norm};
    \node[font=\scriptsize\bfseries, align=right, anchor=east] at (-1.9, -1.8) {No Norm};
    \node[draw=gray!60, fill=gray!15, rounded corners=2pt, minimum width=1.7cm, minimum height=1.6cm, align=center, text width=1.5cm] at (0.1, 0.9) {Norm\\[2pt]{\footnotesize\textcolor{gray!60}{$d \approx 0$}}};
    \node[draw=gray!60, fill=gray!15, rounded corners=2pt, minimum width=1.7cm, minimum height=0.7cm, align=center, text width=1.5cm] at (0.1, -1.8) {NoNorm\\[0pt]{\footnotesize\textcolor{gray!60}{$d \approx 0$}}};
    \node[draw=blue!50, fill=blue!10, rounded corners=2pt, minimum width=1.7cm, minimum height=0.7cm, align=center, text width=1.5cm] at (3.8, 1.8) {ImgNorm\\[0pt]{\footnotesize\textcolor{blue!70}{$d{=}1.83$}}};
    \node[text=gray!40] at (3.8, 0) {---};
    \node[draw=blue!50, fill=blue!10, rounded corners=2pt, minimum width=1.7cm, minimum height=0.7cm, align=center, text width=1.5cm] at (3.8, -1.8) {NoNorm\\[0pt]{\footnotesize\textcolor{blue!70}{$d{=}0.58$}}};
    \node[text=gray!40] at (7.5, 1.8) {---};
    \node[draw=red!50, fill=red!10, rounded corners=2pt, minimum width=1.7cm, minimum height=0.7cm, align=center, text width=1.5cm] at (7.5, 0) {TxtNorm\\[0pt]{\footnotesize\textcolor{red!70}{$d{=}2.11$}}};
    \node[draw=red!50, fill=red!10, rounded corners=2pt, minimum width=1.7cm, minimum height=0.7cm, align=center, text width=1.5cm] at (7.5, -1.8) {NoNorm\\[0pt]{\footnotesize\textcolor{red!70}{$d{=}1.31$}}};
  \end{tikzpicture}%
  }
  \caption{CLIP pre-training framework. Gray: symmetric loss yields $d \approx 0$. Blue/red: asymmetric loss enables magnitude learning on non-query side.}
  \label{fig:clip_framework}
\end{figure}
\begin{table}[!t]
\centering
\small
\caption{CLIP pre-training results on MS-COCO. The TxtNorm/ImgNorm naming refers to the training-time loss normalization; Cohen's $d$ and all retrieval metrics are computed on raw (unnormalized) features at evaluation time, so text/image magnitudes vary across candidates and Cosine and Dot can yield different rankings. Shaded cells indicate the direction(s) trained by the loss function. Symmetric loss yields $d \approx 0$; asymmetric loss enables magnitude learning but sacrifices bidirectional capability.}
\label{tab:clip_pretrain}
\resizebox{0.6\textwidth}{!}{
\begin{tabular}{lccccccc}
\toprule
& \multicolumn{2}{c}{\textbf{Cohen's $d$}} & \multicolumn{2}{c}{\textbf{Cosine R@1}} & \multicolumn{2}{c}{\textbf{Dot R@1}} & $\Delta$ \\
\cmidrule(lr){2-3} \cmidrule(lr){4-5} \cmidrule(lr){6-7}
\textbf{Config} & Image & Text & I$\to$T & T$\to$I & I$\to$T & T$\to$I & \scriptsize(Dot$-$Cos) \\
\midrule
Norm (Sym) & \cellcolor{gray!15}$-$0.55 & \cellcolor{gray!15}$-$0.30 & \cellcolor{gray!15}\textbf{38.4} & \cellcolor{gray!15}\textbf{28.0} & \cellcolor{gray!15}32.8 & \cellcolor{gray!15}26.3 & $-$5.6 \\
NoNorm (Sym) & \cellcolor{gray!15}$-$0.65 & \cellcolor{gray!15}$-$0.06 & \cellcolor{gray!15}23.4 & \cellcolor{gray!15}19.1 & \cellcolor{gray!15}\textbf{25.1} & \cellcolor{gray!15}\textbf{19.5} & +1.7 \\
\midrule
ImgNorm-I2T & 5.72 & \cellcolor{blue!15}\textbf{1.83} & \cellcolor{blue!15}30.5 & 0.3 & \cellcolor{blue!15}\textbf{32.1} & 0.3 & +1.6 \\
NoNorm-I2T & $-$0.49 & \cellcolor{blue!15}\textbf{0.58} & \cellcolor{blue!15}18.0 & 12.4 & \cellcolor{blue!15}\textbf{21.2} & 9.6 & +3.2 \\
TxtNorm-T2I & \cellcolor{red!15}\textbf{2.11} & 10.89 & 1.9 & \cellcolor{red!15}17.6 & 0.2 & \cellcolor{red!15}\textbf{23.4} & +5.8 \\
NoNorm-T2I & \cellcolor{red!15}\textbf{1.31} & 1.21 & 8.6 & \cellcolor{red!15}13.5 & 6.9 & \cellcolor{red!15}\textbf{16.7} & +3.2 \\
\bottomrule
\end{tabular}
}
\end{table}

\paragraph{Contrastive Language-Image Pre-training (CLIP).} Standard CLIP is functionally symmetric not because image and text are interchangeable inputs but because the bidirectional InfoNCE loss aggregates I$\to$T and T$\to$I. On MS-COCO, pre-trained CLIP yields Cohen's $d \approx 0$ for both modalities, and switching to Dot degrades retrieval. Controlled pre-training (\Cref{fig:clip_framework}, \Cref{tab:clip_pretrain}) localizes the symmetry to \emph{the loss}: bidirectional InfoNCE produces $d \approx 0$ even without normalization, while a one-directional loss enables magnitude learning ($d > 0$ on the output side) at the cost of bidirectional retrieval. Details in Appendix~\ref{appendix:clip_sts}.

\paragraph{Knowledge Graph Completion (KGC).} KGC scores triplets via embeddings drawn from a \emph{shared} entity table; although $s(h+r, t)$ is not internally symmetric, the standard filtered MRR averages head and tail prediction, making the aggregate procedure role-exchangeable. We train RotatE \cite{sun2019rotate} and PairRE \cite{chao2021pairre} on FB15k-237 and WN18RR with 3 seeds. The asymmetric variants (\textsc{Dot}, \textsc{QNorm}, \textsc{DNorm}) lose in \emph{all} 16 (model, dataset, scale) configurations; \textsc{Cosine} wins 14/16 and \textsc{Euclidean} wins the remaining 2 (also role-exchangeable). The mean MRR gap to the best asymmetric variant ranges from $0.6\%$ on PairRE WN18RR to $4\times$ on RotatE WN18RR at high $\alpha$. The cross-evaluation matrix also validates the Ranking Equivalence Proposition. Full sweep in Appendix~\ref{appendix:kgc}.

\subsection{Functionally Asymmetric Tasks Beyond Retrieval (Predicted: DNorm/QNorm)}
\label{sec:gen_asymmetric}

\paragraph{Few-Shot Classification.} Prototypical networks rank class prototypes given a query image, with no role exchange; the standard similarity is Euclidean distance, which our framework predicts \textsc{DNorm} should beat. Following \citet{snell2017prototypical}, we train a 2-layer adapter on frozen CLIP features. Across 7 base configurations spanning two datasets (CIFAR-100, \emph{mini}-ImageNet) and varied $N$-way and shot counts, \textsc{DNorm} wins on 6/7 (73.10\% vs.\ Euclidean 72.45\% vs.\ Cosine 69.02\%); in 26/28 individual runs, asymmetric training beats Cosine \emph{even under Cosine evaluation}, refuting the hypothesis that magnitude only encodes an extra dimension cosine evaluation discards. The same pattern holds with DINOv2 \cite{oquab2023dinov2} features (\textsc{DNorm} 3/4, 90.17\% vs.\ 88.54\% vs.\ 86.01\%), isolating the effect from CLIP's image-text pretraining. Appendix~\ref{appendix:protonet}.

\paragraph{Recommender Systems.} Collaborative filtering ranks items for users from separate embedding tables, with defaults split between Dot \cite{he2020lightgcn} and Cosine \cite{mao2021simplex}. Training LightGCN on MovieLens-1M and ML-100K with 3 seeds, asymmetric variants beat Cosine in 6/8 (dataset, scale) configurations; in the remaining 2 (lowest temperature, $\alpha{=}10$), Cosine ties within one standard deviation. At $\alpha=30$ on ML-1M, Cosine reaches $4.61\pm0.04$ NDCG@20 while Dot and DNorm reach $5.35\pm0.05$ and $5.33\pm0.06$. Trained magnitudes follow the predicted pattern ($\|u\|{=}1.92$ vs.\ $\|i\|{=}1.18$ under DNorm), normalizing away item-popularity bias. Appendix~\ref{appendix:rec}.

\section{Related Work}
\label{sec:related}

\paragraph{Magnitude Treatment Across Tasks.}
The similarity choice varies by community: dense retrievers split between cosine \cite{izacard2022contriever,xiao2022retromae,wang2022e5} and dot \cite{karpukhin2020dpr,qwen3embedding2025}; multimodal contrastive methods use cosine \cite{chen2020simclr,gao2021simcse,radford2021clip}; recommenders split between dot \cite{he2020lightgcn} and cosine \cite{mao2021simplex}; KGC \cite{sun2019rotate,chao2021pairre,trouillon2016complex} and few-shot classification \cite{snell2017prototypical} use distance scores. None expose the per-side magnitude choice; we unify them in a 2$\times$2 framework with a structural selection criterion.

\paragraph{When to Use Cosine.}
Whether to keep or discard magnitude has remained open: \citet{wang2020hypersphere} argue for hypersphere normalization, while \citet{steck2024cosine} and \citet{draganov2024pitfalls} expose pathologies of cosine and unnormalized dot. Magnitude carries signal in other neural components \cite{oyama2023norm,kobayashi2020attention,guo2024kvcache,liao2025magnitude} but only as a diagnostic. We instead show that the similarity choice determines whether magnitude can be \emph{learned} as task signal; functional symmetry then selects among cosine, dot, \textsc{QNorm}, and \textsc{DNorm}, with face verification \cite{wang2017normface,ranjan2017l2} the symmetric case.

\section{Conclusion}

A 2$\times$2 normalization framework reveals \textsc{QNorm} and \textsc{DNorm} as new operating points: they beat both cosine and dot on retrieval (up to +72\% out-of-domain and +24\% RAG), beat Cosine and the Euclidean default of Prototypical Networks on few-shot classification, and beat Cosine on recommendation. \emph{Functional symmetry}, the property that scoring treats $Q$ and $C$ as interchangeable, then predicts cosine for symmetric tasks and unilateral normalization for asymmetric ones across six task families, with the FIM condition number further selecting \textsc{QNorm} versus \textsc{DNorm} (Appendix~\ref{appendix:limitations}).

\bibliography{references}

@article{izacard2022contriever,
  title={Unsupervised Dense Information Retrieval with Contrastive Learning},
  author={Izacard, Gautier and Caron, Mathilde and Hosseini, Lucas and Riedel, Sebastian and Bojanowski, Piotr and Joulin, Armand and Grave, Edouard},
  journal={Transactions on Machine Learning Research},
  year={2022}
}

@article{wang2022e5,
  title={Text Embeddings by Weakly-Supervised Contrastive Pre-training},
  author={Wang, Liang and Yang, Nan and Huang, Xiaolong and Jiao, Binxing and Yang, Linjun and Jiang, Daxin and Majumder, Rangan and Wei, Furu},
  journal={arXiv preprint arXiv:2212.03533},
  year={2022}
}

@article{fisher1925theory,
  title={Theory of Statistical Estimation},
  author={Fisher, Ronald Aylmer},
  journal={Mathematical Proceedings of the Cambridge Philosophical Society},
  volume={22},
  number={5},
  pages={700--725},
  year={1925},
  publisher={Cambridge University Press}
}

@inproceedings{kobayashi2020attention,
  title={Attention is Not Only a Weight: Analyzing Transformers with Vector Norms},
  author={Kobayashi, Goro and Kuribayashi, Tatsuki and Yokoi, Sho and Inui, Kentaro},
  booktitle={Proceedings of the 2020 Conference on Empirical Methods in Natural Language Processing (EMNLP)},
  pages={7057--7075},
  year={2020}
}

@inproceedings{liao2025magnitude,
  title={Beyond Cosine Similarity: Magnitude-Aware {CLIP} for No-Reference Image Quality Assessment},
  author={Liao, Zhicheng and Wu, Dongxu and Shi, Zhenshan and Mai, Sijie and Zhu, Hanwei and Zhu, Lingyu and Jiang, Yuncheng and Chen, Baoliang},
  booktitle={Proceedings of the AAAI Conference on Artificial Intelligence},
  year={2026}
}

@inproceedings{guo2024kvcache,
  title={Attention Score is not All You Need for Token Importance Indicator in {KV} Cache Reduction: Value Also Matters},
  author={Guo, Zhiyu and Kamigaito, Hidetaka and Watanabe, Taro},
  booktitle={Proceedings of the 2024 Conference on Empirical Methods in Natural Language Processing (EMNLP)},
  pages={21158--21166},
  year={2024}
}

@inproceedings{thakur2021beir,
  title={{BEIR}: A Heterogeneous Benchmark for Zero-shot Evaluation of Information Retrieval Models},
  author={Thakur, Nandan and Reimers, Nils and R{\"u}ckl{\'e}, Andreas and Srivastava, Abhishek and Gurevych, Iryna},
  booktitle={Thirty-fifth Conference on Neural Information Processing Systems Datasets and Benchmarks Track (Round 2)},
  year={2021}
}

@inproceedings{bright2025,
  title={{BRIGHT}: A Realistic and Challenging Benchmark for Reasoning-Intensive Retrieval},
  author={Su, Hongjin and Yen, Howard and Xia, Mengzhou and Shi, Weijia and Muennighoff, Niklas and Wang, Han-yu and Liu, Haisu and Shi, Quan and Siegel, Zachary S. and Tang, Michael and Sun, Ruoxi and Yoon, Jinsung and Ar{\i}k, Sercan {\"O}. and Chen, Danqi and Yu, Tao},
  booktitle={International Conference on Learning Representations},
  year={2025}
}

@inproceedings{ho2020wikimultihop,
  title={Constructing A Multi-hop QA Dataset for Comprehensive Evaluation of Reasoning Steps},
  author={Ho, Xanh and Duong Nguyen, Anh-Khoa and Sugawara, Saku and Aizawa, Akiko},
  booktitle={Proceedings of the 28th International Conference on Computational Linguistics},
  pages={6609--6625},
  year={2020}
}

@article{trivedi2022musique,
  title={{MuSiQue}: Multihop Questions via Single-hop Question Composition},
  author={Trivedi, Harsh and Balasubramanian, Niranjan and Khot, Tushar and Sabharwal, Ashish},
  journal={Transactions of the Association for Computational Linguistics},
  volume={10},
  pages={539--554},
  year={2022}
}

@inproceedings{yang2018hotpotqa,
  title={{HotpotQA}: A Dataset for Diverse, Explainable Multi-hop Question Answering},
  author={Yang, Zhilin and Qi, Peng and Zhang, Saizheng and Bengio, Yoshua and Cohen, William and Salakhutdinov, Ruslan and Manning, Christopher D},
  booktitle={Proceedings of the 2018 Conference on Empirical Methods in Natural Language Processing},
  pages={2369--2380},
  year={2018}
}

@inproceedings{gupta2025novelhopqa,
  title={{NovelHopQA}: Diagnosing Multi-Hop Reasoning Failures in Long Narrative Contexts},
  author={Gupta, Abhay and Lu, Michael and Zhu, Kevin and O'Brien, Sean and Sharma, Vasu},
  booktitle={Proceedings of the 2025 Conference on Empirical Methods in Natural Language Processing},
  pages={26134--26151},
  year={2025}
}

@inproceedings{nguyen2016msmarco,
  title={{MS MARCO}: A Human Generated Machine Reading Comprehension Dataset},
  author={Nguyen, Tri and Rosenberg, Mir and Song, Xia and Gao, Jianfeng and Tiwary, Saurabh and Majumder, Rangan and Deng, Li},
  booktitle={Proceedings of the Workshop on Cognitive Computation: Integrating Neural and Symbolic Approaches},
  year={2016}
}

@inproceedings{xiao2022retromae,
  title={{RetroMAE}: Pre-Training Retrieval-oriented Language Models Via Masked Auto-Encoder},
  author={Xiao, Shitao and Liu, Zheng and Shao, Yingxia and Cao, Zhao},
  booktitle={Proceedings of the 2022 Conference on Empirical Methods in Natural Language Processing},
  pages={538--548},
  year={2022}
}

@article{qwen3embedding2025,
  title={Qwen3 Embedding: Advancing Text Embedding and Reranking Through Foundation Models},
  author={Zhang, Yanzhao and Li, Mingxin and Long, Dingkun and Zhang, Xin and Lin, Huan and Yang, Baosong and Xie, Pengjun and Yang, An and Liu, Dayiheng and Lin, Junyang and Huang, Fei and Zhou, Jingren},
  journal={arXiv preprint arXiv:2506.05176},
  year={2025}
}

@inproceedings{karpukhin2020dpr,
  title={Dense Passage Retrieval for Open-Domain Question Answering},
  author={Karpukhin, Vladimir and O{\u{g}}uz, Barlas and Min, Sewon and Lewis, Patrick and Wu, Ledell and Edunov, Sergey and Chen, Danqi and Yih, Wen-tau},
  booktitle={Proceedings of the 2020 Conference on Empirical Methods in Natural Language Processing},
  pages={6769--6781},
  year={2020}
}

@inproceedings{chen2020simclr,
  title={A Simple Framework for Contrastive Learning of Visual Representations},
  author={Chen, Ting and Kornblith, Simon and Norouzi, Mohammad and Hinton, Geoffrey},
  booktitle={International Conference on Machine Learning},
  pages={1597--1607},
  year={2020}
}

@article{oord2018infonce,
  title={Representation Learning with Contrastive Predictive Coding},
  author={van den Oord, Aaron and Li, Yazhe and Vinyals, Oriol},
  journal={arXiv preprint arXiv:1807.03748},
  year={2018}
}

@inproceedings{gao2021simcse,
  title={{SimCSE}: Simple Contrastive Learning of Sentence Embeddings},
  author={Gao, Tianyu and Yao, Xingcheng and Chen, Danqi},
  booktitle={Proceedings of the 2021 Conference on Empirical Methods in Natural Language Processing},
  pages={6894--6910},
  year={2021}
}

@inproceedings{radford2021clip,
  title={Learning Transferable Visual Models From Natural Language Supervision},
  author={Radford, Alec and Kim, Jong Wook and Hallacy, Chris and Ramesh, Aditya and Goh, Gabriel and Agarwal, Sandhini and Sastry, Girish and Askell, Amanda and Mishkin, Pamela and Clark, Jack and others},
  booktitle={International Conference on Machine Learning},
  pages={8748--8763},
  year={2021}
}

@inproceedings{wang2020hypersphere,
  title={Understanding Contrastive Representation Learning through Alignment and Uniformity on the Hypersphere},
  author={Wang, Tongzhou and Isola, Phillip},
  booktitle={International Conference on Machine Learning},
  pages={9929--9939},
  year={2020}
}

@inproceedings{wang2017normface,
  title={{NormFace}: {L2} Hypersphere Embedding for Face Verification},
  author={Wang, Feng and Xiang, Xiang and Cheng, Jian and Yuille, Alan L.},
  booktitle={Proceedings of the 25th ACM International Conference on Multimedia},
  pages={1041--1049},
  year={2017}
}

@article{ranjan2017l2,
  title={{L2}-Constrained Softmax Loss for Discriminative Face Verification},
  author={Ranjan, Rajeev and Castillo, Carlos D. and Chellappa, Rama},
  journal={arXiv preprint arXiv:1703.09507},
  year={2017}
}

@inproceedings{steck2024cosine,
  title={Is Cosine-Similarity of Embeddings Really About Similarity?},
  author={Steck, Harald and Ekanadham, Chaitanya and Kallus, Nathan},
  booktitle={Companion Proceedings of the ACM on Web Conference 2024},
  pages={887--890},
  year={2024}
}

@article{draganov2024pitfalls,
  title={The Hidden Pitfalls of the Cosine Similarity Loss},
  author={Draganov, Andrew and Vadgama, Sharvaree and Bekkers, Erik J.},
  journal={arXiv preprint arXiv:2406.16468},
  year={2024}
}

@inproceedings{snell2017prototypical,
  title={Prototypical Networks for Few-shot Learning},
  author={Snell, Jake and Swersky, Kevin and Zemel, Richard},
  booktitle={Advances in Neural Information Processing Systems},
  year={2017}
}

@inproceedings{sun2019rotate,
  title={{RotatE}: Knowledge Graph Embedding by Relational Rotation in Complex Space},
  author={Sun, Zhiqing and Deng, Zhi-Hong and Nie, Jian-Yun and Tang, Jian},
  booktitle={International Conference on Learning Representations},
  year={2019}
}

@inproceedings{chao2021pairre,
  title={{PairRE}: Knowledge Graph Embeddings via Paired Relation Vectors},
  author={Chao, Linlin and He, Jianshan and Wang, Taifeng and Chu, Wei},
  booktitle={Proceedings of the 59th Annual Meeting of the Association for Computational Linguistics},
  pages={4360--4369},
  year={2021}
}

@inproceedings{he2020lightgcn,
  title={{LightGCN}: Simplifying and Powering Graph Convolution Network for Recommendation},
  author={He, Xiangnan and Deng, Kuan and Wang, Xiang and Li, Yan and Zhang, Yongdong and Wang, Meng},
  booktitle={Proceedings of the 43rd International ACM SIGIR Conference},
  pages={639--648},
  year={2020}
}

@inproceedings{mao2021simplex,
  title={{SimpleX}: A Simple and Strong Baseline for Collaborative Filtering},
  author={Mao, Kelong and Zhu, Jieming and Wang, Jinpeng and Dai, Quanyu and Dong, Zhenhua and Xiao, Xi and He, Xiuqiang},
  booktitle={Proceedings of the 30th ACM International Conference on Information \& Knowledge Management},
  pages={1243--1252},
  year={2021}
}

@inproceedings{xiao2024c,
  title={{C-Pack}: Packed Resources For General Chinese Embeddings},
  author={Xiao, Shitao and Liu, Zheng and Zhang, Peitian and Muennighoff, Niklas and Lian, Defu and Nie, Jian-Yun},
  booktitle={Proceedings of the 47th International ACM SIGIR Conference on Research and Development in Information Retrieval},
  year={2024}
}

@inproceedings{trouillon2016complex,
  title={Complex Embeddings for Simple Link Prediction},
  author={Trouillon, Th{\'e}o and Welbl, Johannes and Riedel, Sebastian and Gaussier, {\'E}ric and Bouchard, Guillaume},
  booktitle={International Conference on Machine Learning},
  pages={2071--2080},
  year={2016}
}

@article{oquab2023dinov2,
  title={{DINOv2}: Learning Robust Visual Features without Supervision},
  author={Oquab, Maxime and Darcet, Timoth{\'e}e and Moutakanni, Th{\'e}o and others},
  journal={arXiv preprint arXiv:2304.07193},
  year={2023}
}

@inproceedings{toutanova2015fb15k237,
  title={Observed versus Latent Features for Knowledge Base and Text Inference},
  author={Toutanova, Kristina and Chen, Danqi},
  booktitle={Proceedings of the 3rd Workshop on Continuous Vector Space Models and Their Compositionality},
  pages={57--66},
  year={2015}
}

@inproceedings{dettmers2018conve,
  title={Convolutional 2D Knowledge Graph Embeddings},
  author={Dettmers, Tim and Minervini, Pasquale and Stenetorp, Pontus and Riedel, Sebastian},
  booktitle={Proceedings of the AAAI Conference on Artificial Intelligence},
  year={2018}
}

@techreport{krizhevsky2009cifar,
  title={Learning Multiple Layers of Features from Tiny Images},
  author={Krizhevsky, Alex},
  institution={University of Toronto},
  year={2009}
}

@inproceedings{vinyals2016matching,
  title={Matching Networks for One Shot Learning},
  author={Vinyals, Oriol and Blundell, Charles and Lillicrap, Timothy and Wierstra, Daan and others},
  booktitle={Advances in Neural Information Processing Systems},
  year={2016}
}

@inproceedings{oyama2023norm,
  title={Norm of Word Embedding Encodes Information Gain},
  author={Oyama, Momose and Yokoi, Sho and Shimodaira, Hidetoshi},
  booktitle={Proceedings of the 2023 Conference on Empirical Methods in Natural Language Processing},
  pages={2108--2130},
  year={2023}
}

@inproceedings{wei2022mitigating,
  title={Mitigating Neural Network Overconfidence with Logit Normalization},
  author={Wei, Hongxin and Xie, Renchunzi and Cheng, Hao and Feng, Lei and An, Bo and Li, Yixuan},
  booktitle={International Conference on Machine Learning},
  pages={23631--23644},
  year={2022}
}

@inproceedings{cer2017semeval,
  title={{SemEval-2017 Task 1}: Semantic Textual Similarity Multilingual and Crosslingual Focused Evaluation},
  author={Cer, Daniel and Diab, Mona and Agirre, Eneko and Lopez-Gazpio, I{\~n}igo and Specia, Lucia},
  booktitle={Proceedings of the 11th International Workshop on Semantic Evaluation (SemEval-2017)},
  pages={1--14},
  year={2017}
}

@inproceedings{lewis2020rag,
  title={Retrieval-Augmented Generation for Knowledge-Intensive {NLP} Tasks},
  author={Lewis, Patrick and Perez, Ethan and Piktus, Aleksandra and Petroni, Fabio and Karpukhin, Vladimir and Goyal, Naman and K{\"u}ttler, Heinrich and Lewis, Mike and Yih, Wen-tau and Rockt{\"a}schel, Tim and others},
  booktitle={Advances in Neural Information Processing Systems},
  volume={33},
  pages={9459--9474},
  year={2020}
}
\bibliographystyle{plainnat}

\newpage
\appendix


\section{Pre-trained Model Details}
\label{appendix:model_details}

Table~\ref{tab:model_details} summarizes the pre-training details of the three retrievers used in our experiments.

\begin{table}[h]
\centering
\caption{Pre-training Details of Dense Retrievers Used in This Study. MAE = Masked Auto-Encoder. InfoNCE = Noise Contrastive Estimation loss for contrastive learning. Table information is collected from original papers and official sources (e.g., Hugging Face). For unsupervised/weakly supervised stages, only document sources are shown; query synthesis methods are omitted.}
\label{tab:model_details}
\resizebox{\textwidth}{!}{
\begin{tabular}{lp{4.5cm}p{4.5cm}p{10cm}}
\toprule
 & \textbf{Contriever} & \textbf{RetroMAE} & \textbf{E5} \\
\midrule
Checkpoint & facebook/contriever & Shitao/RetroMAE & intfloat/e5-base-unsupervised \\
Backbone & google-bert/bert-base-uncased & google-bert/bert-base-uncased & google-bert/bert-base-uncased \\
Checkpoint Size & 439\,MB & 438\,MB & 877\,MB \\
Pre-training Loss & InfoNCE & MAE & InfoNCE \\
Pre-training Data & Wikipedia, CCNet & Wikipedia, BookCorpus & CCPairs, including Wikipedia, Reddit, Common Crawl (web pages from MS-MARCO document ranking corpus included), Stackexchange, S2ORC, News, SimpleWiki, GooAQ, WikiHow, Yahoo Answers \\
\bottomrule
\end{tabular}
}
\end{table}

A key distinction among these models is the composition of their pre-training data. Contriever's pre-training uses Wikipedia and CCNet, while RetroMAE uses Wikipedia and BookCorpus. In contrast, E5's pre-training corpus (CCPairs) explicitly includes web pages from the MS-MARCO document ranking corpus through Common Crawl. This data overlap has important implications for our finetuning experiments, as discussed in Section~\ref{sec:experiments}.

\section{Training Configuration}
\label{appendix:training_config}

To ensure fair comparison and reproducibility, all BERT-based retriever models in our experiments were trained using a unified configuration. This section details the hyperparameters and settings used.

\subsection{General Training Hyperparameters}

Table~\ref{tab:training_hyperparams} summarizes the training hyperparameters for all models.

\begin{table}[h]
\centering
\caption{Training Hyperparameters for All Models}
\label{tab:training_hyperparams}
\begin{tabular}{lll}
\toprule
\textbf{Hyperparameter} & \textbf{Contriever/RetroMAE/E5} & \textbf{Qwen3-Base} \\
\midrule
Training data & MS MARCO v1.1 QA & MS MARCO v1.1 QA / Passage Ranking \\
Training samples & 81,795 & 81,795 / 502,939 \\
Number of GPUs & 2 (data parallel) & 2 (data parallel) \\
Per-device batch size & 128 & 64 \\
Effective batch size & 256 & 128 \\
Number of epochs & 100 & 40 \\
Learning rate & $2\text{e-}6$ (Contriever), $5\text{e-}6$ (others) & $5 \times 10^{-6}$ \\
Warmup ratio & 0 & 0 \\
Weight decay & 0.01 & 0.01 \\
Max gradient norm & 1.0 & 1.0 \\
LR scheduler & Cosine decay & Cosine decay \\
Random seeds & 0, 42, 1337 & 0 \\
\midrule
\multicolumn{3}{l}{\textit{AdamW Optimizer: $\beta_1 = 0.9$, $\beta_2 = 0.98$, $\epsilon = 10^{-8}$}} \\
\bottomrule
\end{tabular}
\end{table}

\paragraph{Multi-GPU Training.} All experiments use distributed data parallel training across 2 GPUs via Hugging Face Accelerate. The effective batch size is $128 \times 2 = 256$ samples per optimization step. Training steps are computed based on the effective batch size: steps per epoch $= \lceil 81795 / 256 \rceil = 320$, and 100 epochs corresponds to approximately 32,000 steps. Checkpoints are saved and evaluated every 500 steps (approximately 1.6 epochs).

\paragraph{Learning Rate Selection.} We conduct learning rate sweeps using cosine similarity training to determine the optimal learning rate for each model. Learning rates are selected from $\{2\text{e-}7, 5\text{e-}7, 1\text{e-}6, 2\text{e-}6, 5\text{e-}6, 1\text{e-}5\}$. Table~\ref{tab:lr_selection} shows the results. The selected learning rates are then fixed across all similarity function variants for fair comparison.

\begin{table}[h]
\centering
\caption{Learning Rate Selection via Grid Search (Cosine Similarity Training)}
\label{tab:lr_selection}
\begin{tabular}{lcc}
\toprule
\textbf{Model} & \textbf{Best LR} & \textbf{Val NDCG@10} \\
\midrule
Contriever & $2 \times 10^{-6}$ & 92.87 \\
RetroMAE & $5 \times 10^{-6}$ & 91.87 \\
E5 & $5 \times 10^{-6}$ & 93.97 \\
Qwen3-Base & $5 \times 10^{-6}$ & 94.12 \\
\bottomrule
\end{tabular}
\end{table}

\subsection{Evaluation Configuration}

Table~\ref{tab:eval_config} summarizes the evaluation settings used during training and final evaluation.

\begin{table}[h]
\centering
\caption{Evaluation Configuration}
\label{tab:eval_config}
\begin{tabular}{ll}
\toprule
\textbf{Setting} & \textbf{Value} \\
\midrule
Evaluation batch size & 64 \\
Similarity function & Corresponding to loss function \\
Primary metric & NDCG@10 \\
Checkpoint save interval & 500 steps \\
Evaluation interval & 500 steps \\
\bottomrule
\end{tabular}
\end{table}

Note that each model variant is evaluated using its corresponding similarity function: Cosine-trained models use cosine similarity, Dot-trained models use dot product, QNorm-trained models use query-normalized dot product, DNorm-trained models use document-normalized dot product, and Learnable-trained models use the learned $\gamma$-interpolated similarity.

\paragraph{Training-time validation vs final evaluation.}
For computational efficiency during training, our training-time validation NDCG@10 (Figure~\ref{fig:training_curves}, Table~\ref{tab:seed_sensitivity}) is computed over a small in-domain corpus constructed only from the validation queries' relevant passages (e.g., $\sim$10K docs for MS MARCO v1.1 validation, all of which are relevant to some query). This makes the validation a relatively easy ranking problem ($\sim$92--94\% NDCG@10), but the metric is sufficient for model selection and tracking convergence. In contrast, our final evaluation in the main text (Table~\ref{tab:main_simplified}, Table~\ref{tab:foundation}) uses the full benchmark corpora (e.g., 8.8M docs for MS MARCO Dev, 5.4M for FEVER), which is why final NDCG@10 scores ($\sim$31--33\%) are substantially lower despite the same models. Method rankings are preserved between the two settings, validating our use of the small-corpus validation for early stopping and hyperparameter selection.

\section{Training Convergence Analysis}
\label{appendix:convergence}

To ensure reliable evaluation results, we train all models for sufficient epochs until convergence and select the best checkpoint based on validation performance.

\subsection{Convergence Criteria}

We consider a model converged when both conditions are met: (1) the standard deviation of NDCG@10 over the last 5 evaluation points is less than 0.002, and (2) the difference between the best and final NDCG@10 is less than 0.02. We evaluate validation performance every 500 training steps.

\subsection{Convergence Results}
\label{appendix:first_convergence}

Table~\ref{tab:convergence_all} summarizes the convergence dynamics for all three retrievers trained on MS MARCO v1.1 QA. We report both the first convergence point (when validation performance first stabilizes) and the best checkpoint performance, showing that models continue to improve after initial convergence.

\begin{table}[ht]
\centering
\caption{Convergence Analysis for All Retrievers on MS MARCO v1.1 QA (seed=0). ``First Conv.'' indicates when the model first satisfies convergence criteria (std of last 5 eval points $< 0.002$ and best-current diff $< 0.02$). ``Gain'' shows NDCG@10 improvement from first convergence to best checkpoint. Steps per epoch $\approx 320$.}
\label{tab:convergence_all}
\resizebox{0.6\textwidth}{!}{
\begin{tabular}{llcccccc}
\toprule
& & \multicolumn{3}{c}{\textbf{First Convergence}} & \multicolumn{2}{c}{\textbf{Best Checkpoint}} & \\
\cmidrule(lr){3-5} \cmidrule(lr){6-7}
\textbf{Model} & \textbf{Loss} & \textbf{Step} & \textbf{Epoch} & \textbf{NDCG@10} & \textbf{Epoch} & \textbf{NDCG@10} & \textbf{Gain} \\
\midrule
\multirow{5}{*}{Contriever}
 & Cosine & 3,500 & 11.0 & 92.00 & 56.3 & \textbf{92.82} & +0.82 \\
 & Dot & 3,500 & 11.0 & 92.49 & 75.1 & \textbf{93.50} & +1.01 \\
 & QNorm & 5,000 & 15.6 & 90.63 & 98.6 & \textbf{91.70} & +1.07 \\
 & DNorm & 4,000 & 12.5 & 92.57 & 84.5 & \textbf{93.57} & +1.00 \\
 & Learnable & 5,000 & 15.6 & 92.97 & 82.8 & \textbf{93.58} & +0.61 \\
\midrule
\multirow{5}{*}{RetroMAE}
 & Cosine & 3,500 & 11.0 & 91.07 & 42.3 & \textbf{91.88} & +0.81 \\
 & Dot & 6,500 & 20.3 & 91.88 & 89.2 & \textbf{93.05} & +1.17 \\
 & QNorm & 3,000 & 9.4 & 90.11 & 53.2 & \textbf{91.26} & +1.15 \\
 & DNorm & 4,000 & 12.5 & 91.90 & 68.9 & \textbf{93.08} & +1.18 \\
 & Learnable & 4,500 & 14.1 & 92.04 & 85.9 & \textbf{93.09} & +1.05 \\
\midrule
\multirow{5}{*}{E5}
 & Cosine & 2,500 & 7.8 & 93.91 & 48.5 & \textbf{94.04} & +0.13 \\
 & Dot & 2,500 & 7.8 & 92.94 & 90.8 & \textbf{94.43} & +1.49 \\
 & QNorm & 4,000 & 12.5 & 93.95 & 67.3 & \textbf{94.62} & +0.67 \\
 & DNorm & 3,000 & 9.4 & 94.34 & 70.4 & \textbf{94.80} & +0.46 \\
 & Learnable & 4,000 & 12.5 & 94.23 & 68.8 & \textbf{94.77} & +0.54 \\
\bottomrule
\end{tabular}
}
\end{table}

All fifteen model-loss combinations achieved stable convergence. Models first satisfy convergence criteria within 8--20 epochs (2,500--6,500 steps), but continue to improve by 0.13--1.49 percentage points before reaching their best performance. This justifies our choice of 100 epochs as a conservative upper bound.

\subsection{Convergence Curves}

Figure~\ref{fig:convergence_curves} shows the NDCG@10 progression during training for all model-loss combinations.


\begin{figure*}[t]
\centering
\begin{tikzpicture}
\begin{groupplot}[
    group style={
        group size=3 by 1,
        horizontal sep=1.0cm,
    },
    width=0.36\textwidth,
    height=0.26\textwidth,
    xlabel={Training Steps},
    xmin=0, xmax=35000,
    xtick={0, 5000, 10000, 15000, 20000, 25000, 30000, 35000},
    xticklabels={0, 5k, 10k, 15k, 20k, 25k, 30k, 35k},
    scaled x ticks=false,
    ymin=0.87, ymax=0.96,
    grid=major,
    grid style={dashed,gray!30},
    tick label style={font=\small},
    label style={font=\small},
    title style={font=\normalsize\bfseries},
    every axis plot/.append style={semithick, dashed, mark size=1.5pt},
]

\nextgroupplot[title={Contriever}, ylabel={Val NDCG@10}]
\addplot[blue!80, semithick, dashed, mark=square*, mark size=1.5pt, mark repeat=5]
    coordinates {(500,0.9027) (1500,0.9152) (3000,0.9195) (4000,0.9211) (5500,0.9241) (7000,0.9259) (8000,0.9255) (9500,0.9263) (11000,0.9272) (12000,0.9266) (13500,0.9275) (14500,0.9275) (16000,0.9281) (17500,0.9283) (18500,0.9275) (20000,0.9278) (21500,0.9273) (22500,0.9276) (24000,0.9276) (25000,0.9274) (26500,0.9277) (28000,0.9276) (29000,0.9280) (30500,0.9282) (32000,0.9280)};
\addplot[orange, semithick, dashed, mark=diamond*, mark size=1.5pt, mark repeat=5]
    coordinates {(500,0.9103) (1500,0.9192) (3000,0.9237) (4000,0.9258) (5500,0.9276) (6500,0.9283) (8000,0.9299) (9500,0.9305) (10500,0.9323) (12000,0.9321) (13000,0.9318) (14500,0.9330) (16000,0.9332) (17000,0.9335) (18500,0.9334) (19500,0.9343) (21000,0.9347) (22000,0.9344) (23500,0.9349) (25000,0.9350) (26000,0.9349) (27500,0.9349) (28500,0.9349) (30000,0.9349) (31500,0.9350)};
\addplot[red!80, semithick, dashed, mark=triangle*, mark size=1.5pt, mark repeat=5]
    coordinates {(500,0.8818) (1500,0.8949) (3000,0.9021) (4000,0.9062) (5500,0.9052) (6500,0.9090) (8000,0.9113) (9500,0.9110) (10500,0.9137) (12000,0.9133) (13000,0.9133) (14500,0.9148) (16000,0.9148) (17000,0.9154) (18500,0.9160) (19500,0.9163) (21000,0.9164) (22000,0.9168) (23500,0.9165) (25000,0.9168) (26000,0.9169) (27500,0.9170) (28500,0.9169) (30000,0.9169) (31500,0.9170)};
\addplot[teal, thick, mark=o, mark size=1.5pt, mark repeat=5]
    coordinates {(500,0.9064) (1500,0.9182) (3000,0.9246) (4000,0.9257) (5500,0.9295) (6500,0.9302) (8000,0.9316) (9500,0.9323) (10500,0.9327) (12000,0.9334) (13000,0.9332) (14500,0.9336) (16000,0.9337) (17000,0.9336) (18500,0.9345) (19500,0.9347) (21000,0.9352) (22000,0.9354) (23500,0.9350) (25000,0.9354) (26000,0.9354) (27500,0.9356) (28500,0.9355) (30000,0.9355) (31500,0.9356)};
\addplot[magenta, very thick, mark=*, mark size=1.5pt, mark repeat=5]
    coordinates {(500,0.9126) (1500,0.9202) (3000,0.9263) (4000,0.9291) (5500,0.9302) (6500,0.9316) (8000,0.9322) (9500,0.9321) (10500,0.9329) (12000,0.9331) (13000,0.9333) (14500,0.9345) (16000,0.9344) (17000,0.9345) (18500,0.9347) (19500,0.9348) (21000,0.9354) (22000,0.9356) (23500,0.9358) (25000,0.9357) (26000,0.9356) (27500,0.9355) (28500,0.9355) (30000,0.9357) (31500,0.9357)};

\nextgroupplot[title={RetroMAE}, yticklabels={}]
\addplot[blue!80, semithick, dashed, mark=square*, mark size=1.5pt, mark repeat=5]
    coordinates {(500,0.8967) (1500,0.9064) (3000,0.9105) (4000,0.9133) (5500,0.9132) (6500,0.9160) (8000,0.9165) (9500,0.9165) (10500,0.9170) (12000,0.9170) (13000,0.9179) (14500,0.9174) (16000,0.9171) (17000,0.9179) (18500,0.9176) (19500,0.9168) (21000,0.9174) (22000,0.9166) (23500,0.9174) (25000,0.9178) (26000,0.9181) (27500,0.9176) (28500,0.9176) (30000,0.9173) (31500,0.9174)};
\addplot[orange, semithick, dashed, mark=diamond*, mark size=1.5pt, mark repeat=5]
    coordinates {(500,0.8768) (1500,0.8978) (3000,0.9111) (4000,0.9111) (5500,0.9152) (6500,0.9188) (8000,0.9200) (9500,0.9202) (10500,0.9241) (12000,0.9254) (13000,0.9238) (14500,0.9257) (16000,0.9265) (17000,0.9274) (18500,0.9279) (19500,0.9283) (21000,0.9299) (22000,0.9301) (23500,0.9291) (25000,0.9299) (26000,0.9299) (27500,0.9300) (28500,0.9305) (30000,0.9302) (31500,0.9303)};
\addplot[red!80, semithick, dashed, mark=triangle*, mark size=1.5pt, mark repeat=5]
    coordinates {(500,0.8885) (1500,0.8992) (3000,0.9011) (4000,0.9032) (5500,0.9035) (6500,0.9042) (8000,0.9067) (9500,0.9055) (10500,0.9094) (12000,0.9100) (13000,0.9099) (14500,0.9083) (16000,0.9114) (17000,0.9126) (18500,0.9113) (19500,0.9118) (21000,0.9105) (22000,0.9109) (23500,0.9101) (25000,0.9099) (26000,0.9110) (27500,0.9106) (28500,0.9107) (30000,0.9106) (31500,0.9107)};
\addplot[teal, thick, mark=o, mark size=1.5pt, mark repeat=5]
    coordinates {(500,0.9017) (1500,0.9123) (3000,0.9186) (4000,0.9190) (5500,0.9229) (6500,0.9239) (8000,0.9248) (9500,0.9244) (10500,0.9274) (12000,0.9274) (13000,0.9280) (14500,0.9275) (16000,0.9276) (17000,0.9284) (18500,0.9282) (19500,0.9287) (21000,0.9302) (22000,0.9308) (23500,0.9303) (25000,0.9302) (26000,0.9308) (27500,0.9307) (28500,0.9307) (30000,0.9305) (31500,0.9306)};
\addplot[magenta, very thick, mark=*, mark size=1.5pt, mark repeat=5]
    coordinates {(500,0.9014) (1500,0.9130) (3000,0.9194) (4000,0.9193) (5500,0.9217) (6500,0.9237) (8000,0.9233) (9500,0.9263) (10500,0.9276) (12000,0.9272) (13000,0.9265) (14500,0.9275) (16000,0.9276) (17000,0.9294) (18500,0.9293) (19500,0.9290) (21000,0.9287) (22000,0.9305) (23500,0.9297) (25000,0.9299) (26000,0.9307) (27500,0.9309) (28500,0.9307) (30000,0.9306) (31500,0.9307)};

\nextgroupplot[title={E5}, yticklabels={}]
\addplot[blue!80, semithick, dashed, mark=square*, mark size=1.5pt, mark repeat=5]
    coordinates {(500,0.9336) (1500,0.9369) (3000,0.9376) (4000,0.9378) (5500,0.9380) (6500,0.9382) (8000,0.9390) (9500,0.9383) (10500,0.9388) (12000,0.9391) (13000,0.9393) (14500,0.9391) (16000,0.9389) (17000,0.9393) (18500,0.9395) (19500,0.9394) (21000,0.9392) (22000,0.9393) (23500,0.9390) (25000,0.9390) (26000,0.9388) (27500,0.9388) (28500,0.9389) (30000,0.9387) (31500,0.9388)};
\addplot[orange, semithick, dashed, mark=diamond*, mark size=1.5pt, mark repeat=5]
    coordinates {(500,0.9284) (1500,0.9284) (3000,0.9304) (4000,0.9341) (5500,0.9333) (6500,0.9346) (8000,0.9372) (9500,0.9371) (10500,0.9390) (12000,0.9387) (13000,0.9389) (14500,0.9395) (16000,0.9402) (17000,0.9415) (18500,0.9420) (19500,0.9423) (21000,0.9423) (22000,0.9430) (23500,0.9435) (25000,0.9437) (26000,0.9435) (27500,0.9439) (28500,0.9438) (30000,0.9441) (31500,0.9441)};
\addplot[red!80, semithick, dashed, mark=triangle*, mark size=1.5pt, mark repeat=5]
    coordinates {(500,0.9288) (1500,0.9335) (3000,0.9391) (4000,0.9395) (5500,0.9407) (6500,0.9418) (8000,0.9424) (9500,0.9442) (10500,0.9434) (12000,0.9447) (13000,0.9430) (14500,0.9433) (16000,0.9430) (17000,0.9457) (18500,0.9452) (19500,0.9455) (21000,0.9453) (22000,0.9461) (23500,0.9459) (25000,0.9457) (26000,0.9458) (27500,0.9461) (28500,0.9460) (30000,0.9460) (31500,0.9460)};
\addplot[teal, thick, mark=o, mark size=1.5pt, mark repeat=5]
    coordinates {(500,0.9354) (1500,0.9406) (3000,0.9434) (4000,0.9442) (5500,0.9446) (6500,0.9439) (8000,0.9458) (9500,0.9456) (10500,0.9464) (12000,0.9462) (13000,0.9457) (14500,0.9470) (16000,0.9463) (17000,0.9463) (18500,0.9474) (19500,0.9480) (21000,0.9475) (22000,0.9478) (23500,0.9479) (25000,0.9475) (26000,0.9480) (27500,0.9479) (28500,0.9480) (30000,0.9479) (31500,0.9477)};
\addplot[magenta, very thick, mark=*, mark size=1.5pt, mark repeat=5]
    coordinates {(500,0.9329) (1500,0.9377) (3000,0.9418) (4000,0.9423) (5500,0.9426) (6500,0.9434) (8000,0.9431) (9500,0.9445) (10500,0.9453) (12000,0.9448) (13000,0.9446) (14500,0.9450) (16000,0.9447) (17000,0.9456) (18500,0.9466) (19500,0.9473) (21000,0.9469) (22000,0.9477) (23500,0.9468) (25000,0.9466) (26000,0.9470) (27500,0.9473) (28500,0.9472) (30000,0.9473) (31500,0.9474)};

\end{groupplot}
\end{tikzpicture}
\vspace{0.5em}
\centerline{%
  \tikz{\draw[blue!80, semithick, dashed] (0,0) -- (0.5,0); \node[blue!80, mark size=2pt] at (0.25,0) {\pgfuseplotmark{square*}};} \small~Cosine \quad
  \tikz{\draw[orange, semithick, dashed] (0,0) -- (0.5,0); \node[orange, mark size=2pt] at (0.25,0) {\pgfuseplotmark{diamond*}};} \small~Dot \quad
  \tikz{\draw[red!80, semithick, dashed] (0,0) -- (0.5,0); \node[red!80, mark size=2pt] at (0.25,0) {\pgfuseplotmark{triangle*}};} \small~QNorm \quad
  \tikz{\draw[teal, thick] (0,0) -- (0.5,0); \node[teal, mark size=2pt] at (0.25,0) {\pgfuseplotmark{o}};} \small~DNorm \quad
  \tikz{\draw[magenta, very thick] (0,0) -- (0.5,0); \node[magenta, mark size=2pt] at (0.25,0) {\pgfuseplotmark{*}};} \small~Learnable%
}
\caption{Validation NDCG@10 during training for Contriever, RetroMAE, and E5 with different similarity functions (seed=0). All models demonstrate stable convergence across all loss function variants.}
\label{fig:convergence_curves}
\end{figure*}

The convergence curves reveal consistent patterns across all three retrievers. All variants show rapid initial improvement within the first 5,000--10,000 steps, followed by gradual convergence with performance stabilizing after approximately 15,000--25,000 steps. For Contriever and RetroMAE, the Dot, DNorm, and Learnable variants consistently track above the Cosine baseline throughout training, while QNorm converges to lower performance. For E5, all five variants converge to similar high performance on the validation set, with magnitude-aware variants slightly outperforming Cosine.

\paragraph{Key observations.}
\begin{itemize}[leftmargin=1.5em, itemsep=0.2em]
    \item \textbf{Early convergence}: All models first satisfy convergence criteria within 8--20 epochs (approximately 2,500--6,500 steps), indicating rapid stabilization of training dynamics.
    \item \textbf{Continued improvement}: Despite meeting convergence criteria early, models continue to improve by 0.13--1.49 percentage points in NDCG@10 before reaching their best performance. This ``post-convergence gain'' represents meaningful improvement that would be missed by early stopping at first convergence.
    \item \textbf{Model-specific patterns}: E5 converges fastest (2,500--4,000 steps, 8--12 epochs) but shows variable post-convergence gains. Notably, E5-Dot shows the largest gain (+1.49\%), suggesting that magnitude-aware training requires more optimization for E5's architecture. Contriever and RetroMAE show more consistent post-convergence gains around +1\%.
    \item \textbf{Implications for training budget}: These results suggest that training for approximately 60--90 epochs (rather than stopping at 10--20 epochs when first convergence is detected) captures important gradual improvements. This justifies our choice of 100 epochs as a conservative upper bound that ensures all models reach their optimal performance.
\end{itemize}

\section{Loss Scale Sensitivity Analysis}
\label{appendix:loss_scale}

The InfoNCE loss uses a scale factor $\tau$ (inverse temperature) that controls the sharpness of the softmax distribution over candidates. Our main experiments use $\tau = 20$ following common practice. Here we analyze the sensitivity of different similarity functions to this hyperparameter.

\subsection{Experimental Setup}

We train Contriever with each similarity function using four scale values: $\tau \in \{10, 20, 30, 40\}$ on MS MARCO v1.1 QA (82K samples). All other hyperparameters remain identical to the main experiments (Section~\ref{appendix:training_config}).

\subsection{Results}

Table~\ref{tab:loss_scale} presents the NDCG@10 scores across different scales for in-domain (TREC-DL 2019/2020 average) and three out-of-domain benchmark categories (BEIR, BRIGHT, and Multi-hop).

\begin{table*}[t]
  \caption{Loss scale ($\tau$) sensitivity for Contriever finetuned on MS MARCO 82K.
  NDCG@10 ($\times 100$) for different loss scale values. Default $\tau=20$ is used in main experiments.
  \textbf{Bold} indicates best, \underline{underline} indicates second best per method within each benchmark.}
  \label{tab:loss_scale}
  \begin{center}
    \resizebox{\textwidth}{!}{
      \begin{tabular}{lcccc|cccc|cccc|cccc}
        \toprule
        & \multicolumn{4}{c|}{\textbf{In-Domain}} & \multicolumn{4}{c|}{\textbf{BEIR}} & \multicolumn{4}{c|}{\textbf{BRIGHT}} & \multicolumn{4}{c}{\textbf{Multi-hop}} \\
        \cmidrule(lr){2-5} \cmidrule(lr){6-9} \cmidrule(lr){10-13} \cmidrule(lr){14-17}
        \textbf{Method} & $\tau$=10 & $\tau$=20 & $\tau$=30 & $\tau$=40 & $\tau$=10 & $\tau$=20 & $\tau$=30 & $\tau$=40 & $\tau$=10 & $\tau$=20 & $\tau$=30 & $\tau$=40 & $\tau$=10 & $\tau$=20 & $\tau$=30 & $\tau$=40 \\
        \midrule
        Cosine & 51.4 & \textbf{57.5} & \underline{57.1} & 57.0 & 35.6 & 40.4 & \textbf{43.0} & \underline{42.8} & 6.0 & 7.1 & \underline{9.3} & \textbf{9.5} & 41.2 & 50.4 & \underline{56.7} & \textbf{57.9} \\
        Dot & \textbf{56.7} & \underline{56.4} & 55.9 & 55.7 & 43.1 & \underline{43.5} & 43.0 & \textbf{43.5} & \underline{12.4} & \textbf{12.6} & 12.2 & 12.2 & 57.0 & \underline{58.0} & 58.0 & \textbf{58.2} \\
        QNorm & 56.7 & \textbf{57.7} & \underline{56.8} & 56.3 & 42.1 & \textbf{44.1} & 43.6 & \underline{43.7} & 11.5 & \textbf{13.0} & \underline{12.6} & 12.5 & 55.4 & 57.2 & \underline{57.8} & \textbf{58.3} \\
        DNorm & \textbf{56.7} & \underline{56.2} & 55.2 & 55.4 & \underline{42.4} & \textbf{42.6} & 42.0 & 42.1 & \textbf{10.1} & \underline{9.9} & 9.7 & 9.6 & 56.9 & \textbf{57.7} & \underline{57.4} & 57.3 \\
        Learnable & \textbf{58.4} & \underline{58.3} & 57.5 & 57.4 & 42.4 & 43.7 & \underline{43.8} & \textbf{43.9} & 10.5 & 11.7 & \underline{12.4} & \textbf{12.6} & 55.7 & 57.1 & \underline{57.5} & \textbf{58.0} \\
        \bottomrule
      \end{tabular}
    }
  \end{center}
\end{table*}

\paragraph{Key findings.}

\begin{itemize}[leftmargin=1.5em, itemsep=0.2em]
    \item \textbf{In-domain stability}: Most magnitude-aware methods (Dot, QNorm, DNorm, Learnable) show small performance variation across scales ($\leq 1.5$ NDCG@10 range), indicating robustness to this hyperparameter. Cosine similarity shows larger sensitivity (6.1 range), with poor performance at $\tau=10$.

    \item \textbf{OOD patterns}: Higher scale values ($\tau=30, 40$) generally improve out-of-domain generalization for Cosine, QNorm, and Learnable methods. Dot and DNorm remain stable across scales on OOD benchmarks.

    \item \textbf{Method ranking consistency}: Learnable achieves the best in-domain performance across all scale values, demonstrating that the learned normalization exponents provide consistent benefits regardless of the scale hyperparameter.

    \item \textbf{Default choice justification}: The default $\tau=20$ provides a good balance: it achieves optimal or near-optimal in-domain performance for most methods while maintaining competitive OOD performance. This justifies our choice for the main experiments.
\end{itemize}

These results suggest that while the loss scale does affect absolute performance, the relative ranking of similarity functions remains consistent, and our main conclusions about magnitude-aware training are robust to reasonable variations in this hyperparameter.

\section{Complete Evaluation Results}
\label{appendix:complete_results}

This section presents the complete evaluation results. We first show the full NDCG@10 results across all three training paradigms (Table~\ref{tab:combined_results}), then provide per-dataset results with all three metrics: NDCG@10 (N), Recall@100 (R), and MRR@10 (M).

\subsection{Full Results Across Training Paradigms}

\begin{table*}[t!]
  \caption{NDCG@10 across three training paradigms on MS MARCO. \textbf{Bold} = best, \underline{underline} = second best per model/data scale. PT = pretrained baseline (before finetuning).}
  \label{tab:combined_results}
  \centering

  \begin{subtable}{\textwidth}
    \centering
    \caption{Finetuning Pretrained Models (mean$\pm$std across 3 seeds)}
    \label{tab:finetune}
    \resizebox{\textwidth}{!}{
      \begin{tabular}{lcccccc|cccccc}
        \toprule
        & \multicolumn{6}{c|}{\textbf{Contriever}} & \multicolumn{6}{c}{\textbf{RetroMAE}} \\
        \cmidrule(lr){2-7} \cmidrule(lr){8-13}
        \textbf{Dataset} & \textbf{PT} & \textbf{Cos} & \textbf{Dot} & \textbf{QN} & \textbf{DN} & \textbf{Ln} & \textbf{PT} & \textbf{Cos} & \textbf{Dot} & \textbf{QN} & \textbf{DN} & \textbf{Ln} \\
        \midrule
        \multicolumn{13}{l}{\textit{In-Domain}} \\
        \quad MSM-Dev & 19.5 & 30.9$\pm$0.1 & 31.1$\pm$0.1 & 31.4$\pm$0.0 & \underline{31.4$\pm$0.1} & \textbf{32.2$\pm$0.1} & 2.9 & 30.7$\pm$0.1 & 31.2$\pm$0.2 & 32.2$\pm$0.1 & \textbf{32.9$\pm$0.2} & \underline{32.8$\pm$0.1} \\
        \quad DL-19 & 42.9 & 56.7$\pm$0.3 & 56.7$\pm$0.2 & \textbf{58.0$\pm$0.4} & 55.4$\pm$0.4 & \underline{57.8$\pm$0.1} & 14.1 & 56.1$\pm$0.3 & 56.8$\pm$0.8 & 57.5$\pm$0.4 & \underline{60.0$\pm$2.0} & \textbf{60.4$\pm$1.0} \\
        \quad DL-20 & 40.5 & \underline{57.3$\pm$0.7} & 56.3$\pm$0.2 & 56.9$\pm$0.6 & 57.2$\pm$0.0 & \textbf{58.5$\pm$0.3} & 9.0 & 56.9$\pm$0.6 & 57.8$\pm$0.9 & 58.1$\pm$0.7 & \underline{59.2$\pm$0.8} & \textbf{59.7$\pm$0.9} \\
        \midrule
        \multicolumn{13}{l}{\textit{Out-of-Domain}} \\
        \quad BEIR (14) & 28.7 & 41.0$\pm$0.5 & 43.5$\pm$0.1 & \textbf{44.0$\pm$0.3} & 42.5$\pm$0.2 & \underline{43.6$\pm$0.2} & 14.3 & 38.3$\pm$0.2 & 40.1$\pm$0.2 & 41.1$\pm$0.3 & \underline{41.2$\pm$0.2} & \textbf{41.2$\pm$0.3} \\
        \quad BRIGHT (12) & 3.8 & 7.4$\pm$0.3 & \underline{12.5$\pm$0.2} & \textbf{12.7$\pm$0.2} & 9.8$\pm$0.3 & 11.7$\pm$0.2 & 5.2 & 6.1$\pm$0.2 & 8.7$\pm$0.2 & 9.4$\pm$0.3 & \underline{9.5$\pm$0.2} & \textbf{9.8$\pm$0.1} \\
        \quad MHop (4) & 39.7 & 51.4$\pm$0.8 & \textbf{58.2$\pm$0.2} & 57.4$\pm$0.2 & \underline{57.5$\pm$0.2} & 57.2$\pm$0.1 & 18.5 & 50.7$\pm$0.4 & 52.9$\pm$0.4 & 54.8$\pm$0.3 & \textbf{55.5$\pm$0.6} & \underline{55.2$\pm$0.4} \\
        \bottomrule
      \end{tabular}
    }
  \end{subtable}

  \vspace{0.15cm}

  \begin{subtable}{\textwidth}
    \centering
    \caption{Training from Foundation Model (Qwen3-Base-0.6B, single seed)}
    \label{tab:foundation}
    \resizebox{0.6\textwidth}{!}{
      \begin{tabular}{lccccc|ccccc}
        \toprule
        & \multicolumn{5}{c|}{\textbf{MS MARCO QA (82K)}} & \multicolumn{5}{c}{\textbf{MS MARCO Passage (503K)}} \\
        \cmidrule(lr){2-6} \cmidrule(lr){7-11}
        \textbf{Dataset} & \textbf{Cos} & \textbf{Dot} & \textbf{QN} & \textbf{DN} & \textbf{Ln} & \textbf{Cos} & \textbf{Dot} & \textbf{QN} & \textbf{DN} & \textbf{Ln} \\
        \midrule
        \multicolumn{11}{l}{\textit{In-Domain}} \\
        \quad MSM-Dev & \underline{29.4} & 7.2 & 27.3 & 25.7 & \textbf{31.0} & \underline{31.3} & 13.5 & 26.8 & \textbf{33.5} & 15.8 \\
        \quad DL-19 & \underline{54.1} & 19.1 & 51.7 & 50.0 & \textbf{56.7} & \underline{56.4} & 23.5 & 51.6 & \textbf{61.2} & 39.5 \\
        \quad DL-20 & 54.6 & 10.6 & \underline{55.4} & 49.1 & \textbf{56.3} & \textbf{60.1} & 26.3 & 48.5 & \underline{59.9} & 37.8 \\
        \midrule
        \multicolumn{11}{l}{\textit{Out-of-Domain}} \\
        \quad BEIR (14) & \underline{42.6} & 15.4 & 29.7 & \textbf{42.7} & 40.3 & \underline{44.8} & 13.4 & 21.1 & \textbf{45.3} & 19.3 \\
        \quad BRIGHT (12) & \textbf{10.4} & 0.1 & 0.1 & \underline{9.8} & 2.2 & \underline{11.1} & 0.1 & 0.2 & \textbf{13.5} & 1.7 \\
        \quad MHop (4) & \textbf{49.2} & 18.7 & 37.4 & 44.1 & \underline{48.7} & \underline{54.6} & 16.8 & 47.2 & \textbf{57.6} & 10.7 \\
        \bottomrule
      \end{tabular}
    }
  \end{subtable}

  \vspace{0.15cm}

  \begin{subtable}{\textwidth}
    \centering
    \caption{Training from Random Initialization (mean$\pm$std across 3 seeds)}
    \label{tab:random_init}
    \resizebox{\textwidth}{!}{
      \begin{tabular}{lccccc|ccccc}
        \toprule
        & \multicolumn{5}{c|}{\textbf{Contriever}} & \multicolumn{5}{c}{\textbf{RetroMAE}} \\
        \cmidrule(lr){2-6} \cmidrule(lr){7-11}
        \textbf{Dataset} & \textbf{Cos} & \textbf{Dot} & \textbf{QN} & \textbf{DN} & \textbf{Ln} & \textbf{Cos} & \textbf{Dot} & \textbf{QN} & \textbf{DN} & \textbf{Ln} \\
        \midrule
        \multicolumn{11}{l}{\textit{In-Domain}} \\
        \quad MSM-Dev & \textbf{17.7$\pm$0.1} & 11.1$\pm$0.0 & 9.6$\pm$0.0 & \underline{15.2$\pm$0.2} & 12.0$\pm$0.1 & \textbf{19.2$\pm$0.1} & 14.1$\pm$0.1 & 13.7$\pm$0.2 & \underline{17.3$\pm$0.2} & 14.3$\pm$0.3 \\
        \quad DL-19 & \textbf{38.5$\pm$2.2} & 24.3$\pm$1.4 & 20.6$\pm$1.3 & \underline{29.8$\pm$0.6} & 24.6$\pm$2.4 & \textbf{41.0$\pm$2.0} & 28.8$\pm$1.5 & 27.6$\pm$1.6 & \underline{35.6$\pm$2.8} & 28.9$\pm$1.4 \\
        \quad DL-20 & \textbf{38.0$\pm$1.8} & 24.9$\pm$0.7 & 22.3$\pm$0.6 & \underline{33.6$\pm$0.9} & 27.4$\pm$0.8 & \textbf{41.3$\pm$1.3} & 34.6$\pm$1.4 & 32.9$\pm$1.0 & \underline{40.3$\pm$0.8} & 34.9$\pm$1.5 \\
        \midrule
        \multicolumn{11}{l}{\textit{Out-of-Domain}} \\
        \quad BEIR (14) & \textbf{24.4$\pm$0.5} & 13.4$\pm$0.5 & 15.5$\pm$0.5 & \underline{21.1$\pm$0.4} & 17.9$\pm$0.3 & \textbf{25.7$\pm$0.2} & 18.4$\pm$0.2 & 19.5$\pm$0.4 & \underline{23.1$\pm$0.1} & 20.6$\pm$0.2 \\
        \quad BRIGHT (12) & \underline{3.1$\pm$0.2} & 0.3$\pm$0.1 & 0.3$\pm$0.1 & \textbf{3.8$\pm$0.2} & 1.2$\pm$0.0 & \underline{3.2$\pm$0.2} & 0.8$\pm$0.0 & 1.0$\pm$0.1 & \textbf{3.9$\pm$0.3} & 1.7$\pm$0.1 \\
        \quad MHop (4) & \textbf{27.9$\pm$1.1} & 18.3$\pm$0.1 & 20.3$\pm$0.5 & \underline{27.7$\pm$0.7} & 22.7$\pm$0.3 & \textbf{31.0$\pm$0.4} & 25.6$\pm$0.3 & 26.8$\pm$0.8 & \underline{30.8$\pm$0.8} & 28.3$\pm$0.7 \\
        \bottomrule
      \end{tabular}
    }
  \end{subtable}

\end{table*}

\subsection{Per-Dataset Results with All Metrics}

The following tables present complete results for each model with NDCG@10, Recall@100, and MRR@10. The best result among the five finetuned models is shown in \textbf{bold}, and the second-best is \underline{underlined}.

\subsubsection{Contriever Results}

Table~\ref{tab:complete_contriever} presents the complete results for Contriever across all evaluation benchmarks.

\begin{table}[ht]
\centering
\caption{Complete Evaluation Results for Contriever (MS MARCO v1.1 QA 82K, mean of seeds 0, 42, 1337). Comparing Pretrained and five finetuned methods (Cosine, Dot, QNorm, DNorm, Learnable). N = NDCG@10, R = Recall@100, M = MRR@10. \textbf{Bold} = best among finetuned, \underline{underline} = second best.}
\label{tab:complete_contriever}
\resizebox{\textwidth}{!}{
\begin{tabular}{l|ccc|ccc|ccc|ccc|ccc|ccc}
\toprule
& \multicolumn{3}{c|}{\textbf{Pretrained}} & \multicolumn{3}{c|}{\textbf{Cosine}} & \multicolumn{3}{c|}{\textbf{Dot}} & \multicolumn{3}{c|}{\textbf{QNorm}} & \multicolumn{3}{c|}{\textbf{DNorm}} & \multicolumn{3}{c}{\textbf{Learnable}} \\
\textbf{Dataset} & N & R & M & N & R & M & N & R & M & N & R & M & N & R & M & N & R & M \\
\midrule
\midrule
\multicolumn{19}{l}{\textit{In-Domain}} \\
MSM-Dev & 19.45 & 64.61 & 15.29 & 30.86 & 81.49 & 24.84 & 31.07 & \underline{83.65} & 24.76 & 31.35 & 83.46 & 25.01 & \underline{31.40} & 83.50 & \underline{25.05} & \textbf{32.17} & \textbf{83.82} & \textbf{25.83} \\
DL-19 & 42.85 & 38.18 & 72.54 & 56.65 & 45.55 & \textbf{84.71} & 56.69 & \textbf{50.12} & 80.68 & \textbf{58.02} & 49.85 & 81.82 & 55.41 & 49.31 & 82.19 & \underline{57.76} & \underline{50.04} & \underline{82.30} \\
DL-20 & 40.52 & 41.22 & 71.32 & \underline{57.25} & 52.16 & \textbf{87.45} & 56.32 & 54.86 & 79.66 & 56.87 & \underline{55.39} & 81.49 & 57.17 & 54.06 & \underline{87.13} & \textbf{58.48} & \textbf{55.79} & 86.47 \\
\midrule
\multicolumn{19}{l}{\textit{BEIR (14 datasets)}} \\
ArguAna & 44.94 & 95.16 & 36.49 & \underline{49.70} & \underline{98.17} & \underline{41.03} & 44.96 & 96.30 & 35.92 & 43.32 & 95.26 & 34.52 & \textbf{54.09} & \textbf{98.60} & \textbf{44.62} & 49.08 & 97.73 & 40.08 \\
Climate-FEVER & 7.16 & 23.60 & 9.37 & \underline{20.05} & 49.53 & \underline{27.36} & 19.12 & \underline{49.55} & 25.76 & 19.13 & 48.95 & 25.93 & \textbf{20.73} & \textbf{51.07} & \textbf{28.18} & 19.59 & 48.87 & 26.69 \\
CQADupStack & 20.41 & 49.19 & 19.79 & 29.99 & 62.14 & 29.07 & \textbf{34.57} & \textbf{67.41} & \textbf{33.63} & 34.02 & 66.89 & 33.08 & 34.07 & \underline{67.24} & 33.20 & \underline{34.40} & 67.08 & \underline{33.61} \\
DBPedia & 27.03 & 41.70 & 56.80 & 35.08 & 47.50 & 68.63 & \underline{39.16} & \textbf{53.90} & 71.57 & 39.13 & \underline{53.82} & \textbf{71.94} & 37.29 & 52.36 & 68.61 & \textbf{39.44} & 53.70 & \underline{71.81} \\
FEVER & 27.22 & 64.22 & 24.08 & 67.10 & 92.54 & 64.70 & \textbf{73.62} & \textbf{94.31} & \textbf{71.79} & \underline{72.61} & \underline{94.15} & \underline{70.66} & 62.02 & 93.23 & 57.83 & 69.93 & 93.90 & 67.34 \\
FiQA & 12.41 & 36.34 & 17.26 & 27.05 & 60.82 & 33.69 & \underline{29.98} & \underline{64.74} & \textbf{36.66} & 29.41 & 63.73 & 35.95 & \textbf{29.99} & \textbf{66.36} & \underline{36.37} & 29.23 & 64.41 & 35.86 \\
HotpotQA & 41.01 & 62.82 & 54.23 & 52.90 & 68.15 & 70.44 & \textbf{63.17} & \textbf{79.58} & \textbf{80.42} & \underline{62.78} & \underline{79.41} & \underline{80.08} & 60.77 & 78.02 & 77.55 & 62.78 & 79.16 & 80.04 \\
NFCorpus & 27.11 & 26.11 & 45.92 & 31.45 & 29.48 & 51.58 & \textbf{33.24} & \textbf{29.78} & \textbf{54.09} & 32.84 & 29.62 & 53.42 & \underline{33.06} & \underline{29.67} & \underline{53.80} & 32.87 & 29.43 & 53.27 \\
NQ & 18.05 & 67.08 & 14.60 & 34.55 & 82.23 & 29.34 & 36.37 & \textbf{88.39} & 30.80 & \textbf{38.73} & 88.36 & \textbf{33.43} & 35.23 & 88.15 & 29.52 & \underline{37.12} & \underline{88.38} & \underline{31.60} \\
Quora & 83.37 & 98.53 & 82.66 & 84.45 & 98.51 & 83.70 & \underline{86.15} & \textbf{99.19} & \underline{85.33} & 85.72 & 99.14 & 84.90 & 85.93 & 99.12 & 85.23 & \textbf{86.52} & \underline{99.18} & \textbf{85.76} \\
SCIDOCS & 10.97 & 32.37 & 19.74 & 15.63 & 35.84 & 27.97 & 16.35 & 38.60 & 28.87 & 16.41 & 38.40 & 28.67 & \textbf{17.12} & \textbf{39.64} & \textbf{30.04} & \underline{16.84} & \underline{39.22} & \underline{29.36} \\
SciFact & 57.14 & 89.66 & 53.19 & 61.42 & 90.46 & 56.87 & 68.40 & \textbf{94.47} & 64.55 & 68.18 & 94.13 & 64.14 & \textbf{68.80} & 93.61 & \textbf{65.08} & \underline{68.73} & \underline{94.28} & \underline{64.86} \\
TREC-COVID & 18.16 & 1.91 & 39.04 & 46.10 & \underline{7.92} & 74.29 & \underline{47.27} & 5.98 & \underline{76.08} & \textbf{53.96} & \textbf{8.36} & \textbf{84.10} & 39.66 & 4.56 & 64.75 & 46.45 & 5.80 & 75.03 \\
Touche-2020 & 7.25 & 15.66 & 21.26 & \underline{18.00} & \textbf{40.05} & \underline{34.68} & 16.36 & 39.24 & 31.06 & \textbf{19.88} & \underline{39.80} & \textbf{39.27} & 15.56 & 39.39 & 27.48 & 17.18 & 38.93 & 30.79 \\
\midrule
\multicolumn{19}{l}{\textit{BRIGHT (12 datasets)}} \\
AOPS & 3.46 & 9.00 & 6.31 & 4.46 & 8.98 & 8.01 & \underline{6.44} & \underline{17.35} & \textbf{12.63} & 5.77 & \textbf{17.37} & 10.79 & 5.62 & 14.60 & 9.89 & \textbf{6.96} & 17.09 & \underline{11.52} \\
Biology & 1.58 & 11.09 & 1.96 & 6.72 & 24.65 & 9.78 & \underline{16.03} & \underline{46.74} & \underline{21.49} & \textbf{17.33} & \textbf{48.65} & \textbf{22.85} & 11.76 & 41.93 & 16.45 & 13.16 & 42.47 & 17.59 \\
Earth Sci. & 3.44 & 12.26 & 5.02 & 11.65 & 39.55 & 15.06 & \underline{20.69} & \underline{53.08} & \underline{27.73} & \textbf{22.55} & \textbf{55.73} & \textbf{30.40} & 16.00 & 46.76 & 21.07 & 17.87 & 51.13 & 23.37 \\
Economics & 2.27 & 9.03 & 3.93 & 9.94 & 31.61 & 13.97 & \underline{15.75} & \underline{42.29} & 21.14 & \textbf{17.01} & \textbf{44.26} & \textbf{22.80} & 12.28 & 36.90 & 17.43 & 15.68 & 42.05 & \underline{22.43} \\
LeetCode & 13.34 & 48.31 & 12.33 & 12.32 & 30.87 & \underline{12.55} & \underline{13.54} & \underline{48.18} & 12.43 & 12.83 & 44.03 & 11.95 & \textbf{14.40} & \textbf{49.27} & \textbf{13.43} & 13.29 & 47.65 & 12.02 \\
Pony & 0.07 & 7.25 & 0.13 & 1.80 & 11.02 & 3.86 & \textbf{6.21} & \textbf{21.92} & \textbf{16.41} & 4.28 & \underline{20.41} & 10.05 & 3.14 & 16.45 & 7.70 & \underline{5.09} & 19.68 & \underline{12.12} \\
Psychology & 2.43 & 11.95 & 3.00 & 9.24 & 31.78 & 10.13 & \underline{15.92} & 44.67 & \underline{20.24} & \textbf{16.75} & \textbf{49.51} & \textbf{20.63} & 12.61 & 42.32 & 14.14 & 14.54 & \underline{45.71} & 18.40 \\
Robotics & 5.21 & 14.60 & 7.30 & 6.66 & 21.16 & 9.85 & \underline{14.46} & \textbf{36.93} & \underline{18.55} & \textbf{14.71} & \underline{36.87} & \textbf{19.11} & 9.12 & 28.87 & 12.25 & 12.39 & 34.80 & 15.85 \\
StackOverflow & 4.42 & 16.64 & 4.51 & 7.79 & 26.44 & 9.37 & \underline{11.56} & \underline{41.94} & \textbf{14.53} & 10.61 & \textbf{42.03} & 12.65 & 9.27 & 37.50 & 10.36 & \textbf{11.99} & 39.44 & \underline{13.53} \\
Sust. Living & 1.43 & 14.98 & 2.23 & 8.18 & 35.20 & 10.32 & \underline{14.37} & \underline{53.08} & \underline{18.57} & \textbf{16.15} & \textbf{53.53} & \textbf{20.05} & 9.11 & 46.52 & 10.97 & 14.18 & 51.26 & 16.42 \\
TheoremQA-Q & 7.66 & 18.56 & 8.20 & 7.43 & 15.68 & 8.38 & 11.08 & 22.52 & 12.41 & \underline{11.17} & \textbf{23.54} & \underline{12.54} & 11.10 & 21.63 & \textbf{12.83} & \textbf{11.21} & \underline{22.55} & 12.44 \\
TheoremQA-T & 0.74 & 5.92 & 0.40 & 2.71 & 9.43 & 2.76 & \textbf{4.27} & \textbf{17.10} & \textbf{4.28} & \underline{3.66} & 10.82 & 3.67 & 3.59 & 12.28 & \underline{3.82} & 3.60 & \underline{13.75} & 3.46 \\
\midrule
\multicolumn{19}{l}{\textit{Multi-hop (4 datasets)}} \\
2WikiMHopQA & 47.85 & 66.49 & 65.94 & 62.26 & 72.06 & 87.10 & \textbf{66.88} & \underline{76.99} & \textbf{90.61} & 66.25 & 76.68 & 89.79 & \underline{66.46} & \textbf{77.17} & \underline{90.03} & 66.02 & 76.83 & 89.50 \\
MuSiQue & 32.11 & 60.58 & 47.66 & 33.79 & 66.01 & 49.60 & \textbf{35.96} & 68.44 & \textbf{51.17} & \underline{35.82} & 68.22 & \underline{50.96} & 34.88 & \textbf{68.51} & 49.70 & 35.51 & \underline{68.47} & 50.76 \\
NovelHopQA & 37.84 & 86.70 & 32.19 & 56.53 & 94.51 & 50.69 & \underline{66.64} & \underline{96.67} & \underline{61.24} & 64.69 & 96.59 & 59.19 & \textbf{68.05} & \textbf{96.89} & \textbf{62.86} & 64.38 & 96.54 & 58.82 \\
HotpotQA & 41.01 & 62.82 & 54.23 & 52.90 & 68.15 & 70.44 & \textbf{63.17} & \textbf{79.58} & \textbf{80.42} & \underline{62.78} & \underline{79.41} & \underline{80.08} & 60.77 & 78.02 & 77.55 & 62.78 & 79.16 & 80.04 \\
\midrule
\multicolumn{19}{l}{\textit{Averages}} \\
DL Avg (2) & 41.68 & 39.70 & 71.93 & 56.95 & 48.85 & \textbf{86.08} & 56.50 & 52.49 & 80.17 & \underline{57.45} & \underline{52.62} & 81.66 & 56.29 & 51.68 & \underline{84.66} & \textbf{58.12} & \textbf{52.92} & 84.38 \\
BEIR Avg (14) & 28.73 & 50.31 & 35.32 & 40.96 & 61.67 & 49.52 & 43.48 & \textbf{64.39} & \underline{51.89} & \textbf{44.01} & 64.29 & \textbf{52.86} & 42.45 & \underline{64.36} & 50.16 & \underline{43.58} & 64.29 & 51.86 \\
BRIGHT Avg (12) & 3.84 & 14.97 & 4.61 & 7.41 & 23.86 & 9.50 & \underline{12.53} & \underline{37.15} & \textbf{16.70} & \textbf{12.74} & \textbf{37.23} & \underline{16.46} & 9.83 & 32.92 & 12.53 & 11.66 & 35.63 & 14.93 \\
MHop Avg (4) & 39.70 & 69.15 & 50.00 & 51.37 & 75.18 & 64.46 & \textbf{58.16} & \textbf{80.42} & \textbf{70.86} & 57.38 & 80.22 & 70.01 & \underline{57.54} & 80.14 & \underline{70.04} & 57.17 & \underline{80.25} & 69.78 \\
\bottomrule
\end{tabular}
}
\end{table}

\subsection{RetroMAE Results}

Table~\ref{tab:complete_retromae} presents the complete results for RetroMAE across all evaluation benchmarks.

\begin{table}[ht]
\centering
\caption{Complete Evaluation Results for RetroMAE (MS MARCO v1.1 QA 82K, mean of seeds 0, 42, 1337). Comparing Pretrained and five finetuned methods (Cosine, Dot, QNorm, DNorm, Learnable). N = NDCG@10, R = Recall@100, M = MRR@10. \textbf{Bold} = best among finetuned, \underline{underline} = second best.}
\label{tab:complete_retromae}
\resizebox{\textwidth}{!}{
\begin{tabular}{l|ccc|ccc|ccc|ccc|ccc|ccc}
\toprule
& \multicolumn{3}{c|}{\textbf{Pretrained}} & \multicolumn{3}{c|}{\textbf{Cosine}} & \multicolumn{3}{c|}{\textbf{Dot}} & \multicolumn{3}{c|}{\textbf{QNorm}} & \multicolumn{3}{c|}{\textbf{DNorm}} & \multicolumn{3}{c}{\textbf{Learnable}} \\
\textbf{Dataset} & N & R & M & N & R & M & N & R & M & N & R & M & N & R & M & N & R & M \\
\midrule
\midrule
\multicolumn{19}{l}{\textit{In-Domain}} \\
MSM-Dev & 2.87 & 14.14 & 2.25 & 30.68 & 78.61 & 25.11 & 31.24 & 81.63 & 25.23 & 32.15 & 82.16 & 26.02 & \textbf{32.92} & \textbf{82.66} & \textbf{26.73} & \underline{32.84} & \underline{82.60} & \underline{26.67} \\
DL-19 & 14.05 & 10.35 & 37.44 & 56.11 & 44.31 & 86.31 & 56.82 & 46.67 & 87.60 & 57.51 & \underline{47.34} & 89.47 & \underline{59.95} & \textbf{47.37} & \underline{91.61} & \textbf{60.43} & 46.82 & \textbf{91.73} \\
DL-20 & 9.00 & 8.30 & 23.30 & 56.85 & 48.75 & 83.87 & 57.83 & 52.29 & 85.49 & 58.11 & 52.39 & 86.05 & \underline{59.20} & \underline{52.59} & \underline{86.63} & \textbf{59.69} & \textbf{53.07} & \textbf{87.08} \\
\midrule
\multicolumn{19}{l}{\textit{BEIR (14 datasets)}} \\
ArguAna & 32.70 & 78.66 & 26.64 & 44.06 & 96.06 & 35.90 & 47.93 & 97.30 & 38.92 & \underline{52.03} & \underline{98.08} & \underline{43.12} & 48.43 & 97.72 & 39.72 & \textbf{52.41} & \textbf{98.36} & \textbf{43.71} \\
Climate-FEVER & 5.95 & 20.03 & 8.30 & \textbf{19.92} & 46.51 & \textbf{27.44} & 17.80 & 46.53 & 24.10 & \underline{19.72} & 48.42 & \underline{26.95} & 19.03 & \underline{48.63} & 25.97 & 19.53 & \textbf{48.87} & 26.64 \\
CQADupStack & 10.22 & 27.20 & 10.04 & 26.07 & 56.18 & 25.46 & 30.85 & 61.12 & 30.21 & 31.52 & 61.60 & 30.92 & \textbf{31.97} & \underline{62.06} & \underline{31.39} & \underline{31.96} & \textbf{62.07} & \textbf{31.43} \\
DBPedia & 4.78 & 9.79 & 13.78 & 32.56 & 41.06 & \textbf{67.55} & 29.10 & 44.92 & 57.65 & \underline{33.22} & 46.53 & 64.36 & 33.03 & \textbf{47.07} & 63.73 & \textbf{33.56} & \underline{46.95} & \underline{64.76} \\
FEVER & 4.72 & 13.70 & 4.23 & \underline{62.51} & 91.39 & \underline{59.52} & 59.88 & 90.33 & 57.00 & 61.81 & 91.51 & 58.74 & \textbf{63.17} & \textbf{92.22} & \textbf{59.72} & 61.69 & \underline{91.60} & 58.45 \\
FiQA & 2.75 & 11.21 & 3.60 & 22.41 & 50.69 & 27.64 & 24.91 & 55.21 & 30.31 & 25.07 & \underline{56.73} & 30.21 & \textbf{26.22} & 56.47 & \textbf{31.93} & \underline{25.98} & \textbf{56.88} & \underline{31.37} \\
HotpotQA & 20.88 & 38.21 & 28.44 & 48.29 & 61.70 & 65.68 & 54.67 & 70.97 & 72.46 & 56.03 & 71.83 & 73.90 & \textbf{56.70} & \underline{72.62} & \textbf{74.22} & \underline{56.50} & \textbf{72.78} & \underline{74.04} \\
NFCorpus & 6.22 & 11.25 & 13.73 & 27.36 & 24.69 & 47.73 & 28.52 & \underline{25.21} & \underline{48.37} & \underline{28.88} & \textbf{25.58} & \textbf{48.38} & \textbf{28.98} & 24.91 & 47.94 & 28.57 & 25.13 & 47.48 \\
NQ & 3.73 & 21.80 & 2.86 & 33.20 & 79.68 & 28.69 & 35.33 & 84.35 & 30.09 & 35.60 & 84.26 & 30.43 & \textbf{36.41} & \textbf{85.02} & \textbf{30.97} & \underline{36.11} & \underline{84.72} & \underline{30.85} \\
Quora & 56.76 & 82.95 & 56.47 & 82.72 & 97.95 & 82.10 & 83.14 & 98.16 & 82.50 & 84.51 & 98.72 & 83.79 & \underline{85.16} & \underline{98.84} & \underline{84.45} & \textbf{85.34} & \textbf{98.86} & \textbf{84.68} \\
SCIDOCS & 4.07 & 16.74 & 7.76 & 13.04 & 30.32 & 24.25 & 14.67 & 33.12 & 26.46 & 14.99 & 33.28 & 26.97 & \textbf{15.51} & \underline{33.99} & \textbf{28.49} & \underline{15.34} & \textbf{33.99} & \underline{27.72} \\
SciFact & 30.93 & 67.40 & 27.74 & 51.10 & 85.46 & 47.89 & 57.51 & 89.40 & 53.50 & 58.31 & \textbf{89.99} & 54.95 & \textbf{59.24} & 89.50 & \textbf{55.51} & \underline{58.79} & \underline{89.77} & \underline{55.29} \\
TREC-COVID & 15.72 & 1.68 & 31.36 & 54.02 & 8.69 & 74.94 & \textbf{60.73} & \textbf{10.63} & \textbf{87.50} & \underline{55.89} & \underline{10.12} & \underline{83.05} & 54.21 & 9.33 & 76.60 & 53.87 & 9.38 & 79.84 \\
Touche-2020 & 0.41 & 4.00 & 1.14 & \textbf{18.69} & \underline{39.64} & \textbf{31.76} & 16.18 & 38.01 & \underline{31.21} & 17.20 & 39.15 & 30.47 & \underline{18.10} & \textbf{40.08} & 29.76 & 17.39 & 38.88 & 30.49 \\
\midrule
\multicolumn{19}{l}{\textit{BRIGHT (12 datasets)}} \\
AOPS & 4.00 & 12.44 & 8.54 & 1.99 & 7.84 & 3.95 & 3.36 & 14.22 & 6.66 & 3.95 & \underline{15.42} & 7.82 & \textbf{5.05} & 14.83 & \textbf{9.04} & \underline{4.38} & \textbf{15.70} & \underline{8.09} \\
Biology & 6.10 & 23.57 & 8.54 & 6.02 & 21.99 & 8.72 & 9.62 & 31.40 & \underline{14.79} & 9.19 & 31.28 & 14.40 & \underline{10.05} & \underline{31.85} & 13.85 & \textbf{10.72} & \textbf{32.64} & \textbf{16.30} \\
Earth Sci. & 7.80 & 23.86 & 11.36 & 11.35 & 36.10 & 16.30 & 13.48 & 37.64 & 16.14 & 14.86 & 39.45 & 18.83 & \textbf{16.04} & \textbf{40.75} & \textbf{20.33} & \underline{15.49} & \underline{40.05} & \underline{19.09} \\
Economics & 5.81 & 28.04 & 9.68 & 8.74 & 29.38 & 11.69 & 12.45 & 33.48 & 14.39 & \textbf{13.29} & \underline{36.52} & \textbf{15.96} & 11.13 & 33.52 & 11.94 & \underline{13.13} & \textbf{37.19} & \underline{14.98} \\
LeetCode & 11.44 & 40.01 & 10.60 & 11.59 & 26.51 & 11.74 & 13.49 & \textbf{43.65} & 12.76 & 13.45 & 41.48 & 12.95 & \textbf{14.49} & \underline{43.30} & \textbf{13.75} & \underline{14.08} & 42.69 & \underline{13.46} \\
Pony & 3.17 & 8.06 & 8.83 & 0.58 & 4.14 & 1.38 & 0.92 & 6.69 & 2.35 & 3.51 & \textbf{13.79} & 9.21 & \underline{3.54} & 12.71 & \underline{9.30} & \textbf{3.67} & \underline{13.62} & \textbf{9.54} \\
Psychology & 3.89 & 16.35 & 4.98 & 9.96 & 29.93 & 12.17 & 11.02 & 32.46 & 11.93 & \textbf{13.97} & \underline{35.27} & \underline{15.67} & \underline{13.03} & 33.75 & \textbf{16.33} & 12.97 & \textbf{35.63} & 15.22 \\
Robotics & 6.71 & 21.02 & 9.06 & 6.20 & 19.18 & 9.51 & \textbf{13.42} & 35.42 & \textbf{17.55} & 12.31 & \textbf{35.82} & 16.24 & 11.83 & 35.28 & \underline{16.68} & \underline{12.45} & \underline{35.71} & 16.44 \\
StackOverflow & 3.13 & 13.61 & 3.68 & 4.46 & 19.70 & 5.32 & 6.98 & 29.03 & 9.18 & 8.00 & \underline{30.94} & 9.11 & \underline{8.18} & 30.62 & \underline{10.43} & \textbf{9.18} & \textbf{32.76} & \textbf{11.24} \\
Sust. Living & 6.20 & 25.16 & 8.73 & 7.23 & 32.81 & 8.17 & 11.15 & \underline{42.04} & 13.57 & 11.78 & 41.01 & 13.76 & \textbf{12.51} & \textbf{43.06} & \textbf{16.03} & \underline{12.25} & 41.27 & \underline{14.17} \\
TheoremQA-Q & 3.90 & 8.44 & 4.46 & 3.96 & 11.06 & 4.52 & 6.64 & 13.95 & 8.01 & \underline{7.13} & \textbf{16.47} & \underline{8.93} & 6.36 & 13.20 & 7.94 & \textbf{7.20} & \underline{15.32} & \textbf{9.12} \\
TheoremQA-T & 0.27 & 1.97 & 0.19 & 1.21 & 6.07 & 0.83 & \textbf{2.01} & \underline{9.54} & \textbf{1.84} & \underline{1.90} & \textbf{9.87} & \underline{1.65} & 1.36 & 6.32 & 1.19 & 1.45 & 9.25 & 1.17 \\
\midrule
\multicolumn{19}{l}{\textit{Multi-hop (4 datasets)}} \\
2WikiMHopQA & 19.18 & 36.86 & 27.38 & \underline{63.33} & 70.49 & \textbf{89.60} & 59.92 & 70.91 & 84.74 & \textbf{63.40} & \textbf{72.88} & \underline{89.08} & 62.95 & \underline{72.84} & 88.30 & 62.69 & 72.78 & 88.10 \\
MuSiQue & 12.55 & 38.20 & 17.22 & \textbf{33.40} & 60.40 & \textbf{50.16} & 31.88 & 62.22 & 45.82 & \underline{32.73} & \textbf{63.70} & \underline{47.22} & 32.11 & 63.16 & 46.06 & 32.51 & \underline{63.63} & 46.83 \\
NovelHopQA & 21.20 & 60.32 & 17.96 & 57.64 & 93.25 & 51.98 & 65.27 & 95.54 & 60.15 & 66.92 & 95.92 & 61.81 & \textbf{70.14} & \textbf{96.42} & \textbf{65.31} & \underline{68.93} & \underline{96.30} & \underline{63.95} \\
HotpotQA & 20.88 & 38.21 & 28.44 & 48.29 & 61.70 & 65.68 & 54.67 & 70.97 & 72.46 & 56.03 & 71.83 & 73.90 & \textbf{56.70} & \underline{72.62} & \textbf{74.22} & \underline{56.50} & \textbf{72.78} & \underline{74.04} \\
\midrule
\multicolumn{19}{l}{\textit{Averages}} \\
DL Avg (2) & 11.52 & 9.32 & 30.37 & 56.48 & 46.53 & 85.09 & 57.33 & 49.48 & 86.55 & 57.81 & 49.86 & 87.76 & \underline{59.58} & \textbf{49.98} & \underline{89.11} & \textbf{60.06} & \underline{49.94} & \textbf{89.40} \\
BEIR Avg (14) & 14.27 & 28.90 & 16.86 & 38.28 & 57.86 & 46.18 & 40.09 & 60.38 & 47.88 & 41.05 & 61.13 & \underline{49.02} & \underline{41.15} & \textbf{61.32} & 48.60 & \textbf{41.22} & \underline{61.30} & \textbf{49.05} \\
BRIGHT Avg (12) & 5.20 & 18.54 & 7.39 & 6.11 & 20.39 & 7.86 & 8.71 & 27.46 & 10.76 & 9.44 & \underline{28.95} & 12.05 & \underline{9.46} & 28.27 & \underline{12.24} & \textbf{9.75} & \textbf{29.32} & \textbf{12.40} \\
MHop Avg (4) & 18.45 & 43.40 & 22.75 & 50.67 & 71.46 & 64.35 & 52.93 & 74.91 & 65.79 & 54.77 & 76.08 & 68.00 & \textbf{55.47} & \underline{76.26} & \textbf{68.47} & \underline{55.16} & \textbf{76.37} & \underline{68.23} \\
\bottomrule
\end{tabular}
}
\end{table}

\subsection{E5 Results}

Table~\ref{tab:complete_e5} presents the complete results for E5 across all evaluation benchmarks. Note that magnitude-aware variants (Dot, QNorm, DNorm, Learnable) require removing E5's built-in normalization layer, which disrupts the pre-trained representations.

\begin{table}[ht]
\centering
\caption{Complete Evaluation Results for E5 (MS MARCO v1.1 QA 82K, seed=0). Comparing Pretrained and five finetuned methods (Cosine, Dot, QNorm, DNorm, Learnable). N = NDCG@10, R = Recall@100, M = MRR@10. \textbf{Bold} = best among finetuned, \underline{underline} = second best. Magnitude-aware variants require removing E5's normalization layer.}
\label{tab:complete_e5}
\resizebox{\textwidth}{!}{
\begin{tabular}{l|ccc|ccc|ccc|ccc|ccc|ccc}
\toprule
& \multicolumn{3}{c|}{\textbf{Pretrained}} & \multicolumn{3}{c|}{\textbf{Cosine}} & \multicolumn{3}{c|}{\textbf{Dot}} & \multicolumn{3}{c|}{\textbf{QNorm}} & \multicolumn{3}{c|}{\textbf{DNorm}} & \multicolumn{3}{c}{\textbf{Learnable}} \\
\textbf{Dataset} & N & R & M & N & R & M & N & R & M & N & R & M & N & R & M & N & R & M \\
\midrule
\multicolumn{19}{l}{\textit{In-Domain}} \\
MSM-Dev & 18.13 & 60.11 & 13.98 & \textbf{33.63} & \underline{84.68} & \textbf{27.14} & 16.74 & 71.83 & 12.16 & 9.53 & 48.39 & 6.87 & 26.01 & 75.07 & 21.02 & \underline{28.56} & \textbf{85.66} & \underline{21.96} \\
DL-19 & 33.03 & 27.84 & 53.57 & \textbf{62.27} & \underline{50.41} & \textbf{92.64} & 22.80 & 35.58 & 30.51 & 10.36 & 18.27 & 15.50 & \underline{52.31} & 40.07 & \underline{82.99} & 48.81 & \textbf{50.83} & 64.95 \\
DL-20 & 31.97 & 35.81 & 58.36 & \textbf{63.92} & \underline{57.18} & \textbf{90.74} & 28.00 & 41.45 & 38.73 & 12.79 & 26.80 & 17.12 & \underline{56.95} & 45.60 & \underline{86.07} & 51.32 & \textbf{57.87} & 64.30 \\
\midrule
\multicolumn{19}{l}{\textit{BEIR (14 datasets)}} \\
ArguAna & 35.19 & 91.39 & 27.80 & \textbf{50.11} & \textbf{98.08} & \textbf{41.74} & 1.04 & 13.23 & 0.52 & 18.38 & 67.71 & 13.78 & 16.67 & 63.02 & 12.69 & \underline{31.96} & \underline{88.48} & \underline{24.65} \\
Climate-FEVER & 15.86 & 40.72 & 22.78 & \underline{20.70} & \underline{51.11} & 27.55 & 6.33 & 19.81 & 8.02 & \textbf{26.07} & \textbf{54.11} & \textbf{38.77} & 3.11 & 11.79 & 4.52 & 20.57 & 51.07 & \underline{28.03} \\
CQADupStack & 25.64 & 61.07 & 24.15 & \textbf{40.50} & \underline{74.38} & \textbf{39.52} & 35.23 & 70.58 & \underline{33.85} & 26.17 & 72.36 & 23.32 & 31.19 & 64.41 & 30.15 & \underline{35.31} & \textbf{74.76} & 33.28 \\
DBPedia & 18.19 & 32.00 & 38.00 & \textbf{36.23} & \textbf{47.43} & \textbf{70.21} & 14.20 & 28.93 & 24.68 & 11.91 & 28.46 & 20.39 & \underline{31.33} & 41.43 & \underline{62.56} & 24.35 & \underline{44.99} & 41.22 \\
FEVER & 49.48 & 84.29 & 46.46 & \textbf{67.50} & \textbf{93.05} & \textbf{65.19} & 52.48 & 81.53 & 49.98 & \underline{61.26} & 90.69 & \underline{58.16} & 13.02 & 38.67 & 11.33 & 60.60 & \underline{91.14} & 57.12 \\
FiQA & 26.18 & 62.98 & 29.91 & \underline{34.30} & 67.42 & \textbf{41.55} & 30.19 & 68.36 & 35.95 & 27.03 & \underline{70.51} & 28.66 & 23.68 & 54.19 & 30.26 & \textbf{34.72} & \textbf{71.59} & \underline{40.28} \\
HotpotQA & 51.06 & 67.29 & 66.70 & \underline{52.17} & 66.16 & \underline{69.38} & 49.19 & 64.84 & 63.28 & 47.78 & \underline{68.26} & 60.36 & 29.41 & 48.21 & 38.20 & \textbf{56.62} & \textbf{73.20} & \textbf{72.33} \\
NFCorpus & 30.40 & 30.99 & 48.97 & 34.63 & \textbf{33.68} & 55.63 & \underline{35.35} & 33.34 & \underline{56.46} & 33.77 & 33.23 & 50.85 & 26.86 & 25.67 & 45.90 & \textbf{35.85} & \underline{33.36} & \textbf{56.63} \\
NQ & 25.33 & 73.02 & 21.32 & \textbf{40.74} & \textbf{86.32} & \textbf{35.52} & 5.38 & 22.52 & 4.28 & 2.43 & 12.65 & 1.94 & \underline{33.14} & \underline{79.44} & \underline{28.37} & 17.46 & 64.81 & 14.08 \\
Quora & 83.08 & 97.06 & 82.66 & \textbf{87.16} & \textbf{99.22} & \textbf{86.38} & 59.57 & 88.45 & 58.00 & 77.52 & 94.66 & 77.10 & 74.44 & 94.15 & 73.86 & \underline{81.57} & \underline{96.45} & \underline{81.15} \\
SCIDOCS & 1.94 & 5.36 & 4.44 & \textbf{19.99} & \textbf{46.36} & \textbf{33.81} & 3.53 & 18.95 & 6.39 & 3.74 & 10.26 & 7.25 & 5.88 & \underline{25.72} & 10.64 & \underline{6.12} & 21.09 & \underline{11.15} \\
SciFact & 41.13 & 85.79 & 35.62 & \textbf{68.76} & \textbf{95.33} & \textbf{64.74} & 45.47 & 75.06 & 41.68 & 57.12 & 83.96 & 53.14 & 23.65 & 58.78 & 21.03 & \underline{59.44} & \underline{87.23} & \underline{55.74} \\
TREC-COVID & 31.82 & 3.21 & 60.38 & \textbf{64.80} & \textbf{11.87} & \textbf{91.00} & 25.71 & 2.86 & 44.03 & 31.64 & 2.85 & 56.92 & 29.81 & 3.47 & 52.27 & \underline{37.41} & \underline{3.64} & \underline{62.16} \\
Touche-2020 & 1.52 & 3.69 & 4.66 & \textbf{20.28} & \textbf{43.74} & \textbf{36.80} & 0.00 & 10.71 & 0.00 & 0.28 & 0.64 & 1.02 & \underline{4.06} & \underline{23.96} & \underline{10.66} & 0.37 & 1.60 & 1.43 \\
\midrule
\multicolumn{19}{l}{\textit{BRIGHT (12 datasets)}} \\
AOPS & 2.58 & 7.47 & 4.98 & \textbf{2.68} & \textbf{13.34} & \textbf{4.90} & 0.00 & 0.00 & 0.00 & 0.12 & 1.36 & 0.30 & 0.00 & 0.00 & 0.00 & \underline{1.41} & \underline{4.55} & \underline{2.70} \\
Biology & 15.29 & 50.94 & 20.72 & 7.58 & 34.50 & 11.32 & 0.91 & 7.43 & 1.46 & \textbf{17.44} & \textbf{50.43} & \textbf{25.14} & 5.69 & 20.79 & 8.48 & \underline{12.68} & \underline{40.29} & \underline{17.57} \\
Earth Sci. & 17.69 & 46.08 & 22.11 & 13.77 & 40.94 & 18.79 & 0.73 & 11.32 & 0.64 & \textbf{21.04} & \textbf{47.26} & \textbf{28.34} & 6.28 & 17.43 & 8.34 & \underline{16.58} & \underline{41.08} & \underline{22.11} \\
Economics & 12.24 & 39.33 & 16.18 & \textbf{11.13} & \textbf{36.90} & \textbf{12.78} & 0.00 & 7.09 & 0.00 & 9.30 & 29.35 & 10.91 & 3.64 & 9.20 & 4.53 & \underline{9.77} & \underline{33.08} & \underline{11.16} \\
LeetCode & 14.83 & 48.49 & 14.77 & \textbf{14.63} & \textbf{39.68} & \textbf{14.67} & 0.48 & 1.82 & 0.78 & 9.75 & 25.09 & 10.46 & 5.55 & 12.10 & 6.81 & \underline{12.15} & \underline{34.88} & \underline{12.51} \\
Pony & 1.72 & 10.84 & 5.08 & 0.40 & 3.71 & 0.64 & 10.37 & \underline{30.07} & 23.10 & \textbf{29.32} & \textbf{35.33} & \textbf{65.36} & 7.04 & 27.08 & 14.41 & \underline{13.83} & 24.17 & \underline{37.71} \\
Psychology & 16.66 & 43.57 & 18.79 & 10.92 & 36.73 & \underline{13.14} & 0.26 & 7.81 & 0.25 & \textbf{13.88} & \underline{42.02} & \textbf{14.57} & 2.38 & 11.31 & 2.48 & \underline{11.64} & \textbf{43.15} & 12.72 \\
Robotics & 10.21 & 26.93 & 15.14 & \underline{8.75} & 26.24 & 11.89 & 1.06 & 4.42 & 1.24 & 8.24 & \underline{27.97} & \underline{13.10} & 2.06 & 12.13 & 3.24 & \textbf{10.76} & \textbf{30.80} & \textbf{13.90} \\
StackOverflow & 9.99 & 42.63 & 12.27 & 10.12 & \textbf{43.32} & \underline{13.14} & 1.11 & 9.50 & 1.50 & \textbf{10.62} & 24.15 & \textbf{14.26} & 2.38 & 12.19 & 2.39 & \underline{10.19} & \underline{32.27} & 13.05 \\
Sust. Living & 10.36 & 40.36 & 13.50 & \textbf{8.48} & \textbf{39.44} & \underline{10.56} & 0.45 & 5.07 & 1.06 & 8.03 & 31.22 & 10.38 & 1.26 & 12.83 & 2.81 & \underline{8.25} & \underline{35.30} & \textbf{11.81} \\
TheoremQA-Q & 11.24 & 24.57 & 12.77 & \textbf{12.81} & \textbf{25.31} & \textbf{15.46} & 0.52 & 2.32 & 0.52 & 6.30 & 14.13 & 7.92 & 2.50 & 8.98 & 2.55 & \underline{9.14} & \underline{17.90} & \underline{11.72} \\
TheoremQA-T & 4.96 & 16.23 & 4.62 & \textbf{5.96} & \textbf{28.29} & \underline{5.49} & 0.00 & 4.06 & 0.00 & 3.86 & 9.43 & 3.87 & 1.61 & 3.07 & 2.63 & \underline{5.77} & \underline{12.94} & \textbf{5.92} \\
\midrule
\multicolumn{19}{l}{\textit{Multi-hop (4 datasets)}} \\
2WikiMHopQA & 62.43 & 73.30 & 83.71 & 68.77 & 74.83 & \underline{93.84} & 68.24 & 75.37 & 91.57 & \underline{70.19} & \underline{76.77} & 93.31 & 38.54 & 58.90 & 52.99 & \textbf{71.12} & \textbf{77.01} & \textbf{94.44} \\
MuSiQue & 30.34 & 60.26 & 43.74 & \textbf{34.61} & \textbf{64.95} & \textbf{50.57} & 30.86 & 59.52 & 43.22 & 33.70 & 62.94 & 47.83 & 21.04 & 46.21 & 29.50 & \underline{34.16} & \underline{64.01} & \underline{48.32} \\
NovelHopQA & 61.33 & 94.98 & 55.50 & 54.66 & \textbf{94.27} & 48.69 & 47.75 & 83.04 & 43.57 & \textbf{60.34} & 92.29 & \textbf{55.30} & 23.00 & 64.63 & 20.14 & \underline{60.24} & \underline{92.34} & \underline{55.08} \\
HotpotQA & 51.06 & 67.29 & 66.70 & \underline{52.17} & 66.16 & \underline{69.38} & 49.19 & 64.84 & 63.28 & 47.78 & \underline{68.26} & 60.36 & 29.41 & 48.21 & 38.20 & \textbf{56.62} & \textbf{73.20} & \textbf{72.33} \\
\midrule
\multicolumn{19}{l}{\textit{Averages}} \\
DL Avg (2) & 32.50 & 31.83 & 55.96 & \textbf{63.10} & \underline{53.80} & \textbf{91.69} & 25.40 & 38.52 & 34.62 & 11.58 & 22.54 & 16.31 & \underline{54.63} & 42.84 & \underline{84.53} & 50.06 & \textbf{54.35} & 64.62 \\
BEIR Avg (14) & 31.20 & 52.78 & 36.70 & \textbf{45.56} & \textbf{65.30} & \textbf{54.22} & 25.98 & 42.80 & 30.51 & 30.36 & 49.31 & 35.12 & 24.73 & 45.21 & 30.89 & \underline{35.88} & \underline{57.39} & \underline{41.37} \\
BRIGHT Avg (12) & 10.65 & 33.12 & 13.41 & 8.94 & \textbf{30.70} & 11.07 & 1.32 & 7.58 & 2.55 & \textbf{11.49} & 28.14 & \textbf{17.05} & 3.37 & 12.26 & 4.89 & \underline{10.18} & \underline{29.20} & \underline{14.41} \\
MHop Avg (4) & 51.29 & 73.96 & 62.41 & 52.55 & 75.05 & \underline{65.62} & 49.01 & 70.69 & 60.41 & \underline{53.00} & \underline{75.06} & 64.20 & 28.00 & 54.49 & 35.21 & \textbf{55.54} & \textbf{76.64} & \textbf{67.54} \\
\bottomrule
\end{tabular}
}
\end{table}

\subsection{E5 Results (MS MARCO Passage Ranking)}

Table~\ref{tab:complete_e5_500k} presents the complete results for E5 trained on MS MARCO Passage Ranking (503K, compared to 82K above). With more training data, DNorm overcomes E5's architectural constraint and achieves the best performance on reasoning-intensive tasks (BRIGHT and Multi-hop).

\begin{table}[ht]
\centering
\caption{Complete Evaluation Results for E5 (MS MARCO Passage Ranking 503K, seed=0). Comparing Pretrained and five finetuned methods (Cosine, Dot, QNorm, DNorm, Learnable). N = NDCG@10, R = Recall@100, M = MRR@10. \textbf{Bold} = best among finetuned models, \underline{underline} = second best. Magnitude-aware variants require removing E5's normalization layer.}
\label{tab:complete_e5_500k}
\resizebox{\textwidth}{!}{
\begin{tabular}{l|ccc|ccc|ccc|ccc|ccc|ccc}
\toprule
& \multicolumn{3}{c|}{\textbf{Pretrained}} & \multicolumn{3}{c|}{\textbf{Cosine}} & \multicolumn{3}{c|}{\textbf{Dot}} & \multicolumn{3}{c|}{\textbf{QNorm}} & \multicolumn{3}{c|}{\textbf{DNorm}} & \multicolumn{3}{c}{\textbf{Learnable}} \\
\textbf{Dataset} & N & R & M & N & R & M & N & R & M & N & R & M & N & R & M & N & R & M \\
\midrule
\multicolumn{19}{l}{\textit{In-Domain}} \\
MSM-Dev & 18.13 & 60.11 & 13.98 & \underline{33.86} & \underline{85.85} & \textbf{27.31} & 24.85 & 80.40 & 18.98 & 14.94 & 61.57 & 11.21 & \textbf{33.92} & \textbf{86.06} & \underline{27.05} & 26.86 & 82.39 & 20.96 \\
DL-19 & 33.03 & 27.84 & 53.57 & \textbf{62.63} & \underline{50.50} & \textbf{92.25} & 37.49 & 44.38 & 50.77 & 23.39 & 30.05 & 32.27 & \underline{61.11} & \textbf{51.98} & \underline{86.70} & 47.66 & 48.15 & 67.14 \\
DL-20 & 31.97 & 35.81 & 58.36 & \textbf{63.17} & \underline{55.12} & \textbf{88.12} & 43.62 & 52.97 & 54.65 & 21.38 & 39.65 & 29.44 & \underline{62.72} & \textbf{57.28} & \underline{84.57} & 51.16 & 54.88 & 70.33 \\
\midrule
\multicolumn{19}{l}{\textit{BEIR (14 datasets)}} \\
ArguAna & 35.19 & 91.39 & 27.80 & \textbf{47.20} & \textbf{97.23} & \textbf{38.80} & 9.37 & 45.23 & 7.02 & 18.79 & 72.76 & 13.79 & \underline{40.98} & \underline{94.45} & \underline{32.63} & 15.67 & 64.65 & 11.71 \\
Climate-FEVER & 15.86 & 40.72 & 22.78 & \underline{22.46} & \underline{52.34} & \underline{29.98} & 20.58 & 45.56 & 28.16 & \textbf{24.12} & \textbf{54.04} & \textbf{33.40} & 20.47 & 52.34 & 27.60 & 10.06 & 32.85 & 13.89 \\
CQADupStack & 25.64 & 61.07 & 24.15 & \textbf{40.59} & \underline{73.88} & \textbf{39.57} & 38.52 & 73.70 & 37.18 & 32.96 & 73.53 & 30.70 & \underline{39.77} & \textbf{75.47} & \underline{38.17} & 29.35 & 69.09 & 27.11 \\
DBPedia & 18.19 & 32.00 & 38.00 & \textbf{36.06} & 45.11 & \textbf{70.08} & 11.00 & 25.73 & 18.31 & 15.73 & 37.79 & 26.31 & \underline{32.19} & \textbf{48.49} & \underline{58.10} & 26.11 & \underline{45.74} & 47.71 \\
FEVER & 49.48 & 84.29 & 46.46 & \textbf{67.80} & \textbf{91.73} & \textbf{65.72} & 49.45 & 74.03 & 47.45 & \underline{60.71} & \underline{89.07} & \underline{58.00} & 54.05 & 89.03 & 49.70 & 40.56 & 75.77 & 37.31 \\
FiQA & 26.18 & 62.98 & 29.91 & \underline{32.47} & 65.76 & \underline{39.29} & 31.74 & 67.79 & 37.43 & 27.27 & \underline{69.24} & 29.80 & \textbf{34.76} & \textbf{69.66} & \textbf{42.30} & 24.25 & 64.24 & 25.81 \\
HotpotQA & 51.06 & 67.29 & 66.70 & 53.15 & 65.42 & 71.02 & 53.56 & 69.39 & 68.64 & \underline{56.56} & \underline{73.98} & \underline{71.79} & \textbf{60.14} & \textbf{74.91} & \textbf{77.06} & 53.83 & 70.51 & 69.27 \\
NFCorpus & 30.40 & 30.99 & 48.97 & 34.19 & 32.70 & 53.21 & 36.18 & \textbf{33.59} & 57.29 & \textbf{37.32} & 33.30 & \textbf{59.75} & \underline{36.49} & \underline{33.44} & \underline{58.26} & 35.89 & 31.82 & 56.44 \\
NQ & 25.33 & 73.02 & 21.32 & \textbf{39.93} & \textbf{87.18} & \textbf{34.62} & 6.00 & 26.12 & 4.85 & 6.85 & 31.35 & 5.29 & \underline{28.84} & \underline{85.21} & \underline{23.26} & 19.29 & 72.19 & 15.42 \\
Quora & 83.08 & 97.06 & 82.66 & \textbf{86.91} & \textbf{99.02} & \textbf{86.10} & 76.66 & 94.08 & 75.78 & 78.94 & 95.69 & 78.36 & \underline{86.51} & \underline{98.00} & \underline{86.00} & 74.62 & 94.77 & 73.84 \\
SCIDOCS & 1.94 & 5.36 & 4.44 & \textbf{18.62} & \textbf{43.89} & \textbf{32.31} & 5.15 & 27.85 & 8.57 & 9.10 & 30.96 & 15.27 & \underline{12.42} & \underline{39.26} & \underline{20.43} & 3.56 & 22.57 & 6.60 \\
SciFact & 41.13 & 85.79 & 35.62 & \textbf{68.85} & \textbf{96.67} & \textbf{64.59} & 51.34 & 81.20 & 47.79 & 49.59 & 83.20 & 45.81 & \underline{63.62} & \underline{90.60} & \underline{60.42} & 35.05 & 75.84 & 31.23 \\
TREC-COVID & 31.82 & 3.21 & 60.38 & \textbf{58.25} & \textbf{10.62} & \textbf{84.40} & 30.61 & 3.70 & 53.80 & 29.66 & 2.50 & 53.70 & \underline{39.53} & \underline{4.48} & \underline{67.31} & 24.58 & 2.03 & 45.22 \\
Touche-2020 & 1.52 & 3.69 & 4.66 & \textbf{16.46} & \textbf{40.07} & \textbf{31.23} & 0.09 & 2.15 & 0.34 & 0.00 & 0.08 & 0.00 & \underline{1.42} & \underline{13.83} & \underline{3.39} & 0.00 & 0.31 & 0.00 \\
\midrule
\multicolumn{19}{l}{\textit{BRIGHT (12 datasets)}} \\
AOPS & 2.58 & 7.47 & 4.98 & 1.50 & \textbf{11.03} & 3.19 & \underline{1.69} & 5.24 & \underline{3.84} & 0.39 & 2.24 & 1.03 & \textbf{2.23} & \underline{9.35} & \textbf{4.21} & 0.07 & 1.40 & 0.10 \\
Biology & 15.29 & 50.94 & 20.72 & 7.95 & \underline{35.81} & 10.62 & 3.29 & 26.38 & 3.59 & \textbf{17.38} & \textbf{54.11} & \textbf{22.93} & \underline{9.08} & 34.25 & \underline{13.92} & 8.30 & 30.88 & 11.27 \\
Earth Sci. & 17.69 & 46.08 & 22.11 & \underline{15.51} & 42.18 & \underline{20.74} & 7.85 & \underline{48.27} & 8.24 & \textbf{23.06} & \textbf{49.91} & \textbf{30.75} & 15.26 & 42.19 & 19.88 & 7.77 & 27.75 & 10.94 \\
Economics & 12.24 & 39.33 & 16.18 & \textbf{11.07} & \underline{36.22} & \textbf{13.86} & 0.00 & 26.02 & 0.00 & 7.06 & 26.75 & 10.09 & \underline{9.66} & \textbf{38.04} & \underline{11.90} & 2.66 & 12.32 & 2.85 \\
LeetCode & 14.83 & 48.49 & 14.77 & \underline{14.89} & \underline{35.14} & \underline{15.26} & 5.51 & 23.73 & 6.71 & 8.77 & 20.50 & 10.02 & \textbf{15.60} & \textbf{45.73} & \textbf{15.28} & 4.15 & 11.24 & 4.79 \\
Pony & 1.72 & 10.84 & 5.08 & 0.38 & 3.15 & 0.67 & 4.71 & \underline{28.41} & 8.04 & \textbf{22.09} & 27.09 & \textbf{52.83} & \underline{11.90} & \textbf{29.23} & \underline{30.68} & 3.62 & 23.57 & 7.05 \\
Psychology & 16.66 & 43.57 & 18.79 & 11.47 & 35.49 & \underline{13.42} & 0.00 & 23.57 & 0.00 & \underline{11.94} & \underline{40.20} & 13.11 & \textbf{14.29} & \textbf{40.49} & \textbf{16.22} & 2.95 & 20.16 & 3.12 \\
Robotics & 10.21 & 26.93 & 15.14 & \underline{8.05} & 28.17 & \underline{11.65} & 5.36 & 29.53 & 7.22 & 7.90 & \underline{30.93} & 10.15 & \textbf{12.47} & \textbf{34.85} & \textbf{18.18} & 1.77 & 11.75 & 3.96 \\
StackOverflow & 9.99 & 42.63 & 12.27 & \textbf{11.44} & \underline{39.73} & \textbf{13.25} & 6.94 & 22.58 & 9.05 & 9.58 & 31.99 & \underline{12.72} & \underline{10.65} & \textbf{47.66} & 12.46 & 2.51 & 10.33 & 3.55 \\
Sust. Living & 10.36 & 40.36 & 13.50 & \textbf{9.33} & \textbf{39.86} & \underline{11.78} & 5.69 & 31.74 & 4.97 & 7.65 & 29.96 & 10.37 & \underline{9.31} & \underline{37.19} & \textbf{12.41} & 1.90 & 18.43 & 2.58 \\
TheoremQA-Q & 11.24 & 24.57 & 12.77 & \textbf{12.77} & \underline{25.16} & \textbf{14.39} & 2.62 & 11.00 & 2.74 & 7.49 & 20.61 & 8.90 & \underline{11.35} & \textbf{25.36} & \underline{13.94} & 5.68 & 14.51 & 7.62 \\
TheoremQA-T & 4.96 & 16.23 & 4.62 & \underline{6.04} & \textbf{21.82} & \textbf{5.47} & 1.93 & 12.39 & 1.68 & 1.79 & 12.17 & 2.41 & \textbf{6.47} & \underline{18.64} & \underline{5.42} & 1.32 & 3.29 & 1.97 \\
\midrule
\multicolumn{19}{l}{\textit{Multi-hop (4 datasets)}} \\
2WikiMHopQA & 62.43 & 73.30 & 83.71 & 69.24 & 74.48 & 94.57 & 71.65 & 76.12 & 95.62 & \underline{72.58} & \underline{76.97} & \underline{95.99} & \textbf{72.91} & \textbf{77.03} & \textbf{96.61} & 66.98 & 74.44 & 90.10 \\
MuSiQue & 30.34 & 60.26 & 43.74 & 34.01 & 64.54 & 49.81 & 34.87 & 65.59 & 49.38 & \textbf{36.56} & \underline{66.78} & \textbf{52.32} & \underline{36.33} & \textbf{68.38} & \underline{51.44} & 34.23 & 63.25 & 49.45 \\
NovelHopQA & 61.33 & 94.98 & 55.50 & 54.19 & 93.26 & 48.27 & \underline{65.26} & \underline{94.82} & \underline{60.06} & 61.67 & 94.29 & 56.48 & \textbf{66.55} & \textbf{95.88} & \textbf{61.27} & 49.61 & 88.95 & 44.26 \\
HotpotQA & 51.06 & 67.29 & 66.70 & 53.15 & 65.42 & 71.02 & 53.56 & 69.39 & 68.64 & \underline{56.56} & \underline{73.98} & \underline{71.79} & \textbf{60.14} & \textbf{74.91} & \textbf{77.06} & 53.83 & 70.51 & 69.27 \\
\midrule
\multicolumn{19}{l}{\textit{Averages}} \\
DL Avg (2) & 32.50 & 31.83 & 55.96 & \textbf{62.90} & \underline{52.81} & \textbf{90.18} & 40.56 & 48.68 & 52.71 & 22.38 & 34.85 & 30.86 & \underline{61.92} & \textbf{54.63} & \underline{85.64} & 49.41 & 51.52 & 68.74 \\
BEIR Avg (14) & 31.20 & 52.78 & 36.70 & \textbf{44.50} & \textbf{64.40} & \textbf{52.92} & 30.02 & 47.87 & 35.19 & 31.97 & 53.39 & 37.28 & \underline{39.37} & \underline{62.08} & \underline{46.04} & 28.06 & 51.60 & 32.97 \\
BRIGHT Avg (12) & 10.65 & 33.12 & 13.41 & 9.20 & \underline{29.48} & 11.19 & 3.80 & 24.07 & 4.67 & \underline{10.42} & 28.87 & \textbf{15.44} & \textbf{10.69} & \textbf{33.58} & \underline{14.54} & 3.56 & 15.47 & 4.98 \\
MHop Avg (4) & 51.29 & 73.96 & 62.41 & 52.65 & 74.42 & 65.92 & 56.34 & 76.48 & 68.42 & \underline{56.84} & \underline{78.00} & \underline{69.14} & \textbf{58.98} & \textbf{79.05} & \textbf{71.60} & 51.16 & 74.29 & 63.27 \\
\bottomrule
\end{tabular}
}
\end{table}

\subsection{Training from Foundation Model}

The following tables present the complete evaluation results for training from foundation model. Unlike random initialization below, these models start from pretrained LLM weights (Qwen3-0.6B-Base) without retrieval-specific pretraining, following standard practice for LLM-based retrievers.

\subsubsection{Qwen3-Base-0.6B (82K)}

Table~\ref{tab:complete_qwen_80k} presents the complete results for Qwen3-Base-0.6B trained from foundation model (Qwen3-0.6B-Base) on MS MARCO v1.1 QA (82K samples).

\begin{table}[ht]
\centering
\caption{Complete Evaluation Results for Qwen3-Base-0.6B trained from foundation model (MS MARCO v1.1 QA, 82K samples, single seed). Comparing five training methods (Cosine, Dot, QNorm, DNorm, Learnable). N = NDCG@10, R = Recall@100, M = MRR@10. \textbf{Bold} = best, \underline{underline} = second best.}
\label{tab:complete_qwen_80k}
\resizebox{\textwidth}{!}{
\begin{tabular}{l|ccc|ccc|ccc|ccc|ccc}
\toprule
& \multicolumn{3}{c|}{\textbf{Cosine}} & \multicolumn{3}{c|}{\textbf{Dot}} & \multicolumn{3}{c|}{\textbf{QNorm}} & \multicolumn{3}{c|}{\textbf{DNorm}} & \multicolumn{3}{c}{\textbf{Learnable}} \\
\textbf{Dataset} & N & R & M & N & R & M & N & R & M & N & R & M & N & R & M \\
\midrule
\multicolumn{16}{l}{\textit{In-Domain}} \\
MSM-Dev & \underline{29.42} & \underline{79.77} & \underline{23.83} & 7.17 & 53.63 & 4.76 & 27.28 & 79.59 & 21.27 & 25.71 & 73.33 & 20.94 & \textbf{30.96} & \textbf{82.34} & \textbf{24.98} \\
DL-19 & \underline{54.14} & 44.91 & \underline{81.98} & 19.13 & 29.79 & 26.72 & 51.66 & \underline{47.67} & 71.24 & 50.00 & 41.23 & 80.14 & \textbf{56.65} & \textbf{48.84} & \textbf{85.66} \\
DL-20 & 54.55 & 48.66 & \underline{81.08} & 10.58 & 26.84 & 11.93 & \underline{55.43} & \underline{50.09} & 80.83 & 49.10 & 42.16 & 75.79 & \textbf{56.32} & \textbf{50.46} & \textbf{82.21} \\
\midrule
\multicolumn{16}{l}{\textit{BEIR (14 datasets)}} \\
ArguAna & \underline{48.27} & \textbf{97.44} & \textbf{40.12} & 0.00 & 1.14 & 0.00 & 0.03 & 7.89 & 0.02 & 42.32 & 95.31 & 34.14 & \textbf{48.35} & \underline{96.94} & \underline{39.76} \\
Climate-FEVER & \underline{21.88} & \underline{52.00} & \underline{29.54} & 1.18 & 7.33 & 1.22 & 6.34 & 14.93 & 8.32 & \textbf{28.30} & \textbf{56.89} & \textbf{37.78} & 5.94 & 23.45 & 7.47 \\
CQADupStack & \underline{34.93} & \textbf{70.90} & \underline{33.55} & 20.69 & 64.87 & 17.55 & 33.62 & 68.89 & 31.82 & 33.63 & 68.35 & 32.41 & \textbf{36.41} & \underline{70.57} & \textbf{35.27} \\
DBPedia & \underline{30.03} & 38.29 & \underline{64.19} & 13.10 & 28.58 & 23.42 & 25.28 & \textbf{39.83} & 47.18 & \textbf{30.73} & \underline{38.56} & \textbf{66.20} & 24.88 & 32.91 & 53.06 \\
FEVER & \underline{63.11} & \underline{89.61} & \underline{60.99} & 1.41 & 9.45 & 1.08 & 48.06 & 77.78 & 44.43 & \textbf{66.63} & \textbf{89.86} & \textbf{64.85} & 38.84 & 72.49 & 36.30 \\
FiQA & \underline{30.26} & \underline{64.10} & \underline{37.40} & 17.93 & 55.29 & 19.34 & 29.66 & 62.40 & 35.81 & 27.95 & 59.90 & 34.49 & \textbf{32.51} & \textbf{65.39} & \textbf{38.22} \\
HotpotQA & 42.54 & 55.94 & 58.83 & 13.00 & 36.99 & 15.41 & 32.19 & 52.14 & 39.83 & \underline{44.14} & \underline{59.74} & \underline{59.99} & \textbf{50.08} & \textbf{65.54} & \textbf{67.10} \\
NFCorpus & \textbf{31.05} & \textbf{30.79} & \textbf{50.82} & 28.76 & 27.44 & 47.18 & \underline{29.41} & 27.53 & \underline{48.82} & 28.52 & 27.50 & 47.84 & 29.13 & \underline{28.22} & 48.96 \\
NQ & 36.03 & 81.94 & 31.26 & 1.56 & 17.24 & 1.03 & 25.09 & 80.92 & 19.48 & \underline{36.54} & \underline{82.02} & \underline{31.74} & \textbf{37.26} & \textbf{84.37} & \textbf{31.79} \\
Quora & \textbf{85.17} & \textbf{98.81} & \textbf{84.26} & 70.93 & 90.32 & 69.46 & 81.67 & 95.06 & 81.00 & \underline{84.39} & \underline{98.56} & \underline{83.71} & 82.72 & 96.38 & 82.21 \\
SCIDOCS & \textbf{18.24} & \underline{40.98} & \textbf{31.94} & 6.67 & 26.63 & 11.56 & 16.03 & 36.45 & 28.27 & \underline{18.08} & \textbf{42.11} & 31.08 & 17.93 & 38.57 & \underline{31.75} \\
SciFact & 63.43 & \textbf{93.67} & 59.48 & 18.78 & 32.25 & 17.23 & 40.59 & 61.54 & 38.69 & \underline{63.86} & 90.66 & \underline{60.82} & \textbf{64.90} & \underline{90.99} & \textbf{61.94} \\
TREC-COVID & \underline{72.39} & \textbf{13.45} & \underline{88.42} & 19.90 & 2.44 & 39.46 & 39.59 & 8.89 & 61.03 & 69.30 & \underline{13.03} & \textbf{88.67} & \textbf{72.81} & 13.08 & 86.95 \\
Touche-2020 & 18.42 & 43.80 & 36.37 & 1.16 & 12.75 & 3.62 & 8.86 & 40.99 & 17.83 & \textbf{23.06} & \underline{44.55} & \textbf{42.63} & \underline{22.05} & \textbf{48.06} & \underline{39.13} \\
\midrule
\multicolumn{16}{l}{\textit{BRIGHT (12 datasets)}} \\
AOPS & \textbf{4.73} & \textbf{20.90} & \underline{8.28} & 0.00 & 0.00 & 0.00 & 0.00 & 0.00 & 0.00 & \underline{4.43} & \underline{19.04} & \textbf{8.56} & 1.96 & 18.73 & 4.52 \\
Biology & \underline{7.97} & \underline{33.49} & \underline{11.59} & 0.14 & 0.78 & 0.24 & 0.24 & 0.97 & 0.36 & \textbf{11.21} & \textbf{44.50} & \textbf{17.58} & 0.75 & 3.62 & 0.94 \\
Earth Sci. & \textbf{19.27} & \textbf{51.29} & \textbf{23.24} & 0.92 & 1.72 & 0.65 & 0.54 & 1.72 & 0.43 & \underline{11.06} & \underline{36.87} & \underline{13.37} & 0.36 & 5.36 & 0.41 \\
Economics & \textbf{11.99} & \textbf{39.28} & \textbf{13.80} & 0.21 & 0.01 & 0.97 & 0.00 & 0.00 & 0.00 & \underline{11.07} & \underline{36.85} & \underline{12.95} & 1.96 & 9.23 & 2.73 \\
LeetCode & \textbf{16.52} & \underline{48.84} & \textbf{17.61} & 0.00 & 0.00 & 0.00 & 0.00 & 0.00 & 0.00 & \underline{15.02} & \textbf{49.74} & \underline{15.03} & 7.46 & 21.53 & 8.37 \\
Pony & \underline{3.48} & \underline{15.56} & \underline{6.32} & 0.00 & 0.00 & 0.00 & 0.00 & 0.00 & 0.00 & \textbf{12.77} & \textbf{22.68} & \textbf{29.09} & 0.53 & 3.77 & 1.07 \\
Psychology & \underline{9.83} & \textbf{40.04} & \underline{12.53} & 0.00 & 0.00 & 0.00 & 0.00 & 0.00 & 0.00 & \textbf{10.23} & \underline{36.38} & \textbf{13.81} & 0.90 & 4.08 & 1.44 \\
Robotics & \textbf{10.58} & \textbf{34.83} & \textbf{13.34} & 0.00 & 0.00 & 0.00 & 0.00 & 0.00 & 0.00 & \underline{7.32} & \underline{24.09} & \underline{9.88} & 2.54 & 9.66 & 5.66 \\
StackOverflow & \textbf{10.60} & \textbf{48.33} & \textbf{12.28} & 0.00 & 0.00 & 0.00 & 0.00 & 0.00 & 0.00 & \underline{3.30} & \underline{30.72} & \underline{4.56} & 0.00 & 1.72 & 0.00 \\
Sust. Living & \underline{10.46} & \underline{40.07} & \underline{12.21} & 0.10 & 0.07 & 0.31 & 0.00 & 0.00 & 0.00 & \textbf{11.03} & \textbf{45.51} & \textbf{13.87} & 0.84 & 10.77 & 2.14 \\
TheoremQA-Q & \underline{12.66} & \underline{26.92} & \underline{14.32} & 0.00 & 0.00 & 0.00 & 0.52 & 0.52 & 0.52 & \textbf{13.25} & \textbf{30.76} & \textbf{14.56} & 6.06 & 15.92 & 7.31 \\
TheoremQA-T & \underline{6.58} & \textbf{31.99} & \textbf{8.30} & 0.00 & 0.00 & 0.00 & 0.00 & 0.00 & 0.00 & \textbf{6.92} & \underline{26.32} & \underline{6.41} & 3.00 & 11.95 & 2.63 \\
\midrule
\multicolumn{16}{l}{\textit{Multi-hop (4 datasets)}} \\
2WikiMHopQA & \underline{61.50} & 69.45 & \textbf{88.65} & 33.44 & 48.99 & 45.99 & 61.15 & \underline{71.03} & 83.82 & 55.17 & 67.22 & 80.34 & \textbf{63.62} & \textbf{72.41} & \underline{88.37} \\
MuSiQue & \underline{30.86} & \underline{58.91} & \underline{45.73} & 16.72 & 31.41 & 23.81 & 27.96 & 49.70 & 40.08 & 25.28 & 54.55 & 36.06 & \textbf{31.98} & \textbf{59.55} & \textbf{46.90} \\
NovelHopQA & \textbf{61.86} & \textbf{95.26} & \textbf{55.91} & 11.58 & 27.20 & 9.84 & 28.44 & 42.39 & 26.12 & \underline{51.75} & \underline{91.35} & \underline{45.64} & 49.21 & 90.63 & 43.45 \\
HotpotQA & 42.54 & 55.94 & 58.83 & 13.00 & 36.99 & 15.41 & 32.19 & 52.14 & 39.83 & \underline{44.14} & \underline{59.74} & \underline{59.99} & \textbf{50.08} & \textbf{65.54} & \textbf{67.10} \\
\bottomrule
\end{tabular}
}
\end{table}

\subsubsection{Qwen3-Base-0.6B (503K)}

Table~\ref{tab:complete_qwen_500k} presents the complete results for Qwen3-Base-0.6B trained from foundation model (Qwen3-0.6B-Base) on MS MARCO Passage Ranking (503K samples).

\begin{table}[ht]
\centering
\caption{Complete Evaluation Results for Qwen3-Base-0.6B trained from foundation model (MS MARCO Passage Ranking, 503K samples, single seed). Comparing five training methods (Cosine, Dot, QNorm, DNorm, Learnable). N = NDCG@10, R = Recall@100, M = MRR@10. \textbf{Bold} = best, \underline{underline} = second best.}
\label{tab:complete_qwen_500k}
\resizebox{\textwidth}{!}{
\begin{tabular}{l|ccc|ccc|ccc|ccc|ccc}
\toprule
& \multicolumn{3}{c|}{\textbf{Cosine}} & \multicolumn{3}{c|}{\textbf{Dot}} & \multicolumn{3}{c|}{\textbf{QNorm}} & \multicolumn{3}{c|}{\textbf{DNorm}} & \multicolumn{3}{c}{\textbf{Learnable}} \\
\textbf{Dataset} & N & R & M & N & R & M & N & R & M & N & R & M & N & R & M \\
\midrule
\multicolumn{16}{l}{\textit{In-Domain}} \\
MSM-Dev & \underline{31.32} & \underline{83.69} & \underline{25.06} & 13.54 & 70.76 & 9.27 & 26.80 & 80.80 & 20.47 & \textbf{33.52} & \textbf{85.34} & \textbf{26.85} & 15.83 & 58.71 & 12.36 \\
DL-19 & \underline{56.39} & \underline{46.77} & 80.31 & 23.54 & 40.96 & 28.13 & 51.61 & 46.64 & 69.64 & \textbf{61.18} & \textbf{49.68} & \textbf{84.71} & 39.52 & 38.59 & \underline{70.78} \\
DL-20 & \textbf{60.11} & \underline{50.67} & \textbf{89.11} & 26.30 & 41.79 & 30.93 & 48.51 & 49.31 & 67.36 & \underline{59.93} & \textbf{52.72} & \underline{85.03} & 37.75 & 34.62 & 74.17 \\
\midrule
\multicolumn{16}{l}{\textit{BEIR (14 datasets)}} \\
ArguAna & \textbf{50.29} & \textbf{98.01} & \textbf{41.42} & 0.02 & 1.14 & 0.01 & 1.17 & 16.93 & 0.69 & \underline{49.79} & \underline{97.01} & \underline{41.23} & 21.51 & 73.26 & 16.67 \\
Climate-FEVER & \underline{27.12} & \underline{58.63} & \underline{36.70} & 0.71 & 2.26 & 1.00 & 2.29 & 6.53 & 2.86 & \textbf{28.58} & \textbf{59.32} & \textbf{38.41} & 0.52 & 2.73 & 0.56 \\
CQADupStack & \underline{37.74} & \textbf{73.27} & \underline{36.34} & 22.96 & 54.24 & 21.37 & 31.97 & 68.50 & 29.91 & \textbf{39.41} & \underline{72.38} & \textbf{38.48} & 24.10 & 57.18 & 22.87 \\
DBPedia & \underline{32.74} & \underline{41.42} & \textbf{66.55} & 14.02 & 27.15 & 27.22 & 20.55 & 32.08 & 37.94 & \textbf{32.66} & \textbf{41.53} & \underline{64.17} & 12.19 & 19.76 & 28.83 \\
FEVER & \underline{68.90} & \underline{93.54} & \underline{66.41} & 0.19 & 0.34 & 0.17 & 0.96 & 1.96 & 0.89 & \textbf{70.44} & \textbf{93.48} & \textbf{68.08} & 2.04 & 12.61 & 1.65 \\
FiQA & \underline{31.91} & \underline{64.01} & \underline{39.07} & 14.73 & 39.74 & 16.86 & 27.43 & 60.83 & 32.25 & \textbf{32.86} & \textbf{65.84} & \textbf{39.44} & 18.36 & 49.50 & 22.58 \\
HotpotQA & \underline{51.44} & \underline{63.80} & \underline{69.28} & 8.26 & 17.00 & 10.79 & 31.02 & 57.37 & 37.84 & \textbf{55.86} & \textbf{69.19} & \textbf{73.85} & 3.93 & 11.96 & 5.20 \\
NFCorpus & \textbf{31.88} & \textbf{31.55} & \textbf{52.41} & 27.36 & 25.54 & 47.36 & 28.41 & 26.49 & 47.51 & \underline{30.55} & \underline{28.90} & \underline{49.88} & 27.58 & 27.31 & 46.35 \\
NQ & \underline{37.98} & \underline{85.56} & \underline{32.74} & 4.41 & 38.84 & 3.03 & 12.34 & 67.92 & 8.81 & \textbf{38.88} & \textbf{86.79} & \textbf{33.36} & 20.37 & 70.97 & 16.00 \\
Quora & \textbf{85.76} & \textbf{98.99} & \textbf{84.85} & 66.54 & 82.71 & 65.47 & 81.04 & 94.32 & 80.54 & \underline{85.12} & \underline{97.76} & \underline{84.79} & 48.51 & 74.57 & 47.02 \\
SCIDOCS & \textbf{18.02} & \textbf{40.76} & \textbf{31.31} & 4.39 & 15.97 & 7.94 & 12.57 & 32.80 & 21.75 & \underline{17.91} & \underline{38.69} & \underline{31.28} & 8.75 & 24.83 & 16.54 \\
SciFact & \textbf{67.72} & \underline{93.33} & \textbf{64.18} & 1.80 & 4.50 & 1.61 & 8.87 & 33.76 & 7.49 & \underline{66.81} & \textbf{94.00} & \underline{63.55} & 30.06 & 69.73 & 26.69 \\
TREC-COVID & \textbf{68.19} & \underline{12.65} & \textbf{88.50} & 21.66 & 2.40 & 46.32 & 33.23 & 3.51 & 59.01 & \underline{68.14} & \textbf{13.37} & \underline{85.92} & 39.17 & 6.05 & 66.07 \\
Touche-2020 & \underline{16.94} & \underline{41.32} & 29.84 & 1.05 & 13.85 & 3.97 & 4.12 & 34.93 & 6.62 & \textbf{17.63} & \textbf{41.88} & \textbf{32.76} & 12.84 & 35.03 & \underline{23.80} \\
\midrule
\multicolumn{16}{l}{\textit{BRIGHT (12 datasets)}} \\
AOPS & \underline{2.91} & \underline{17.02} & \textbf{6.09} & 0.00 & 0.00 & 0.00 & 0.00 & 0.00 & 0.00 & \textbf{3.07} & \textbf{20.13} & \underline{5.72} & 0.00 & 0.00 & 0.00 \\
Biology & \underline{8.09} & \underline{36.48} & \underline{11.29} & 0.21 & 0.19 & 0.48 & 0.16 & 0.58 & 0.32 & \textbf{12.68} & \textbf{39.60} & \textbf{18.62} & 1.07 & 3.90 & 1.66 \\
Earth Sci. & \underline{20.17} & \underline{50.71} & \underline{25.06} & 0.54 & 0.95 & 0.43 & 1.13 & 2.16 & 0.97 & \textbf{24.97} & \textbf{56.14} & \textbf{31.91} & 0.16 & 3.62 & 0.10 \\
Economics & \underline{13.66} & \underline{42.87} & \underline{15.56} & 0.00 & 0.05 & 0.00 & 0.00 & 0.01 & 0.00 & \textbf{17.27} & \textbf{46.75} & \textbf{19.42} & 0.58 & 4.53 & 0.91 \\
LeetCode & \underline{17.90} & \underline{57.31} & \textbf{18.62} & 0.00 & 0.00 & 0.00 & 0.00 & 0.00 & 0.00 & \textbf{17.96} & \textbf{57.76} & \underline{16.48} & 2.54 & 6.00 & 4.03 \\
Pony & \underline{4.22} & \underline{19.44} & 9.07 & 0.00 & 0.00 & 0.00 & 0.00 & 0.00 & 0.00 & 4.96 & 26.34 & \underline{12.30} & \textbf{10.39} & \textbf{37.33} & \textbf{28.64} \\
Psychology & \underline{11.34} & \underline{40.54} & \underline{13.76} & 0.00 & 0.41 & 0.00 & 0.00 & 0.00 & 0.00 & \textbf{13.34} & \textbf{44.17} & \textbf{15.40} & 1.43 & 6.75 & 1.90 \\
Robotics & \underline{13.11} & \underline{34.14} & \underline{17.25} & 0.00 & 0.00 & 0.00 & 0.00 & 0.00 & 0.00 & \textbf{17.44} & \textbf{42.13} & \textbf{23.11} & 0.80 & 6.89 & 1.63 \\
StackOverflow & \underline{10.92} & \underline{48.39} & \underline{11.21} & 0.00 & 0.09 & 0.00 & 0.00 & 0.00 & 0.00 & \textbf{16.32} & \textbf{55.20} & \textbf{18.82} & 0.78 & 4.60 & 1.51 \\
Sust. Living & \underline{9.50} & \underline{44.81} & \underline{11.18} & 0.00 & 0.02 & 0.00 & 0.00 & 0.00 & 0.00 & \textbf{11.47} & \textbf{47.38} & \textbf{14.20} & 1.28 & 6.79 & 1.97 \\
TheoremQA-Q & \textbf{13.78} & \textbf{28.12} & \underline{14.83} & 0.00 & 0.00 & 0.00 & 0.58 & 0.64 & 0.57 & \underline{12.65} & \underline{26.58} & \textbf{14.95} & 0.52 & 0.77 & 0.52 \\
TheoremQA-T & \underline{7.36} & \textbf{34.87} & \underline{7.60} & 0.00 & 0.00 & 0.00 & 0.00 & 1.32 & 0.00 & \textbf{9.52} & \underline{36.09} & \textbf{10.70} & 0.86 & 4.61 & 0.66 \\
\midrule
\multicolumn{16}{l}{\textit{Multi-hop (4 datasets)}} \\
2WikiMHopQA & \underline{69.09} & \underline{72.77} & \textbf{95.57} & 29.37 & 37.35 & 42.11 & 62.84 & 72.32 & 85.62 & \textbf{68.70} & \textbf{75.10} & \underline{93.84} & 17.73 & 37.90 & 23.67 \\
MuSiQue & \textbf{35.04} & \textbf{63.22} & \textbf{52.53} & 17.54 & 28.61 & 25.33 & 32.12 & 59.28 & 47.26 & \underline{34.56} & \underline{62.75} & \underline{50.39} & 8.16 & 23.07 & 11.55 \\
NovelHopQA & \underline{62.65} & \underline{96.27} & \underline{56.53} & 12.16 & 23.80 & 10.83 & 62.85 & 92.57 & 57.46 & \textbf{71.16} & \textbf{96.96} & \textbf{65.94} & 13.15 & 54.80 & 10.46 \\
HotpotQA & \underline{51.44} & \underline{63.80} & \underline{69.28} & 8.26 & 17.00 & 10.79 & 31.02 & 57.37 & 37.84 & \textbf{55.86} & \textbf{69.19} & \textbf{73.85} & 3.93 & 11.96 & 5.20 \\
\bottomrule
\end{tabular}
}
\end{table}

\subsection{Training from Random Initialization}
\label{appendix:random_init_curves}

The following tables present the complete evaluation results for training from random initialization on MS MARCO v1.1 QA. Unlike the finetuning and foundation model experiments above, these models use only the model architecture (BERT-base) with randomly initialized weights, without any pretrained representations. Training curves are shown in \Cref{fig:random_init_curves}.

\input{floats/figure_random_init_curves}

\subsubsection{Contriever (Random Init)}

Table~\ref{tab:complete_contriever_random_init} presents the complete results for Contriever trained from random initialization.

\begin{table}[ht]
\centering
\caption{Complete Evaluation Results for Contriever Training from Random Initialization (MS MARCO v1.1 QA 82K, mean of seeds 0, 42, 1337). N = NDCG@10, R = Recall@100, M = MRR@10. \textbf{Bold} = best, \underline{underline} = second best.}
\label{tab:complete_contriever_random_init}
\resizebox{\textwidth}{!}{
\begin{tabular}{l|ccc|ccc|ccc|ccc|ccc}
\toprule
& \multicolumn{3}{c|}{\textbf{Cosine}} & \multicolumn{3}{c|}{\textbf{Dot}} & \multicolumn{3}{c|}{\textbf{QNorm}} & \multicolumn{3}{c|}{\textbf{DNorm}} & \multicolumn{3}{c}{\textbf{Learnable}} \\
\textbf{Dataset} & N & R & M & N & R & M & N & R & M & N & R & M & N & R & M \\
\midrule
\multicolumn{16}{l}{\textit{In-Domain}} \\
MSM-Dev & \textbf{17.7} & \textbf{49.4} & \textbf{14.6} & 11.1 & 41.2 & 8.8 & 9.6 & 40.0 & 7.5 & \underline{15.2} & \underline{48.4} & \underline{12.3} & 12.0 & 43.9 & 9.6 \\
DL-19 & \textbf{38.5} & \textbf{25.4} & \textbf{68.4} & 24.3 & 20.8 & 48.8 & 20.6 & 19.3 & 40.9 & \underline{29.8} & \underline{24.3} & \underline{61.6} & 24.6 & 21.6 & 49.9 \\
DL-20 & \textbf{38.0} & \underline{30.7} & \textbf{71.5} & 24.9 & 25.9 & 54.5 & 22.3 & 24.9 & 47.3 & \underline{33.6} & \textbf{30.8} & \underline{66.6} & 27.4 & 27.4 & 55.6 \\
\midrule
\multicolumn{16}{l}{\textit{BEIR (14 datasets)}} \\
ArguAna & \textbf{18.8} & \textbf{61.9} & \textbf{14.4} & 0.7 & 9.3 & 0.4 & 4.9 & 24.3 & 3.6 & 14.8 & 40.8 & 11.9 & \underline{17.8} & \underline{54.4} & \underline{14.3} \\
Climate-FEVER & \textbf{6.8} & \textbf{18.1} & \textbf{9.4} & 1.1 & 3.9 & 1.4 & 1.4 & 5.7 & 2.1 & \underline{4.2} & \underline{12.0} & \underline{5.9} & 2.0 & 8.1 & 2.8 \\
CQADupStack & \textbf{14.8} & \underline{36.5} & \textbf{14.4} & 11.6 & 34.0 & 11.0 & 11.6 & 34.8 & 10.9 & \underline{14.8} & \textbf{37.5} & \underline{14.4} & 12.9 & 36.4 & 12.3 \\
DBPedia & \textbf{19.0} & \textbf{20.9} & \textbf{44.4} & 11.2 & 17.2 & 27.1 & 10.3 & 17.3 & 24.8 & \underline{15.4} & \underline{20.6} & \underline{36.2} & 11.2 & 18.2 & 27.5 \\
FEVER & \textbf{38.8} & \textbf{70.9} & \textbf{36.1} & 10.5 & 32.9 & 9.1 & 9.3 & 33.9 & 7.9 & \underline{23.1} & \underline{56.8} & \underline{20.5} & 11.1 & 38.4 & 9.4 \\
FiQA & \textbf{7.5} & \textbf{22.1} & \textbf{10.3} & 4.1 & 14.0 & 5.6 & 5.0 & 17.3 & 6.7 & \underline{6.4} & \underline{20.0} & \underline{8.9} & 5.4 & 18.3 & 7.2 \\
HotpotQA & \underline{29.5} & \underline{39.2} & \textbf{42.8} & 17.7 & 30.0 & 25.0 & 19.1 & 33.6 & 26.8 & \textbf{29.6} & \textbf{40.7} & \underline{42.4} & 21.7 & 36.7 & 30.3 \\
NFCorpus & \underline{20.5} & \textbf{18.7} & 38.0 & 18.9 & 17.9 & 35.4 & 19.3 & 17.9 & 36.9 & \textbf{20.7} & 18.5 & \textbf{39.1} & 20.0 & \underline{18.5} & \underline{38.3} \\
NQ & \textbf{18.0} & \textbf{50.6} & \textbf{15.4} & 6.9 & 32.9 & 5.5 & 8.2 & 39.6 & 6.6 & \underline{13.7} & \underline{46.9} & \underline{11.4} & 7.7 & 38.6 & 6.1 \\
Quora & \textbf{68.3} & \underline{90.0} & \textbf{68.2} & 47.9 & 65.7 & 48.3 & 63.0 & 86.2 & 62.8 & 66.5 & 86.7 & 66.5 & \underline{66.9} & \textbf{90.9} & \underline{66.6} \\
SCIDOCS & \underline{8.0} & 19.3 & \underline{15.4} & 6.3 & 17.9 & 11.9 & 7.2 & 19.4 & 13.2 & \textbf{8.4} & \underline{20.2} & \textbf{15.9} & 7.5 & \textbf{20.5} & 13.8 \\
SciFact & \underline{40.0} & \underline{70.4} & \underline{37.2} & 28.2 & 62.0 & 25.1 & 32.5 & 67.4 & 29.2 & \textbf{41.1} & 68.8 & \textbf{38.4} & 40.0 & \textbf{73.2} & 36.5 \\
TREC-COVID & \textbf{37.0} & \textbf{5.5} & \textbf{57.5} & 16.1 & 1.8 & 33.2 & 19.7 & 2.5 & 33.8 & \underline{26.8} & \underline{3.8} & \underline{46.2} & 22.2 & 1.9 & 42.5 \\
Touche-2020 & \textbf{14.1} & \textbf{29.7} & \textbf{31.5} & 6.1 & 12.5 & 17.6 & 4.9 & 14.1 & 13.8 & \underline{9.5} & \underline{21.7} & \underline{24.5} & 4.3 & 13.2 & 10.9 \\
\midrule
\multicolumn{16}{l}{\textit{BRIGHT (12 datasets)}} \\
AOPS & \textbf{2.9} & \underline{7.0} & \textbf{6.4} & 0.4 & 1.4 & 0.7 & 0.0 & 2.3 & 0.0 & \underline{2.7} & \textbf{7.2} & \underline{5.8} & 0.7 & 4.8 & 1.7 \\
Biology & \underline{2.5} & \underline{14.0} & \underline{4.1} & 0.2 & 1.1 & 0.2 & 0.5 & 1.9 & 0.8 & \textbf{3.8} & \textbf{15.6} & \textbf{6.2} & 0.9 & 3.8 & 1.4 \\
Earth Sci. & \underline{5.7} & \underline{18.6} & \underline{9.1} & 0.1 & 0.6 & 0.2 & 0.4 & 1.0 & 1.2 & \textbf{7.9} & \textbf{19.6} & \textbf{11.2} & 1.0 & 3.1 & 1.6 \\
Economics & \underline{2.2} & \underline{13.4} & \underline{3.2} & 0.0 & 1.1 & 0.0 & 0.1 & 1.7 & 0.0 & \textbf{3.2} & \textbf{13.7} & \textbf{5.2} & 0.4 & 5.6 & 0.6 \\
LeetCode & \textbf{7.0} & \textbf{25.7} & \textbf{7.6} & 2.0 & 12.8 & 2.1 & 0.6 & 14.1 & 0.5 & \underline{5.9} & \underline{24.3} & \underline{6.1} & 4.3 & 20.0 & 4.6 \\
Pony & \underline{1.1} & \textbf{5.6} & \underline{3.1} & 0.2 & 1.8 & 0.4 & 0.1 & 2.4 & 0.3 & \textbf{4.4} & \underline{4.3} & \textbf{13.4} & 0.3 & 3.7 & 0.7 \\
Psychology & \underline{2.6} & \textbf{13.4} & \underline{3.5} & 0.0 & 0.1 & 0.0 & 0.0 & 0.6 & 0.0 & \textbf{3.1} & \underline{7.9} & \textbf{4.9} & 0.7 & 3.2 & 1.3 \\
Robotics & \underline{4.1} & \underline{16.1} & \underline{6.0} & 0.0 & 0.8 & 0.0 & 0.8 & 3.8 & 0.9 & \textbf{5.0} & \textbf{16.2} & \textbf{8.6} & 1.7 & 10.4 & 2.9 \\
StackOverflow & \textbf{2.7} & \textbf{12.7} & \underline{2.9} & 0.0 & 0.1 & 0.0 & 0.2 & 1.2 & 0.4 & \underline{2.4} & \underline{10.1} & \textbf{3.3} & 0.6 & 4.9 & 0.8 \\
Sust. Living & \underline{4.8} & \textbf{22.3} & \underline{5.8} & 0.2 & 2.1 & 0.3 & 0.4 & 5.9 & 0.6 & \textbf{5.7} & \underline{17.4} & \textbf{7.6} & 2.0 & 13.2 & 3.1 \\
TheoremQA-Q & \textbf{1.5} & \textbf{3.9} & \textbf{2.0} & 0.6 & 0.6 & 0.9 & 0.7 & 1.1 & 1.0 & 1.1 & \underline{3.2} & \underline{1.4} & \underline{1.1} & 2.3 & 1.3 \\
TheoremQA-T & \textbf{0.0} & \textbf{0.9} & \textbf{0.0} & \underline{0.0} & \underline{0.2} & \underline{0.0} & 0.0 & 0.0 & 0.0 & 0.0 & 0.2 & 0.0 & 0.0 & 0.2 & 0.0 \\
\midrule
\multicolumn{16}{l}{\textit{Multi-hop (4 datasets)}} \\
2WikiMHopQA & \underline{43.9} & 55.3 & \underline{62.6} & 34.0 & 51.6 & 47.9 & 35.0 & 53.6 & 48.4 & \textbf{44.4} & \textbf{57.8} & \textbf{62.9} & 37.5 & \underline{56.4} & 51.9 \\
MuSiQue & 17.7 & 32.9 & \underline{27.3} & 14.0 & 26.5 & 20.9 & 16.4 & 30.6 & 24.8 & \textbf{18.2} & \underline{32.9} & \textbf{27.5} & \underline{17.9} & \textbf{33.6} & 27.2 \\
NovelHopQA & \textbf{20.4} & \textbf{59.5} & \textbf{16.8} & 7.4 & 26.8 & 6.3 & 10.5 & 37.6 & 8.7 & \underline{18.8} & \underline{55.2} & \underline{15.7} & 13.8 & 51.2 & 11.3 \\
HotpotQA & \underline{29.5} & \underline{39.2} & \textbf{42.8} & 17.7 & 30.0 & 25.0 & 19.1 & 33.6 & 26.8 & \textbf{29.6} & \textbf{40.7} & \underline{42.4} & 21.7 & 36.7 & 30.3 \\
\midrule
\multicolumn{16}{l}{\textit{Averages}} \\
DL Avg (2) & \textbf{38.2} & \textbf{28.1} & \textbf{69.9} & 24.6 & 23.3 & 51.6 & 21.4 & 22.1 & 44.1 & \underline{31.7} & \underline{27.5} & \underline{64.1} & 26.0 & 24.5 & 52.8 \\
BEIR Avg (14) & \textbf{24.4} & \textbf{39.6} & \textbf{31.1} & 13.4 & 25.1 & 18.3 & 15.5 & 29.6 & 19.9 & \underline{21.1} & \underline{35.4} & \underline{27.3} & 17.9 & 33.4 & 22.7 \\
BRIGHT Avg (12) & \underline{3.1} & \textbf{12.8} & \underline{4.5} & 0.3 & 1.9 & 0.4 & 0.3 & 3.0 & 0.5 & \textbf{3.8} & \underline{11.6} & \textbf{6.1} & 1.2 & 6.3 & 1.7 \\
MHop Avg (4) & \textbf{27.9} & \textbf{46.7} & \textbf{37.4} & 18.3 & 33.7 & 25.0 & 20.3 & 38.8 & 27.2 & \underline{27.7} & \underline{46.6} & \underline{37.1} & 22.7 & 44.5 & 30.2 \\
\bottomrule
\end{tabular}
}
\end{table}

\subsubsection{RetroMAE (Random Init)}

Table~\ref{tab:complete_retromae_random_init} presents the complete results for RetroMAE trained from random initialization.

\begin{table}[ht]
\centering
\caption{Complete Evaluation Results for RetroMAE Training from Random Initialization (MS MARCO v1.1 QA 82K, mean of seeds 0, 42, 1337). N = NDCG@10, R = Recall@100, M = MRR@10. \textbf{Bold} = best, \underline{underline} = second best.}
\label{tab:complete_retromae_random_init}
\resizebox{\textwidth}{!}{
\begin{tabular}{l|ccc|ccc|ccc|ccc|ccc}
\toprule
& \multicolumn{3}{c|}{\textbf{Cosine}} & \multicolumn{3}{c|}{\textbf{Dot}} & \multicolumn{3}{c|}{\textbf{QNorm}} & \multicolumn{3}{c|}{\textbf{DNorm}} & \multicolumn{3}{c}{\textbf{Learnable}} \\
\textbf{Dataset} & N & R & M & N & R & M & N & R & M & N & R & M & N & R & M \\
\midrule
\multicolumn{16}{l}{\textit{In-Domain}} \\
MSM-Dev & \textbf{19.2} & \textbf{53.0} & \textbf{16.0} & 14.1 & 47.8 & 11.3 & 13.7 & 47.6 & 11.0 & \underline{17.2} & \underline{52.4} & \underline{14.1} & 14.3 & 48.8 & 11.6 \\
DL-19 & \textbf{41.0} & \underline{26.8} & \textbf{71.0} & 28.8 & 23.4 & 59.0 & 27.6 & 24.2 & 56.8 & \underline{35.6} & \textbf{27.5} & \underline{65.8} & 28.9 & 24.6 & 58.4 \\
DL-20 & \textbf{41.2} & \underline{33.0} & \textbf{73.4} & 34.6 & 30.6 & 66.1 & 32.8 & 31.6 & 64.1 & \underline{40.3} & \textbf{34.9} & \underline{73.3} & 34.9 & 32.7 & 65.8 \\
\midrule
\multicolumn{16}{l}{\textit{BEIR (14 datasets)}} \\
ArguAna & \textbf{20.7} & \textbf{67.4} & \textbf{15.8} & 6.3 & 28.4 & 4.7 & 7.6 & 31.5 & 5.8 & \underline{16.6} & 44.4 & \underline{13.3} & 16.5 & \underline{48.7} & 13.1 \\
Climate-FEVER & \textbf{8.3} & \textbf{21.1} & \textbf{11.5} & 2.9 & 8.9 & 3.9 & 3.9 & 11.2 & 5.4 & \underline{4.9} & \underline{13.0} & \underline{6.7} & 4.0 & 12.6 & 5.4 \\
CQADupStack & \underline{14.6} & 36.4 & \underline{14.2} & 13.7 & 37.0 & 13.1 & 13.7 & 37.3 & 13.1 & \textbf{15.4} & \textbf{38.1} & \textbf{15.0} & 14.3 & \underline{37.9} & 13.7 \\
DBPedia & \textbf{20.1} & \textbf{22.6} & \textbf{45.5} & 14.4 & 20.6 & 33.9 & 15.0 & 20.8 & 34.9 & \underline{17.4} & \underline{21.8} & \underline{39.5} & 14.6 & 20.9 & 33.3 \\
FEVER & \textbf{42.2} & \textbf{74.8} & \textbf{39.4} & 17.8 & 48.8 & 15.9 & 18.4 & 51.8 & 16.2 & \underline{27.1} & \underline{63.5} & \underline{24.5} & 17.9 & 52.5 & 15.6 \\
FiQA & \textbf{8.6} & \textbf{26.4} & \textbf{11.3} & 5.8 & 20.1 & 7.8 & 6.5 & 22.2 & 8.9 & \underline{7.3} & 22.8 & \underline{9.9} & 6.8 & \underline{22.9} & 9.0 \\
HotpotQA & \underline{31.5} & 41.2 & \underline{45.5} & 26.6 & 39.2 & 37.7 & 27.7 & 41.0 & 39.0 & \textbf{32.4} & \textbf{43.4} & \textbf{46.4} & 28.0 & \underline{42.1} & 39.3 \\
NFCorpus & \textbf{21.4} & \textbf{19.0} & \textbf{38.8} & 20.2 & 18.6 & 37.2 & 20.2 & 18.7 & 36.8 & \underline{20.9} & 18.0 & \underline{38.3} & 20.4 & \underline{19.0} & 36.8 \\
NQ & \textbf{18.4} & \textbf{53.1} & \textbf{15.6} & 11.0 & 44.1 & 8.9 & 12.4 & 48.4 & 10.2 & \underline{15.6} & \underline{50.8} & \underline{13.0} & 10.9 & 46.7 & 8.7 \\
Quora & 68.7 & \underline{90.3} & 68.6 & 62.8 & 81.6 & 63.0 & 66.7 & 87.3 & 66.6 & \underline{69.6} & 89.5 & \textbf{69.5} & \textbf{69.9} & \textbf{92.5} & \underline{69.4} \\
SCIDOCS & \underline{8.7} & 19.6 & \underline{16.6} & 8.0 & 19.7 & 15.3 & 8.5 & \underline{20.5} & 16.2 & \textbf{8.9} & 20.2 & \textbf{17.0} & 8.6 & \textbf{21.1} & 16.1 \\
SciFact & 41.3 & 74.5 & 37.9 & 38.4 & 72.2 & 35.1 & 40.1 & \underline{74.7} & 36.9 & \textbf{44.1} & 73.6 & \textbf{41.4} & \underline{43.7} & \textbf{76.7} & \underline{40.1} \\
TREC-COVID & \textbf{40.7} & \textbf{6.2} & \textbf{61.5} & 21.5 & 3.1 & 42.8 & 24.5 & 3.9 & 44.6 & \underline{31.7} & \underline{4.8} & \underline{54.3} & 25.9 & 2.9 & 51.8 \\
Touche-2020 & \textbf{14.9} & \textbf{31.0} & \textbf{32.5} & 8.4 & 19.6 & 26.0 & 8.4 & 21.1 & 23.6 & \underline{11.0} & \underline{24.9} & \underline{29.9} & 6.5 & 18.8 & 16.9 \\
\midrule
\multicolumn{16}{l}{\textit{BRIGHT (12 datasets)}} \\
AOPS & \underline{3.0} & \textbf{8.6} & \underline{6.2} & 1.0 & 4.3 & 1.9 & 0.8 & 4.8 & 1.3 & \textbf{3.2} & \underline{8.1} & \textbf{6.9} & 1.2 & 5.6 & 2.3 \\
Biology & \underline{2.7} & \underline{13.1} & \underline{4.8} & 0.8 & 2.7 & 1.6 & 1.3 & 4.3 & 2.6 & \textbf{4.5} & \textbf{16.9} & \textbf{6.7} & 1.5 & 8.8 & 3.0 \\
Earth Sci. & \underline{6.0} & \textbf{22.5} & \underline{9.6} & 0.9 & 2.4 & 1.5 & 1.5 & 4.8 & 2.2 & \textbf{8.1} & \underline{22.3} & \textbf{11.8} & 2.4 & 8.3 & 3.8 \\
Economics & \underline{3.1} & \underline{15.2} & \underline{4.4} & 0.2 & 3.4 & 0.3 & 0.7 & 7.2 & 0.6 & \textbf{4.2} & \textbf{15.5} & \textbf{6.0} & 2.1 & 12.7 & 3.1 \\
LeetCode & \textbf{7.0} & \textbf{26.9} & \textbf{6.8} & 3.6 & 19.3 & 3.9 & 2.9 & 17.6 & 3.2 & \underline{6.1} & \underline{23.1} & \underline{6.2} & 4.1 & 20.4 & 4.4 \\
Pony & \underline{0.2} & \underline{2.5} & \underline{0.7} & 0.1 & 1.3 & 0.5 & 0.0 & 1.2 & 0.1 & \textbf{0.8} & \textbf{2.6} & \textbf{2.0} & 0.0 & 1.1 & 0.1 \\
Psychology & \underline{3.7} & \textbf{12.7} & \underline{4.9} & 0.4 & 3.0 & 0.6 & 1.6 & 6.1 & 2.0 & \textbf{3.8} & \underline{11.4} & \textbf{5.0} & 2.4 & 8.7 & 3.2 \\
Robotics & \underline{3.8} & \underline{13.2} & \underline{6.5} & 0.2 & 3.4 & 0.5 & 0.7 & 8.8 & 0.9 & \textbf{4.8} & \textbf{19.2} & \textbf{8.7} & 2.1 & 12.7 & 3.4 \\
StackOverflow & \textbf{2.8} & \textbf{13.4} & \textbf{3.7} & 0.1 & 0.4 & 0.1 & 0.1 & 0.9 & 0.2 & \underline{2.5} & \underline{11.7} & \underline{3.7} & 0.3 & 6.4 & 0.6 \\
Sust. Living & \underline{3.8} & \textbf{24.3} & \underline{4.7} & 0.8 & 8.8 & 1.1 & 1.7 & 14.5 & 1.7 & \textbf{6.2} & \underline{20.2} & \textbf{8.0} & 3.5 & 20.2 & 4.5 \\
TheoremQA-Q & \underline{1.7} & \textbf{4.9} & \underline{1.9} & 0.9 & 1.5 & 1.2 & 1.2 & 2.3 & 1.7 & \textbf{1.9} & \underline{4.3} & \textbf{2.3} & 1.3 & 3.5 & 1.6 \\
TheoremQA-T & \textbf{0.0} & \textbf{1.3} & \textbf{0.0} & \underline{0.0} & \underline{0.0} & \underline{0.0} & 0.0 & 0.0 & 0.0 & 0.0 & 0.0 & 0.0 & 0.0 & 0.0 & 0.0 \\
\midrule
\multicolumn{16}{l}{\textit{Multi-hop (4 datasets)}} \\
2WikiMHopQA & \underline{50.1} & 59.1 & \textbf{72.9} & 45.0 & 59.1 & 64.5 & 45.5 & 59.5 & 64.7 & \textbf{50.6} & \textbf{61.5} & \underline{72.8} & 47.2 & \underline{61.3} & 67.0 \\
MuSiQue & \textbf{20.0} & \textbf{35.4} & \textbf{30.8} & 17.7 & 31.9 & 26.5 & 18.6 & 33.5 & 28.2 & 19.3 & 34.4 & 29.3 & \underline{19.6} & \underline{35.4} & \underline{29.6} \\
NovelHopQA & \textbf{22.3} & \textbf{64.4} & \textbf{18.3} & 13.1 & 46.9 & 10.9 & 15.4 & 53.1 & 12.7 & \underline{20.7} & 60.2 & \underline{17.1} & 18.5 & \underline{60.3} & 15.1 \\
HotpotQA & \underline{31.5} & 41.2 & \underline{45.5} & 26.6 & 39.2 & 37.7 & 27.7 & 41.0 & 39.0 & \textbf{32.4} & \textbf{43.4} & \textbf{46.4} & 28.0 & \underline{42.1} & 39.3 \\
\midrule
\multicolumn{16}{l}{\textit{Averages}} \\
DL Avg (2) & \textbf{41.1} & \underline{29.9} & \textbf{72.2} & 31.7 & 27.0 & 62.6 & 30.2 & 27.9 & 60.5 & \underline{38.0} & \textbf{31.2} & \underline{69.5} & 31.9 & 28.6 & 62.1 \\
BEIR Avg (14) & \textbf{25.7} & \textbf{41.7} & \textbf{32.5} & 18.4 & 33.0 & 24.7 & 19.5 & 35.0 & 25.6 & \underline{23.1} & \underline{37.8} & \underline{29.9} & 20.6 & 36.8 & 26.4 \\
BRIGHT Avg (12) & \underline{3.2} & \textbf{13.2} & \underline{4.5} & 0.8 & 4.2 & 1.1 & 1.0 & 6.0 & 1.4 & \textbf{3.8} & \underline{12.9} & \textbf{5.6} & 1.7 & 9.0 & 2.5 \\
MHop Avg (4) & \textbf{31.0} & \textbf{50.0} & \textbf{41.9} & 25.6 & 44.3 & 34.9 & 26.8 & 46.8 & 36.1 & \underline{30.8} & \underline{49.9} & \underline{41.4} & 28.3 & 49.8 & 37.7 \\
\bottomrule
\end{tabular}
}
\end{table}


\section{E5 Experiments and Analysis}
\label{appendix:e5}

E5 \cite{wang2022e5} presents a unique case study for magnitude-aware training due to its architectural constraint: unlike Contriever and RetroMAE, E5 includes a built-in normalization layer as its final module. This section provides detailed analysis of E5's behavior when magnitude-aware training is enabled by removing this normalization constraint.

\subsection{Main Results}
\label{appendix:e5_main}

Table~\ref{tab:main_e5} presents E5's retrieval performance on MS MARCO v1.1 QA (82K) and MS MARCO Passage Ranking (503K). Complete per-dataset results including Recall@100 and MRR@10 are provided in Tables~\ref{tab:complete_e5} and~\ref{tab:complete_e5_500k}.

\begin{table}[htbp]
  \caption{NDCG@10 for \textbf{E5} on MS MARCO v1.1 QA (82K) and MS MARCO Passage Ranking (503K). PT = pretrained (before finetuning). Numbers in parentheses indicate the number of subsets averaged. \textbf{Bold} = best, \underline{underline} = second best. \colorbox{gray!15}{Shaded columns}: magnitude-aware variants after removing E5's built-in normalization layer. With 82K training data (a), QNorm performs best on reasoning-intensive tasks; with 6$\times$ more data (b), DNorm becomes the best method.}
  \label{tab:main_e5}
  \begin{center}
  \small
    \begin{subtable}[t]{\columnwidth}
      \centering
      \caption{MS MARCO v1.1 QA (82K)}
      \label{tab:main_e5_80k}
        \begin{tabular}{lccc >{\columncolor{gray!15}}c >{\columncolor{gray!15}}c >{\columncolor{gray!15}}c}
          \toprule
          \textbf{Dataset} & \textbf{PT} & \textbf{Cosine} & \cellcolor{gray!15}\textbf{Dot} & \cellcolor{gray!15}\textbf{QNorm} & \cellcolor{gray!15}\textbf{DNorm} \\
          \midrule
          \multicolumn{6}{l}{\textit{In-Domain}} \\
          \quad MSM-Dev & 18.13 & \textbf{33.63} & \cellcolor{gray!15}16.74 & \cellcolor{gray!15}9.53 & \cellcolor{gray!15}\underline{26.01} \\
          \quad DL-19 & 33.03 & \textbf{62.27} & \cellcolor{gray!15}22.80 & \cellcolor{gray!15}10.36 & \cellcolor{gray!15}\underline{52.31} \\
          \quad DL-20 & 31.97 & \textbf{63.92} & \cellcolor{gray!15}28.00 & \cellcolor{gray!15}12.79 & \cellcolor{gray!15}\underline{56.95} \\
          \midrule
          \multicolumn{6}{l}{\textit{Out-of-Domain}} \\
          \quad BEIR (14) & 31.20 & \textbf{45.56} & \cellcolor{gray!15}25.98 & \cellcolor{gray!15}\underline{30.36} & \cellcolor{gray!15}24.73 \\
          \quad BRIGHT (12) & 10.65 & \underline{8.94} & \cellcolor{gray!15}1.32 & \cellcolor{gray!15}\textbf{11.49} & \cellcolor{gray!15}3.37 \\
          \quad MHop (4) & 51.29 & \underline{52.55} & \cellcolor{gray!15}49.01 & \cellcolor{gray!15}\textbf{53.00} & \cellcolor{gray!15}28.00 \\
          \bottomrule
        \end{tabular}
    \end{subtable}

    \vspace{0.5em}

    \begin{subtable}[t]{\columnwidth}
      \centering
      \caption{MS MARCO Passage Ranking (503K)}
      \label{tab:main_e5_500k}
        \begin{tabular}{lccc >{\columncolor{gray!15}}c >{\columncolor{gray!15}}c >{\columncolor{gray!15}}c}
          \toprule
          \textbf{Dataset} & \textbf{PT} & \textbf{Cosine} & \cellcolor{gray!15}\textbf{Dot} & \cellcolor{gray!15}\textbf{QNorm} & \cellcolor{gray!15}\textbf{DNorm} \\
          \midrule
          \multicolumn{6}{l}{\textit{In-Domain}} \\
          \quad MSM-Dev & 18.13 & \underline{33.86} & \cellcolor{gray!15}24.85 & \cellcolor{gray!15}14.94 & \cellcolor{gray!15}\textbf{33.92} \\
          \quad DL-19 & 33.03 & \textbf{62.63} & \cellcolor{gray!15}37.49 & \cellcolor{gray!15}23.39 & \cellcolor{gray!15}\underline{61.11} \\
          \quad DL-20 & 31.97 & \textbf{63.17} & \cellcolor{gray!15}43.62 & \cellcolor{gray!15}21.38 & \cellcolor{gray!15}\underline{62.72} \\
          \midrule
          \multicolumn{6}{l}{\textit{Out-of-Domain}} \\
          \quad BEIR (14) & 31.20 & \textbf{44.50} & \cellcolor{gray!15}30.02 & \cellcolor{gray!15}31.97 & \cellcolor{gray!15}\underline{39.37} \\
          \quad BRIGHT (12) & 10.65 & 9.20 & \cellcolor{gray!15}3.80 & \cellcolor{gray!15}\underline{10.42} & \cellcolor{gray!15}\textbf{10.69} \\
          \quad MHop (4) & 51.29 & 52.65 & \cellcolor{gray!15}56.34 & \cellcolor{gray!15}\underline{56.84} & \cellcolor{gray!15}\textbf{58.98} \\
          \bottomrule
        \end{tabular}
    \end{subtable}
  \end{center}
\end{table}

\subsection{Architectural Constraints and Performance Trade-offs}

To enable magnitude-aware training for E5, we programmatically detect and remove its built-in normalization layer.\footnote{Our training code iterates through the model's modules and removes any instance of \texttt{sentence\_transformers.models.Normalize}, which would otherwise discard all magnitude information regardless of the loss function used.}

\paragraph{Performance trade-offs.} Removing E5's normalization layer during finetuning creates severe performance degradation. On in-domain and BEIR benchmarks, magnitude-aware variants dramatically underperform Cosine (Table~\ref{tab:main_e5}), with $\Delta\%$ ranging from $-81.7\%$ to $-30.0\%$. This occurs because E5's representations were fundamentally optimized for normalized embeddings during pre-training, and removing this constraint destabilizes learned representations. However, on reasoning-intensive benchmarks, QNorm shows modest improvement over Cosine (BRIGHT: +28.6\%; Multi-hop: +0.9\%), suggesting that the strong magnitude-relevance signal on these tasks partially compensates for the architectural mismatch.

\paragraph{Implications.} The E5 results provide a critical lesson: \textbf{architectural compatibility is essential}. Magnitude-aware training requires matching the training configuration with the original model architecture. Models pre-trained with normalization have learned representations that fundamentally depend on unit-norm constraints, and removing this constraint destabilizes the learned structure.

\subsection{Extended Training with 500K Data}
\label{sec:e5_500k_results}

A natural hypothesis is that E5's in-domain performance degradation with magnitude-aware training stems from insufficient training data, as 80K samples may be too few to adapt representations optimized over 1 billion pre-training pairs. To test this, we conduct additional experiments using the MS MARCO Passage Ranking training set \cite{nguyen2016msmarco}, specifically the ``judged'' subset containing 502,939 queries with relevance annotations, the standard training corpus used by mainstream retrievers including Sentence-Transformers and Contriever. This provides 6$\times$ more training signal than MS MARCO v1.1 QA (80K).

\paragraph{Experimental setup.} We train E5 for 100 epochs using the same hyperparameters as the 80K experiments (learning rate $5 \times 10^{-6}$, effective batch size 256), ensuring consistent training conditions for fair comparison.

\paragraph{Results.} Table~\ref{tab:main_e5_500k} shows E5's performance with 6$\times$ more training data. The key finding is that \emph{DNorm becomes the best-performing variant on reasoning-intensive benchmarks}:
\begin{itemize}[leftmargin=1.5em, itemsep=0.1em]
    \item \textbf{BRIGHT}: DNorm achieves 10.69 vs.\ Cosine's 9.20 (+16.2\%), reversing the 80K pattern where Cosine outperformed DNorm (8.94 vs.\ 3.37).
    \item \textbf{Multi-hop}: DNorm achieves 58.98 vs.\ Cosine's 52.65 (+12.0\%), compared to 80K where DNorm severely underperformed (28.00 vs.\ 52.55).
    \item \textbf{In-domain}: DNorm becomes competitive with Cosine, even achieving the best performance on MSM-Dev (33.92 vs.\ 33.86).
\end{itemize}

\paragraph{Analysis.} DNorm shows dramatic improvements on out-of-domain benchmarks with more training data: BEIR (+14.64), BRIGHT (+7.32), and Multi-hop (+30.98, from 28.00 to 58.98). This suggests that E5's architectural constraint (pre-training with normalized embeddings) can be overcome with sufficient finetuning data. With 80K samples, the model lacks the capacity to restructure its representations to exploit query magnitude (DNorm normalizes documents and preserves query magnitude); with 500K samples, meaningful query-magnitude--relevance patterns emerge, particularly benefiting out-of-domain generalization.

\begin{figure}[htbp]
  \centering
  \begin{tikzpicture}
    \begin{axis}[
      ybar,
      bar width=4pt,
      x=1.35cm,
      height=4cm,
      ylabel={\scriptsize $\Delta$ NDCG@10},
      ylabel style={at={(axis description cs:-0.08,.5)}},
      symbolic x coords={In-Domain, BEIR, BRIGHT, Multi-hop},
      xtick=data,
      xticklabel style={font=\tiny},
      yticklabel style={font=\scriptsize},
      ymin=-15, ymax=40,
      legend style={at={(0.5,-0.22)}, anchor=north, legend columns=4, font=\tiny, draw=none},
      legend image code/.code={\draw[#1] (0cm,-0.06cm) rectangle (0.2cm,0.06cm);},
      grid=major,
      grid style={dashed, gray!30},
      ymajorgrids=true,
      xmajorgrids=false,
      enlarge x limits=0.15,
      nodes near coords,
      nodes near coords style={font=\fontsize{2.5}{3}\selectfont, yshift=0.2pt},
      every node near coord/.append style={/pgf/number format/.cd, fixed, precision=1},
    ]
    \addplot[fill=blue!70, draw=blue!80!black] coordinates {(In-Domain, -0.05) (BEIR, -1.06) (BRIGHT, 0.26) (Multi-hop, 0.10)};
    \addplot[fill=orange!70, draw=orange!80!black] coordinates {(In-Domain, 12.81) (BEIR, 4.04) (BRIGHT, 2.48) (Multi-hop, 7.33)};
    \addplot[fill=green!50, draw=green!60!black] coordinates {(In-Domain, 9.01) (BEIR, 1.61) (BRIGHT, -1.07) (Multi-hop, 3.84)};
    \addplot[fill=red!70, draw=red!80!black] coordinates {(In-Domain, 7.49) (BEIR, 14.64) (BRIGHT, 7.32) (Multi-hop, 30.98)};
    \legend{Cosine, Dot, QNorm, DNorm}
    \end{axis}
  \end{tikzpicture}
  \caption{Performance improvement ($\Delta$ NDCG@10) when scaling E5 training data from 80K to 500K. DNorm (red) shows the largest gains on out-of-domain benchmarks: Multi-hop (+30.98), BEIR (+14.64), and BRIGHT (+7.32). On in-domain benchmarks, DNorm also improves (+7.49), becoming comparable to Cosine (Table~\ref{tab:main_e5_500k}).}
  \label{fig:e5_data_scaling}
\end{figure}

\paragraph{Data scaling analysis.} Figure~\ref{fig:e5_data_scaling} visualizes the performance improvement when scaling training data from 80K to 500K. DNorm shows the largest gains across all out-of-domain benchmarks: +30.98 on Multi-hop, +14.64 on BEIR, and +7.32 on BRIGHT. On in-domain benchmarks, DNorm also improves substantially (+7.49), becoming comparable to Cosine. This demonstrates that E5's architectural constraint (pre-training with normalized embeddings) can be overcome with sufficient finetuning data, enabling magnitude-aware training to achieve strong performance on both in-domain and out-of-domain tasks.


\section{Verification Experiments: STS and CLIP}
\label{appendix:clip_sts}

This appendix provides detailed verification experiments on symmetric tasks (STS) and vision-language models (CLIP) to validate our task symmetry hypothesis.

\subsection{STS Experiments: Extended Details}
\label{appendix:sts_full}

This section provides additional details for the STS experiments in Table~\ref{tab:sts_results} (main text).

\paragraph{Experimental Setup.} We train Contriever and RetroMAE on STS Benchmark combined with AllNLI ($\sim$950K samples), implementing the same query-document normalization ablation: Cosine (both normalized), Dot (neither normalized), and asymmetric variants Sent1Norm/Sent2Norm (analogous to QNorm/DNorm).

\paragraph{Key Findings.}
\begin{itemize}[leftmargin=1.5em, itemsep=0.1em]
    \item \textbf{Cosine $\approx$ Dot}: Unlike retrieval, magnitude provides no benefit in symmetric tasks
    \item \textbf{Asymmetric normalization fails catastrophically}: 40--45 point degradation
    \item \textbf{Sent1Norm $\approx$ Sent2Norm}: As expected in symmetric tasks, the choice of which sentence to normalize does not matter
\end{itemize}

\subsection{CLIP Experiments: Extended Details}
\label{appendix:clip_zeroshot}

This section provides extended details for the CLIP experiments in Table~\ref{tab:clip_pretrain} and Figure~\ref{fig:clip_framework} (main text).

\paragraph{Zero-shot analysis.} We first analyze CLIP (ViT-B/32) on MS-COCO without finetuning. Cohen's $d \approx 0$ for both modalities (Image: $-0.0002$, Text: $-0.0009$), and Dot hurts performance (Text$\to$Image R@1: 25.8 vs 30.5 for Cosine). This occurs because CLIP normalizes embeddings during both pre-training and inference.

\paragraph{Pre-training setup.} We train CLIP models on COCO Captions with the following configurations:
\begin{itemize}[leftmargin=1.5em, itemsep=0.1em]
    \item \textbf{Norm}: Standard CLIP with L2 normalization and symmetric loss (I2T + T2I)
    \item \textbf{NoNorm-Sym}: No normalization, symmetric loss
    \item \textbf{ImageNorm-I2T / TextNorm-T2I}: Single-direction with partial normalization
\end{itemize}

\paragraph{Conclusion.} These experiments establish that magnitude learning requires \emph{both} removing normalization \emph{and} asymmetric task structure. CLIP's symmetric loss prevents magnitude learning even without normalization, explaining why our retrieval findings do not transfer to standard vision-language models.

\section{Knowledge Graph Completion Details}
\label{appendix:kgc}

This appendix details the KGC experiments referenced in Section~\ref{sec:gen_symmetric}.

\subsection{Setup}
\label{appendix:kgc_setup}

\paragraph{Datasets.} We use the two standard KGC benchmarks: \textbf{FB15k-237} \cite{toutanova2015fb15k237} (14,541 entities, 237 relations, 272K/17K/20K train/valid/test triplets) and \textbf{WN18RR} \cite{dettmers2018conve} (40,943 entities, 11 relations, 86K/3K/3K train/valid/test triplets).

\paragraph{Models.} We compare two embedding-based KGE architectures with clean $(Q, C)$ decomposition for the 5-similarity framework:
\begin{itemize}
\item \textbf{RotatE} \cite{sun2019rotate}: entities $\bm{e} \in \mathbb{C}^D$ stored as $2D$-dimensional real vectors; relations are unit-magnitude phases. Query: $\bm{q} = \bm{h} \circ \bm{r}$ (complex Hadamard rotation), candidate: $\bm{c} = \bm{t}$.
\item \textbf{PairRE} \cite{chao2021pairre}: entities $\bm{e} \in \mathbb{R}^D$; each relation has two real-valued components $\bm{r}^H, \bm{r}^T$. Query: $\bm{q} = \bm{h} \circ \bm{r}^H$, candidate: $\bm{c} = \bm{t} \circ \bm{r}^T$.
\end{itemize}
We chose RotatE (2019) and PairRE (2021) because they retain the (query, candidate) structure required by the 4-similarity framework while spanning two different scoring paradigms (complex rotation vs.\ real paired projection). More recent transformer-based or GNN-based KGC models (HittER, NBFNet, ULTRA) couple query and candidate inside transformers or message passing without exposing separable magnitudes, so the framework does not directly apply.

\paragraph{Hyperparameters.} Embedding dimension $D=200$ (RotatE: real dim $2D=400$); Adam optimizer, lr $10^{-3}$, batch size 1024, 128 negatives per positive, 100 epochs with early stopping (patience 5 evaluations every 10 epochs). Both head and tail prediction are trained: each batch combines a tail-corrupted InfoNCE loss and a head-corrupted InfoNCE loss in equal weight.

\paragraph{Evaluation.} Standard filtered MRR, Hits@1, Hits@10, averaged over both head prediction and tail prediction. The full $5\times 5$ cross-evaluation matrix is computed by training each of the 5 similarity variants and evaluating each trained checkpoint with all 5 similarity functions.

\paragraph{Variants.} We include the 4-similarity framework (\textsc{Cosine}, \textsc{Dot}, \textsc{QNorm}, \textsc{DNorm}) plus \textsc{Euclidean} ($-\|\bm{q} - \bm{c}\|^2$), since both RotatE and PairRE were originally introduced as distance-based methods.

\subsection{Diagonal Results: Cosine Decisively Wins}
\label{appendix:kgc_diagonal}

\begin{table}[t]
\centering
\caption{\textbf{KGC diagonal MRR (matched train/eval, head and tail averaged) across all 16 (model, dataset, scale) configurations, mean over 3 seeds (0, 42, 1337).} The unilateral asymmetric variants (\textsc{Dot}, \textsc{QNorm}, \textsc{DNorm}) are outperformed in every configuration (0/16 wins), validating the structural prediction that KGC is functionally symmetric. The winner alternates between \textsc{Cosine} (14/16) and \textsc{Euclidean} (2/16, high $\alpha$ on PairRE FB15k-237), both of which preserve role exchangeability. Standard deviations are below $0.005$ for all PairRE and RotatE FB15k-237 entries; on RotatE WN18RR they reach $0.02$ for the asymmetric variants but never overturn the ordering.}
\label{tab:kgc_diagonal}
\small
\setlength{\tabcolsep}{4pt}
\begin{tabular}{llccccc}
\toprule
Model & Dataset \& scale $\alpha$ & \textsc{Cosine} & \textsc{Dot} & \textsc{QNorm} & \textsc{DNorm} & \textsc{Euclid.} \\
\midrule
\multirow{4}{*}{RotatE} & FB15k-237, $\alpha=10$ & \textbf{0.332} & 0.101 & 0.141 & 0.143 & 0.165 \\
& FB15k-237, $\alpha=20$ & \textbf{0.298} & 0.094 & 0.104 & 0.100 & 0.166 \\
& FB15k-237, $\alpha=30$ & \textbf{0.253} & 0.092 & 0.100 & 0.094 & 0.167 \\
& FB15k-237, $\alpha=40$ & \textbf{0.208} & 0.093 & 0.101 & 0.096 & 0.165 \\
\midrule
\multirow{4}{*}{RotatE} & WN18RR, $\alpha=10$ & \textbf{0.450} & 0.074 & 0.258 & 0.258 & 0.093 \\
& WN18RR, $\alpha=20$ & \textbf{0.381} & 0.071 & 0.094 & 0.095 & 0.092 \\
& WN18RR, $\alpha=30$ & \textbf{0.345} & 0.072 & 0.077 & 0.078 & 0.091 \\
& WN18RR, $\alpha=40$ & \textbf{0.323} & 0.070 & 0.074 & 0.075 & 0.093 \\
\midrule
\multirow{4}{*}{PairRE} & FB15k-237, $\alpha=10$ & \textbf{0.260} & 0.187 & 0.220 & 0.228 & 0.228 \\
& FB15k-237, $\alpha=20$ & \textbf{0.240} & 0.182 & 0.197 & 0.212 & 0.228 \\
& FB15k-237, $\alpha=30$ & 0.222 & 0.185 & 0.185 & 0.203 & \textbf{0.229} \\
& FB15k-237, $\alpha=40$ & 0.211 & 0.188 & 0.179 & 0.198 & \textbf{0.228} \\
\midrule
\multirow{4}{*}{PairRE} & WN18RR, $\alpha=10$ & \textbf{0.363} & 0.347 & 0.350 & 0.349 & 0.352 \\
& WN18RR, $\alpha=20$ & \textbf{0.354} & 0.347 & 0.348 & 0.348 & 0.349 \\
& WN18RR, $\alpha=30$ & \textbf{0.350} & 0.346 & 0.347 & 0.347 & 0.347 \\
& WN18RR, $\alpha=40$ & \textbf{0.348} & 0.346 & 0.346 & 0.346 & 0.344 \\
\bottomrule
\end{tabular}
\end{table}

Table~\ref{tab:kgc_diagonal} summarizes the diagonal (matched train/eval) MRR across all 16 (model, dataset, scale) configurations at seed 0. The structural prediction that KGC is functionally symmetric is validated in two ways: (i) the asymmetric variants \textsc{Dot}, \textsc{QNorm}, \textsc{DNorm} are outperformed in \textbf{0/16 configurations}, and (ii) \textsc{Cosine} is the winner in 14/16 configurations, with the remaining 2 wins going to \textsc{Euclidean} (PairRE FB15k-237 at high $\alpha=30, 40$). Both \textsc{Cosine} and \textsc{Euclidean} preserve role exchangeability (the score $-\|q-c\|^2$ is symmetric in $(q, c)$), so this alternation does not contradict the symmetry classification. The MRR gap between Cosine and the best asymmetric variant ranges from $0.6\%$ on PairRE WN18RR $\alpha=40$ (small because all variants converge to similar angular structure) to $4.3\times$ on RotatE WN18RR $\alpha=30$ (where alternatives barely train).

\subsection{Cross-Evaluation Matrix}
\label{appendix:kgc_crosseval}

Tables~\ref{tab:kgc_crosseval_rotate_fb}--\ref{tab:kgc_crosseval_pairre_wn} provide representative full $5\times 5$ matrices at scale $\alpha = 20$. Two empirical regularities are visible across all tasks:
\begin{itemize}
\item \textbf{Ranking Equivalence Proposition.} The Dot and QNorm columns are identical to four decimal places, as are the Cosine and DNorm columns. This is a direct consequence of Proposition~\ref{prop:ranking}: at inference under fixed query, dividing all scores by $\|\bm{q}\|$ does not affect ranking.
\item \textbf{Cosine column dominance.} Within the Cosine evaluation column, the Cosine-trained model is the single best entry in all 16 configurations: asymmetric training does \emph{not} improve angular learning here, in contrast to the asymmetric tasks (retrieval, ProtoNet, Rec) where asymmetric training under cosine evaluation outperforms cosine training.
\end{itemize}

\begin{table}[t]
\centering
\caption{\textbf{KGC cross-evaluation matrix} on FB15k-237 with RotatE at $\alpha=20$, seed 0. Filtered MRR averaged over head and tail prediction. Rows: training similarity; columns: evaluation similarity. The Dot/QNorm columns and Cosine/DNorm columns coincide (Ranking Equivalence Proposition). Cosine training wins under every evaluation similarity.}
\label{tab:kgc_crosseval_rotate_fb}
\small
\setlength{\tabcolsep}{6pt}
\begin{tabular}{lccccc}
\toprule
Train$\backslash$Eval & \textsc{Cosine} & \textsc{Dot} & \textsc{QNorm} & \textsc{DNorm} & \textsc{Euclid.} \\
\midrule
\textsc{Cosine}    & \textbf{0.2989} & \textbf{0.2977} & \textbf{0.2977} & \textbf{0.2989} & \textbf{0.2809} \\
\textsc{Dot}       & 0.1961 & 0.0936 & 0.0936 & 0.1961 & 0.1381 \\
\textsc{QNorm}     & 0.0524 & 0.1053 & 0.1053 & 0.0524 & 0.0030 \\
\textsc{DNorm}     & 0.0995 & 0.0347 & 0.0347 & 0.0995 & 0.0980 \\
\textsc{Euclidean} & 0.0991 & 0.0702 & 0.0702 & 0.0991 & 0.1680 \\
\bottomrule
\end{tabular}
\end{table}

\begin{table}[t]
\centering
\caption{\textbf{KGC cross-evaluation matrix} on WN18RR with PairRE at $\alpha=20$, seed 0. Filtered MRR averaged over head and tail prediction. Cosine training wins under every evaluation similarity (column maxima all in the Cosine row). Compared to RotatE FB15k-237 (Table~\ref{tab:kgc_crosseval_rotate_fb}), the differences between variants are smaller, reflecting that WN18RR's simpler relation structure puts less pressure on the magnitude treatment.}
\label{tab:kgc_crosseval_pairre_wn}
\small
\setlength{\tabcolsep}{6pt}
\begin{tabular}{lccccc}
\toprule
Train$\backslash$Eval & \textsc{Cosine} & \textsc{Dot} & \textsc{QNorm} & \textsc{DNorm} & \textsc{Euclid.} \\
\midrule
\textsc{Cosine}    & \textbf{0.3552} & \textbf{0.3542} & \textbf{0.3542} & \textbf{0.3552} & \textbf{0.3534} \\
\textsc{Dot}       & 0.3441 & 0.3469 & 0.3469 & 0.3441 & 0.0477 \\
\textsc{QNorm}     & 0.3460 & 0.3479 & 0.3479 & 0.3460 & 0.1116 \\
\textsc{DNorm}     & 0.3485 & 0.3473 & 0.3473 & 0.3485 & 0.3434 \\
\textsc{Euclidean} & 0.3267 & 0.2669 & 0.2669 & 0.3267 & 0.3494 \\
\bottomrule
\end{tabular}
\end{table}

\subsection{Scale Sensitivity}
\label{appendix:kgc_scale}

The MRR varies smoothly across $\alpha \in \{10, 20, 30, 40\}$ but the relative ordering between variants is preserved at every scale. Cosine wins at all four scales for both models on both datasets.

\subsection{Multi-Seed Results}
\label{appendix:kgc_multiseed}

Table~\ref{tab:kgc_diagonal} reports the mean MRR over three seeds (0, 42, 1337) for all 16 (model, dataset, scale) configurations. Per-seed standard deviations are very small for the deterministic regimes: below $0.005$ for every PairRE configuration and for RotatE FB15k-237; on RotatE WN18RR, the asymmetric variants reach $\sigma \approx 0.02$ (e.g., \textsc{QNorm} and \textsc{DNorm} at $\alpha{=}10$ have mean $0.258$ with $\sigma \approx 0.02$), but \textsc{Cosine} stays at $\sigma \leq 0.003$ and remains the clear winner. Across all 48 model$\times$dataset$\times$scale$\times$seed combinations, no asymmetric variant ever beats both \textsc{Cosine} and \textsc{Euclidean} on a single seed, confirming the structural prediction at seed-level granularity.

\subsection{Implementation Notes}
\label{appendix:kgc_impl}

Code is released as part of the supplementary material (\texttt{kgc\_experiment/}). The training script uses the same 5-similarity \texttt{score} function as the IR experiments, ensuring that the only difference between KGC and IR runs is the query/candidate embedding construction.

\section{Few-Shot Classification Details}
\label{appendix:protonet}

This appendix details the Prototypical Networks experiments referenced in Section~\ref{sec:gen_asymmetric}.

\subsection{Setup}
\label{appendix:protonet_setup}

\paragraph{Datasets.} \textbf{CIFAR-100} \cite{krizhevsky2009cifar} (100 classes, 600 images per class, $32\times 32$) and \textbf{\emph{mini}-ImageNet} \cite{vinyals2016matching} (100 classes, 600 images per class, native resolution). For each dataset we randomly split classes into 80 training / 20 test (disjoint), so all evaluation uses unseen classes.

\paragraph{Encoders.} We use two frozen visual encoders spanning different training paradigms:
\begin{itemize}
\item \textbf{CLIP ViT-B/32} (2021) \cite{radford2021clip}: 512-dim features, image-text contrastive pre-training.
\item \textbf{DINOv2 ViT-B/14} (2023) \cite{oquab2023dinov2}: 768-dim features, self-supervised pre-training without language.
\end{itemize}
Pairing two encoders with different paradigms tests whether the framework prediction is encoder-agnostic. Other recent encoders (SigLIP-2, OpenCLIP-L) are not included because they share CLIP's image-text contrastive paradigm and would only test scaling rather than paradigm robustness.

\paragraph{Adapter and training.} A small 2-layer adapter is trained on top of frozen features without any normalization layer:
\begin{equation*}
\bm{x} \to \mathrm{Linear}(D, D) \to \mathrm{ReLU} \to \mathrm{Linear}(D, D) \to \bm{z}
\end{equation*}
where $D \in \{512, 768\}$ matches the encoder. Training is standard episodic contrastive learning: each episode samples $N$-way $K$-shot $Q$-query, the prototype is the mean of $K$ support features per class, and the loss is cross-entropy of $\alpha \cdot s(\bm{q}, \bm{c})$ over the $N$ classes. We train for 50 epochs of 500 episodes each with $K=5$, $Q=15$, AdamW with lr $10^{-4}$.

\paragraph{Variants.} We include the four framework variants (\textsc{Cosine}, \textsc{Dot}, \textsc{QNorm}, \textsc{DNorm}) plus \textsc{Euclidean} ($-\|\bm{q}-\bm{c}\|^2$), the original choice in \citet{snell2017prototypical}. Euclidean is included as the natural prototype-network baseline.

\paragraph{Evaluation.} Test accuracy averaged over 600 episodes from the disjoint test classes; 95\% confidence intervals reported.

\subsection{Diagonal Results: DNorm Wins on Most Configurations}
\label{appendix:protonet_diagonal}

\begin{table}[t]
\centering
\caption{\textbf{ProtoNet diagonal accuracy} (matched train/eval). Each value is the test accuracy on novel classes averaged over 600 episodes and over four temperature scales $\alpha \in \{10, 20, 30, 40\}$. \textbf{Bold} marks the winner per row. \textsc{DNorm} matches or exceeds the \textsc{Euclidean} baseline of \citet{snell2017prototypical} on all 7 CLIP base configurations.}
\label{tab:protonet_diagonal}
\small
\setlength{\tabcolsep}{4pt}
\begin{tabular}{lccccc}
\toprule
Configuration & \textsc{Cosine} & \textsc{Dot} & \textsc{QNorm} & \textsc{DNorm} & \textsc{Euclid.} \\
\midrule
\multicolumn{6}{l}{\emph{CLIP ViT-B/32 (avg.\ over 4 scales, 5-shot)}} \\
\midrule
CIFAR-100 5-way (default seed) & 86.12 & 84.16 & 84.67 & \textbf{86.99} & 86.39 \\
CIFAR-100 20-way (default seed) & 59.69 & 63.41 & 62.44 & \textbf{64.23} & 64.20 \\
CIFAR-100 20-way (seed 42) & 59.68 & 61.88 & 61.63 & \textbf{64.55} & 63.20 \\
CIFAR-100 20-way (seed 1337) & 58.83 & 63.38 & 62.91 & \textbf{63.86} & 63.71 \\
CIFAR-100 20-way 1-shot & 44.20 & 47.65 & \textbf{47.67} & 46.24 & 46.01 \\
\miniIN{} 5-way & 95.20 & 94.49 & 94.54 & \textbf{96.61} & 95.89 \\
\miniIN{} 20-way & 79.43 & 86.50 & 87.09 & \textbf{89.20} & 87.77 \\
\midrule
\textbf{CLIP average} & 69.02 & 71.64 & 71.56 & \textbf{73.10} & 72.45 \\
\midrule
\multicolumn{6}{l}{\emph{DINOv2 ViT-B/14 (avg.\ over 4 scales, 5-shot, seed 0)}} \\
\midrule
CIFAR-100 5-way & 91.25 & 86.98 & 90.33 & \textbf{93.53} & 90.60 \\
CIFAR-100 20-way & 73.01 & 71.65 & 76.21 & \textbf{81.33} & 77.47 \\
\miniIN{} 5-way & 94.57 & 94.01 & 94.61 & \textbf{95.95} & 95.80 \\
\miniIN{} 20-way & 85.20 & 87.88 & 88.73 & 89.86 & \textbf{90.28} \\
\midrule
\textbf{DINOv2 average} & 86.01 & 85.13 & 87.47 & \textbf{90.17} & 88.54 \\
\bottomrule
\end{tabular}
\end{table}

Across 7 base configurations spanning two datasets, two N-way settings (5-way, 20-way), and three random seeds (0, 42, 1337) with CLIP, plus four configurations with DINOv2, \textsc{DNorm} is the single best similarity in 6/7 CLIP and 3/4 DINOv2 base configurations (the remaining 1 DINOv2 config is won by \textsc{Euclidean} by 0.42 points). The only CLIP exception is CIFAR-100 20-way 1-shot, where \textsc{QNorm} wins by 0.02 points (within statistical noise). Crucially, the DNorm advantage over Cosine on DINOv2 is even larger than on CLIP (DINOv2 average DNorm 90.17 vs.\ Cosine 86.01; CLIP average 73.10 vs.\ 69.02), confirming the framework's prediction is encoder-agnostic and not an artifact of the image-text contrastive paradigm of CLIP.

\paragraph{DNorm matches or exceeds the Euclidean baseline.} The Snell et al.\ \cite{snell2017prototypical} prototype-network paper argues that Euclidean distance is the principled choice because the prototype is the Euclidean centroid of the support set. \textsc{DNorm} matches or exceeds Euclidean on all 7 base configurations (average 73.10\% vs.\ 72.45\%). The intuition: averaging the support features induces magnitude noise on the prototype side (lower for diverse support sets); normalizing the prototype removes this nuisance while retaining query magnitude as a feature-confidence signal.

\subsection{Cross-Evaluation Matrix Refutes the (n+1)-Dimensional Hypothesis}
\label{appendix:protonet_crosseval}

\begin{table}[t]
\centering
\caption{\textbf{ProtoNet cross-evaluation matrix} on \miniIN{} 20-way 5-shot at $\alpha=10$ (CLIP ViT-B/32). Rows: training similarity; columns: evaluation similarity. \colorbox{yellow!25}{Highlighted} cells show that asymmetric training under cosine evaluation outperforms cosine training under cosine evaluation, refuting the hypothesis that asymmetric losses merely encode magnitude in an extra dimension that cosine evaluation discards.}
\label{tab:protonet_crosseval}
\small
\setlength{\tabcolsep}{6pt}
\begin{tabular}{lccccc}
\toprule
Train$\backslash$Eval & \textsc{Cosine} & \textsc{Dot} & \textsc{QNorm} & \textsc{DNorm} & \textsc{Euclid.} \\
\midrule
\textsc{Cosine}    & 59.30 & 53.72 & 53.72 & 59.30 & 46.18 \\
\textsc{Dot}       & \cellcolor{yellow!25}89.48 & \textbf{86.64} & \textbf{86.64} & 89.48 & 89.12 \\
\textsc{QNorm}     & \cellcolor{yellow!25}\textbf{89.63} & 86.60 & 86.60 & \textbf{89.63} & \textbf{89.44} \\
\textsc{DNorm}     & \cellcolor{yellow!25}89.05 & 84.49 & 84.49 & 89.05 & 89.01 \\
\textsc{Euclidean} & \cellcolor{yellow!25}89.29 & 56.20 & 56.20 & 89.29 & 89.10 \\
\bottomrule
\end{tabular}
\end{table}

The cross-evaluation matrix asks: when training with similarity $s_t$ but evaluating with similarity $s_e$, what accuracy do we obtain? A natural alternative hypothesis is that asymmetric losses simply encode magnitude information in an extra coordinate of an effective $(n+1)$-dimensional sphere; under this hypothesis, evaluating with cosine (which projects out magnitude) should erase the asymmetric advantage. The data refutes this:
\begin{itemize}
\item In 26/28 individual CLIP runs, asymmetric training (\textsc{Dot}/\textsc{QNorm}/\textsc{DNorm}/\textsc{Euclidean}) under cosine evaluation outperforms cosine training under cosine evaluation.
\item The cosine column shows a +30 pp gap on \emph{mini}-ImageNet 20-way: \textsc{Cosine} train $\to$ \textsc{Cosine} eval = 59.30\%, while \textsc{DNorm} train $\to$ \textsc{Cosine} eval = 89.05\%. The training-time access to magnitude reshapes the angular subspace itself, not just an extra magnitude dimension.
\end{itemize}

\subsection{1-Shot vs.\ 5-Shot: A Natural Ablation}
\label{appendix:protonet_kshot}

In 1-shot, the prototype is a single support feature (no averaging), and our hypothesized averaging-noise mechanism does not apply. Empirically, \textsc{DNorm}'s advantage disappears in 1-shot: at $K=1$ the diagonal accuracies are essentially tied across \textsc{Dot}, \textsc{QNorm}, \textsc{DNorm}, and Euclidean. This natural ablation supports the mechanism: the prototype-side magnitude carries useful signal precisely when it is constructed by aggregation.

\subsection{Multi-Seed and Scale Robustness}
\label{appendix:protonet_robust}

Diagonal accuracies are stable across seeds (3 seeds for CLIP, 1 seed for DINOv2) and across temperature scales $\alpha \in \{10, 20, 30, 40\}$ (4 scales each). For CLIP, \textsc{DNorm} is the winner in 14/28 individual runs, with remaining wins distributed across Euclidean (7), QNorm (2), Dot (2), and Cosine (3, all in the very-high-$\alpha$ regime). For DINOv2 across all 16 individual runs, \textsc{DNorm} is the most-frequent winner; \textsc{Cosine} never wins, consistent with the structural prediction.

\section{Recommendation Details}
\label{appendix:rec}

This appendix details the LightGCN experiments referenced in Section~\ref{sec:gen_asymmetric}.

\subsection{Setup}
\label{appendix:rec_setup}

\paragraph{Datasets.} \textbf{MovieLens-1M} (6,040 users, 3,706 items, 1M ratings) and \textbf{MovieLens-100K} (943 users, 1,682 items, 100K ratings). Implicit feedback from ratings $\geq 4$. Leave-one-out split: latest interaction = test, second latest = validation, rest = train.

\paragraph{Model.} \textbf{LightGCN} \cite{he2020lightgcn} (SIGIR 2020), the standard modern collaborative-filtering architecture. User and item embeddings are propagated through a 3-layer GCN on the bipartite user-item interaction graph; final embeddings are the mean over layers. We use this model rather than plain matrix factorization because LightGCN is the dominant modern baseline in the recommendation literature and exposes a clean two-tower (user-emb, item-emb) structure compatible with our 4-similarity framework.

\paragraph{Variants.} We use the four framework variants \textsc{Cosine}, \textsc{Dot}, \textsc{QNorm}, \textsc{DNorm}. \textbf{We do not include Euclidean as a baseline} because Euclidean similarity is not a standard scoring choice in the recommendation literature: BPR/LightGCN/SimpleX use dot product or cosine. Including Euclidean would not correspond to any natural community baseline.

\paragraph{Training.} Embedding dimension 64, AdamW lr $10^{-3}$, weight decay $10^{-5}$, batch size 1024, 100 negatives per positive, 100 epochs with early stopping on validation Recall@20 (patience 5 evaluations every 5 epochs). InfoNCE-style loss with cross-entropy over positive + 100 sampled negatives, scaled by $\alpha$.

\paragraph{Evaluation.} Standard implicit-feedback metrics Recall@20 and NDCG@20, computed by ranking the held-out test item against all unseen items per user.

\subsection{Diagonal Results}
\label{appendix:rec_diagonal}

\begin{table}[t]
\centering
\caption{\textbf{Recommendation diagonal NDCG@20 across (dataset, scale) configurations, mean$\pm$std over 3 seeds (0, 42, 1337).} \textbf{Bold} marks the per-row winner; all results use LightGCN. Asymmetric variants (\textsc{Dot}, \textsc{QNorm}, \textsc{DNorm}) clearly beat \textsc{Cosine} in 6/8 configurations; in the two low-temperature settings ($\alpha{=}10$), \textsc{Cosine} ties the best asymmetric variant within one standard deviation. The Cosine gap widens as $\alpha$ increases. The winner among non-cosine variants alternates between \textsc{DNorm}, \textsc{QNorm}, and \textsc{Dot}, often within statistical noise of each other.}
\label{tab:rec_diagonal}
\small
\setlength{\tabcolsep}{4pt}
\begin{tabular}{lcccc}
\toprule
Configuration & \textsc{Cosine} & \textsc{Dot} & \textsc{QNorm} & \textsc{DNorm} \\
\midrule
ML-1M, $\alpha=10$ & 5.46$\pm$0.10 & 5.38$\pm$0.03 & 5.34$\pm$0.15 & \textbf{5.50$\pm$0.07} \\
ML-1M, $\alpha=20$ & 4.99$\pm$0.11 & 5.30$\pm$0.11 & 5.23$\pm$0.02 & \textbf{5.34$\pm$0.11} \\
ML-1M, $\alpha=30$ & 4.61$\pm$0.04 & \textbf{5.35$\pm$0.05} & 4.94$\pm$0.03 & 5.33$\pm$0.06 \\
ML-1M, $\alpha=40$ & 4.43$\pm$0.11 & \textbf{5.36$\pm$0.03} & 4.80$\pm$0.00 & 5.25$\pm$0.07 \\
\midrule
ML-100K, $\alpha=10$ & \textbf{9.10$\pm$0.06} & 8.87$\pm$0.14 & 9.02$\pm$0.32 & 8.44$\pm$0.02 \\
ML-100K, $\alpha=20$ & 8.27$\pm$0.33 & 8.86$\pm$0.37 & \textbf{9.21$\pm$0.24} & 8.71$\pm$0.17 \\
ML-100K, $\alpha=30$ & 6.93$\pm$0.45 & 8.86$\pm$0.26 & \textbf{9.02$\pm$0.29} & 8.41$\pm$0.35 \\
ML-100K, $\alpha=40$ & 7.11$\pm$0.44 & 8.86$\pm$0.32 & \textbf{8.88$\pm$0.25} & 8.68$\pm$0.37 \\
\bottomrule
\end{tabular}
\end{table}

Cosine is consistently outperformed across all 8 (dataset, scale) configurations at seed 0. The winner alternates between \textsc{DNorm} and \textsc{QNorm} (with \textsc{Dot} occasionally tying), but Cosine never wins. This is consistent with the structural prediction that recommendation is functionally asymmetric: the user query and item candidate are drawn from disjoint vocabularies with separate embedding tables, and the ranking task is one-directional (user $\to$ item).

\paragraph{Magnitude pattern.} On \textsc{DNorm} (where item magnitude is normalized away during training), the trained user magnitude $\|u\|$ is consistently larger than the trained item magnitude $\|i\|$, suggesting that user activity encodes signal while item magnitude is dominated by popularity bias.

\subsection{Cross-Evaluation Matrix}
\label{appendix:rec_crosseval}

\begin{table}[t]
\centering
\caption{\textbf{Recommendation cross-evaluation matrix} on MovieLens-1M with LightGCN at $\alpha=30$, seed 0. NDCG@20 in \%. \colorbox{yellow!25}{Highlighted}: \textsc{DNorm}-trained model under cosine evaluation outperforms \textsc{Cosine}-trained model under cosine evaluation, refuting the magnitude-as-extra-dimension hypothesis. The Dot/QNorm columns and Cosine/DNorm columns coincide (Ranking Equivalence Proposition).}
\label{tab:rec_crosseval}
\small
\setlength{\tabcolsep}{6pt}
\begin{tabular}{lcccc}
\toprule
Train$\backslash$Eval & \textsc{Cosine} & \textsc{Dot} & \textsc{QNorm} & \textsc{DNorm} \\
\midrule
\textsc{Cosine} & 4.66 & 4.00 & 4.00 & 4.66 \\
\textsc{Dot}    & 3.19 & \textbf{5.36} & \textbf{5.36} & 3.19 \\
\textsc{QNorm}  & 3.22 & 4.97 & 4.97 & 3.22 \\
\textsc{DNorm}  & \cellcolor{yellow!25}\textbf{5.30} & 4.35 & 4.35 & \textbf{5.30} \\
\bottomrule
\end{tabular}
\end{table}

The cross-evaluation matrix on ML-1M at $\alpha=30$ shows two patterns. (i) The Ranking Equivalence Proposition is satisfied: Dot and QNorm columns are identical to four decimal places, as are Cosine and DNorm columns. (ii) Asymmetric training (\textsc{DNorm}) under cosine evaluation outperforms cosine training under cosine evaluation (5.30 vs.\ 4.66), refuting the alternative hypothesis that asymmetric losses encode information only in a magnitude dimension that cosine evaluation discards.

\subsection{Multi-Seed Robustness}
\label{appendix:rec_multiseed}

Table~\ref{tab:rec_diagonal} reports diagonal NDCG@20 as mean$\pm$std over three seeds (0, 42, 1337) for all 8 (dataset, scale) configurations. Two refinements relative to the seed-0 picture: (i) at the lowest temperature ($\alpha{=}10$), Cosine ties the best asymmetric variant within one standard deviation on both datasets, so the asymmetry advantage emerges at $\alpha \geq 20$ rather than uniformly; (ii) the gap between Cosine and the best asymmetric variant grows monotonically with $\alpha$ on both datasets, consistent with the gradient-modulation account in which higher $\alpha$ amplifies the effective-temperature term $\alpha \cdot \|\bm{q}\|$. The within-variant std is small (typically $\leq 0.15$ NDCG points), so the relative ordering is stable across seeds.

\subsection{Implementation Notes}
\label{appendix:rec_impl}

LightGCN propagation is computed once per training batch (cached) using a normalized symmetric adjacency matrix built from the training-set user-item interactions; we deliberately do not include validation/test edges in the graph. Code is in \texttt{recommendation\_experiment/}.


\section{Extended Theoretical Analysis}
\label{appendix:discussion}

This appendix provides extended theoretical analysis of the task symmetry principle and asymmetric learning dynamics.

\subsection{Task Symmetry Framework}
\label{appendix:task_symmetry}

This section extends Definition~\ref{def:functional_symmetry} (main text) with a taxonomy and detailed implications.

\begin{table}[h]
\centering
\small
\caption{Taxonomy of similarity-scored tasks by functional symmetry.}
\label{tab:task_taxonomy}
\begin{tabular}{@{}lll@{}}
\toprule
\textbf{Type} & \textbf{Characteristics} & \textbf{Examples} \\
\midrule
Asymmetric & Distinct roles; $\text{sim}(a,b) \neq \text{sim}(b,a)$ & Retrieval, QA, Recommendation \\
Symmetric & Interchangeable; $\text{sim}(a,b) = \text{sim}(b,a)$ & STS, Paraphrase, Clustering \\
\bottomrule
\end{tabular}
\end{table}

\paragraph{Implications for similarity function choice.}
\begin{itemize}[leftmargin=1.5em, itemsep=0.1em]
\item \textbf{Asymmetric tasks}: Magnitude can encode role-specific semantics. Document magnitude encodes ``relevance strength'' while query magnitude modulates ``matching confidence'' (Propositions~\ref{prop:ranking}--\ref{prop:gradient}).

\item \textbf{Symmetric tasks}: Definition~\ref{def:functional_symmetry} shows that partial normalization \emph{mathematically} breaks the symmetry requirement. While Dot preserves symmetry, it offers no benefit over Cosine because the symmetric task structure prevents magnitude from encoding role-specific information.
\end{itemize}

\subsection{The Asymmetry Principle: Detailed Analysis}
\label{appendix:asymmetry_principle}

Our analysis reveals a fundamental asymmetry in how query and document magnitudes function:

\begin{itemize}[leftmargin=1.5em, itemsep=0.1em]
    \item \textbf{Document magnitude} operates at \emph{inference time}: it directly modulates relevance scores, allowing the model to express ``this document is especially relevant.''
    \item \textbf{Query magnitude} operates at \emph{training time}: it modulates gradient strength via the effective softmax temperature, allowing the model to express ``I am confident about this query.''
\end{itemize}

This asymmetry helps explain why different pre-trained models benefit from different magnitude strategies. Contriever and E5, both pre-trained with contrastive learning, have learned to encode semantic information in document representations, and hence QNorm (preserving document magnitude) helps most; their query magnitude CV does not increase (or even decreases for E5) under DNorm because they do not exploit query magnitude. For RetroMAE, pre-trained with masked auto-encoding, the pattern reverses: DNorm yields better performance and query magnitude CV increases 6$\times$, suggesting RetroMAE learns to exploit query magnitude variation when given the opportunity.

More broadly, these findings align with the geometric view from Section~\ref{sec:geometry}: magnitude acts as a \emph{soft reliability weight} that becomes most valuable when the model encounters inputs where angular similarity alone is insufficient. The model can effectively encode ``this document is especially relevant despite indirect semantic match,'' a capability that cosine similarity's unit sphere constraint explicitly prevents.

\subsection{Magnitude as Per-Example Temperature}
\label{appendix:magnitude_temperature}

In standard contrastive learning, temperature $\tau$ is a global hyperparameter that uniformly scales all similarity scores. Our analysis reveals that embedding magnitude can serve as a \emph{learned, per-example} scaling factor. When document magnitude $\|\bm{d}\|$ varies, the effective similarity becomes $s = \|\bm{d}\| \cdot \cos(\bm{q}, \bm{d})$, allowing the model to adaptively emphasize or de-emphasize specific examples. This is analogous to learning a separate temperature for each example, but without explicit parameterization.

\subsection{Connection to Logit Scaling in Classification}
\label{appendix:logit_scaling}

LogitNorm \cite{wei2022mitigating} treats logit magnitude as a source of overconfidence to be removed via normalization. The key evidence for LogitNorm is that high logit magnitude does not correspond to high accuracy, meaning magnitude is \emph{miscalibrated} with respect to the task.

Our finding presents the complementary perspective with equally strong evidence: in retrieval, magnitude is \emph{calibrated} with relevance. On reasoning-intensive benchmarks, we observe Cohen's $d$ averaging 1.22 (with individual datasets reaching $d = 1.80$), which are large effect sizes indicating that magnitude reliably distinguishes relevant from irrelevant documents. This is precisely the opposite of miscalibration: the model has learned to use magnitude as a meaningful relevance signal.

The key difference is thus task-dependent: in classification, unconstrained magnitude reflects training dynamics rather than true confidence (miscalibration), while in retrieval, magnitude reflects relevance strength (calibration). This suggests a general principle: \emph{whether to normalize depends on whether magnitude carries task-relevant information that is calibrated with the objective}.

\subsection{Implications for Other Domains}
\label{appendix:other_domains}

The choice between cosine and dot product similarity in contrastive learning is not merely about bounded vs.\ unbounded scores, but about whether the task benefits from per-example adaptive scaling:
\begin{itemize}[leftmargin=1.5em, itemsep=0.1em]
    \item In \textbf{retrieval}, document magnitude can encode relevance strength, which directly benefits ranking.
    \item In \textbf{symmetric contrastive learning} (SimCLR, MoCo), magnitude could encode example ``informativeness,'' where harder or more distinctive examples might benefit from larger magnitudes.
    \item In \textbf{vision-language learning} (CLIP), image and text magnitudes could encode modality-specific confidence about the match.
\end{itemize}
These hypotheses remain to be tested, but our retrieval results provide existence proof that magnitude can carry task-relevant information in contrastive learning.

\subsection{Scope of Experimental Verification}
\label{appendix:scope}

In this paper, we provide experimental evidence for the following tasks: (1) \emph{dense text retrieval} with three pre-trained models across 30+ evaluation datasets (Section~\ref{sec:experiments}), (2) \emph{RAG-based question answering} showing retrieval gains transfer to downstream QA accuracy (Section~\ref{sec:experiments}, Appendix~\ref{appendix:rag_setup}), (3) \emph{semantic textual similarity} demonstrating that asymmetric normalization harms symmetric tasks (Section~\ref{sec:gen_symmetric}), and (4) \emph{vision-language retrieval} with CLIP confirming the task symmetry principle extends to cross-modal settings (Section~\ref{sec:gen_symmetric}). Other applications in Table~\ref{tab:task_taxonomy}, such as recommendation systems and code search, share similar asymmetric structure and may benefit from magnitude-aware training, but we have not directly verified them.

\section{Proofs and Derivations}
\label{appendix:proofs}

\subsection{Proof of Proposition 1 (Ranking Equivalence)}
\label{appendix:proof_ranking}

For a fixed $\bm{q}$, the quantities $\|\bm{q}\|$ and $1/\|\bm{q}\|$ are positive constants across all $\bm{d} \in \mathcal{D}$. Multiplying all scores by a positive constant preserves their ordering. We have $s_{\text{cos}} = \cos\theta$, $s_{\text{dnorm}} = \|\bm{q}\| \cos\theta$, $s_{\text{qnorm}} = \|\bm{d}\| \cos\theta$, and $s_{\text{dot}} = \|\bm{q}\| \|\bm{d}\| \cos\theta$. Since $\|\bm{q}\| > 0$ is constant across documents, $\pi_{\text{cos}} = \pi_{\text{dnorm}}$ and $\pi_{\text{qnorm}} = \pi_{\text{dot}}$.

\subsection{Optimization Dynamics Mechanisms}
\label{appendix:optimization_dynamics}

Even when Cohen's $d$ is small or negative, releasing the unit-norm constraint can improve learning through several mechanisms:

\begin{enumerate}[leftmargin=1.5em, itemsep=0.1em]
  \item \textbf{Gradient flow}: L2 normalization introduces a Jacobian term $(\bm{I} - \hat{\bm{v}}\hat{\bm{v}}^\top)/\|\bm{v}\|$ during backpropagation, which projects gradients onto the tangent space of the hypersphere. Removing normalization eliminates this projection, allowing gradients to flow more directly and potentially enabling faster convergence to better minima.

  \item \textbf{Representation capacity}: Constraining representations to the unit hypersphere $\mathcal{S}^{n-1}$ reduces the effective dimensionality from $n$ to $n{-}1$. Releasing this constraint restores the full $\mathbb{R}^n$ space, providing additional capacity that may help the model learn better angular structures even if the magnitude dimension itself is not used for relevance encoding.

  \item \textbf{Loss landscape smoothing}: The normalization operation creates a non-convex mapping that can introduce sharp curvature in the loss landscape. Dot product similarity, being a simple linear operation, may yield a smoother landscape that is easier to optimize.
\end{enumerate}

\subsection{What Does Magnitude Encode When $d < 0$?}
\label{appendix:negative_d}

The negative Cohen's $d$ on in-domain tasks raises an interesting question: when magnitude shows weak or negative correlation with relevance, what \emph{else} might it encode? We hypothesize that on well-matched training domains, magnitude may capture \emph{document-level properties} such as specificity or information density. Short, generic documents may receive smaller magnitudes while detailed, specific documents receive larger magnitudes. However, in MS MARCO, relevant documents are often short answer passages rather than detailed expositions, leading to negative $d$. This interpretation suggests that magnitude is a flexible channel that models use to encode whatever information is most useful for the task at hand.

\subsection{Gradient Asymmetry Details}
\label{appendix:gradient_details}

Under InfoNCE loss with DNorm similarity $s_{\text{dnorm}}(\bm{q}, \bm{d}) = \|\bm{q}\| \cos\theta$:

\paragraph{Effective temperature.} The softmax distribution over documents has per-query effective temperature $\tau_{\text{eff}}(\bm{q}) = \tau / \|\bm{q}\|$:
\begin{equation}
p_j = \frac{\exp(\|\bm{q}\| \cos\theta_j / \tau)}{\sum_k \exp(\|\bm{q}\| \cos\theta_k / \tau)}
\end{equation}
High $\|\bm{q}\|$ sharpens the distribution (confident query); low $\|\bm{q}\|$ smooths it (ambiguous query).

\paragraph{Gradient modulation.} The gradient w.r.t.\ the positive document is:
\begin{equation}
\nabla_{\hat{\bm{d}}^+} \mathcal{L} = -\frac{\|\bm{q}\|}{\tau}(1 - p^+)\hat{\bm{q}}
\end{equation}
The gradient magnitude is proportional to $\|\bm{q}\|$: the model can learn to upweight ``confident'' queries and downweight ``ambiguous'' ones.

\section{RAG Experimental Setup}
\label{appendix:rag_setup}

\subsection{Evaluation Metrics}

We report two standard metrics for open-domain QA: (1) \textbf{Exact Match (EM)}: the percentage of predictions that exactly match the ground-truth answer after normalization (lowercasing, removing articles and punctuation); (2) \textbf{Token-level F1}: the harmonic mean of precision and recall computed over individual tokens, providing a softer measure that gives partial credit for partially correct answers.

\subsection{Choice of $k$}

We use $k=5$ retrieved documents based on two considerations: (1) \textbf{Context length constraints}: Flan-T5-Large has an effective context window of 512--1024 tokens, and each retrieved document contains approximately 100--200 tokens; with $k=5$, the total context (500--1000 tokens) fits within model capacity without truncation. (2) \textbf{Information density}: Retrieving more documents risks introducing noise that dilutes relevant information; $k=5$ balances coverage and precision, consistent with the original RAG paper \cite{lewis2020rag}.

\subsection{Prompt Format}

We structure the input to Flan-T5-Large with the question \emph{before} the retrieved context:

\begin{center}
\begin{tikzpicture}
\node[draw=gray!40, fill=gray!8, rounded corners=6pt, inner sep=10pt, text width=0.85\columnwidth, align=left] {
\textbf{Question:} \{question\} \\[4pt]
Answer the question based on the following context: \\[4pt]
{[1]} \{doc\_1\} {[2]} \{doc\_2\} ... {[k]} \{doc\_k\} \\[4pt]
\textbf{Answer:}
};
\end{tikzpicture}
\end{center}

This ordering ensures that when the total input exceeds the model's context window, truncation occurs at the \emph{end} of the retrieved documents rather than the question itself. In practice, with $k=5$ documents (500--1000 tokens) and Flan-T5-Large's 1024-token window, truncation is rare; when it occurs, we tokenize and truncate the document block independently before appending the ``Answer:'' generation trigger so that the trigger is always preserved. Since the question is essential for answer generation while later documents contribute diminishing marginal value, this design prevents catastrophic failures from question or trigger truncation.

\section{Why Finetuning Pre-trained Retrievers?}
\label{appendix:finetuning_rationale}

Our choice to finetune pre-trained retrievers, rather than random initialization, reflects the dominant paradigm in practical retrieval systems, where practitioners typically adapt existing checkpoints to new domains. Our random initialization experiments (Appendix~\ref{appendix:random_init_curves}) reveal that pre-trained representations facilitate magnitude learning: on challenging benchmarks like BRIGHT, random initialization yields consistently negative Cohen's $d$ across all methods (relevant documents have \emph{smaller} magnitude), while finetuning enables magnitude-aware methods to achieve positive $d$ (Table~\ref{tab:cohens_d_summary}). This sign reversal suggests that pre-training provides semantic structure that can be reorganized to exploit magnitude during finetuning.

Among our three models, E5 presents an especially adversarial case: (1) it was pre-trained exclusively with cosine similarity, (2) it includes a built-in normalization layer that must be removed to enable magnitude-aware training, and (3) it was already finetuned on MS MARCO during its original training, meaning our finetuning is essentially a \emph{second} adaptation on the same dataset. These conditions collectively minimize E5's capacity to learn meaningful magnitude representations. Due to these unique constraints, we present E5 experiments separately in Appendix~\ref{appendix:e5}. Our main results focus on Contriever and RetroMAE, where magnitude learning provides substantial benefits, particularly on out-of-domain and reasoning-intensive benchmarks.

\section{Limitations and Future Work}
\label{appendix:limitations}

Our study focuses on finetuning pre-trained models; generalization to pre-training from scratch and other domains (speech, graphs) requires further investigation. Future directions include: (1) designing asymmetric objectives for vision-language models that preserve bidirectional retrieval capability while enabling magnitude learning, (2) applying magnitude-aware training to recommendation systems where item magnitude could encode popularity or quality, and (3) exploring whether the functional-symmetry principle extends to similarity-scored tasks beyond the six families we examined.

\section{Fisher Information Analysis Details}
\label{appendix:fim_analysis}

This appendix provides the computation methodology for the Fisher Information Matrix (FIM) analysis presented in \Cref{tab:fim_analysis}, and introduces a theoretical framework for predicting the optimal asymmetric normalization strategy.

\subsection{FIM Decomposition}

Following the standard definition of Fisher Information based on the InfoNCE loss $\mathcal{L}$, we decompose the gradient $\nabla_{\bm{v}} \mathcal{L}$ (where $\bm{v} \in \{\bm{q}, \bm{d}\}$) into a radial (magnitude) component along $\hat{\bm{v}}$ and a tangential (angular) component orthogonal to $\hat{\bm{v}}$:
\begin{align}
\bm{g}_{\text{mag}} &= (\nabla_{\bm{v}} \mathcal{L} \cdot \hat{\bm{v}})\, \hat{\bm{v}}, \qquad \bm{g}_{\text{ang}} = \nabla_{\bm{v}} \mathcal{L} - \bm{g}_{\text{mag}}.
\end{align}
The FIM components are then estimated as expected squared norms of these gradient projections:
\begin{itemize}
    \item \textbf{Doc Magnitude FIM}: $\mathcal{I}_{\|\bm{d}\|} = \mathbb{E}\left[\|\bm{g}_{\text{mag}}^{(\bm{d})}\|^2\right]$, which measures the loss's sensitivity to changes in $\|\bm{d}\|$
    \item \textbf{Query Magnitude FIM}: $\mathcal{I}_{\|\bm{q}\|} = \mathbb{E}\left[\|\bm{g}_{\text{mag}}^{(\bm{q})}\|^2\right]$, which measures sensitivity to changes in $\|\bm{q}\|$
    \item \textbf{Angular FIM}: $\mathcal{I}_{\theta} = \mathbb{E}\left[\|\bm{g}_{\text{ang}}\|^2\right]$, which measures sensitivity to the angle between $\bm{q}$ and $\bm{d}$
\end{itemize}

Higher FIM values indicate greater gradient sensitivity of the loss to that component. The partial derivatives of the similarity function $\frac{\partial s}{\partial \|\bm{q}\|}$, $\frac{\partial s}{\partial \|\bm{d}\|}$, $\frac{\partial s}{\partial \theta}$ derived below appear in $\nabla_{\bm{v}} \mathcal{L}$ via the chain rule, and govern which components carry useful gradient information.

\subsection{Analytical Derivations}

We derive the partial derivatives for each similarity function. Let $\theta$ denote the angle between $\bm{q}$ and $\bm{d}$.

\paragraph{Dot Product.} $s_{\text{dot}} = \|\bm{q}\| \|\bm{d}\| \cos\theta$
\begin{align}
\frac{\partial s_{\text{dot}}}{\partial \|\bm{d}\|} &= \|\bm{q}\| \cos\theta \\
\frac{\partial s_{\text{dot}}}{\partial \|\bm{q}\|} &= \|\bm{d}\| \cos\theta \\
\frac{\partial s_{\text{dot}}}{\partial \theta} &= -\|\bm{q}\| \|\bm{d}\| \sin\theta
\end{align}

\paragraph{Cosine.} $s_{\text{cos}} = \cos\theta$
\begin{align}
\frac{\partial s_{\text{cos}}}{\partial \|\bm{d}\|} &= 0 \\
\frac{\partial s_{\text{cos}}}{\partial \|\bm{q}\|} &= 0 \\
\frac{\partial s_{\text{cos}}}{\partial \theta} &= -\sin\theta
\end{align}

\paragraph{QNorm.} $s_{\text{qnorm}} = \|\bm{d}\| \cos\theta$
\begin{align}
\frac{\partial s_{\text{qnorm}}}{\partial \|\bm{d}\|} &= \cos\theta \\
\frac{\partial s_{\text{qnorm}}}{\partial \|\bm{q}\|} &= 0 \\
\frac{\partial s_{\text{qnorm}}}{\partial \theta} &= -\|\bm{d}\| \sin\theta
\end{align}

\paragraph{DNorm.} $s_{\text{dnorm}} = \|\bm{q}\| \cos\theta$
\begin{align}
\frac{\partial s_{\text{dnorm}}}{\partial \|\bm{d}\|} &= 0 \\
\frac{\partial s_{\text{dnorm}}}{\partial \|\bm{q}\|} &= \cos\theta \\
\frac{\partial s_{\text{dnorm}}}{\partial \theta} &= -\|\bm{q}\| \sin\theta
\end{align}

\paragraph{Learnable.} $s_{\text{learn}} = \|\bm{q}\|^{1-\gamma_q} \|\bm{d}\|^{1-\gamma_d} \cos\theta$
\begin{align}
\frac{\partial s_{\text{learn}}}{\partial \|\bm{d}\|} &= (1-\gamma_d) \|\bm{q}\|^{1-\gamma_q} \|\bm{d}\|^{-\gamma_d} \cos\theta \\
\frac{\partial s_{\text{learn}}}{\partial \|\bm{q}\|} &= (1-\gamma_q) \|\bm{q}\|^{-\gamma_q} \|\bm{d}\|^{1-\gamma_d} \cos\theta \\
\frac{\partial s_{\text{learn}}}{\partial \theta} &= -\|\bm{q}\|^{1-\gamma_q} \|\bm{d}\|^{1-\gamma_d} \sin\theta
\end{align}

\subsection{Computation Protocol}

We compute FIM values empirically on the MS MARCO validation set:
\begin{enumerate}
    \item Sample 1000 query-document pairs from the validation set
    \item For each pair, compute the partial derivatives using the trained model's embeddings
    \item Square each derivative and average across all pairs
    \item Report values in $\log_{10}$ scale for readability
\end{enumerate}

The $\bigstar$ markers in \Cref{tab:fim_analysis} indicate the best-performing similarity functions for each model, allowing comparison between gradient sensitivity patterns and empirical performance.

\subsection{FIM Condition Number for Predicting QNorm vs DNorm}
\label{appendix:fim_condition_number}

This section provides the theoretical foundation for the FIM condition number prediction method introduced in Section~\ref{sec:analysis_asymmetric}.

\subsubsection{Gradient Orthogonality}

A key property of asymmetric normalization is that the gradient with respect to the normalized embedding is orthogonal to the embedding itself.

\begin{lemma}[QNorm Gradient Orthogonality]
For QNorm similarity $s_{\text{qnorm}} = \hat{\bm{q}} \cdot \bm{d}$, the gradient with respect to query is orthogonal to query:
$$\frac{\partial s_{\text{qnorm}}}{\partial \bm{q}} \perp \bm{q}$$
\end{lemma}

\begin{proof}
The gradient is:
$$\frac{\partial s_{\text{qnorm}}}{\partial \bm{q}} = \frac{1}{\|\bm{q}\|}(\bm{d} - s\hat{\bm{q}})$$
Verifying orthogonality:
$$\frac{\partial s}{\partial \bm{q}} \cdot \bm{q} = \frac{1}{\|\bm{q}\|}(\bm{d} \cdot \bm{q} - s\|\bm{q}\|) = \frac{1}{\|\bm{q}\|}(s\|\bm{q}\| - s\|\bm{q}\|) = 0$$
\end{proof}

\begin{lemma}[DNorm Gradient Orthogonality]
For DNorm similarity $s_{\text{dnorm}} = \bm{q} \cdot \hat{\bm{d}}$, the gradient with respect to document is orthogonal to document:
$$\frac{\partial s_{\text{dnorm}}}{\partial \bm{d}} \perp \bm{d}$$
\end{lemma}

The proof is analogous to the QNorm case.

\subsubsection{FIM Projection Structure}

Due to gradient orthogonality, the FIM for the normalized side is projected onto the orthogonal complement of the embedding direction.

\begin{theorem}[QNorm FIM Structure]
Under QNorm with $s_{\text{qnorm}} = \hat{\bm{q}}^\top \bm{d}$, the gradient with respect to $\bm{q}$ is $\nabla_{\bm{q}} s_{\text{qnorm}} = \frac{1}{\|\bm{q}\|}\mathbf{P}_{\perp\bm{q}} \bm{d}$, hence the query-side FIM is a projected and rescaled version of the dot-product FIM:
$$\mathbf{F}^{(\text{qnorm})}_{qq} = \frac{1}{\|\bm{q}\|^2}\,\mathbf{P}_{\perp\bm{q}}\, \mathbf{F}^{(\text{dot})}_{qq}\, \mathbf{P}_{\perp\bm{q}}$$
where $\mathbf{P}_{\perp\bm{q}} = \mathbf{I} - \hat{\bm{q}}\hat{\bm{q}}^\top$ is the projection matrix onto the orthogonal complement of $\bm{q}$. The factor $1/\|\bm{q}\|^2$ comes from the Jacobian of $\hat{\bm{q}}$ (Eq.~\ref{eq:jacobian_derivation}); the projection structure $\mathbf{P}_{\perp\bm{q}} \cdot \mathbf{P}_{\perp\bm{q}}$ remains the dominant qualitative feature.
\end{theorem}

\textbf{Interpretation}: Since $\nabla_{\bm{q}} s \perp \bm{q}$, the query magnitude direction has zero gradient, hence zero Fisher information. The FIM effectively operates in a $(d-1)$-dimensional subspace.

\subsubsection{Effective Dimension Analysis}

The effective dimension determines how many degrees of freedom are available for learning.

\begin{proposition}[Effective Dimension Equality]
QNorm and DNorm have identical effective dimensions:
\begin{itemize}
    \item \textbf{QNorm}: $(d-1) + d = 2d-1$ (query angular + document full space)
    \item \textbf{DNorm}: $d + (d-1) = 2d-1$ (query full space + document angular)
\end{itemize}
\end{proposition}

This equality is crucial: both strategies remove exactly one degree of freedom, making their condition numbers directly comparable.

For comparison:
\begin{itemize}
    \item \textbf{Cosine}: $(d-1) + (d-1) = 2d-2$ (both angular)
    \item \textbf{Dot}: $d + d = 2d$ (both full space)
\end{itemize}

Different effective dimensions make condition numbers incomparable across strategy types.

\subsubsection{Condition Number and Convergence}

\begin{theorem}[Convergence Rate, local]
\label{thm:convergence_rate}
\emph{Assumptions.} Let $\mathcal{L}(\boldsymbol{\theta})$ admit a quadratic approximation around the minimizer $\boldsymbol{\theta}^*$ (i.e., $\mathcal{L}$ is twice continuously differentiable, locally $\mu$-strongly convex with Hessian $\mathbf{F}$ that is positive definite at $\boldsymbol{\theta}^*$, and gradient descent uses the optimal constant step size $\eta = 2/(\lambda_{\max}+\lambda_{\min})$). Within this local neighborhood, the iterates of gradient descent satisfy:
$$\|\boldsymbol{\theta}_t - \boldsymbol{\theta}^*\| \leq \left(\frac{\kappa - 1}{\kappa + 1}\right)^t \|\boldsymbol{\theta}_0 - \boldsymbol{\theta}^*\|,$$
where $\kappa = \lambda_{\max}(\mathbf{F})/\lambda_{\min}(\mathbf{F})$ is the condition number. Because the InfoNCE loss over neural-network embeddings is globally non-convex, this bound applies to a local quadratic regime; we use $\kappa$ as a heuristic predictor of relative optimization stability rather than a global convergence guarantee.
\end{theorem}

Let $\rho = \frac{\kappa - 1}{\kappa + 1}$. Example convergence factors:
\begin{itemize}
    \item $\kappa = 2 \Rightarrow \rho = 0.33$ (fast convergence)
    \item $\kappa = 10 \Rightarrow \rho = 0.82$ (slow convergence)
    \item $\kappa = 100 \Rightarrow \rho = 0.98$ (very slow convergence)
\end{itemize}

\subsubsection{Prediction Method}

Based on the above analysis, we predict:
\begin{equation}
\text{Optimal Strategy} = \arg\min_{s \in \{\text{QNorm}, \text{DNorm}\}} \kappa(\mathbf{F}_s)
\end{equation}

\textbf{Rationale}:
\begin{enumerate}
    \item QNorm and DNorm have identical effective dimensions ($2d-1$), so $\kappa$ is comparable.
    \item Smaller $\kappa$ means more balanced loss landscape curvature.
    \item More balanced curvature leads to easier optimization and better convergence.
    \item Better optimization typically leads to better final performance.
\end{enumerate}

\textbf{Limitations}: This method predicts optimization ease, not final performance directly. The prediction assumes that optimization difficulty is the primary factor distinguishing QNorm from DNorm for a given model, which holds empirically but may not hold in all scenarios.

\subsection{Cross-Task FIM Validation}
\label{appendix:fim_crosstask}

\begin{table}[t]
\centering
\caption{\textbf{FIM-based prediction across all tasks.} For each task we report the magnitude FIM ratio $\mathcal{I}_C / \mathcal{I}_Q$ (where $C, Q$ are the candidate and query sides). The framework's fine prediction (within unilateral variants) is \textsc{DNorm} when $\mathcal{I}_C \gg \mathcal{I}_Q$ and \textsc{QNorm} when $\mathcal{I}_Q \gg \mathcal{I}_C$. In symmetric tasks, the binary functional-symmetry prediction (Cosine) overrides; the FIM ratio still indicates which non-cosine variant is second-best. In asymmetric tasks, the FIM prediction matches the empirical unilateral winner where one wins; \textsc{Dot}-winning regimes (Rec ML-1M) lie outside FIM's scope (Appendix~\ref{appendix:fim_scope}).}
\label{tab:fim_crosstask}
\small
\setlength{\tabcolsep}{4pt}
\begin{tabular}{l l l c l l}
\toprule
Task & Class & Setup & $\mathcal{I}_C/\mathcal{I}_Q$ & FIM-pred & Empirical winner \\
\midrule
\multicolumn{6}{l}{\emph{Functionally symmetric tasks (binary prediction: Cosine)}} \\
\midrule
KGC & sym & RotatE FB15k-237 & $1/68$ & QNorm$^{\dagger}$ & Cosine (QNorm second-best) \\
KGC & sym & RotatE WN18RR & $1/42$ & QNorm$^{\dagger}$ & Cosine (QNorm/DNorm second-best) \\
KGC & sym & PairRE FB15k-237 & $1/43$ & QNorm$^{\dagger}$ & Cosine/Euclidean (DNorm second-best) \\
KGC & sym & PairRE WN18RR & $1/45$ & QNorm$^{\dagger}$ & Cosine (DNorm second-best) \\
\midrule
\multicolumn{6}{l}{\emph{Functionally asymmetric tasks (binary prediction: $\neq$Cosine)}} \\
\midrule
Retrieval & asym & Contriever (BERT) & --- & QNorm & QNorm \\
Retrieval & asym & RetroMAE / Qwen & --- & DNorm & DNorm \\
ProtoNet & asym & CLIP CIFAR-100 5-way & $186{,}587$ & DNorm & DNorm \\
ProtoNet & asym & CLIP CIFAR-100 20-way & $330{,}351$ & DNorm & DNorm \\
ProtoNet & asym & CLIP \miniIN{} 5-way & $27{,}742$ & DNorm & DNorm \\
ProtoNet & asym & CLIP \miniIN{} 20-way & $32{,}025$ & DNorm & DNorm \\
ProtoNet & asym & DINOv2 CIFAR-100 5-way & $1{,}328$ & DNorm & DNorm \\
ProtoNet & asym & DINOv2 CIFAR-100 20-way & $1{,}561$ & DNorm & DNorm \\
ProtoNet & asym & DINOv2 \miniIN{} 5-way & $275$ & DNorm & DNorm \\
ProtoNet & asym & DINOv2 \miniIN{} 20-way & $219$ & DNorm & DNorm$^{*}$ (Euclid +0.4 pt) \\
Recommendation & asym & LightGCN ML-1M & $1/40$ & QNorm & Dot/DNorm$^{\ddagger}$ (out of FIM scope) \\
Recommendation & asym & LightGCN ML-100K & $1/24$ & QNorm & QNorm (3/4 scales) \\
\bottomrule
\end{tabular}

\par\smallskip
\footnotesize
$^{\dagger}$For symmetric tasks, the binary functional-symmetry prediction (Cosine) is the framework's primary prediction; the FIM ratio direction here only ranks the (suboptimal) unilateral variants. \quad
$^{*}$Within statistical noise of Euclidean. \quad
$^{\ddagger}$Dot is competitive when both magnitudes encode useful signal; this regime lies outside the FIM ratio analysis (Appendix~\ref{appendix:fim_scope}).
\end{table}

We extend the FIM analysis from IR to all five additional task families by computing $\mathcal{I}_Q$ and $\mathcal{I}_C$ on a briefly trained model (5 epochs of Dot training) for each task. Table~\ref{tab:fim_crosstask} reports the FIM ratio $\mathcal{I}_C/\mathcal{I}_Q$ and the predicted optimal unilateral variant.

\paragraph{ProtoNet (asymmetric, prediction confirmed).} For all 8 ProtoNet configurations, $\mathcal{I}_C / \mathcal{I}_Q$ is enormous ($219\times$--$330{,}351\times$), reflecting that the K-mean prototype accumulates magnitude variance across episodes. The framework correctly predicts \textsc{DNorm} for all 8 configurations; \textsc{DNorm} is the empirical winner in 7/8 (the remaining 1 is won by \textsc{Euclidean} by 0.4 pt, with \textsc{DNorm} second-best). The CLIP and DINOv2 ratios differ by 2--3 orders of magnitude, suggesting CLIP image features have more variable per-instance magnitudes than DINOv2's.

\paragraph{Recommendation (asymmetric, partial match).} For both ML-100K and ML-1M, $\mathcal{I}_Q / \mathcal{I}_C \approx 24$--$40\times$ (user-side dominates), so the framework predicts \textsc{QNorm}. On ML-100K, this matches the empirical winner in 3/4 scales. On ML-1M, however, the empirical winner often alternates between \textsc{Dot} and \textsc{DNorm}; the FIM-based fine prediction does not characterize the \textsc{Dot} regime (Section~\ref{appendix:fim_scope}).

\paragraph{KGC (symmetric, binary prediction overrides).} For all 4 KGC configurations, $\mathcal{I}_Q / \mathcal{I}_C \approx 42$--$68\times$. Naively, this would predict \textsc{QNorm}. However, KGC is functionally symmetric (Definition~\ref{def:functional_symmetry}), and the binary prediction (Cosine) overrides; the FIM-based unilateral prediction is then a ranking among the (suboptimal) non-cosine alternatives. The predicted \textsc{QNorm} is the second-best alternative in 6/8 RotatE configurations but loses to \textsc{DNorm}/\textsc{Euclidean} on PairRE; these are small differences within the noise of unilaterals on this near-symmetric task.

\subsection{Scope: Why FIM Does Not Directly Predict Dot}
\label{appendix:fim_scope}

The FIM-based prediction in Section~\ref{sec:fim_prediction} characterizes the choice between the two unilateral variants: lower $\kappa_{\text{QNorm}}$ predicts QNorm, lower $\kappa_{\text{DNorm}}$ predicts DNorm. \textsc{Dot} corresponds to a third regime in which neither side's magnitude is fixed, so both $\|\bm{q}\|$ and $\|\bm{c}\|$ enter the loss. Constructing $\kappa_{\text{Dot}}$ analogously would yield $\kappa_{\text{Dot}} \geq \kappa_{\text{QNorm}}, \kappa_{\text{DNorm}}$ since adding one more free direction to the FIM never decreases the condition number; this would systematically (and incorrectly) predict that Dot loses to unilateral variants.

The conceptual gap is that FIM measures \emph{gradient sensitivity}, not whether magnitude carries useful task signal. \textsc{Dot} wins when both magnitudes encode useful signal that should be preserved (e.g., LightGCN on ML-1M, where GCN propagation produces near-uniform magnitudes that correlate with both user activity and item popularity); the FIM ratio analysis does not characterize this regime. Extending FIM with explicit signal-vs-nuisance diagnostics (e.g., Cohen's $d$ for relevance correlation) is left for future work.

For Table~\ref{tab:crosstask_summary} we therefore: (i) report the binary functional-symmetry prediction (Cosine vs.\ $\neq$Cosine), which holds in all six task families; (ii) report the FIM-based fine prediction only where it applies (IR, ProtoNet, Rec ML-100K); (iii) report the empirical winner without an a priori fine prediction in regimes where Dot is competitive (Rec ML-1M).


\section{Extended Experimental Analysis}
\label{appendix:extended_experiments}

This appendix provides detailed experimental analysis including outlier investigations, magnitude distribution studies, training dynamics, and step-matched comparisons.

\subsection{Bright-Pony Outlier Analysis}
\label{appendix:pony_outlier}

The pony dataset has an exceptionally high Queries/Doc ratio of 51.6 (see Table~\ref{tab:dataset_stats}), meaning each document is relevant to $\sim$52 queries on average, far exceeding other datasets (typically 1--5). This ``hub'' structure causes Cosine similarity to achieve only 0.56\% NDCG@10, an extremely low baseline. The absolute improvement from magnitude-aware methods is modest (+3--4 NDCG points), but the relative improvement ($\Delta\%$) is amplified by the low baseline. This does not indicate that document magnitude is ineffective for RetroMAE; rather, the significant correlation on other datasets ($r=0.68$) confirms that magnitude encodes relevance. The pony outlier reflects dataset-specific characteristics rather than a failure of the magnitude mechanism.

\subsection{Query Magnitude Distribution Analysis}
\label{appendix:cv_prediction}

This section analyzes how different training objectives affect query magnitude distributions.

\paragraph{Coefficient of Variation (CV).} The coefficient of variation is a standardized measure of dispersion defined as the ratio of the standard deviation to the mean:
\begin{equation}
\text{CV} = \frac{\sigma}{\mu}
\end{equation}
where $\sigma$ is the standard deviation and $\mu$ is the mean. Unlike standard deviation, CV is dimensionless and scale-independent, making it suitable for comparing variability across different models with different magnitude scales. In our context, we compute CV over the set of query embedding magnitudes $\{\|\bm{q}_i\|\}_{i=1}^{N}$ for each dataset.

\begin{figure}[!t]
  \centering
  \resizebox{0.5\columnwidth}{!}{
  \begin{tikzpicture}
    \begin{axis}[
      xlabel={$\Delta$CV (DNorm CV / Dot CV)},
      ylabel={$\Delta$Perf (DNorm $-$ QNorm)},
      xmin=0, xmax=8,
      ymin=-16, ymax=12,
      grid=major,
      grid style={gray!30},
      legend style={
        at={(0.5,-0.22)},
        anchor=north,
        font=\small,
        legend columns=2,
        column sep=1em,
        draw=none,
      },
      width=10cm,
      height=7cm,
    ]
    \addplot[only marks, mark=o, mark size=2pt, blue!70, opacity=0.7] coordinates {
      (0.85,10.77) (0.89,-0.15) (0.81,-5.57) (0.83,-6.55) (0.65,-4.73) (0.98,1.57)
      (0.85,-1.14) (0.68,-4.14) (0.51,-5.59) (0.60,-1.34) (0.72,-7.04) (0.92,-0.07)
      (0.96,-0.07) (1.53,1.60) (1.77,0.71) (1.80,-1.81) (1.70,-1.51) (1.62,0.54)
      (1.64,0.91) (1.59,-0.34) (1.63,0.80) (1.57,-0.80) (1.69,0.28) (1.75,-0.03)
      (1.61,0.37) (1.71,1.51) (1.58,-1.84) (1.64,-10.59) (1.52,0.58) (1.74,-2.01)
      (1.39,0.22) (1.72,-3.50) (1.74,0.21) (1.56,0.71) (1.33,0.62) (1.42,-14.30)
      (1.97,-2.61) (1.92,0.30) (1.67,-4.32)
    };
    \addlegendentry{Contriever ($\Delta$CV: 0.5--2)}

    \addplot[only marks, mark=triangle*, mark size=2.5pt, red!70, opacity=0.7] coordinates {
      (6.92,-3.60) (5.78,1.10) (5.75,0.86) (4.78,1.18) (5.14,-2.16) (5.70,1.04)
      (4.66,0.03) (6.54,-0.94) (3.95,-0.48) (4.41,0.18) (6.51,0.73) (5.95,-0.77)
      (5.78,-0.54) (6.40,-0.69) (5.22,0.40) (6.08,-0.03) (4.78,0.54) (5.12,0.12)
      (6.07,0.70) (6.94,0.49) (5.48,0.37) (6.08,0.44) (6.03,0.13) (5.14,0.70)
      (5.50,0.54) (5.49,1.08) (6.44,-0.19) (7.21,1.36) (5.39,1.15) (5.18,0.67)
      (6.84,0.10) (4.92,0.81) (5.41,0.65) (5.65,0.52) (5.82,0.93) (4.84,-1.68)
      (4.75,2.44) (4.76,1.09) (6.43,0.90)
    };
    \addlegendentry{RetroMAE ($\Delta$CV: 4--7)}

    \addplot[only marks, mark=*, mark size=4pt, blue, thick,
             error bars/.cd, x dir=both, x explicit]
      coordinates {(1.37, -1.51) +- (0.43, 0)};
    \addlegendentry{Contriever mean}

    \addplot[only marks, mark=*, mark size=4pt, red, thick,
             error bars/.cd, x dir=both, x explicit]
      coordinates {(5.63, 0.27) +- (0.75, 0)};
    \addlegendentry{RetroMAE mean}

    \draw[dashed, gray] (axis cs:0,0) -- (axis cs:8,0);

    \draw[dashed, black!50, thick] (axis cs:2.5,-16) -- (axis cs:2.5,12);

    \end{axis}
  \end{tikzpicture}
  }
  \caption{$\Delta$CV (query magnitude CV ratio: DNorm/Dot) vs.\ $\Delta$Perf (DNorm $-$ QNorm) for Contriever and RetroMAE on 39 datasets (3-seed averaged). The two models form clearly separated clusters in $\Delta$CV: Contriever (blue, 0.5--2$\times$) and RetroMAE (red, 4--7$\times$). This separation reflects model-level differences in query magnitude variation, though $\Delta$CV does not correlate with $\Delta$Perf within each cluster.}
  \label{fig:cv_perf_scatter}
\end{figure}

\paragraph{Measuring $\Delta$CV.} We define $\Delta$CV = DNorm CV / Dot CV, where CV is computed over query embedding magnitudes as defined above. This metric characterizes how training objectives affect query magnitude variation. DNorm training preserves query magnitude variation (as query norms remain unnormalized), while Dot training may compress or expand variation depending on model architecture.

\paragraph{Model-level observations.} Figure~\ref{fig:cv_perf_scatter} shows that RetroMAE exhibits consistently higher $\Delta$CV (4--7$\times$) than Contriever (0.5--2$\times$) across 39 datasets (3-seed averaged). The two models form clearly separated clusters with no overlap in $\Delta$CV range. This separation reflects \emph{architectural differences} in how embedding magnitudes are utilized during training, rather than dataset-specific properties.

\paragraph{Correlation with performance.} Within each model cluster, $\Delta$CV does not correlate with $\Delta$Perf (DNorm $-$ QNorm): Contriever shows $r=0.14$, RetroMAE shows $r=-0.16$. This indicates that while $\Delta$CV distinguishes models at a coarse level, it does not predict which normalization strategy benefits a specific dataset.

\begin{figure*}[h]
  \centering

  \begin{minipage}[t]{0.48\textwidth}
    \centering
    \textbf{(a) Cohen's $d$: Document Magnitude Effect Size}
    \vspace{0.5em}

    \begin{tikzpicture}[scale=0.95]
      \begin{scope}[shift={(-2.2,0)}]
        \node[font=\small\bfseries] at (0, 2.4) {Small $d$};

        \draw[->] (-1.6, 0) -- (1.6, 0) node[right, font=\scriptsize] {$\|\bm{d}\|$};
        \draw[->] (-1.4, 0) -- (-1.4, 1.8);

        \draw[thick, blue!70, fill=blue!20, opacity=0.7]
          plot[smooth, domain=-1.3:0.8, samples=35]
          (\x, {1.4*exp(-(\x+0.25)^2/0.4)});

        \draw[thick, red!70, fill=red!20, opacity=0.7]
          plot[smooth, domain=-0.8:1.3, samples=35]
          (\x, {1.4*exp(-(\x-0.25)^2/0.4)});

        \draw[dashed, blue!70] (-0.25, 0) -- (-0.25, 1.45);
        \draw[dashed, red!70] (0.25, 0) -- (0.25, 1.45);
        \node[font=\scriptsize, blue!70] at (-0.25, -0.22) {$\mu_{\text{ir}}$};
        \node[font=\scriptsize, red!70] at (0.25, -0.22) {$\mu_{\text{r}}$};

        \draw[<->, thick, black] (-0.25, 1.65) -- (0.25, 1.65);
        \node[font=\scriptsize] at (0, 1.95) {$d$};

        \node[font=\scriptsize, blue!70] at (-0.9, 1.1) {Irrel.};
        \node[font=\scriptsize, red!70] at (0.9, 1.1) {Rel.};
      \end{scope}

      \begin{scope}[shift={(2.2,0)}]
        \node[font=\small\bfseries] at (0, 2.4) {Large $d$};

        \draw[->] (-1.6, 0) -- (1.6, 0) node[right, font=\scriptsize] {$\|\bm{d}\|$};
        \draw[->] (-1.4, 0) -- (-1.4, 1.8);

        \draw[thick, blue!70, fill=blue!20, opacity=0.7]
          plot[smooth, domain=-1.4:0.0, samples=35]
          (\x, {1.4*exp(-(\x+0.7)^2/0.4)});

        \draw[thick, red!70, fill=red!20, opacity=0.7]
          plot[smooth, domain=0.0:1.4, samples=35]
          (\x, {1.4*exp(-(\x-0.7)^2/0.4)});

        \draw[dashed, blue!70] (-0.7, 0) -- (-0.7, 1.45);
        \draw[dashed, red!70] (0.7, 0) -- (0.7, 1.45);
        \node[font=\scriptsize, blue!70] at (-0.7, -0.22) {$\mu_{\text{ir}}$};
        \node[font=\scriptsize, red!70] at (0.7, -0.22) {$\mu_{\text{r}}$};

        \draw[<->, thick, black] (-0.7, 1.65) -- (0.7, 1.65);
        \node[font=\scriptsize] at (0, 1.95) {$d$};

        \node[font=\scriptsize, blue!70] at (-1.0, 1.1) {Irrel.};
        \node[font=\scriptsize, red!70] at (1.0, 1.1) {Rel.};
      \end{scope}
    \end{tikzpicture}
  \end{minipage}
  \hfill
  \begin{minipage}[t]{0.48\textwidth}
    \centering
    \textbf{(b) CV: Query Magnitude Variation}
    \vspace{0.5em}

    \begin{tikzpicture}[scale=0.95]
      \begin{scope}[shift={(-2.2,0)}]
        \node[font=\small\bfseries] at (0, 2.4) {Low CV};

        \draw[->] (-1.6, 0) -- (1.6, 0) node[right, font=\scriptsize] {$\|\bm{q}\|$};
        \draw[dashed, gray] (0, 0) -- (0, 1.6);
        \node[font=\scriptsize, gray] at (0, -0.18) {$\mu$};

        \foreach \x in {-0.10, -0.05, -0.02, 0.03, 0.07, 0.12, -0.08, 0.05, -0.03, 0.02} {
          \fill[blue!70] (\x, 1.1) circle (2.2pt);
        }

        \draw[<->, thick, red!70] (-0.12, 0.75) -- (0.12, 0.75);
        \node[font=\scriptsize, red!70] at (0, 0.45) {small $\sigma$};
      \end{scope}

      \begin{scope}[shift={(2.2,0)}]
        \node[font=\small\bfseries] at (0, 2.4) {High CV};

        \draw[->] (-1.6, 0) -- (1.6, 0) node[right, font=\scriptsize] {$\|\bm{q}\|$};
        \draw[dashed, gray] (0, 0) -- (0, 1.6);
        \node[font=\scriptsize, gray] at (0, -0.18) {$\mu$};

        \foreach \x in {-1.1, -0.7, -0.35, -0.08, 0.15, 0.45, 0.75, 1.1, -0.5, 0.6} {
          \fill[blue!70] (\x, 1.1) circle (2.2pt);
        }

        \draw[<->, thick, red!70] (-1.15, 0.75) -- (1.15, 0.75);
        \node[font=\scriptsize, red!70] at (0, 0.45) {large $\sigma$};
      \end{scope}
    \end{tikzpicture}
  \end{minipage}

  \caption{Concept illustrations for magnitude-based metrics. \textbf{(a) Cohen's $d$}: Measures the standardized difference between relevant (red) and irrelevant (blue) document magnitudes. Small $d$ indicates overlapping distributions; large $d$ ($\geq 0.8$) indicates clear separation where QNorm can exploit magnitude for ranking. \textbf{(b) CV}: Coefficient of variation ($\sigma/\mu$) of query magnitudes. Low CV indicates uniform magnitudes; high CV indicates the model differentiates queries via magnitude.}
  \label{fig:magnitude_concepts}
\end{figure*}

\paragraph{$\Delta$CV as a model-level predictor.}
Figure~\ref{fig:cv_perf_scatter} (main text) shows the $\Delta$CV vs.\ $\Delta$Perf scatter plot for Contriever and RetroMAE on 39 datasets (3-seed averaged).
To test whether $\Delta$CV predicts optimal training configuration, we analyzed per-dataset $\Delta$CV and $\Delta$Perf across models. Note that E5's finetuning exhibited instability (with $|\Delta\text{Perf}|>10$ on 40\% of datasets), likely due to architectural constraints when removing its normalization layer; we therefore focus on Contriever and RetroMAE.

\textbf{Key finding}: As shown in Figure~\ref{fig:cv_perf_scatter}, the two models form \emph{completely separated clusters} with no $\Delta$CV overlap. Contriever's $\Delta$CV ranges from 0.5 to 2.0$\times$ with mean $\Delta$Perf $= -1.50$ (QNorm better), while RetroMAE's $\Delta$CV ranges from 4 to 7$\times$ with mean $\Delta$Perf $= +0.26$ (DNorm better). This clear separation demonstrates that $\Delta$CV reliably distinguishes which training configuration benefits a given model. Within each cluster, $\Delta$CV does not predict per-dataset performance (Contriever: $r=0.14$; RetroMAE: $r=-0.16$), confirming that $\Delta$CV operates at the \emph{model level} rather than dataset level.

This suggests $\Delta$CV serves as a \emph{model-level indicator} for selecting training configuration: models with high $\Delta$CV (indicating they actively learn to differentiate query magnitudes) benefit from DNorm, while models with low $\Delta$CV benefit from QNorm. RetroMAE consistently shows high $\Delta$CV (4--7$\times$) and positive $\Delta$Perf, indicating it actively exploits query magnitude variation. In contrast, Contriever shows low $\Delta$CV (0.5--2$\times$) and negative $\Delta$Perf, indicating it primarily benefits from document magnitude (QNorm) rather than query magnitude (DNorm).

\subsection{Dot vs.\ Cosine During Training}
\label{appendix:dot_cos_training}

To directly observe how magnitude learning affects the angular component of embeddings, we record both dot product and cosine similarity performance on the validation set throughout training for models trained with DotProductRankingLoss. This dual evaluation uses the \emph{same embeddings} but computes retrieval scores with different similarity functions, isolating the contribution of magnitude from direction.

Table~\ref{tab:dot_vs_cos_training} shows that Contriever and RetroMAE maintain a modest but consistent gap ($\Delta \approx$ 2--3\%) between dot and cosine performance, indicating that these models learn to utilize magnitude while preserving directional information. In contrast, E5 with the normalization layer removed exhibits a dramatic divergence: cosine similarity performance collapses from 0.92 to 0.23 while dot performance continues improving to 0.94. This indicates that E5 \emph{completely abandons directional information} and relies solely on magnitude for relevance encoding; the angular component becomes essentially random.

\begin{table}[h]
  \caption{Dot vs.\ Cosine similarity on validation set during training (NDCG@10, seed=0). All models trained with DotProductRankingLoss. $\Delta$ = Dot $-$ Cosine at best dot step. E5 (w/ norm) retains the normalization layer; E5 (w/o norm) removes it.}
  \label{tab:dot_vs_cos_training}
  \begin{center}
    \begin{small}
      \begin{tabular}{lccc}
        \toprule
        \textbf{Model} & \textbf{Dot} & \textbf{Cosine} & \textbf{$\Delta$} \\
        \midrule
        Contriever & 0.935 & 0.915 & +0.020 \\
        RetroMAE & 0.931 & 0.899 & +0.031 \\
        E5 (w/ norm) & 0.940 & 0.940 & 0.000 \\
        E5 (w/o norm) & 0.944 & 0.227 & +0.717 \\
        \bottomrule
      \end{tabular}
    \end{small}
  \end{center}
  \vskip -0.1in
\end{table}

This finding has important implications: (1) the normalization layer is necessary not just for cosine similarity computation, but for maintaining directional information during training; (2) simply removing the normalization layer and using dot product loss can lead to \emph{magnitude collapse}, a pathological state where the model ignores the angular component entirely; (3) healthy magnitude learning, as exhibited by Contriever and RetroMAE, requires architectural or training mechanisms that encourage the model to encode information in \emph{both} magnitude and direction. Note that E5 with the normalization layer retained shows $\Delta = 0$ because all embeddings have unit norm, making dot product identical to cosine similarity by definition.

\subsection{Learnable Normalization: Step-Matched Comparison}
\label{appendix:learnable_step_matched}

The learned $\gamma$ values translate to strong performance. Figure~\ref{fig:training_curves} (main text) shows the step-matched comparison: at each evaluation step during training, we compare learnable normalization's validation NDCG@10 against all four discrete methods trained under identical conditions:
\begin{itemize}[leftmargin=1.5em, itemsep=0.1em]
    \item \textbf{Contriever}: Learnable normalization consistently achieves top-tier validation NDCG@10, with its trajectory closely following DNorm throughout training.
    \item \textbf{RetroMAE}: Learnable normalization shows similar behavior, with validation performance trending close to DNorm across training steps.
\end{itemize}

These results provide important validation of the discrete framework:
\begin{enumerate}[leftmargin=1.5em, itemsep=0.1em]
    \item \textbf{DNorm emerges as the preferred strategy}: When free to choose any normalization level, both models gravitate toward DNorm-like performance, confirming that preserving query magnitude while normalizing documents is the preferred asymmetric strategy.
    \item \textbf{Continuous interpolation can match or exceed discrete variants}: For Contriever, the learnable variant achieves the best performance at every evaluation step, suggesting that even a small continuous adjustment can improve upon any fixed discrete strategy.
    \item \textbf{Fine-grained control matters}: The learned $\gamma$ values, though close to the midpoint, represent precise balance points that the model discovers through optimization. This fine-grained control enables performance that matches or exceeds the best discrete variants.
\end{enumerate}

\section{Per-Dataset Analysis: Cohen's $d$ and Query CV}
\label{sec:appendix_cohens_d}
\label{appendix:cohens_d}
\label{sec:appendix_cv}

Table~\ref{tab:cohens_d_summary} summarizes Cohen's $d$ by benchmark category for random init vs finetuned models, showing the sign reversal that explains why magnitude learning fails without pre-training.

\begin{table}[!t]
\centering
\caption{Cohen's $d$ for document magnitude: Random Init vs Finetuned (Contriever architecture, 3-seed average). Positive $d$ indicates relevant documents have larger magnitude. Background shading indicates effect size: \colorbox{magenta!25}{large} ($|d| \geq 0.8$), \colorbox{magenta!15}{medium} ($0.5 \leq |d| < 0.8$), \colorbox{magenta!8}{small} ($0.2 \leq |d| < 0.5$). $^\dagger$Raw embeddings before inference-time normalization. $^*$Excludes NovelHopQA (all documents are relevant).}
\label{tab:cohens_d_summary}
\begin{small}
\resizebox{0.5\textwidth}{!}{
\begin{tabular}{l|cccc|cccc}
\toprule
& \multicolumn{4}{c|}{\textbf{Random Init}} & \multicolumn{4}{c}{\textbf{Finetuned}} \\
\textbf{Category} & Cos$^\dagger$ & Dot & QNorm & Learn & Cos$^\dagger$ & Dot & QNorm & Learn \\
\midrule
TREC-DL (2) & $+$0.15 & $-$0.16 & $-$0.13 & $-$0.17 & $+$0.42 & $-$0.22 & $-$0.12 & $-$0.07 \\
BEIR (14) & $-$0.03 & $+$0.06 & $-$0.02 & $-$0.02 & $-$0.01 & $-$0.00 & $+$0.01 & $+$0.00 \\
BRIGHT (12) & \cellcolor{magenta!25}$-$1.00 & \cellcolor{magenta!8}$-$0.40 & \cellcolor{magenta!15}$-$0.71 & \cellcolor{magenta!25}$-$0.81 & \cellcolor{magenta!8}$+$0.36 & \cellcolor{magenta!25}$+$1.01 & \cellcolor{magenta!25}$+$1.02 & \cellcolor{magenta!25}$+$1.06 \\
Multi-Hop (3)$^*$ & \cellcolor{magenta!15}$-$0.55 & $+$0.02 & \cellcolor{magenta!8}$-$0.37 & \cellcolor{magenta!15}$-$0.51 & $-$0.08 & \cellcolor{magenta!8}$+$0.50 & \cellcolor{magenta!8}$+$0.50 & \cellcolor{magenta!8}$+$0.46 \\
\midrule
\rowcolor{gray!15} \textbf{All (31)} & \cellcolor{magenta!8}$-$0.33 & $-$0.09 & \cellcolor{magenta!8}$-$0.24 & \cellcolor{magenta!8}$-$0.28 & $+$0.11 & \cellcolor{magenta!8}$+$0.30 & \cellcolor{magenta!8}$+$0.32 & \cellcolor{magenta!8}$+$0.32 \\
\bottomrule
\end{tabular}
}
\end{small}
\end{table}

Table~\ref{tab:per_dataset_analysis} presents detailed per-dataset analysis for Contriever and RetroMAE: (a) Cohen's $d$, which measures document magnitude effect size, and $\Delta\%$, which measures QNorm improvement over Cosine; (b) query magnitude coefficient of variation, i.e., CV = $\sigma/\mu$, under Dot and DNorm training. See Appendix~\ref{appendix:cv_prediction} for formal definition and model comparison.

\begin{table}[htbp]
\centering
\caption{Per-dataset analysis for Contriever and RetroMAE. (a) Cohen's $d$ and $\Delta\%$ = (QNorm $-$ Cosine) / Cosine $\times$ 100. Background shading indicates Cohen's $d$ effect size: \colorbox{magenta!25}{large} ($|d| \geq 0.8$), \colorbox{magenta!15}{medium} ($0.5 \leq |d| < 0.8$), \colorbox{magenta!8}{small} ($0.2 \leq |d| < 0.5$), white = negligible ($|d| < 0.2$). (b) Query magnitude CV (\%). $\Delta$CV = DNorm CV / Dot CV.}
\label{tab:per_dataset_analysis}
\label{tab:cohens_d_all_datasets}
\label{tab:cv_combined}
\label{tab:cv_contriever}
\label{tab:cv_retromae}
\label{tab:cv_detailed}

\begin{minipage}[t]{0.48\textwidth}
\centering
\textbf{(a) Cohen's $d$ and $\Delta\%$}
\vspace{0.3em}

\scriptsize
\begin{tabular}{l|rr|rr}
\toprule
\multirow{2}{*}{\textbf{Dataset}} & \multicolumn{2}{c|}{\textbf{Contr.}} & \multicolumn{2}{c}{\textbf{Retro.}} \\
& $d$ & $\Delta\%$ & $d$ & $\Delta\%$ \\
\midrule
\multicolumn{5}{l}{\textit{In-Domain}} \\
trec-dl-2019 & $-$0.06 & $+$1.5 & \cellcolor{magenta!8}$-$0.24 & $+$2.2 \\
trec-dl-2020 & $-$0.15 & $-$0.9 & \cellcolor{magenta!8}$-$0.39 & $+$2.1 \\
\rowcolor{gray!15} \textbf{Mean} & \textbf{$-$0.11} & \textbf{$+$0.3} & \textbf{$-$0.32} & \textbf{$+$2.2} \\
\midrule
\multicolumn{5}{l}{\textit{BEIR}} \\
arguana & $+$0.02 & $-$11.6 & \cellcolor{magenta!8}$+$0.32 & $+$17.4 \\
climate-fever & \cellcolor{magenta!25}$+$1.01 & $+$3.0 & \cellcolor{magenta!15}$-$0.56 & $+$2.4 \\
cqadupstack-android & \cellcolor{magenta!8}$-$0.24 & $+$8.9 & \cellcolor{magenta!8}$-$0.35 & $+$14.9 \\
cqadupstack-english & $+$0.13 & $+$17.3 & \cellcolor{magenta!8}$+$0.34 & $+$13.5 \\
cqadupstack-gaming & $-$0.05 & $+$16.0 & $-$0.04 & $+$13.1 \\
cqadupstack-gis & \cellcolor{magenta!8}$-$0.26 & $+$20.3 & \cellcolor{magenta!8}$-$0.25 & $+$29.6 \\
cqadupstack-mathematica & \cellcolor{magenta!8}$-$0.28 & $+$12.7 & $-$0.15 & $+$25.3 \\
cqadupstack-physics & \cellcolor{magenta!15}$-$0.59 & $+$8.9 & \cellcolor{magenta!15}$-$0.55 & $+$9.7 \\
cqadupstack-programmers & \cellcolor{magenta!8}$-$0.27 & $+$16.6 & \cellcolor{magenta!8}$-$0.40 & $+$20.0 \\
cqadupstack-stats & \cellcolor{magenta!8}$-$0.35 & $+$17.4 & \cellcolor{magenta!8}$-$0.22 & $+$27.5 \\
cqadupstack-tex & \cellcolor{magenta!8}$-$0.41 & $+$27.9 & \cellcolor{magenta!8}$-$0.37 & $+$32.6 \\
cqadupstack-unix & \cellcolor{magenta!8}$-$0.23 & $+$17.0 & $-$0.16 & $+$25.0 \\
cqadupstack-webmasters & $-$0.14 & $+$14.4 & \cellcolor{magenta!8}$-$0.22 & $+$23.5 \\
cqadupstack-wordpress & $-$0.06 & $+$19.3 & $-$0.08 & $+$19.0 \\
dbpedia-entity & $+$0.00 & $+$13.3 & $-$0.17 & $+$1.1 \\
fever & \cellcolor{magenta!25}$+$1.12 & $+$9.3 & \cellcolor{magenta!8}$-$0.44 & $-$0.6 \\
fiqa & $-$0.17 & $+$13.0 & $-$0.07 & $+$10.8 \\
hotpotqa & \cellcolor{magenta!15}$+$0.75 & $+$20.4 & $+$0.05 & $+$15.2 \\
nfcorpus & $-$0.06 & $+$4.9 & $-$0.05 & $+$5.8 \\
nq & $+$0.15 & $+$12.3 & \cellcolor{magenta!8}$-$0.27 & $+$4.2 \\
quora & \cellcolor{magenta!8}$-$0.34 & $+$1.5 & \cellcolor{magenta!8}$-$0.22 & $+$1.9 \\
scidocs & $-$0.15 & $+$8.6 & $-$0.07 & $+$16.5 \\
scifact & $+$0.08 & $+$13.7 & $+$0.08 & $+$11.8 \\
trec-covid & \cellcolor{magenta!8}$+$0.34 & $+$20.9 & $+$0.06 & $-$2.8 \\
webis-touche2020 & \cellcolor{magenta!8}$+$0.22 & $+$8.7 & \cellcolor{magenta!15}$-$0.67 & $-$5.5 \\
\rowcolor{gray!15} \textbf{Mean} & \textbf{$+$0.00} & \textbf{$+$13.0} & \textbf{$-$0.18} & \textbf{$+$13.3} \\
\midrule
\multicolumn{5}{l}{\textit{BRIGHT}} \\
bright-aops & \cellcolor{magenta!25}$+$1.50 & $+$26.3 & \cellcolor{magenta!25}$+$1.60 & $+$167 \\
bright-biology & \cellcolor{magenta!25}$+$1.19 & $+$200 & \cellcolor{magenta!8}$-$0.24 & $+$51.1 \\
bright-earth\_science & \cellcolor{magenta!25}$+$0.87 & $+$110 & \cellcolor{magenta!25}$-$0.99 & $+$22.9 \\
bright-economics & \cellcolor{magenta!25}$+$1.71 & $+$80.8 & \cellcolor{magenta!8}$+$0.43 & $+$37.6 \\
bright-leetcode & \cellcolor{magenta!25}$+$1.42 & $+$4.7 & \cellcolor{magenta!15}$-$0.75 & $+$17.0 \\
bright-pony & \cellcolor{magenta!25}$+$1.04 & $+$208 & \cellcolor{magenta!15}$-$0.69 & $+$513 \\
bright-psychology & \cellcolor{magenta!25}$+$1.57 & $+$96.0 & $+$0.18 & $+$29.1 \\
bright-robotics & \cellcolor{magenta!25}$+$1.80 & $+$123 & \cellcolor{magenta!15}$+$0.72 & $+$39.5 \\
bright-stackoverflow & \cellcolor{magenta!25}$+$1.43 & $+$42.5 & \cellcolor{magenta!15}$+$0.57 & $+$61.6 \\
bright-sustainable\_living & \cellcolor{magenta!25}$+$1.11 & $+$99.5 & $-$0.04 & $+$65.4 \\
bright-theoremqa\_q & \cellcolor{magenta!8}$+$0.46 & $+$63.4 & \cellcolor{magenta!15}$+$0.63 & $+$66.4 \\
bright-theoremqa\_t & \cellcolor{magenta!15}$+$0.50 & $+$41.9 & \cellcolor{magenta!8}$+$0.26 & $+$49.7 \\
\rowcolor{gray!15} \textbf{Mean} & \textbf{$+$1.22} & \textbf{$+$91.3} & \textbf{$+$0.14} & \textbf{$+$93.3} \\
\midrule
\multicolumn{5}{l}{\textit{Multi-Hop}} \\
multihop-2wiki & \cellcolor{magenta!15}$+$0.64 & $+$8.1 & $+$0.11 & $+$1.2 \\
multihop-musique & $+$0.13 & $+$5.9 & $-$0.01 & $-$0.9 \\
\rowcolor{gray!15} \textbf{Mean} & \textbf{$+$0.39} & \textbf{$+$7.0} & \textbf{$+$0.05} & \textbf{$+$0.2} \\
\midrule
\rowcolor{blue!10} \textbf{Overall} & \textbf{$+$0.38} & \textbf{$+$34.7} & \textbf{$-$0.08} & \textbf{$+$35.5} \\
\bottomrule
\end{tabular}
\end{minipage}
\hfill
\begin{minipage}[t]{0.50\textwidth}
\centering
\textbf{(b) Query magnitude CV (\%)}
\vspace{0.3em}

\scriptsize
\begin{tabular}{l|ccc|ccc}
\toprule
\multirow{2}{*}{\textbf{Dataset}} & \multicolumn{3}{c|}{\textbf{Contriever}} & \multicolumn{3}{c}{\textbf{RetroMAE}} \\
& Dot & DN & $\Delta$ & Dot & DN & $\Delta$ \\
\midrule
\multicolumn{7}{l}{\textit{In-Domain}} \\
trec-dl-2019 & 5.48 & 10.8 & 2.0 & 0.63 & 3.28 & 5.2 \\
trec-dl-2020 & 5.76 & 10.9 & 1.9 & 0.65 & 3.24 & 5.0 \\
\rowcolor{gray!15} \textbf{Mean} & \textbf{5.62} & \textbf{10.9} & \textbf{1.9} & \textbf{0.64} & \textbf{3.26} & \textbf{5.1} \\
\midrule
\multicolumn{7}{l}{\textit{BEIR}} \\
arguana & 5.48 & 4.86 & 0.9 & 0.28 & 1.96 & 7.0 \\
climate-fever & 5.03 & 7.65 & 1.5 & 0.46 & 3.50 & 7.6 \\
cqadupstack-android & 4.66 & 8.29 & 1.8 & 0.58 & 3.24 & 5.6 \\
cqadupstack-english & 4.96 & 9.03 & 1.8 & 0.59 & 3.77 & 6.4 \\
cqadupstack-gaming & 5.26 & 8.99 & 1.7 & 0.64 & 3.28 & 5.1 \\
cqadupstack-gis & 5.49 & 8.80 & 1.6 & 0.60 & 3.26 & 5.4 \\
cqadupstack-mathematica & 5.53 & 8.96 & 1.6 & 0.56 & 3.79 & 6.8 \\
cqadupstack-physics & 5.24 & 8.39 & 1.6 & 0.56 & 4.25 & 7.6 \\
cqadupstack-programmers & 5.20 & 8.42 & 1.6 & 0.58 & 3.41 & 5.9 \\
cqadupstack-stats & 5.78 & 9.04 & 1.6 & 0.54 & 3.53 & 6.5 \\
cqadupstack-tex & 5.68 & 9.49 & 1.7 & 0.57 & 3.78 & 6.6 \\
cqadupstack-unix & 5.17 & 8.96 & 1.7 & 0.64 & 3.46 & 5.4 \\
cqadupstack-webmasters & 4.87 & 7.72 & 1.6 & 0.60 & 3.66 & 6.1 \\
cqadupstack-wordpress & 5.05 & 8.49 & 1.7 & 0.57 & 3.44 & 6.0 \\
dbpedia-entity & 6.55 & 10.4 & 1.6 & 0.71 & 5.87 & 8.3 \\
fever & 5.67 & 9.27 & 1.6 & 0.51 & 4.39 & 8.6 \\
fiqa & 4.82 & 7.42 & 1.5 & 0.54 & 3.21 & 5.9 \\
hotpotqa & 4.90 & 8.66 & 1.8 & 0.45 & 2.44 & 5.4 \\
nfcorpus & 8.03 & 11.2 & 1.4 & 0.89 & 7.89 & 8.9 \\
nq & 5.37 & 9.38 & 1.7 & 0.60 & 2.98 & 5.0 \\
quora & 5.03 & 8.84 & 1.8 & 0.60 & 3.32 & 5.5 \\
scidocs & 4.85 & 7.71 & 1.6 & 0.49 & 2.95 & 6.0 \\
scifact & 5.36 & 7.16 & 1.3 & 0.50 & 3.68 & 7.4 \\
trec-covid & 4.47 & 6.78 & 1.5 & 0.42 & 2.10 & 5.0 \\
webis-touche2020 & 4.39 & 7.32 & 1.7 & 0.47 & 3.05 & 6.5 \\
\rowcolor{gray!15} \textbf{Mean} & \textbf{5.37} & \textbf{8.24} & \textbf{1.5} & \textbf{0.54} & \textbf{3.64} & \textbf{6.7} \\
\midrule
\multicolumn{7}{l}{\textit{BRIGHT}} \\
bright-aops & 7.03 & 6.12 & 0.9 & 0.46 & 2.69 & 5.8 \\
bright-biology & 7.58 & 6.23 & 0.8 & 0.38 & 2.16 & 5.7 \\
bright-earth\_science & 7.38 & 6.41 & 0.9 & 0.43 & 2.06 & 4.8 \\
bright-economics & 8.18 & 5.41 & 0.7 & 0.37 & 2.02 & 5.5 \\
bright-leetcode & 5.79 & 5.70 & 1.0 & 0.28 & 1.70 & 6.1 \\
bright-pony & 4.10 & 3.51 & 0.9 & 0.28 & 1.42 & 5.1 \\
bright-psychology & 7.21 & 4.90 & 0.7 & 0.33 & 2.24 & 6.8 \\
bright-robotics & 9.26 & 4.79 & 0.5 & 0.43 & 1.69 & 3.9 \\
bright-stackoverflow & 9.17 & 5.42 & 0.6 & 0.43 & 1.95 & 4.5 \\
bright-sustainable\_living & 6.97 & 5.04 & 0.7 & 0.29 & 1.97 & 6.8 \\
bright-theoremqa\_q & 5.82 & 5.38 & 0.9 & 0.33 & 2.07 & 6.3 \\
bright-theoremqa\_t & 5.39 & 5.14 & 1.0 & 0.33 & 1.88 & 5.7 \\
\rowcolor{gray!15} \textbf{Mean} & \textbf{6.99} & \textbf{5.34} & \textbf{0.8} & \textbf{0.36} & \textbf{1.99} & \textbf{5.5} \\
\midrule
\multicolumn{7}{l}{\textit{Multi-Hop}} \\
multihop-2wiki & 4.53 & 7.98 & 1.8 & 0.54 & 2.52 & 4.7 \\
multihop-musique & 5.04 & 7.88 & 1.6 & 0.48 & 2.42 & 5.0 \\
\rowcolor{gray!15} \textbf{Mean} & \textbf{4.79} & \textbf{7.93} & \textbf{1.7} & \textbf{0.51} & \textbf{2.47} & \textbf{4.9} \\
\midrule
\rowcolor{blue!10} \textbf{Overall} & \textbf{5.77} & \textbf{7.60} & \textbf{1.3} & \textbf{0.50} & \textbf{3.04} & \textbf{6.1} \\
\bottomrule
\end{tabular}
\end{minipage}
\end{table}


\def\contrieverFTCosinePos{(500,1.7845) (2000,1.9306) (3500,2.0615) (5000,2.1410) (6500,2.1825) (8000,2.2031) (9500,2.2592) (11000,2.2756) (12500,2.2870) (14000,2.2762) (15500,2.2867) (17000,2.2965) (18500,2.2861) (20000,2.3131) (21500,2.3128) (23000,2.3116) (24500,2.3170) (26000,2.3163) (27500,2.3181) (29000,2.3167) (30500,2.3171) (31500,2.3172)}
\def\contrieverFTCosineNeg{(500,1.7014) (2000,1.8374) (3500,1.9537) (5000,2.0289) (6500,2.0717) (8000,2.0933) (9500,2.1427) (11000,2.1608) (12500,2.1725) (14000,2.1635) (15500,2.1764) (17000,2.1850) (18500,2.1774) (20000,2.1995) (21500,2.2025) (23000,2.2020) (24500,2.2076) (26000,2.2072) (27500,2.2090) (29000,2.2081) (30500,2.2084) (31500,2.2085)}
\def\contrieverFTDotPos{(500,1.8185) (2000,1.9607) (3500,1.9764) (5000,2.0622) (6500,2.0664) (8000,2.0847) (9500,2.0828) (11000,2.1242) (12500,2.1136) (14000,2.1592) (15500,2.1146) (17000,2.1158) (18500,2.1342) (20000,2.1435) (21500,2.1449) (23000,2.1556) (24500,2.1439) (26000,2.1509) (27500,2.1497) (29000,2.1504) (30500,2.1497) (31500,2.1496)}
\def\contrieverFTDotNeg{(500,1.7484) (2000,1.8936) (3500,1.9149) (5000,1.9959) (6500,1.9978) (8000,2.0192) (9500,2.0139) (11000,2.0474) (12500,2.0409) (14000,2.0803) (15500,2.0392) (17000,2.0422) (18500,2.0594) (20000,2.0655) (21500,2.0668) (23000,2.0780) (24500,2.0697) (26000,2.0733) (27500,2.0722) (29000,2.0733) (30500,2.0729) (31500,2.0728)}
\def\contrieverFTDNormPos{(500,2.0675) (2000,2.2503) (3500,2.3553) (5000,2.4026) (6500,2.3807) (8000,2.4299) (9500,2.4102) (11000,2.4761) (12500,2.4878) (14000,2.4801) (15500,2.4647) (17000,2.4681) (18500,2.5019) (20000,2.4793) (21500,2.4745) (23000,2.4934) (24500,2.4953) (26000,2.4830) (27500,2.4787) (29000,2.4803) (30500,2.4792) (31500,2.4791)}
\def\contrieverFTDNormNeg{(500,1.9691) (2000,2.1618) (3500,2.2667) (5000,2.3281) (6500,2.3208) (8000,2.3629) (9500,2.3443) (11000,2.4135) (12500,2.4273) (14000,2.4238) (15500,2.4080) (17000,2.4167) (18500,2.4493) (20000,2.4275) (21500,2.4253) (23000,2.4430) (24500,2.4466) (26000,2.4352) (27500,2.4317) (29000,2.4333) (30500,2.4325) (31500,2.4324)}

\def\contrieverSCCosinePos{(500,4.5378) (2000,4.5435) (3500,4.9282) (5000,5.7139) (6500,6.3862) (8000,6.8170) (9500,7.1776) (11000,7.5157) (12500,7.8308) (14000,8.1101) (15500,8.3178) (17000,8.4924) (18500,8.6453) (20000,8.7549) (21500,8.9047) (23000,8.9509) (24500,8.9878) (26000,9.0293) (27500,9.0626) (29000,9.0614) (30500,9.0667) (31500,9.0627)}
\def\contrieverSCCosineNeg{(500,5.4565) (2000,5.5116) (3500,5.9881) (5000,6.9031) (6500,7.6718) (8000,8.1695) (9500,8.5863) (11000,8.9688) (12500,9.3089) (14000,9.6206) (15500,9.8570) (17000,10.0508) (18500,10.2202) (20000,10.3433) (21500,10.4951) (23000,10.5516) (24500,10.5931) (26000,10.6413) (27500,10.6769) (29000,10.6764) (30500,10.6818) (31500,10.6781)}
\def\contrieverSCDotPos{(500,27.6395) (2000,27.5988) (3500,27.5844) (5000,27.5720) (6500,27.5669) (8000,27.5560) (9500,27.5485) (11000,27.5387) (12500,27.5308) (14000,27.5282) (15500,27.5220) (17000,27.5193) (18500,27.5155) (20000,27.5135) (21500,27.5101) (23000,27.5093) (24500,27.5057) (26000,27.5041) (27500,27.5041) (29000,27.5040) (30500,27.5042) (31500,27.5042)}
\def\contrieverSCDotNeg{(500,27.6391) (2000,27.5922) (3500,27.5784) (5000,27.5690) (6500,27.5657) (8000,27.5555) (9500,27.5484) (11000,27.5390) (12500,27.5314) (14000,27.5290) (15500,27.5231) (17000,27.5204) (18500,27.5167) (20000,27.5149) (21500,27.5115) (23000,27.5108) (24500,27.5072) (26000,27.5056) (27500,27.5057) (29000,27.5056) (30500,27.5057) (31500,27.5057)}
\def\contrieverSCDNormPos{(500,18.5326) (2000,16.2647) (3500,15.8130) (5000,15.5228) (6500,15.2034) (8000,15.4372) (9500,15.0003) (11000,14.2569) (12500,14.0845) (14000,14.4023) (15500,14.2220) (17000,14.2163) (18500,14.0972) (20000,13.9830) (21500,14.0831) (23000,14.0319) (24500,14.0023) (26000,13.8616) (27500,13.8496) (29000,13.8116) (30500,13.8475) (31500,13.8359)}
\def\contrieverSCDNormNeg{(500,18.7041) (2000,16.8430) (3500,16.5808) (5000,16.4298) (6500,16.2079) (8000,16.4650) (9500,16.0805) (11000,15.4072) (12500,15.2672) (14000,15.5836) (15500,15.4388) (17000,15.4363) (18500,15.3334) (20000,15.2310) (21500,15.3273) (23000,15.2854) (24500,15.2669) (26000,15.1369) (27500,15.1296) (29000,15.0956) (30500,15.1291) (31500,15.1184)}

\def\contrieverFTCosineCD{(0,0.75) (500,0.5526) (1500,0.5743) (2500,0.5755) (3500,0.5739) (4500,0.5777) (5500,0.5630) (6500,0.5554) (7500,0.5476) (8500,0.5637) (9500,0.5443) (10500,0.5376) (11500,0.5374) (12500,0.5384) (13500,0.5356) (14500,0.5263) (15500,0.5219) (16500,0.5256) (17500,0.5253) (18500,0.5155) (19500,0.5191) (20500,0.5191) (21500,0.5132) (22500,0.5126) (23500,0.5111) (24500,0.5115) (25500,0.5099) (26500,0.5097) (27500,0.5094) (28500,0.5083) (29500,0.5085) (30500,0.5084) (31500,0.5084)}
\def\contrieverFTDotCD{(0,0.75) (500,0.4089) (1500,0.3706) (2500,0.3677) (3500,0.3339) (4500,0.3516) (5500,0.3377) (6500,0.3335) (7500,0.3049) (8500,0.3352) (9500,0.3176) (10500,0.3482) (11500,0.3385) (12500,0.3424) (13500,0.3300) (14500,0.3402) (15500,0.3421) (16500,0.3361) (17500,0.3445) (18500,0.3313) (19500,0.3419) (20500,0.3432) (21500,0.3408) (22500,0.3312) (23500,0.3363) (24500,0.3290) (25500,0.3388) (26500,0.3381) (27500,0.3393) (28500,0.3392) (29500,0.3376) (30500,0.3378) (31500,0.3376)}
\def\contrieverFTDNormCD{(0,0.75) (500,0.5310) (1500,0.5152) (2500,0.4615) (3500,0.4382) (4500,0.3810) (5500,0.3404) (6500,0.3098) (7500,0.3308) (8500,0.3376) (9500,0.3333) (10500,0.3151) (11500,0.2933) (12500,0.2917) (13500,0.2837) (14500,0.2707) (15500,0.2828) (16500,0.2738) (17500,0.2579) (18500,0.2535) (19500,0.2483) (20500,0.2455) (21500,0.2423) (22500,0.2437) (23500,0.2445) (24500,0.2355) (25500,0.2351) (26500,0.2324) (27500,0.2312) (28500,0.2301) (29500,0.2300) (30500,0.2295) (31500,0.2295)}
\def\contrieverSCCosineCD{(500,-0.5179) (1500,-0.5109) (2500,-0.5089) (3500,-0.5045) (4500,-0.5067) (5500,-0.5024) (6500,-0.5099) (7500,-0.5129) (8500,-0.5145) (9500,-0.5149) (10500,-0.5157) (11500,-0.5162) (12500,-0.5158) (13500,-0.5218) (14500,-0.5196) (15500,-0.5196) (16500,-0.5225) (17500,-0.5248) (18500,-0.5232) (19500,-0.5240) (20500,-0.5228) (21500,-0.5243) (22500,-0.5232) (23500,-0.5247) (24500,-0.5244) (25500,-0.5255) (26500,-0.5258) (27500,-0.5264) (28500,-0.5259) (29500,-0.5257) (30500,-0.5257) (31500,-0.5256)}
\def\contrieverSCDotCD{(500,0.1620) (1500,0.1909) (2500,0.1861) (3500,0.1836) (4500,0.1568) (5500,0.1218) (6500,0.0764) (7500,0.0470) (8500,0.0241) (9500,0.0064) (10500,-0.0191) (11500,-0.0314) (12500,-0.0466) (13500,-0.0635) (14500,-0.0788) (15500,-0.0930) (16500,-0.0982) (17500,-0.0973) (18500,-0.1010) (19500,-0.1120) (20500,-0.1211) (21500,-0.1202) (22500,-0.1263) (23500,-0.1283) (24500,-0.1254) (25500,-0.1302) (26500,-0.1309) (27500,-0.1334) (28500,-0.1333) (29500,-0.1324) (30500,-0.1321) (31500,-0.1324)}
\def\contrieverSCDNormCD{(500,-0.3320) (1500,-0.4415) (2500,-0.4811) (3500,-0.5174) (4500,-0.5398) (5500,-0.5503) (6500,-0.5611) (7500,-0.5667) (8500,-0.5704) (9500,-0.5708) (10500,-0.5734) (11500,-0.5723) (12500,-0.5739) (13500,-0.5744) (14500,-0.5724) (15500,-0.5733) (16500,-0.5733) (17500,-0.5706) (18500,-0.5713) (19500,-0.5702) (20500,-0.5699) (21500,-0.5700) (22500,-0.5704) (23500,-0.5704) (24500,-0.5717) (25500,-0.5705) (26500,-0.5715) (27500,-0.5714) (28500,-0.5714) (29500,-0.5712) (30500,-0.5713) (31500,-0.5713)}

\def\contrieverFTQNormPos{(500,1.6306) (2000,1.6839) (3500,1.8127) (5000,1.8944) (6500,1.8899) (8000,1.9115) (9500,1.8885) (11000,1.8869) (12500,1.9076) (14000,1.9361) (15500,1.9227) (17000,1.9427) (18500,1.9169) (20000,1.9207) (21500,1.9420) (23000,1.9406) (24500,1.9377) (26000,1.9394) (27500,1.9417) (29000,1.9472) (30500,1.9471) (31500,1.9472)}
\def\contrieverFTQNormNeg{(500,1.5755) (2000,1.6335) (3500,1.7444) (5000,1.8189) (6500,1.8154) (8000,1.8322) (9500,1.8179) (11000,1.8216) (12500,1.8368) (14000,1.8618) (15500,1.8516) (17000,1.8691) (18500,1.8487) (20000,1.8550) (21500,1.8738) (23000,1.8705) (24500,1.8685) (26000,1.8703) (27500,1.8721) (29000,1.8764) (30500,1.8763) (31500,1.8764)}
\def\contrieverFTQNormCD{(0,0.75) (500,0.3784) (1500,0.3335) (2500,0.3284) (3500,0.3759) (4500,0.3643) (5500,0.3674) (6500,0.3731) (7500,0.3938) (8500,0.3721) (9500,0.3699) (10500,0.3666) (11500,0.3832) (12500,0.3624) (13500,0.3429) (14500,0.3523) (15500,0.3406) (16500,0.3487) (17500,0.3377) (18500,0.3298) (19500,0.3252) (20500,0.3239) (21500,0.3259) (22500,0.3425) (23500,0.3221) (24500,0.3279) (25500,0.3288) (26500,0.3266) (27500,0.3291) (28500,0.3307) (29500,0.3307) (30500,0.3307) (31500,0.3307)}
\def\contrieverSCQNormPos{(500,21.0094) (2000,19.0157) (3500,18.7247) (5000,18.7531) (6500,18.7582) (8000,18.9194) (9500,19.0162) (11000,18.9421) (12500,18.9537) (14000,19.1051) (15500,19.1127) (17000,19.1234) (18500,19.1541) (20000,19.1770) (21500,19.2283) (23000,19.2021) (24500,19.1835) (26000,19.2055) (27500,19.2152) (29000,19.1920) (30500,19.1986) (31500,19.1975)}
\def\contrieverSCQNormNeg{(500,20.9501) (2000,18.9796) (3500,18.7199) (5000,18.7956) (6500,18.8420) (8000,19.0209) (9500,19.1279) (11000,19.0723) (12500,19.0849) (14000,19.2386) (15500,19.2539) (17000,19.2669) (18500,19.2961) (20000,19.3207) (21500,19.3716) (23000,19.3487) (24500,19.3316) (26000,19.3543) (27500,19.3637) (29000,19.3412) (30500,19.3473) (31500,19.3464)}
\def\contrieverSCQNormCD{(500,0.1380) (1500,0.1103) (2500,0.0631) (3500,0.0132) (4500,-0.0657) (5500,-0.1899) (6500,-0.3082) (7500,-0.3269) (8500,-0.3870) (9500,-0.4099) (10500,-0.4374) (11500,-0.4540) (12500,-0.4482) (13500,-0.4491) (14500,-0.4607) (15500,-0.4778) (16500,-0.4826) (17500,-0.4686) (18500,-0.4736) (19500,-0.4859) (20500,-0.4801) (21500,-0.4777) (22500,-0.4806) (23500,-0.4837) (24500,-0.4882) (25500,-0.4902) (26500,-0.4849) (27500,-0.4868) (28500,-0.4862) (29500,-0.4859) (30500,-0.4861) (31500,-0.4862)}

\def\contrieverFTLnPos{(500,2.4109) (2000,2.7850) (3500,3.0667) (5000,3.2386) (6500,3.2507) (8000,3.2854) (9500,3.2992) (11000,3.3432) (12500,3.3387) (14000,3.3921) (15500,3.3941) (17000,3.4641) (18500,3.4338) (20000,3.4356) (21500,3.4668) (23000,3.4420) (24500,3.4509) (26000,3.4525) (27500,3.4597) (29000,3.4643) (30500,3.4651) (31500,3.4653)}
\def\contrieverFTLnNeg{(500,2.2520) (2000,2.6227) (3500,2.8896) (5000,3.0692) (6500,3.0714) (8000,3.1086) (9500,3.1248) (11000,3.1782) (12500,3.1794) (14000,3.2244) (15500,3.2295) (17000,3.2954) (18500,3.2663) (20000,3.2763) (21500,3.3019) (23000,3.2763) (24500,3.2851) (26000,3.2892) (27500,3.2950) (29000,3.2991) (30500,3.2995) (31500,3.2997)}
\def\contrieverFTLnCD{(0,0.75) (500,0.4886) (1500,0.4442) (2500,0.4111) (3500,0.4111) (4500,0.3869) (5500,0.3771) (6500,0.3806) (7500,0.3812) (8500,0.3607) (9500,0.3704) (10500,0.3619) (11500,0.3627) (12500,0.3386) (13500,0.3341) (14500,0.3374) (15500,0.3262) (16500,0.3225) (17500,0.3202) (18500,0.3255) (19500,0.3197) (20500,0.3175) (21500,0.3235) (22500,0.3303) (23500,0.3220) (24500,0.3250) (25500,0.3253) (26500,0.3220) (27500,0.3236) (28500,0.3238) (29500,0.3237) (30500,0.3239) (31500,0.3239)}
\def\contrieverSCLnPos{(500,18.9565) (2000,16.4910) (3500,16.1228) (5000,16.1492) (6500,15.9859) (8000,16.0475) (9500,15.9933) (11000,15.7893) (12500,15.6504) (14000,15.8129) (15500,15.7236) (17000,15.6705) (18500,15.7179) (20000,15.7359) (21500,15.7243) (23000,15.6837) (24500,15.6944) (26000,15.6751) (27500,15.6821) (29000,15.6691) (30500,15.6678) (31500,15.6658)}
\def\contrieverSCLnNeg{(500,19.0178) (2000,16.6448) (3500,16.3272) (5000,16.3988) (6500,16.2754) (8000,16.3540) (9500,16.3088) (11000,16.1267) (12500,15.9954) (14000,16.1606) (15500,16.0823) (17000,16.0315) (18500,16.0795) (20000,16.0978) (21500,16.0885) (23000,16.0517) (24500,16.0631) (26000,16.0475) (27500,16.0548) (29000,16.0427) (30500,16.0410) (31500,16.0391)}
\def\contrieverSCLnCD{(500,-0.2130) (1500,-0.4146) (2500,-0.4858) (3500,-0.5108) (4500,-0.5130) (5500,-0.5234) (6500,-0.5273) (7500,-0.5297) (8500,-0.5198) (9500,-0.5223) (10500,-0.5250) (11500,-0.5234) (12500,-0.5142) (13500,-0.5171) (14500,-0.5150) (15500,-0.5224) (16500,-0.5179) (17500,-0.5124) (18500,-0.5163) (19500,-0.5185) (20500,-0.5109) (21500,-0.5142) (22500,-0.5146) (23500,-0.5156) (24500,-0.5152) (25500,-0.5179) (26500,-0.5179) (27500,-0.5171) (28500,-0.5169) (29500,-0.5164) (30500,-0.5167) (31500,-0.5167)}

\definecolor{colorCos}{RGB}{51,102,204}     
\definecolor{colorDot}{RGB}{204,51,51}      
\definecolor{colorQN}{RGB}{0,128,128}       
\definecolor{colorDN}{RGB}{255,165,0}       
\definecolor{colorLn}{RGB}{255,0,255}       
\definecolor{shadingGreen}{RGB}{200,230,200}
\definecolor{shadingRed}{RGB}{255,220,220}

\pgfplotsset{
    magnitude axis/.style={
        width=0.30\textwidth,
        height=0.24\textwidth,
        xmin=0, xmax=32000,
        xtick={0,10000,20000,30000},
        xticklabels={0,10k,20k,30k},
        scaled x ticks=false,
        grid=both,
        grid style={line width=0.1pt, draw=gray!20},
        tick label style={font=\tiny},
        label style={font=\tiny},
        title style={font=\scriptsize\bfseries, yshift=-0.5ex},
        legend style={font=\tiny, column sep=3pt, draw=none},
        legend columns=-1,
        clip=false,
    },
    every axis plot/.append style={line width=0.6pt, mark=none},
}

\begin{figure*}[t]
\centering
\caption{Contriever: Document magnitude dynamics during training. \textbf{Left}: Mean magnitude of relevant (positive) documents. \textbf{Center}: Mean magnitude of irrelevant (negative) documents. \textbf{Right}: Cohen's $d$ (effect size). \textbf{Top row}: Finetuning maintains $\mu_{\text{pos}} > \mu_{\text{neg}}$ throughout training. \textbf{Bottom row}: Random initialization shows $\mu_{\text{pos}} < \mu_{\text{neg}}$, explaining negative Cohen's $d$.}
\label{fig:contriever_dynamics_appendix}

\begin{tikzpicture}
\begin{groupplot}[
    group style={
        group size=3 by 2,
        horizontal sep=1.3cm,
        vertical sep=0.9cm,
    },
    magnitude axis,
]

\nextgroupplot[title={Pos Mean (Finetuned)}, ylabel={$\mu_{\text{pos}}$}, ymin=1.5, ymax=3.8, ytick={1.5, 2.0, 2.5, 3.0, 3.5}, xticklabels={}]
\addplot[colorCos] coordinates {\contrieverFTCosinePos};
\addplot[colorDot] coordinates {\contrieverFTDotPos};
\addplot[colorQN] coordinates {\contrieverFTQNormPos};
\addplot[colorDN] coordinates {\contrieverFTDNormPos};
\addplot[colorLn] coordinates {\contrieverFTLnPos};

\nextgroupplot[title={Neg Mean (Finetuned)}, ylabel={$\mu_{\text{neg}}$}, ymin=1.5, ymax=3.8, ytick={1.5, 2.0, 2.5, 3.0, 3.5}, xticklabels={}]
\addplot[colorCos] coordinates {\contrieverFTCosineNeg};
\addplot[colorDot] coordinates {\contrieverFTDotNeg};
\addplot[colorQN] coordinates {\contrieverFTQNormNeg};
\addplot[colorDN] coordinates {\contrieverFTDNormNeg};
\addplot[colorLn] coordinates {\contrieverFTLnNeg};

\nextgroupplot[title={Cohen's $d$ (Finetuned)}, ylabel={Cohen's $d$}, ymin=0, ymax=0.8, xticklabels={}]
\fill[shadingGreen] (axis cs:0,0) rectangle (axis cs:32000,0.8);
\addplot[colorCos] coordinates {\contrieverFTCosineCD};
\addplot[colorDot] coordinates {\contrieverFTDotCD};
\addplot[colorQN] coordinates {\contrieverFTQNormCD};
\addplot[colorDN] coordinates {\contrieverFTDNormCD};
\addplot[colorLn] coordinates {\contrieverFTLnCD};
\draw[dashed, black, thin] (axis cs:0,0) -- (axis cs:32000,0);
\draw[black, thick, fill=white] (axis cs:0,0.75) circle (2pt);
\node[fill=yellow, font=\tiny, inner sep=1pt, anchor=west] (ptlabel1) at (axis cs:4000,0.70) {PT: +0.75};
\draw[->, thick, black] (ptlabel1.west) -- (axis cs:0,0.75);
\node[font=\tiny, green!50!black, anchor=east] at (axis cs:31000,0.65) {$\|\text{Rel}\|{>}\|\text{Irr}\|$};

\nextgroupplot[title={Pos Mean (Random Init)}, ylabel={$\mu_{\text{pos}}$}, xlabel={Steps}, ymin=4, ymax=30, ytick={5, 10, 15, 20, 25}]
\addplot[colorCos] coordinates {\contrieverSCCosinePos};
\addplot[colorDot] coordinates {\contrieverSCDotPos};
\addplot[colorQN] coordinates {\contrieverSCQNormPos};
\addplot[colorDN] coordinates {\contrieverSCDNormPos};
\addplot[colorLn] coordinates {\contrieverSCLnPos};

\nextgroupplot[title={Neg Mean (Random Init)}, ylabel={$\mu_{\text{neg}}$}, xlabel={Steps}, ymin=4, ymax=30, ytick={5, 10, 15, 20, 25}]
\addplot[colorCos] coordinates {\contrieverSCCosineNeg};
\addplot[colorDot] coordinates {\contrieverSCDotNeg};
\addplot[colorQN] coordinates {\contrieverSCQNormNeg};
\addplot[colorDN] coordinates {\contrieverSCDNormNeg};
\addplot[colorLn] coordinates {\contrieverSCLnNeg};

\nextgroupplot[title={Cohen's $d$ (Random Init)}, ylabel={Cohen's $d$}, xlabel={Steps}, ymin=-0.8, ymax=0.2, legend to name=contrieverlegend]
\fill[shadingGreen] (axis cs:0,0) rectangle (axis cs:32000,0.2);
\fill[shadingRed] (axis cs:0,-0.8) rectangle (axis cs:32000,0);
\addplot[colorCos] coordinates {\contrieverSCCosineCD}; \addlegendentry{Cos}
\addplot[colorDot] coordinates {\contrieverSCDotCD}; \addlegendentry{Dot}
\addplot[colorQN] coordinates {\contrieverSCQNormCD}; \addlegendentry{QN}
\addplot[colorDN] coordinates {\contrieverSCDNormCD}; \addlegendentry{DN}
\addplot[colorLn] coordinates {\contrieverSCLnCD}; \addlegendentry{Ln}
\draw[dashed, black, thin] (axis cs:0,0) -- (axis cs:32000,0);
\node[font=\tiny, red!70!black, anchor=east] at (axis cs:31000,-0.7) {$\|\text{Irr}\|{>}\|\text{Rel}\|$};

\end{groupplot}
\node[draw=none, fill=none, inner sep=0pt] at ($(group c1r2.south)!0.5!(group c3r2.south) + (0,-1.1cm)$) {\pgfplotslegendfromname{contrieverlegend}};
\end{tikzpicture}
\end{figure*}

\def\retromaeFTCosinePos{(500,12.0483) (2000,12.8165) (3500,13.1190) (5000,13.3429) (6500,13.2862) (8000,13.5617) (9500,13.4682) (11000,13.6043) (12500,13.7114) (14000,13.6046) (15500,13.8001) (17000,13.8082) (18500,13.8292) (20000,13.8734) (21500,13.9093) (23000,13.8486) (24500,13.8683) (26000,13.8546) (27500,13.8769) (29000,13.8959) (30500,13.8934) (31500,13.8942)}
\def\retromaeFTCosineNeg{(500,12.2757) (2000,13.0602) (3500,13.3504) (5000,13.5610) (6500,13.5056) (8000,13.7860) (9500,13.6962) (11000,13.8090) (12500,13.9096) (14000,13.7990) (15500,13.9662) (17000,13.9817) (18500,13.9965) (20000,14.0469) (21500,14.0704) (23000,14.0162) (24500,14.0387) (26000,14.0264) (27500,14.0436) (29000,14.0638) (30500,14.0608) (31500,14.0617)}
\def\retromaeFTDotPos{(500,7.0608) (2000,6.8174) (3500,6.8328) (5000,6.7720) (6500,6.8052) (8000,6.7979) (9500,6.7698) (11000,6.7597) (12500,6.7610) (14000,6.7602) (15500,6.8033) (17000,6.7915) (18500,6.7823) (20000,6.7721) (21500,6.7544) (23000,6.7720) (24500,6.7634) (26000,6.7605) (27500,6.7600) (29000,6.7572) (30500,6.7589) (31500,6.7592)}
\def\retromaeFTDotNeg{(500,7.0652) (2000,6.8158) (3500,6.8355) (5000,6.7794) (6500,6.8092) (8000,6.8031) (9500,6.7741) (11000,6.7619) (12500,6.7667) (14000,6.7649) (15500,6.8074) (17000,6.7931) (18500,6.7866) (20000,6.7760) (21500,6.7602) (23000,6.7772) (24500,6.7687) (26000,6.7656) (27500,6.7655) (29000,6.7625) (30500,6.7641) (31500,6.7644)}
\def\retromaeFTDNormPos{(500,8.1865) (2000,8.8273) (3500,9.0781) (5000,9.2083) (6500,9.2405) (8000,9.2860) (9500,9.4695) (11000,9.5360) (12500,9.5414) (14000,9.6263) (15500,9.7498) (17000,9.6130) (18500,9.7195) (20000,9.7063) (21500,9.7677) (23000,9.7719) (24500,9.7273) (26000,9.7273) (27500,9.7354) (29000,9.7457) (30500,9.7462) (31500,9.7462)}
\def\retromaeFTDNormNeg{(500,8.2174) (2000,8.8480) (3500,9.1065) (5000,9.2199) (6500,9.2740) (8000,9.3337) (9500,9.4908) (11000,9.5506) (12500,9.5811) (14000,9.6653) (15500,9.7712) (17000,9.6374) (18500,9.7390) (20000,9.7307) (21500,9.7929) (23000,9.8008) (24500,9.7525) (26000,9.7579) (27500,9.7637) (29000,9.7707) (30500,9.7711) (31500,9.7711)}

\def\retromaeSCCosinePos{(500,4.5180) (2000,5.5320) (3500,6.9951) (5000,8.1488) (6500,8.9749) (8000,9.6625) (9500,10.2175) (11000,10.7730) (12500,11.1747) (14000,11.5304) (15500,11.8617) (17000,11.9585) (18500,12.2415) (20000,12.1893) (21500,12.3492) (23000,12.4363) (24500,12.5058) (26000,12.5121) (27500,12.5485) (29000,12.5280) (30500,12.5181) (31500,12.5236)}
\def\retromaeSCCosineNeg{(500,5.4330) (2000,6.6708) (3500,8.3147) (5000,9.5834) (6500,10.5075) (8000,11.2371) (9500,11.8223) (11000,12.4121) (12500,12.8253) (14000,13.1796) (15500,13.5246) (17000,13.6440) (18500,13.9363) (20000,13.9002) (21500,14.0527) (23000,14.1525) (24500,14.2241) (26000,14.2328) (27500,14.2632) (29000,14.2474) (30500,14.2388) (31500,14.2445)}
\def\retromaeSCDotPos{(500,27.6169) (2000,27.5835) (3500,27.5682) (5000,27.5443) (6500,27.5343) (8000,27.5267) (9500,27.5178) (11000,27.5212) (12500,27.5160) (14000,27.5224) (15500,27.5124) (17000,27.5252) (18500,27.5161) (20000,27.5257) (21500,27.5219) (23000,27.5187) (24500,27.5180) (26000,27.5166) (27500,27.5181) (29000,27.5165) (30500,27.5174) (31500,27.5173)}
\def\retromaeSCDotNeg{(500,27.6153) (2000,27.5822) (3500,27.5692) (5000,27.5462) (6500,27.5364) (8000,27.5290) (9500,27.5203) (11000,27.5239) (12500,27.5190) (14000,27.5252) (15500,27.5154) (17000,27.5280) (18500,27.5190) (20000,27.5285) (21500,27.5248) (23000,27.5217) (24500,27.5210) (26000,27.5197) (27500,27.5211) (29000,27.5196) (30500,27.5204) (31500,27.5204)}
\def\retromaeSCDNormPos{(500,17.6733) (2000,16.2548) (3500,16.1676) (5000,15.8557) (6500,15.6886) (8000,16.1947) (9500,15.8998) (11000,14.9013) (12500,15.5251) (14000,15.6189) (15500,15.7670) (17000,15.6782) (18500,15.5028) (20000,15.6370) (21500,15.3615) (23000,15.1500) (24500,15.2451) (26000,15.1948) (27500,14.9361) (29000,14.8551) (30500,14.9134) (31500,14.8872)}
\def\retromaeSCDNormNeg{(500,17.8987) (2000,16.9925) (3500,17.0063) (5000,16.7622) (6500,16.5783) (8000,17.0760) (9500,16.8066) (11000,15.8564) (12500,16.4484) (14000,16.5568) (15500,16.7281) (17000,16.6490) (18500,16.4852) (20000,16.6287) (21500,16.3643) (23000,16.1835) (24500,16.2969) (26000,16.2690) (27500,16.0293) (29000,15.9551) (30500,16.0135) (31500,15.9892)}

\def\retromaeFTCosineCD{(0,-0.60) (500,-0.2817) (1500,-0.2984) (2500,-0.3519) (3500,-0.3276) (4500,-0.3259) (5500,-0.3139) (6500,-0.3197) (7500,-0.3377) (8500,-0.3409) (9500,-0.3398) (10500,-0.3580) (11500,-0.3209) (12500,-0.3243) (13500,-0.3022) (14500,-0.3151) (15500,-0.2703) (16500,-0.2872) (17500,-0.2679) (18500,-0.2718) (19500,-0.2797) (20500,-0.2872) (21500,-0.2711) (22500,-0.2692) (23500,-0.2727) (24500,-0.2812) (25500,-0.2876) (26500,-0.2819) (27500,-0.2751) (28500,-0.2804) (29500,-0.2782) (30500,-0.2780) (31500,-0.2780)}
\def\retromaeFTDotCD{(0,-0.60) (500,-0.0368) (1500,0.0024) (2500,-0.0898) (3500,-0.0621) (4500,-0.1188) (5500,-0.1311) (6500,-0.1046) (7500,-0.1319) (8500,-0.2095) (9500,-0.1243) (10500,-0.0913) (11500,-0.1134) (12500,-0.1633) (13500,-0.1218) (14500,-0.0622) (15500,-0.0933) (16500,-0.0651) (17500,-0.1337) (18500,-0.1165) (19500,-0.0974) (20500,-0.0521) (21500,-0.1498) (22500,-0.1680) (23500,-0.1208) (24500,-0.1327) (25500,-0.1362) (26500,-0.1392) (27500,-0.1419) (28500,-0.1397) (29500,-0.1331) (30500,-0.1313) (31500,-0.1310)}
\def\retromaeFTDNormCD{(0,-0.60) (500,-0.1386) (1500,-0.1018) (2500,-0.1251) (3500,-0.1117) (4500,-0.0338) (5500,-0.1119) (6500,-0.1251) (7500,-0.1393) (8500,-0.1696) (9500,-0.0826) (10500,-0.1088) (11500,-0.1572) (12500,-0.1547) (13500,-0.1364) (14500,-0.1375) (15500,-0.0772) (16500,-0.0578) (17500,-0.0332) (18500,-0.0737) (19500,-0.0777) (20500,-0.1199) (21500,-0.0948) (22500,-0.0937) (23500,-0.0942) (24500,-0.0923) (25500,-0.1178) (26500,-0.1038) (27500,-0.1030) (28500,-0.0948) (29500,-0.0919) (30500,-0.0906) (31500,-0.0907)}
\def\retromaeSCCosineCD{(500,-0.5008) (1500,-0.4954) (2500,-0.5018) (3500,-0.4947) (4500,-0.4996) (5500,-0.5068) (6500,-0.5128) (7500,-0.5056) (8500,-0.5134) (9500,-0.5136) (10500,-0.5121) (11500,-0.5104) (12500,-0.5113) (13500,-0.5170) (14500,-0.5144) (15500,-0.5116) (16500,-0.5169) (17500,-0.5216) (18500,-0.5215) (19500,-0.5230) (20500,-0.5189) (21500,-0.5186) (22500,-0.5196) (23500,-0.5190) (24500,-0.5210) (25500,-0.5190) (26500,-0.5204) (27500,-0.5213) (28500,-0.5386) (29500,-0.5209) (30500,-0.5207) (31500,-0.5211)}
\def\retromaeSCDotCD{(500,0.1647) (1500,0.1492) (2500,0.0072) (3500,-0.1159) (4500,-0.2026) (5500,-0.2493) (6500,-0.2700) (7500,-0.3199) (8500,-0.3327) (9500,-0.3313) (10500,-0.3600) (11500,-0.3534) (12500,-0.3847) (13500,-0.3736) (14500,-0.3696) (15500,-0.3850) (16500,-0.3836) (17500,-0.3760) (18500,-0.3897) (19500,-0.3967) (20500,-0.3894) (21500,-0.3970) (22500,-0.3993) (23500,-0.4064) (24500,-0.4135) (25500,-0.4090) (26500,-0.4081) (27500,-0.4059) (28500,-0.4075) (29500,-0.4072) (30500,-0.4072) (31500,-0.4069)}
\def\retromaeSCDNormCD{(500,-0.2848) (1500,-0.4796) (2500,-0.5275) (3500,-0.5412) (4500,-0.5486) (5500,-0.5465) (6500,-0.5342) (7500,-0.5334) (8500,-0.5349) (9500,-0.5284) (10500,-0.5234) (11500,-0.5244) (12500,-0.5255) (13500,-0.5277) (14500,-0.5287) (15500,-0.5250) (16500,-0.5287) (17500,-0.5257) (18500,-0.5254) (19500,-0.5212) (20500,-0.5169) (21500,-0.5191) (22500,-0.5179) (23500,-0.5200) (24500,-0.5195) (25500,-0.5196) (26500,-0.5204) (27500,-0.5202) (28500,-0.5189) (29500,-0.5194) (30500,-0.5194) (31500,-0.5192)}

\def\retromaeFTQNormPos{(500,8.0599) (2000,8.4594) (3500,8.6991) (5000,8.6958) (6500,8.8142) (8000,8.8855) (9500,8.9149) (11000,8.8603) (12500,8.9677) (14000,8.9420) (15500,9.0596) (17000,9.0463) (18500,9.1005) (20000,9.0844) (21500,9.0544) (23000,9.1179) (24500,9.0913) (26000,9.0791) (27500,9.0832) (29000,9.0674) (30500,9.0724) (31500,9.0721)}
\def\retromaeFTQNormNeg{(500,8.1322) (2000,8.5440) (3500,8.7708) (5000,8.7563) (6500,8.8778) (8000,8.9475) (9500,8.9644) (11000,8.9056) (12500,9.0259) (14000,8.9930) (15500,9.0971) (17000,9.0764) (18500,9.1367) (20000,9.1239) (21500,9.0949) (23000,9.1543) (24500,9.1289) (26000,9.1181) (27500,9.1217) (29000,9.1056) (30500,9.1102) (31500,9.1099)}
\def\retromaeFTQNormCD{(0,-0.60) (500,-0.3114) (1500,-0.3427) (2500,-0.4561) (3500,-0.2990) (4500,-0.2193) (5500,-0.2078) (6500,-0.2627) (7500,-0.2799) (8500,-0.2542) (9500,-0.2139) (10500,-0.2949) (11500,-0.1880) (12500,-0.2680) (13500,-0.2040) (14500,-0.2001) (15500,-0.1394) (16500,-0.1367) (17500,-0.1250) (18500,-0.1484) (19500,-0.1917) (20500,-0.1440) (21500,-0.1645) (22500,-0.1750) (23500,-0.1640) (24500,-0.1483) (25500,-0.1695) (26500,-0.1538) (27500,-0.1516) (28500,-0.1525) (29500,-0.1483) (30500,-0.1478) (31500,-0.1478)}
\def\retromaeSCQNormPos{(500,20.9314) (2000,20.2689) (3500,20.7271) (5000,20.8811) (6500,21.2932) (8000,21.6540) (9500,21.7655) (11000,22.0215) (12500,22.2244) (14000,22.4551) (15500,22.4836) (17000,22.6617) (18500,22.6944) (20000,22.7008) (21500,22.7157) (23000,22.7454) (24500,22.7733) (26000,22.7396) (27500,22.7344) (29000,22.6936) (30500,22.7005) (31500,22.6991)}
\def\retromaeSCQNormNeg{(500,20.8579) (2000,20.2844) (3500,20.8021) (5000,20.9817) (6500,21.3972) (8000,21.7473) (9500,21.8715) (11000,22.1295) (12500,22.3261) (14000,22.5542) (15500,22.5860) (17000,22.7599) (18500,22.7893) (20000,22.7980) (21500,22.8133) (23000,22.8433) (24500,22.8699) (26000,22.8388) (27500,22.8318) (29000,22.7941) (30500,22.7999) (31500,22.7986)}
\def\retromaeSCQNormCD{(500,0.1418) (1500,0.0670) (2500,-0.2115) (3500,-0.3352) (4500,-0.4026) (5500,-0.4366) (6500,-0.4281) (7500,-0.4384) (8500,-0.4160) (9500,-0.4406) (10500,-0.4450) (11500,-0.4208) (12500,-0.4452) (13500,-0.4138) (14500,-0.4307) (15500,-0.4391) (16500,-0.4318) (17500,-0.4558) (18500,-0.4395) (19500,-0.4440) (20500,-0.4430) (21500,-0.4454) (22500,-0.4478) (23500,-0.4362) (24500,-0.4429) (25500,-0.4463) (26500,-0.4415) (27500,-0.4445) (28500,-0.4510) (29500,-0.4499) (30500,-0.4497) (31500,-0.4498)}

\def\retromaeFTLnPos{(500,8.0311) (2000,8.4324) (3500,8.9757) (5000,9.0501) (6500,9.2089) (8000,9.1526) (9500,9.0753) (11000,9.2560) (12500,9.2001) (14000,9.2194) (15500,9.3467) (17000,9.3847) (18500,9.3041) (20000,9.3385) (21500,9.2883) (23000,9.3230) (24500,9.3416) (26000,9.3497) (27500,9.3568) (29000,9.3495) (30500,9.3477) (31500,9.3470)}
\def\retromaeFTLnNeg{(500,8.1214) (2000,8.5403) (3500,9.0742) (5000,9.1297) (6500,9.2650) (8000,9.2298) (9500,9.1438) (11000,9.3273) (12500,9.2742) (14000,9.2879) (15500,9.4082) (17000,9.4540) (18500,9.3759) (20000,9.4094) (21500,9.3618) (23000,9.3907) (24500,9.4078) (26000,9.4181) (27500,9.4242) (29000,9.4151) (30500,9.4133) (31500,9.4125)}
\def\retromaeFTLnCD{(0,-0.60) (500,-0.3455) (1500,-0.3436) (2500,-0.4107) (3500,-0.2968) (4500,-0.2579) (5500,-0.1898) (6500,-0.1730) (7500,-0.2255) (8500,-0.2249) (9500,-0.2453) (10500,-0.2301) (11500,-0.1999) (12500,-0.2603) (13500,-0.2089) (14500,-0.1608) (15500,-0.1937) (16500,-0.2375) (17500,-0.1904) (18500,-0.2176) (19500,-0.2239) (20500,-0.2015) (21500,-0.2304) (22500,-0.2026) (23500,-0.2139) (24500,-0.2022) (25500,-0.1985) (26500,-0.2111) (27500,-0.2052) (28500,-0.2014) (29500,-0.2018) (30500,-0.2010) (31500,-0.2009)}
\def\retromaeSCLnPos{(500,17.7678) (2000,16.9697) (3500,17.0816) (5000,16.9592) (6500,17.2124) (8000,17.7064) (9500,17.8450) (11000,18.1136) (12500,18.1173) (14000,18.2499) (15500,18.4015) (17000,18.5169) (18500,18.5944) (20000,18.4109) (21500,18.4591) (23000,18.4741) (24500,18.4984) (26000,18.4494) (27500,18.3583) (29000,18.3667) (30500,18.3783) (31500,18.3770)}
\def\retromaeSCLnNeg{(500,17.8584) (2000,17.1856) (3500,17.3499) (5000,17.2561) (6500,17.5112) (8000,17.9968) (9500,18.1409) (11000,18.4042) (12500,18.4072) (14000,18.5462) (15500,18.6951) (17000,18.8068) (18500,18.8742) (20000,18.7033) (21500,18.7609) (23000,18.7673) (24500,18.7986) (26000,18.7460) (27500,18.6596) (29000,18.6673) (30500,18.6788) (31500,18.6778)}
\def\retromaeSCLnCD{(500,-0.3080) (1500,-0.4558) (2500,-0.4828) (3500,-0.4893) (4500,-0.4757) (5500,-0.4666) (6500,-0.4831) (7500,-0.4862) (8500,-0.4817) (9500,-0.4952) (10500,-0.4892) (11500,-0.4944) (12500,-0.4924) (13500,-0.4876) (14500,-0.5072) (15500,-0.4984) (16500,-0.4891) (17500,-0.4879) (18500,-0.4905) (19500,-0.4972) (20500,-0.4999) (21500,-0.5079) (22500,-0.4996) (23500,-0.5036) (24500,-0.5020) (25500,-0.4932) (26500,-0.4943) (27500,-0.4978) (28500,-0.4975) (29500,-0.4977) (30500,-0.4980) (31500,-0.4981)}

\begin{figure*}[t]
\centering
\caption{RetroMAE: Document magnitude dynamics during training. Unlike Contriever, RetroMAE shows $\mu_{\text{pos}} < \mu_{\text{neg}}$ even after finetuning, resulting in negative Cohen's $d$. This explains why RetroMAE benefits from DNorm (query magnitude) rather than QNorm (document magnitude).}
\label{fig:retromae_dynamics_appendix}

\begin{tikzpicture}
\begin{groupplot}[
    group style={
        group size=3 by 2,
        horizontal sep=1.3cm,
        vertical sep=0.9cm,
    },
    magnitude axis,
]

\nextgroupplot[title={Pos Mean (Finetuned)}, ylabel={$\mu_{\text{pos}}$}, ymin=6, ymax=15, ytick={6, 8, 10, 12, 14}, xticklabels={}]
\addplot[colorCos] coordinates {\retromaeFTCosinePos};
\addplot[colorDot] coordinates {\retromaeFTDotPos};
\addplot[colorQN] coordinates {\retromaeFTQNormPos};
\addplot[colorDN] coordinates {\retromaeFTDNormPos};
\addplot[colorLn] coordinates {\retromaeFTLnPos};

\nextgroupplot[title={Neg Mean (Finetuned)}, ylabel={$\mu_{\text{neg}}$}, ymin=6, ymax=15, ytick={6, 8, 10, 12, 14}, xticklabels={}]
\addplot[colorCos] coordinates {\retromaeFTCosineNeg};
\addplot[colorDot] coordinates {\retromaeFTDotNeg};
\addplot[colorQN] coordinates {\retromaeFTQNormNeg};
\addplot[colorDN] coordinates {\retromaeFTDNormNeg};
\addplot[colorLn] coordinates {\retromaeFTLnNeg};

\nextgroupplot[title={Cohen's $d$ (Finetuned)}, ylabel={Cohen's $d$}, ymin=-0.6, ymax=0.1, xticklabels={}]
\fill[shadingGreen] (axis cs:0,0) rectangle (axis cs:32000,0.1);
\fill[shadingRed] (axis cs:0,-0.6) rectangle (axis cs:32000,0);
\addplot[colorCos] coordinates {\retromaeFTCosineCD};
\addplot[colorDot] coordinates {\retromaeFTDotCD};
\addplot[colorQN] coordinates {\retromaeFTQNormCD};
\addplot[colorDN] coordinates {\retromaeFTDNormCD};
\addplot[colorLn] coordinates {\retromaeFTLnCD};
\draw[dashed, black, thin] (axis cs:0,0) -- (axis cs:32000,0);
\draw[black, thick, fill=white] (axis cs:0,-0.60) circle (2pt);
\node[fill=yellow, font=\tiny, inner sep=1pt, anchor=west] (ptlabel2) at (axis cs:4000,-0.55) {PT: $-$0.60};
\draw[->, thick, black] (ptlabel2.west) -- (axis cs:0,-0.60);
\node[font=\tiny, red!70!black, anchor=east] at (axis cs:31000,-0.5) {$\|\text{Irr}\|{>}\|\text{Rel}\|$};

\nextgroupplot[title={Pos Mean (Random Init)}, ylabel={$\mu_{\text{pos}}$}, xlabel={Steps}, ymin=4, ymax=30, ytick={5, 10, 15, 20, 25}]
\addplot[colorCos] coordinates {\retromaeSCCosinePos};
\addplot[colorDot] coordinates {\retromaeSCDotPos};
\addplot[colorQN] coordinates {\retromaeSCQNormPos};
\addplot[colorDN] coordinates {\retromaeSCDNormPos};
\addplot[colorLn] coordinates {\retromaeSCLnPos};

\nextgroupplot[title={Neg Mean (Random Init)}, ylabel={$\mu_{\text{neg}}$}, xlabel={Steps}, ymin=4, ymax=30, ytick={5, 10, 15, 20, 25}]
\addplot[colorCos] coordinates {\retromaeSCCosineNeg};
\addplot[colorDot] coordinates {\retromaeSCDotNeg};
\addplot[colorQN] coordinates {\retromaeSCQNormNeg};
\addplot[colorDN] coordinates {\retromaeSCDNormNeg};
\addplot[colorLn] coordinates {\retromaeSCLnNeg};

\nextgroupplot[title={Cohen's $d$ (Random Init)}, ylabel={Cohen's $d$}, xlabel={Steps}, ymin=-0.8, ymax=0.2, legend to name=retromaelegend]
\fill[shadingGreen] (axis cs:0,0) rectangle (axis cs:32000,0.2);
\fill[shadingRed] (axis cs:0,-0.8) rectangle (axis cs:32000,0);
\addplot[colorCos] coordinates {\retromaeSCCosineCD}; \addlegendentry{Cos}
\addplot[colorDot] coordinates {\retromaeSCDotCD}; \addlegendentry{Dot}
\addplot[colorQN] coordinates {\retromaeSCQNormCD}; \addlegendentry{QN}
\addplot[colorDN] coordinates {\retromaeSCDNormCD}; \addlegendentry{DN}
\addplot[colorLn] coordinates {\retromaeSCLnCD}; \addlegendentry{Ln}
\draw[dashed, black, thin] (axis cs:0,0) -- (axis cs:32000,0);
\node[font=\tiny, red!70!black, anchor=east] at (axis cs:31000,-0.7) {$\|\text{Irr}\|{>}\|\text{Rel}\|$};

\end{groupplot}
\node[draw=none, fill=none, inner sep=0pt] at ($(group c1r2.south)!0.5!(group c3r2.south) + (0,-1.1cm)$) {\pgfplotslegendfromname{retromaelegend}};
\end{tikzpicture}
\end{figure*}

\begin{figure*}[t]
  \centering
  \begin{tikzpicture}
    \begin{axis}[
      name=contriever,
      width=0.48\textwidth,
      height=0.264\textwidth,
      xlabel={Cohen's $d$ (mean $\pm$ std across 3 seeds)},
      ylabel={$\Delta\%$},
      xmin=-1.2, xmax=2.2,
      ymin=-40, ymax=250,
      grid=major,
      grid style={gray!30},
      title={Contriever: $r=0.57$, $p<0.001$},
      title style={font=\small},
      set layers,
      axis on top,
    ]
    \fill[magenta!25, opacity=0.5] (axis cs:-1.2,-40) rectangle (axis cs:-0.8,250);
    \fill[magenta!25, opacity=0.5] (axis cs:0.8,-40) rectangle (axis cs:2.2,250);
    \fill[magenta!15, opacity=0.5] (axis cs:-0.8,-40) rectangle (axis cs:-0.5,250);
    \fill[magenta!15, opacity=0.5] (axis cs:0.5,-40) rectangle (axis cs:0.8,250);
    \fill[magenta!8, opacity=0.5] (axis cs:-0.5,-40) rectangle (axis cs:-0.2,250);
    \fill[magenta!8, opacity=0.5] (axis cs:0.2,-40) rectangle (axis cs:0.5,250);

    \addplot[only marks, mark=square*, mark size=2.5pt, blue!70, opacity=0.8,
             error bars/.cd, x dir=both, x explicit, error bar style={blue!50, thick}] coordinates {
        (-0.088,1.5) +- (0.021,0) (-0.155,-0.9) +- (0.006,0)
    };

    \addplot[only marks, mark=*, mark size=2pt, green!60!black, opacity=0.7,
             error bars/.cd, x dir=both, x explicit, error bar style={green!40!black, thick}] coordinates {
        (0.019,-11.6) +- (0.009,0) (0.959,3.0) +- (0.051,0) (-0.240,8.9) +- (0.003,0)
        (0.129,17.3) +- (0.004,0) (-0.054,16.0) +- (0.003,0) (-0.258,20.3) +- (0.001,0)
        (-0.283,12.7) +- (0.002,0) (-0.591,8.9) +- (0.002,0) (-0.258,16.6) +- (0.011,0)
        (-0.350,17.4) +- (0.002,0) (-0.412,27.9) +- (0.003,0) (-0.221,17.0) +- (0.003,0)
        (-0.135,14.4) +- (0.003,0) (-0.061,19.3) +- (0.003,0) (0.001,13.3) +- (0.004,0)
        (1.109,9.3) +- (0.014,0) (0.169,13.0) +- (0.028,0) (0.753,20.4) +- (0.005,0)
        (-0.069,4.9) +- (0.011,0) (0.144,12.3) +- (0.008,0) (-0.341,1.5) +- (0.003,0)
        (-0.141,8.6) +- (0.013,0) (0.086,13.7) +- (0.002,0) (0.348,20.9) +- (0.012,0)
        (0.217,8.7) +- (0.005,0)
    };

    \addplot[only marks, mark=triangle*, mark size=3pt, red!70, opacity=0.7,
             error bars/.cd, x dir=both, x explicit, error bar style={red!50, thick}] coordinates {
        (1.481,26.3) +- (0.011,0) (1.136,199.5) +- (0.042,0) (0.776,109.8) +- (0.069,0)
        (1.684,80.8) +- (0.023,0) (1.320,4.7) +- (0.071,0) (1.014,207.8) +- (0.029,0)
        (1.559,96.0) +- (0.012,0) (1.792,122.8) +- (0.010,0) (1.427,42.5) +- (0.007,0)
        (1.089,99.5) +- (0.018,0) (0.461,63.4) +- (0.029,0) (0.510,41.9) +- (0.005,0)
    };

    \addplot[only marks, mark=pentagon*, mark size=3pt, orange!80, opacity=0.8,
             error bars/.cd, x dir=both, x explicit, error bar style={orange!60, thick}] coordinates {
        (0.753,20.4) +- (0.005,0) (0.642,8.1) +- (0.002,0) (0.124,5.9) +- (0.007,0)
    };

    \addplot[dashed, gray, thick, domain=-0.59:1.80] {42.32*x + 18.50};

    \draw[gray, thin] (axis cs:-1.2,0) -- (axis cs:2.2,0);
    \draw[gray, thin] (axis cs:0,-40) -- (axis cs:0,250);

    \end{axis}
  \end{tikzpicture}
  \hfill
  \begin{tikzpicture}
    \begin{axis}[
      name=retromae,
      width=0.48\textwidth,
      height=0.264\textwidth,
      xlabel={Cohen's $d$ (mean $\pm$ std across 3 seeds)},
      ylabel={},
      yticklabels={},
      xmin=-1.2, xmax=2.2,
      ymin=-40, ymax=250,
      grid=major,
      grid style={gray!30},
      title={RetroMAE: $r=0.68$, $p<0.001$ (excl.\ pony)},
      title style={font=\small},
      set layers,
      axis on top,
    ]
    \fill[magenta!25, opacity=0.5] (axis cs:-1.2,-40) rectangle (axis cs:-0.8,250);
    \fill[magenta!25, opacity=0.5] (axis cs:0.8,-40) rectangle (axis cs:2.2,250);
    \fill[magenta!15, opacity=0.5] (axis cs:-0.8,-40) rectangle (axis cs:-0.5,250);
    \fill[magenta!15, opacity=0.5] (axis cs:0.5,-40) rectangle (axis cs:0.8,250);
    \fill[magenta!8, opacity=0.5] (axis cs:-0.5,-40) rectangle (axis cs:-0.2,250);
    \fill[magenta!8, opacity=0.5] (axis cs:0.2,-40) rectangle (axis cs:0.5,250);

    \addplot[only marks, mark=square*, mark size=2.5pt, blue!70, opacity=0.8,
             error bars/.cd, x dir=both, x explicit, error bar style={blue!50, thick}] coordinates {
        (-0.291,2.2) +- (0.043,0) (-0.387,2.1) +- (0.044,0)
    };

    \addplot[only marks, mark=*, mark size=2pt, green!60!black, opacity=0.7,
             error bars/.cd, x dir=both, x explicit, error bar style={green!40!black, thick}] coordinates {
        (0.303,17.4) +- (0.012,0) (-0.658,2.4) +- (0.080,0) (-0.316,14.9) +- (0.041,0)
        (0.304,13.5) +- (0.026,0) (-0.022,13.1) +- (0.016,0) (-0.234,29.6) +- (0.014,0)
        (-0.177,25.3) +- (0.017,0) (-0.566,9.7) +- (0.009,0) (-0.392,20.0) +- (0.039,0)
        (-0.233,27.5) +- (0.010,0) (-0.354,32.6) +- (0.017,0) (-0.166,25.0) +- (0.010,0)
        (-0.151,23.5) +- (0.051,0) (-0.090,19.0) +- (0.013,0) (-0.151,1.1) +- (0.011,0)
        (-0.404,-0.6) +- (0.023,0) (-0.113,10.8) +- (0.041,0) (0.099,15.2) +- (0.035,0)
        (-0.063,5.8) +- (0.020,0) (-0.224,4.2) +- (0.036,0) (-0.236,1.9) +- (0.010,0)
        (-0.088,16.5) +- (0.016,0) (0.078,11.8) +- (0.008,0) (0.095,-2.8) +- (0.029,0)
        (-0.628,-5.5) +- (0.027,0)
    };

    \addplot[only marks, mark=triangle*, mark size=3pt, red!70, opacity=0.7,
             error bars/.cd, x dir=both, x explicit, error bar style={red!50, thick}] coordinates {
        (1.495,167.3) +- (0.138,0) (-0.185,51.1) +- (0.065,0) (-1.022,22.9) +- (0.058,0)
        (0.519,37.6) +- (0.068,0) (-0.563,17.0) +- (0.226,0)
        (0.277,29.1) +- (0.070,0) (0.859,39.5) +- (0.100,0) (0.728,61.6) +- (0.113,0)
        (-0.001,65.4) +- (0.055,0) (0.656,66.4) +- (0.060,0) (0.314,49.7) +- (0.061,0)
    };

    \addplot[only marks, mark=pentagon*, mark size=3pt, orange!80, opacity=0.8,
             error bars/.cd, x dir=both, x explicit, error bar style={orange!60, thick}] coordinates {
        (0.099,15.2) +- (0.035,0) (0.177,1.2) +- (0.048,0) (0.010,-0.9) +- (0.011,0)
    };

    \addplot[dashed, gray, thick, domain=-1.02:1.50] {44.62*x + 25.61};

    \draw[gray, thin] (axis cs:-1.2,0) -- (axis cs:2.2,0);
    \draw[gray, thin] (axis cs:0,-40) -- (axis cs:0,250);

    \end{axis}
  \end{tikzpicture}

  \vspace{0.2cm}
  \begin{tikzpicture}
    \node[draw=none, font=\scriptsize] at (0,0) {
      \tikz\draw[blue!70,fill=blue!70] (0,0) rectangle (0.1,0.1); In-Domain \hspace{0.12cm}
      \tikz\fill[green!60!black] (0,0) circle (0.05); BEIR \hspace{0.12cm}
      \tikz\fill[red!70] (0,-0.05) -- (0.1,-0.05) -- (0.05,0.065) -- cycle; BRIGHT \hspace{0.12cm}
      \tikz\fill[orange!80] (0,0) -- (0.05,0.07) -- (0.1,0.025) -- (0.08,-0.055) -- (0.02,-0.055) -- cycle; MultiHop \hspace{0.12cm}
      \tikz\draw[dashed, gray, thick] (0,0) -- (0.3,0); Regression
    };
  \end{tikzpicture}

  \caption{Per-dataset Cohen's $d$ vs.\ $\Delta\%$ for Contriever (left) and RetroMAE (right). Cohen's $d = (\mu_{\mathrm{rel}} - \mu_{\mathrm{irrel}}) / \sigma_{\mathrm{pooled}}$ measures the standardized difference between relevant ($\mu_{\mathrm{rel}}$) and irrelevant ($\mu_{\mathrm{irrel}}$) document embedding magnitudes, computed from the QueryNorm-trained model. $\Delta\% = (\text{QNorm} - \text{Cosine}) / \text{Cosine} \times 100$ is the relative performance improvement. Linear regression shows significant correlation: Contriever ($r=0.57$, $p<0.001$) and RetroMAE ($r=0.68$, $p<0.001$, excluding bright-pony outlier). Background shading indicates effect size: \colorbox{magenta!25}{large} ($|d| \geq 0.8$), \colorbox{magenta!15}{medium} ($0.5 \leq |d| < 0.8$), \colorbox{magenta!8}{small} ($0.2 \leq |d| < 0.5$), negligible (white). See Appendix~\ref{appendix:cohens_d} for per-dataset values.}
  \label{fig:cohens_d_scatter}
\end{figure*}

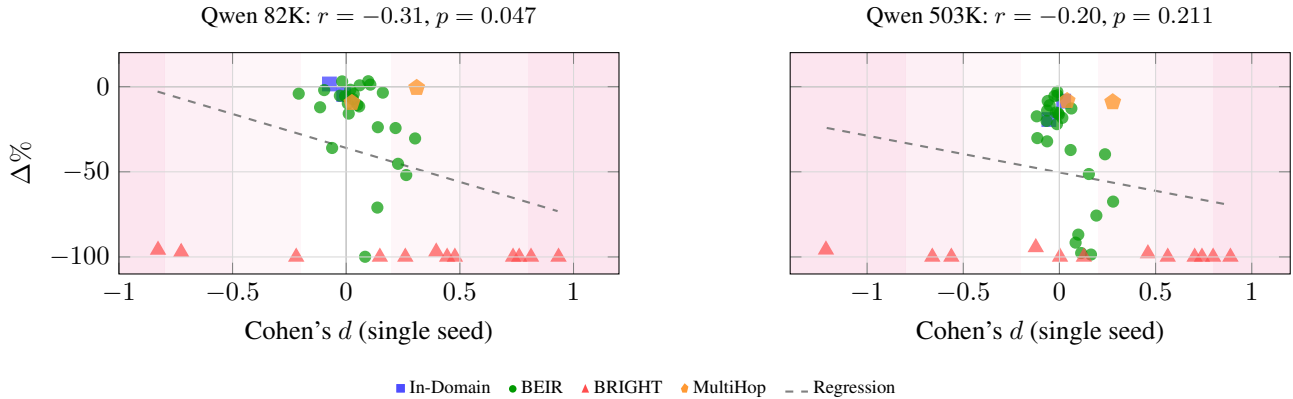
\begin{figure*}[t]
  \centering
  \begin{tikzpicture}
    \begin{axis}[
      name=qwen82k,
      width=0.48\textwidth,
      height=0.264\textwidth,
      xlabel={Cohen's $d$ (single seed)},
      ylabel={$\Delta\%$},
      xmin=-1.0, xmax=1.2,
      ymin=-110, ymax=20,
      grid=major,
      grid style={gray!30},
      title={Qwen 82K: $r=-0.31$, $p=0.047$},
      title style={font=\small},
      set layers,
      axis on top,
    ]
    \fill[magenta!25, opacity=0.5] (axis cs:-1.0,-110) rectangle (axis cs:-0.8,20);
    \fill[magenta!25, opacity=0.5] (axis cs:0.8,-110) rectangle (axis cs:1.2,20);
    \fill[magenta!15, opacity=0.5] (axis cs:-0.8,-110) rectangle (axis cs:-0.5,20);
    \fill[magenta!15, opacity=0.5] (axis cs:0.5,-110) rectangle (axis cs:0.8,20);
    \fill[magenta!8, opacity=0.5] (axis cs:-0.5,-110) rectangle (axis cs:-0.2,20);
    \fill[magenta!8, opacity=0.5] (axis cs:0.2,-110) rectangle (axis cs:0.5,20);

    \addplot[only marks, mark=square*, mark size=2.5pt, blue!70, opacity=0.8] coordinates {
        (-0.011,-4.6) (-0.074,1.6)
    };

    \addplot[only marks, mark=*, mark size=2pt, green!60!black, opacity=0.7] coordinates {
        (0.018,-1.8) (0.264,-51.9) (0.137,-71.0) (0.303,-30.4) (-0.209,-4.1)
        (0.217,-24.3) (-0.062,-36.0) (0.029,-7.3) (0.033,-4.3) (0.059,0.8)
        (-0.018,3.1) (-0.115,-12.1) (-0.005,-5.5) (0.011,-15.8) (0.228,-45.3)
        (0.083,-99.9) (0.049,-10.7) (0.106,1.2) (-0.028,-5.3) (0.139,-23.8)
        (0.097,3.2) (-0.097,-2.0) (0.007,-9.9) (0.161,-3.5) (0.057,-11.6)
    };

    \addplot[only marks, mark=triangle*, mark size=3pt, red!70, opacity=0.7] coordinates {
        (0.934,-100.0) (0.396,-97.0) (0.260,-100.0) (0.148,-100.0) (0.813,-100.0)
        (-0.220,-100.0) (0.444,-100.0) (-0.828,-95.9) (0.478,-100.0) (0.734,-100.0)
        (-0.726,-97.2) (0.761,-100.0)
    };

    \addplot[only marks, mark=pentagon*, mark size=3pt, orange!80, opacity=0.8] coordinates {
        (0.025,-9.4) (0.310,-0.6)
    };

    \addplot[dashed, gray, thick, domain=-0.83:0.93] {-39.90*x + -35.99};

    \draw[gray, thin] (axis cs:-1.0,0) -- (axis cs:1.2,0);
    \draw[gray, thin] (axis cs:0,-110) -- (axis cs:0,20);

    \end{axis}
  \end{tikzpicture}
  \hfill
  \begin{tikzpicture}
    \begin{axis}[
      name=qwen503k,
      width=0.48\textwidth,
      height=0.264\textwidth,
      xlabel={Cohen's $d$ (single seed)},
      ylabel={},
      yticklabels={},
      xmin=-1.4, xmax=1.2,
      ymin=-110, ymax=20,
      grid=major,
      grid style={gray!30},
      title={Qwen 503K: $r=-0.20$, $p=0.211$},
      title style={font=\small},
      set layers,
      axis on top,
    ]
    \fill[magenta!25, opacity=0.5] (axis cs:-1.4,-110) rectangle (axis cs:-0.8,20);
    \fill[magenta!25, opacity=0.5] (axis cs:0.8,-110) rectangle (axis cs:1.2,20);
    \fill[magenta!15, opacity=0.5] (axis cs:-0.8,-110) rectangle (axis cs:-0.5,20);
    \fill[magenta!15, opacity=0.5] (axis cs:0.5,-110) rectangle (axis cs:0.8,20);
    \fill[magenta!8, opacity=0.5] (axis cs:-0.5,-110) rectangle (axis cs:-0.2,20);
    \fill[magenta!8, opacity=0.5] (axis cs:0.2,-110) rectangle (axis cs:0.5,20);

    \addplot[only marks, mark=square*, mark size=2.5pt, blue!70, opacity=0.8] coordinates {
        (0.022,-8.5) (-0.055,-19.3)
    };

    \addplot[only marks, mark=*, mark size=2pt, green!60!black, opacity=0.7] coordinates {
        (-0.118,-17.4) (0.192,-75.7) (0.084,-91.6) (0.279,-67.5) (-0.026,-5.5)
        (0.237,-39.7) (0.098,-86.9) (-0.019,-15.2) (-0.061,-8.4) (-0.010,-5.1)
        (-0.011,-3.0) (-0.115,-30.2) (-0.067,-20.1) (0.058,-37.2) (0.153,-51.3)
        (0.113,-97.7) (-0.013,-22.0) (0.062,-12.8) (-0.051,-10.9) (0.164,-98.6)
        (-0.007,-15.6) (-0.064,-14.0) (-0.003,-16.3) (0.015,-18.3) (-0.064,-32.1)
    };

    \addplot[only marks, mark=triangle*, mark size=3pt, red!70, opacity=0.7] coordinates {
        (0.889,-100.0) (0.460,-98.0) (0.127,-100.0) (-0.562,-100.0) (0.799,-100.0)
        (0.003,-100.0) (-0.661,-100.0) (-1.214,-95.8) (0.563,-100.0) (0.703,-100.0)
        (-0.123,-94.4) (0.741,-100.0)
    };

    \addplot[only marks, mark=pentagon*, mark size=3pt, orange!80, opacity=0.8] coordinates {
        (0.040,-8.3) (0.277,-9.0)
    };

    \addplot[dashed, gray, thick, domain=-1.21:0.89] {-21.61*x + -50.37};

    \draw[gray, thin] (axis cs:-1.4,0) -- (axis cs:1.2,0);
    \draw[gray, thin] (axis cs:0,-110) -- (axis cs:0,20);

    \end{axis}
  \end{tikzpicture}

  \vspace{0.2cm}
  \begin{tikzpicture}
    \node[draw=none, font=\scriptsize] at (0,0) {
      \tikz\draw[blue!70,fill=blue!70] (0,0) rectangle (0.1,0.1); In-Domain \hspace{0.12cm}
      \tikz\fill[green!60!black] (0,0) circle (0.05); BEIR \hspace{0.12cm}
      \tikz\fill[red!70] (0,-0.05) -- (0.1,-0.05) -- (0.05,0.065) -- cycle; BRIGHT \hspace{0.12cm}
      \tikz\fill[orange!80] (0,0) -- (0.05,0.07) -- (0.1,0.025) -- (0.08,-0.055) -- (0.02,-0.055) -- cycle; MultiHop \hspace{0.12cm}
      \tikz\draw[dashed, gray, thick] (0,0) -- (0.3,0); Regression
    };
  \end{tikzpicture}

  \caption{Per-dataset Cohen's $d$ vs.\ $\Delta\%$ for Qwen 82K (left) and Qwen 503K (right). Unlike Contriever and RetroMAE (Figure~\ref{fig:cohens_d_scatter}), which were finetuned from pre-trained retrievers, Qwen was trained from a foundation LLM. Qwen shows \textbf{negative} $\Delta\%$ values (QNorm performs worse than Cosine) and a \textbf{negative correlation} between Cohen's $d$ and $\Delta\%$. Although relevant documents have slightly larger magnitudes on average (Cohen's $d > 0$ for most datasets), the effect size is small and the distributions overlap substantially. This suggests that foundation LLMs, lacking retrieval-specific pre-training, have difficulty learning to leverage document magnitude for retrieval.}
  \label{fig:cohens_d_scatter_qwen}
\end{figure*}

Key observations:
\begin{itemize}
  \item \textbf{Cohen's $d$}: Contriever shows consistently positive Cohen's $d$ for BRIGHT datasets ($d = 0.46$--$1.80$), indicating relevant documents have substantially larger magnitudes. RetroMAE shows mixed effect sizes. The mean Cohen's $d$ differs substantially: Contriever ($+$0.38) vs.\ RetroMAE ($-$0.08).
  \item \textbf{Query CV}: RetroMAE shows consistently high $\Delta$CV (3.9--8.9$\times$) across all datasets, with Dot CV uniformly low ($<1\%$). Contriever shows modest $\Delta$CV (0.5--2.0$\times$), but BRIGHT shows $\Delta$CV $<1$, indicating CV decreases. The non-overlapping $\Delta$CV ranges support $\Delta$CV as a model-level predictor.
  \item \textbf{Training from foundation model (Qwen)}: Figure~\ref{fig:cohens_d_scatter_qwen} shows that Qwen, which is trained from a foundation LLM rather than a pre-trained retriever, exhibits a \emph{negative} correlation between Cohen's $d$ and $\Delta\%$, with QNorm performing worse than Cosine despite relevant documents having slightly larger magnitudes on average. This contrasts with Contriever and RetroMAE (Figure~\ref{fig:cohens_d_scatter}), which were finetuned from retrievers with contrastive or autoencoding pre-training. The result suggests that foundation LLMs, lacking retrieval-specific pre-training, have difficulty learning to leverage document magnitude for retrieval.
\end{itemize}

\section{Dataset Statistics}
\label{sec:appendix_datasets}

\subsection{Training Data}

Table~\ref{tab:training_data_stats} provides statistics for the training datasets used in our experiments. MS MARCO v1.1 QA is the question-answering version containing 82K query-answer pairs, which we use for main experiments. MS MARCO Passage Ranking (judged) is the passage retrieval version containing 503K queries with relevance annotations, which we use for extended E5 experiments in Section~\ref{sec:e5_500k_results}. All statistics are computed on the training split.

\begin{table}[htbp]
\centering
\caption{Statistics of training datasets. Length statistics are measured in characters (Ave = Average, Min = Minimum, Max = Maximum). ``---'' indicates not applicable (MS MARCO Passage Ranking has no answer annotations as it is a retrieval task).}
\label{tab:training_data_stats}
\resizebox{\textwidth}{!}{%
\begin{tabular}{l|rrr|ccc|ccc|ccc}
\toprule
\multirow{2}{*}{\textbf{Dataset}} & \multirow{2}{*}{\textbf{\#Samples}} & \multirow{2}{*}{\textbf{\#Queries}} & \multirow{2}{*}{\textbf{\#Docs}} & \multicolumn{3}{c|}{\textbf{Query Length}} & \multicolumn{3}{c|}{\textbf{Doc Length}} & \multicolumn{3}{c}{\textbf{Answer Length}} \\
& & & & Ave & Min & Max & Ave & Min & Max & Ave & Min & Max \\
\midrule
MS MARCO v1.1 QA & 82,326 & 82,326 & 80,304 & 34.0 & 8 & 144 & 478 & 38 & 3,570 & 82.2 & 0 & 990 \\
MS MARCO Passage Ranking & 502,939 & 502,939 & 489,288 & 33.2 & 5 & 215 & 370 & 25 & 3,602 & --- & --- & --- \\
\bottomrule
\end{tabular}%
}
\end{table}

\subsection{Evaluation Data}

Table~\ref{tab:dataset_stats} provides detailed statistics for all evaluation datasets. Note that TREC-DL 2019 and 2020 contain only 43 and 54 queries respectively, which may lead to high variance in evaluation results. We therefore include MS MARCO Dev (6,980 queries) as a complementary in-domain benchmark to provide more stable estimates.

\begin{table}[htbp]
\centering
\resizebox{\textwidth}{!}{%
\begin{tabular}{l|l|rr|ccc|ccc|cc}
\toprule
\multirow{2}{*}{\textbf{Type}} & \multirow{2}{*}{\textbf{Dataset}} & \multirow{2}{*}{\textbf{\#Queries}} & \multirow{2}{*}{\textbf{\#Docs}} & \multicolumn{3}{c|}{\textbf{Query Length}} & \multicolumn{3}{c|}{\textbf{Doc Length}} & \textbf{Docs/} & \textbf{Queries/} \\
& & & & Ave & Min & Max & Ave & Min & Max & \textbf{Query} & \textbf{Doc} \\
\midrule
\multirow{3}{*}{\textbf{In-Domain}} & MSM-Dev & 6,980 & 8,841,823 & 33 & 9 & 186 & 335 & 3 & 1,665 & 1.07 & 1.00 \\
& trec-dl-2019 & 43 & 8,841,823 & 33 & 16 & 55 & 335 & 3 & 1,665 & 215 & 1.01 \\
& trec-dl-2020 & 54 & 8,841,823 & 34 & 12 & 70 & 335 & 3 & 1,665 & 211 & 1.01 \\
\midrule
\multirow{14}{*}{\textbf{BEIR}} & arguana & 1,406 & 8,674 & 1,193 & 251 & 5,500 & 1,030 & 2 & 6,673 & 1 & 1.00 \\
& climate-fever & 1,535 & 5,416,593 & 123 & 26 & 406 & 539 & 1 & 374,597 & 3.05 & 3.48 \\
& cqadupstack & 13,145 & 457,199 & 50 & 15 & 149 & 932 & 41 & 43,874 & 1.8 & 1.00 \\
& dbpedia-entity & 400 & 4,635,922 & 34 & 6 & 88 & 310 & 8 & 42,899 & 109 & 1.07 \\
& fever & 6,666 & 5,416,568 & 50 & 14 & 189 & 539 & 1 & 374,597 & 1.19 & 5.29 \\
& fiqa & 648 & 57,638 & 63 & 16 & 147 & 767 & 0 & 16,990 & 2.63 & 1.00 \\
& hotpotqa & 7,405 & 5,233,329 & 92 & 32 & 288 & 289 & 8 & 8,276 & 2 & 1.07 \\
& nfcorpus & 323 & 3,633 & 22 & 3 & 72 & 1,591 & 123 & 10,090 & 38 & 3.94 \\
& nq & 3,452 & 2,681,468 & 48 & 25 & 100 & 493 & 4 & 17,008 & 1.22 & 1.00 \\
& quora & 10,000 & 522,931 & 52 & 2 & 258 & 62 & 1 & 1,169 & 1.57 & 1.00 \\
& scidocs & 1,000 & 25,657 & 72 & 16 & 206 & 1,204 & 10 & 10,169 & 30 & 1.17 \\
& scifact & 300 & 5,183 & 90 & 28 & 204 & 1,499 & 221 & 10,127 & 1.13 & 1.20 \\
& trec-covid & 50 & 171,332 & 69 & 30 & 165 & 1,118 & 0 & 122,459 & 1,327 & 1.87 \\
& webis-touche2020 & 49 & 382,545 & 43 & 16 & 83 & 1,720 & 3 & 106,072 & 45 & 1.05 \\
\midrule
\multirow{12}{*}{\textbf{BRIGHT}} & aops & 111 & 188,002 & 320 & 85 & 1,167 & 754 & 58 & 7,334 & 4.72 & 4.72 \\
& biology & 103 & 57,359 & 523 & 89 & 2,195 & 330 & 1 & 31,131 & 3.61 & 1.00 \\
& earth\_science & 116 & 121,249 & 477 & 83 & 1,565 & 338 & 2 & 233,623 & 5.04 & 1.00 \\
& economics & 103 & 50,220 & 740 & 164 & 2,223 & 395 & 3 & 39,672 & 7.77 & 1.00 \\
& leetcode & 142 & 413,932 & 1,459 & 422 & 3,964 & 1,059 & 75 & 103,665 & 1.85 & 1.21 \\
& pony & 112 & 7,894 & 389 & 182 & 946 & 260 & 8 & 2,583 & 19.8 & \textbf{51.6} \\
& psychology & 101 & 52,835 & 693 & 166 & 2,334 & 384 & 3 & 226,941 & 6.85 & 1.01 \\
& robotics & 101 & 61,961 & 2,180 & 165 & 19,341 & 291 & 3 & 28,640 & 5.15 & 1.00 \\
& stackoverflow & 117 & 107,081 & 1,293 & 185 & 12,432 & 1,715 & 1 & 4,000 & 4.09 & 1.01 \\
& sustainable\_living & 108 & 60,792 & 683 & 158 & 2,843 & 344 & 1 & 158,299 & 5.33 & 1.00 \\
& theoremqa\_questions & 194 & 188,002 & 426 & 84 & 1,255 & 754 & 58 & 7,334 & 3.18 & 1.90 \\
& theoremqa\_theorems & 76 & 23,839 & 416 & 159 & 846 & 874 & 74 & 19,106 & 1.99 & 1.59 \\
\midrule
\multirow{4}{*}{\textbf{Multi-hop}} & 2wikimultihopqa & 12,576 & 125,237 & 68 & 23 & 179 & 377 & 18 & 9,780 & 2.44 & 1.00 \\
& musique & 2,417 & 48,315 & 102 & 26 & 275 & 524 & 116 & 2,044 & 2.65 & 1.00 \\
& novelhopqa & 4,345 & 4,345 & 138 & 43 & 334 & 2,336 & 152 & 42,658 & 1 & 1.00 \\
& hotpotqa & 7,405 & 5,233,329 & 92 & 32 & 288 & 289 & 8 & 8,276 & 2 & 1.07 \\
\bottomrule
\end{tabular}%
}
\caption{Statistics of evaluation benchmarks. Length statistics are measured in characters (Ave = Average, Min = Minimum, Max = Maximum). Docs/Query indicates the average number of relevant documents per query; Queries/Doc indicates the average number of queries each document is relevant to. The pony dataset has an exceptionally high Queries/Doc ratio (\textbf{51.6}), indicating that each document is relevant to many queries on average. This ``hub'' structure may affect magnitude learning patterns and explains its outlier behavior in Figure~\ref{fig:cohens_d_scatter}. HotpotQA appears in both BEIR and Multi-hop categories.}
\label{tab:dataset_stats}
\end{table}

\section{Seed Sensitivity Analysis}
\label{appendix:seed_sensitivity}

To verify the robustness of our findings, we conduct experiments with three different random seeds (0, 42, 1337) for both Contriever and RetroMAE on MS MARCO v1.1 QA. Table~\ref{tab:seed_sensitivity} reports the validation NDCG@10 statistics across seeds.

\begin{table}[htbp]
\centering
\caption{Seed sensitivity analysis (3 random seeds). We report mean and standard deviation of validation NDCG@10. Low standard deviations ($<0.06$) for Cosine, Dot, DNorm, and Learnable indicate stable convergence. QNorm exhibits higher variance due to one seed showing different early-training dynamics.}
\label{tab:seed_sensitivity}
\label{tab:seed_sensitivity_contriever}
\label{tab:seed_sensitivity_retromae}
\begin{tabular}{lcc|cc}
\toprule
& \multicolumn{2}{c|}{\textbf{Contriever}} & \multicolumn{2}{c}{\textbf{RetroMAE}} \\
\textbf{Method} & \textbf{Mean} & \textbf{Std} & \textbf{Mean} & \textbf{Std} \\
\midrule
Cosine & 92.84 & $\pm$0.03 & 91.88 & $\pm$0.02 \\
Dot & 93.49 & $\pm$0.02 & 93.02 & $\pm$0.06 \\
QNorm & 92.87 & $\pm$1.02 & 92.45 & $\pm$1.03 \\
DNorm & 93.55 & $\pm$0.02 & 93.06 & $\pm$0.04 \\
Learnable & 93.59 & $\pm$0.01 & 93.08 & $\pm$0.02 \\
\bottomrule
\end{tabular}
\end{table}

\paragraph{Key observations.}
\begin{itemize}[leftmargin=1.5em, itemsep=0.2em]
    \item \textbf{Stable performance}: For both models, Cosine, Dot, DNorm, and Learnable show very low variance with std $\leq 0.06$, indicating that our main conclusions are robust to random seed selection.
    \item \textbf{Consistent ranking}: The relative ordering of methods, specifically Learnable $\approx$ DNorm $\approx$ Dot $>$ Cosine for in-domain validation, is preserved across all seeds for both models.
    \item \textbf{QNorm variance}: Both models show higher QNorm variance due to one seed exhibiting different early-training dynamics. This variance does not affect our main conclusions.
    \item \textbf{Best step variation}: While the optimal checkpoint step varies across seeds due to different convergence speeds, all seeds reach similar final performance, validating our early stopping strategy based on validation NDCG@10.
\end{itemize}

These results demonstrate that our findings, particularly that magnitude-aware training with Dot, DNorm, and Learnable outperforms Cosine similarity, are reproducible across different random initializations for both Contriever and RetroMAE.

\subsection{Statistical Significance on Evaluation Benchmarks}

Beyond validation performance stability, we verify that the improvements on evaluation benchmarks are statistically significant. Table~\ref{tab:statistical_significance} reports per-dataset paired t-test results: for each dataset, we first average results across 3 random seeds, then treat datasets within each benchmark as observations.

\begin{table*}[htbp]
\centering
\caption{Statistical significance analysis using per-dataset paired t-test. For each dataset, we first average results across 3 random seeds, then treat datasets within each benchmark as observations for paired t-tests comparing magnitude-aware methods against Cosine. We report mean $\pm$ std of NDCG@10 across datasets. \textbf{Bold}: best; \underline{underline}: second best; $^*$: $p < 0.05$; $^{**}$: $p < 0.01$.}
\label{tab:statistical_significance}
\resizebox{\textwidth}{!}{%
\begin{tabular}{l|ccccc|ccccc}
\toprule
& \multicolumn{5}{c|}{\textbf{Contriever}} & \multicolumn{5}{c}{\textbf{RetroMAE}} \\
\textbf{Benchmark} & \textbf{Cos} & \textbf{Dot} & \textbf{QNorm} & \textbf{DNorm} & \textbf{Learn} & \textbf{Cos} & \textbf{Dot} & \textbf{QNorm} & \textbf{DNorm} & \textbf{Learn} \\
\midrule
In-Domain(3) & 48.3$\pm$15.1 & 48.0$\pm$14.7 & \underline{48.7$\pm$15.1} & 48.0$\pm$14.4 & \textbf{49.5$\pm$15.0}$^{**}$ & 47.9$\pm$14.9 & 48.6$\pm$15.1$^{*}$ & 49.3$\pm$14.8$^{**}$ & \underline{50.7$\pm$15.4}$^{*}$ & \textbf{51.0$\pm$15.7}$^{*}$ \\
BEIR(14) & 41.0$\pm$20.2 & 43.5$\pm$21.9$^{*}$ & \textbf{44.0$\pm$21.6}$^{*}$ & 42.5$\pm$20.9 & \underline{43.6$\pm$21.5}$^{**}$ & 38.3$\pm$19.7 & 40.1$\pm$20.7 & 41.1$\pm$20.5$^{**}$ & \underline{41.2$\pm$20.4}$^{**}$ & \textbf{41.2$\pm$20.5}$^{**}$ \\
BRIGHT(12) & 7.4$\pm$3.2 & \underline{12.5$\pm$4.8}$^{**}$ & \textbf{12.7$\pm$5.9}$^{**}$ & 9.8$\pm$4.1$^{**}$ & 11.7$\pm$4.3$^{**}$ & 6.1$\pm$3.8 & 8.7$\pm$4.6$^{**}$ & 9.4$\pm$4.5$^{**}$ & \underline{9.5$\pm$4.6}$^{**}$ & \textbf{9.7$\pm$4.6}$^{**}$ \\
MultiHop(4) & 51.4$\pm$12.3 & \textbf{58.2$\pm$14.9}$^{*}$ & 57.4$\pm$14.4$^{*}$ & \underline{57.5$\pm$15.4} & 57.2$\pm$14.5 & 50.7$\pm$13.1 & 52.9$\pm$14.7 & 54.8$\pm$15.4 & \textbf{55.5$\pm$16.5} & \underline{55.2$\pm$15.9} \\
\bottomrule
\end{tabular}%
}
\end{table*}

\section{STS Verification Experiment Configuration}
\label{appendix:sts_config}

This appendix provides detailed hyperparameter rationale for the STS verification experiments described in Section~\ref{sec:gen_symmetric}. The key differences from retrieval experiments arise from task and loss function differences.

\subsection{Hyperparameter Comparison}

Table~\ref{tab:sts_hyperparams} compares the hyperparameters used in retrieval and STS experiments.

\begin{table}[htbp]
\centering
\caption{Hyperparameter comparison between retrieval and STS experiments. Differences arise from task characteristics and loss function requirements.}
\label{tab:sts_hyperparams}
\begin{tabular}{lcc}
\toprule
\textbf{Hyperparameter} & \textbf{Retrieval} & \textbf{STS} \\
\midrule
Loss function & InfoNCE & MSE \\
Learning rate & $2$--$5 \times 10^{-6}$ & $2 \times 10^{-5}$ \\
Batch size & 128 & 64 \\
Epochs & 100 & 4 \\
Warmup ratio & 0 & 0.1 \\
Scale factor ($\lambda$) & 20 & 0.01 \\
\midrule
Training data & MS MARCO ($\sim$80K) & AllNLI + STS-B ($\sim$950K) \\
Evaluation & NDCG@10 & Spearman correlation \\
\bottomrule
\end{tabular}
\end{table}

\subsection{Rationale for Key Differences}

\paragraph{Learning rate ($2 \times 10^{-5}$ vs.\ $2$--$5 \times 10^{-6}$).}
The STS experiments use 10$\times$ higher learning rate because: (1) STS training uses only 4 epochs compared to retrieval's 100 epochs, requiring faster learning; (2) $2 \times 10^{-5}$ is the standard learning rate in SentenceTransformers for BERT-based models.

\paragraph{Scale factor ($\lambda = 0.01$ vs.\ $20$).}
This is the most critical difference, arising from fundamentally different loss functions:
\begin{itemize}[leftmargin=1.5em, itemsep=0.1em]
    \item \textbf{Retrieval (InfoNCE)}: The scale factor $\alpha = 20$ acts as inverse temperature ($\tau = 1/20 = 0.05$), sharpening the softmax distribution over candidates. Higher values increase the penalty for ranking errors.
    \item \textbf{STS (MSE)}: The scale factor $\lambda = 0.01$ maps dot product values (typically in the range 100--1000 for 768-dimensional embeddings) to the target similarity range $[0, 1]$. Since dot product $\langle \bm{s}_1, \bm{s}_2 \rangle \approx \|\bm{s}_1\| \|\bm{s}_2\| \cos\theta$, and embedding norms are typically 10--30, the dot product can reach several hundred. Multiplying by 0.01 brings this into the $[0, 1]$ range for MSE loss.
\end{itemize}

\paragraph{Batch size (64 vs.\ 128).}
STS uses smaller batch size because: (1) the STS-B dataset is smaller ($\sim$6K samples without AllNLI); (2) 64 is the default in SentenceTransformers.

\paragraph{Epochs (4 vs.\ 100).}
STS converges much faster because: (1) CosineSimilarityLoss in SentenceTransformers is well-optimized for this task; (2) the combined AllNLI + STS-B dataset is larger ($\sim$950K samples), so fewer epochs are needed; (3) 4 epochs is standard practice in SentenceTransformers tutorials.

\section{Detailed Derivations}

\subsection{Ranking Non-Preservation Under the \texorpdfstring{$\mathbb{R}^n \to \mathbb{S}^n$}{R\^{}n to S\^{}n} Mapping}
\label{appendix:ranking_reversal}

A natural objection to studying magnitude is that any unnormalized vector $\bm{v} \in \mathbb{R}^n$ with $\|\bm{v}\| \leq M$ can be losslessly encoded on the $(n{+}1)$-dimensional unit sphere via
\begin{equation}
\varphi(\bm{v}) = \left(\frac{\bm{v}}{M},\; \sqrt{1 - \|\bm{v}\|^2 / M^2}\right) \in \mathbb{S}^{n}.
\end{equation}
The map $\varphi$ is injective and recovers $\bm{v}$ from its first $n$ coordinates, so unnormalized embeddings appear ``equivalent'' to normalized embeddings of one higher dimension. The objection then asks: \emph{isn't this just a coordinate change?}

\paragraph{Counterexample (ranking reversal).} The mapping is bijective on the parameter space, but it does \emph{not} preserve similarity ranking. Take $n=2$, $M=10$, and:
\begin{equation*}
\bm{q} = (1, 0), \quad \bm{d}_1 = (8, 6), \quad \bm{d}_2 = (3, 0).
\end{equation*}
In $\mathbb{R}^n$ with the dot product:
\begin{equation*}
\bm{q} \cdot \bm{d}_1 = 8 \;>\; 3 = \bm{q} \cdot \bm{d}_2,
\end{equation*}
so $\bm{d}_1 \succ \bm{d}_2$. After mapping to $\mathbb{S}^n$:
\begin{align*}
\varphi(\bm{q}) &= (0.1,\; 0,\; 0.995), \\
\varphi(\bm{d}_1) &= (0.8,\; 0.6,\; 0), \\
\varphi(\bm{d}_2) &= (0.3,\; 0,\; 0.954),
\end{align*}
and cosine on $\mathbb{S}^n$ (which is dot product of unit vectors) gives:
\begin{equation*}
\varphi(\bm{q}) \cdot \varphi(\bm{d}_1) = 0.08 \;<\; 0.98 = \varphi(\bm{q}) \cdot \varphi(\bm{d}_2),
\end{equation*}
so $\bm{d}_2 \succ \bm{d}_1$. \textbf{The ranking is reversed.}

\paragraph{Geometric intuition.} The $(n{+}1)$-th coordinate $\sqrt{1-\|\bm{v}\|^2/M^2}$ is large for low-magnitude vectors (mapping them near the ``pole'' $x_{n+1}\approx 1$) and small for high-magnitude vectors (mapping them onto the ``equator'' $x_{n+1}=0$). High-magnitude documents thus sit nearly orthogonal to low-magnitude queries on $\mathbb{S}^n$, fundamentally distorting cosine similarity. The mapping is therefore a structural change to the similarity geometry, not a coordinate relabeling, and induces different retrieval rankings and different optimization dynamics in the original $\mathbb{R}^n$ space.

\subsection{Relationship Between Similarity Functions}

The four similarity functions can be understood through a unified framework. Let $\bm{q}, \bm{d} \in \mathbb{R}^n$ be the query and document embeddings. We can write:

\begin{align}
s_{\text{dot}}(\bm{q}, \bm{d}) &= \bm{q}^\top \bm{d} = \|\bm{q}\| \|\bm{d}\| \cos\theta \\
s_{\text{cos}}(\bm{q}, \bm{d}) &= \cos\theta \\
s_{\text{q-norm}}(\bm{q}, \bm{d}) &= \|\bm{d}\| \cos\theta \\
s_{\text{d-norm}}(\bm{q}, \bm{d}) &= \|\bm{q}\| \cos\theta
\end{align}

where $\theta$ is the angle between $\bm{q}$ and $\bm{d}$.

This decomposition reveals that:
\begin{itemize}
    \item Cosine similarity captures only the angular component
    \item Query-only normalization additionally captures document magnitude
    \item Document-only normalization additionally captures query magnitude
    \item Dot product captures both magnitudes and angle
\end{itemize}

\subsection{Gradient Analysis}
\label{appendix:gradient_analysis}

We provide a detailed analysis of the gradient dynamics for cosine and dot product similarity, explaining why the normalization constraint fundamentally limits magnitude learning.

\paragraph{Dot Product Gradient.}
For the InfoNCE loss with dot product similarity, the gradient with respect to the query embedding $\bm{q}$ is:
\begin{equation}
\frac{\partial \mathcal{L}_{\text{dot}}}{\partial \bm{q}} = \alpha \left(-\bm{d}^+ + \sum_{j} p_j \bm{d}^j\right),
\label{eq:grad_dot}
\end{equation}
where $\alpha$ is the inverse temperature (loss scale), $p_j = \frac{\exp(\alpha\, \bm{q}^\top \bm{d}^j)}{\sum_k \exp(\alpha\, \bm{q}^\top \bm{d}^k)}$ is the softmax probability, and $\bm{d}^+$ is the positive document. This gradient operates in the full $\mathbb{R}^n$ space and can adjust both the direction and magnitude of $\bm{q}$.

\paragraph{Cosine Similarity Gradient and Tangent Space Projection.}
For cosine similarity $s_{\cos}(\bm{q}, \bm{d}) = \hat{\bm{q}}^\top \hat{\bm{d}}$ where $\hat{\bm{v}} = \bm{v}/\|\bm{v}\|$, the gradient computation involves the Jacobian of the normalization operation. We derive this Jacobian explicitly.

\textbf{Derivation of the Normalization Jacobian.} Let $\hat{\bm{v}} = \bm{v}/\|\bm{v}\|$. The $(i,j)$-th element of the Jacobian matrix $\bm{J} = \frac{\partial \hat{\bm{v}}}{\partial \bm{v}}$ is:
\begin{equation}
J_{ij} = \frac{\partial}{\partial v_j} \left( \frac{v_i}{\|\bm{v}\|} \right) = \frac{\delta_{ij}}{\|\bm{v}\|} - \frac{v_i v_j}{\|\bm{v}\|^3} = \frac{1}{\|\bm{v}\|} \left( \delta_{ij} - \hat{v}_i \hat{v}_j \right),
\end{equation}
where $\delta_{ij}$ is the Kronecker delta. In matrix form:
\begin{equation}
\bm{J} = \frac{\partial \hat{\bm{v}}}{\partial \bm{v}} = \frac{1}{\|\bm{v}\|} \left( \bm{I} - \hat{\bm{v}} \hat{\bm{v}}^\top \right) = \frac{1}{\|\bm{v}\|} \bm{P}_{\bm{v}}.
\label{eq:jacobian_derivation}
\end{equation}

For a function $f(\hat{\bm{v}})$ composed with L2 normalization, the chain rule (with column-vector gradients) gives:
\begin{equation}
\nabla_{\bm{v}} f = \left(\frac{\partial \hat{\bm{v}}}{\partial \bm{v}}\right)^{\!\top} \nabla_{\hat{\bm{v}}} f = \frac{1}{\|\bm{v}\|}\, \bm{P}_{\bm{v}}\, \nabla_{\hat{\bm{v}}} f,
\label{eq:norm_jacobian}
\end{equation}
where $\bm{P}_{\bm{v}} = \bm{I} - \hat{\bm{v}}\hat{\bm{v}}^\top$ is the symmetric \emph{projection matrix onto the tangent space} of the unit sphere at $\hat{\bm{v}}$, so $\bm{P}_{\bm{v}}^{\top} = \bm{P}_{\bm{v}}$.

\subsection{Learnable Normalization: Gradient Analysis}
\label{appendix:learnable_gradient}

We extend the gradient analysis to the learnable normalization framework introduced in Section~\ref{par:learnable_norm}. Recall that the parameterized similarity is:
\begin{equation}
s_{\gamma_q, \gamma_d}(\bm{q}, \bm{d}) = \frac{\bm{q}^\top}{\|\bm{q}\|^{\gamma_q}} \cdot \frac{\bm{d}}{\|\bm{d}\|^{\gamma_d}}
= \|\bm{q}\|^{1-\gamma_q} \cdot \|\bm{d}\|^{1-\gamma_d} \cdot \cos\theta,
\label{eq:partial_norm_appendix}
\end{equation}
where $\gamma_q, \gamma_d \in [0, 1]$ control the degree of query-side and document-side normalization, respectively.

\begin{proposition}[Normalization Gradient]
\label{prop:norm_gradient}
\label{prop:gradient}
The partial derivatives of $s_{\gamma_q, \gamma_d}$ with respect to the normalization exponents are:
\begin{align}
\frac{\partial\, s_{\gamma_q, \gamma_d}}{\partial \gamma_q} &= -\ln\|\bm{q}\| \cdot s_{\gamma_q, \gamma_d}(\bm{q}, \bm{d}) \label{eq:grad_gamma_q} \\[4pt]
\frac{\partial\, s_{\gamma_q, \gamma_d}}{\partial \gamma_d} &= -\ln\|\bm{d}\| \cdot s_{\gamma_q, \gamma_d}(\bm{q}, \bm{d}) \label{eq:grad_gamma_d}
\end{align}
\end{proposition}

\begin{proof}
From Eq.~\eqref{eq:partial_norm_appendix}, $s = \|\bm{q}\|^{1-\gamma_q} \cdot \|\bm{d}\|^{1-\gamma_d} \cdot \cos\theta$. Taking the derivative with respect to $\gamma_q$:
\begin{equation*}
\frac{\partial s}{\partial \gamma_q} = \frac{\partial}{\partial \gamma_q} \|\bm{q}\|^{1-\gamma_q} \cdot \|\bm{d}\|^{1-\gamma_d} \cos\theta = -\ln\|\bm{q}\| \cdot \|\bm{q}\|^{1-\gamma_q} \cdot \|\bm{d}\|^{1-\gamma_d} \cos\theta = -\ln\|\bm{q}\| \cdot s
\end{equation*}
The derivation for $\gamma_d$ is analogous.
\end{proof}

\subsection{Qwen 500K Learnable: Gamma Stagnation}
\label{appendix:qwen_500k_gamma}

\Cref{fig:qwen_500k_gamma} compares the learned $\gamma$ trajectories for Qwen3-Base trained on 82K vs 500K data using \emph{identical hyperparameters} (LR $5\!\times\!10^{-6}$, AdamW, no parameter-specific learning rate for $\gamma$). With 82K samples, $\gamma$ successfully drifts toward Cosine ($\gamma \approx 0.503$) within 20 epochs. In contrast, 500K training shows complete $\gamma$ stagnation at the initialization value ($\gamma \approx 0.5000$) throughout all 40 epochs, where the drift magnitude is $\sim$5000$\times$ smaller ($\Delta\gamma < 10^{-5}$ vs $\Delta\gamma \approx 0.003$). Since the same optimizer setup yields meaningful drift at 82K but stagnation at 500K, this is unlikely to be a pure learning-rate issue: the gradient signal w.r.t.\ $\gamma$ is itself attenuated under the larger and more diverse 500K data, suggesting the loss landscape for $\gamma$ has flattened. This explains the flat performance curve in \Cref{fig:training_curves}: without meaningful $\gamma$ updates, the learnable normalization cannot adapt to the data, while DNorm, which has no learnable $\gamma$ parameters, transfers successfully across both scales.

\begin{figure}[t]
\centering
\begin{tikzpicture}
\begin{axis}[
  width=0.5\columnwidth, height=5cm,
  xlabel={Training Epochs},
  ylabel={$\gamma$},
  xmin=0, xmax=40, ymin=0.4995, ymax=0.504,
  grid=major, grid style={dashed,gray!30},
  xtick={0, 10, 20, 30, 40},
  ytick={0.4995, 0.500, 0.501, 0.502, 0.503, 0.504},
  scaled x ticks=false,
  scaled y ticks=false,
  y tick label style={/pgf/number format/fixed, /pgf/number format/precision=3},
  legend style={
    at={(0.5,-0.28)},
    anchor=north,
    font=\small,
    legend columns=2,
    column sep=1em,
    draw=none,
  },
]
\addplot[teal, thick, mark=none] coordinates {(0,0.5000120997) (1,0.5000420000) (5,0.5008397698) (10,0.5015639663) (15,0.5020936728) (20,0.5024048686) (25,0.5026083589) (30,0.5026869774) (35,0.5027138591) (40,0.5027164221)};
\addlegendentry{82K $\gamma_q$}
\addplot[teal, thick, dashed, mark=none] coordinates {(0,0.5000119805) (1,0.5000403523) (5,0.5008971691) (10,0.5016829967) (15,0.5022608638) (20,0.5026019216) (25,0.5028162003) (30,0.5028995872) (35,0.5029264092) (40,0.5029286742)};
\addlegendentry{82K $\gamma_d$}
\addplot[orange, thick, mark=none] coordinates {(0,0.5000149012) (1,0.5000154376) (5,0.5000154376) (10,0.5000154376) (15,0.5000154376) (20,0.5000154376) (25,0.5000154376) (30,0.5000154376) (35,0.5000154376) (40,0.5000154376)};
\addlegendentry{500K $\gamma_q$}
\addplot[orange, thick, dashed, mark=none] coordinates {(0,0.5000158548) (1,0.5000165105) (5,0.5000165105) (10,0.5000165105) (15,0.5000165105) (20,0.5000165105) (25,0.5000165105) (30,0.5000165105) (35,0.5000165105) (40,0.5000165105)};
\addlegendentry{500K $\gamma_d$}
\draw[gray, dashed, thin] (axis cs:0,0.5) -- (axis cs:40,0.5);
\node[gray, font=\tiny] at (axis cs:20,0.4998) {Init=0.5};
\end{axis}
\end{tikzpicture}
\caption{Comparison of learned normalization strengths $\gamma_q, \gamma_d$ for Qwen3-Base trained on 82K vs 500K data. With 82K samples, $\gamma$ drifts toward Cosine ($\gamma \approx 0.503$) within 20 epochs. With 500K samples, $\gamma$ remains stagnant at initialization ($\gamma \approx 0.5000$) throughout 40 epochs, indicating the learnable parameters fail to update. The 500K drift magnitude ($\Delta\gamma < 10^{-5}$) is $\sim$5000$\times$ smaller than 82K ($\Delta\gamma \approx 0.003$).}
\label{fig:qwen_500k_gamma}
\end{figure}
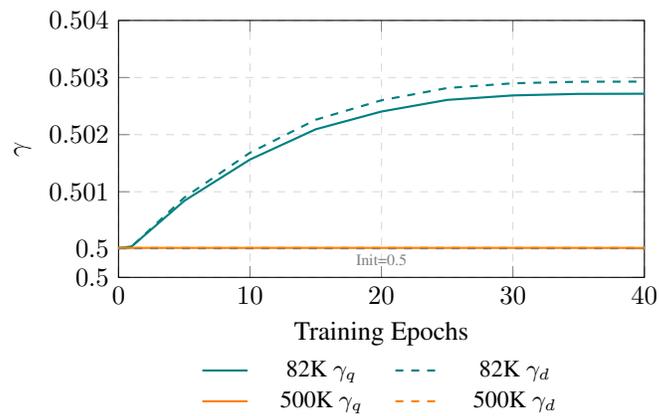


\newpage
\section*{NeurIPS Paper Checklist}

The checklist is designed to encourage best practices for responsible machine learning research, addressing issues of reproducibility, transparency, research ethics, and societal impact. Do not remove the checklist: {\bf The papers not including the checklist will be desk rejected.} The checklist should follow the references and follow the (optional) supplemental material.  The checklist does NOT count towards the page
limit. 

Please read the checklist guidelines carefully for information on how to answer these questions. For each question in the checklist:
\begin{itemize}
    \item You should answer \answerYes{}, \answerNo{}, or \answerNA{}.
    \item \answerNA{} means either that the question is Not Applicable for that particular paper or the relevant information is Not Available.
    \item Please provide a short (1--2 sentence) justification right after your answer (even for \answerNA). 
\end{itemize}

{\bf The checklist answers are an integral part of your paper submission.} They are visible to the reviewers, area chairs, senior area chairs, and ethics reviewers. You will also be asked to include it (after eventual revisions) with the final version of your paper, and its final version will be published with the paper.

The reviewers of your paper will be asked to use the checklist as one of the factors in their evaluation. While \answerYes{} is generally preferable to \answerNo{}, it is perfectly acceptable to answer \answerNo{} provided a proper justification is given (e.g., error bars are not reported because it would be too computationally expensive'' or ``we were unable to find the license for the dataset we used''). In general, answering \answerNo{} or \answerNA{} is not grounds for rejection. While the questions are phrased in a binary way, we acknowledge that the true answer is often more nuanced, so please just use your best judgment and write a justification to elaborate. All supporting evidence can appear either in the main paper or the supplemental material, provided in appendix. If you answer \answerYes{} to a question, in the justification please point to the section(s) where related material for the question can be found.

IMPORTANT, please:
\begin{itemize}
    \item {\bf Delete this instruction block, but keep the section heading ``NeurIPS Paper Checklist"},
    \item  {\bf Keep the checklist subsection headings, questions/answers and guidelines below.}
    \item {\bf Do not modify the questions and only use the provided macros for your answers}.
\end{itemize}


\begin{enumerate}

\item {\bf Claims}
    \item[] Question: Do the main claims made in the abstract and introduction accurately reflect the paper's contributions and scope?
    \item[] Answer: \answerTODO{} 
    \item[] Justification: \justificationTODO{}

\item {\bf Limitations}
    \item[] Question: Does the paper discuss the limitations of the work performed by the authors?
    \item[] Answer: \answerTODO{} 
    \item[] Justification: \justificationTODO{}

\item {\bf Theory assumptions and proofs}
    \item[] Question: For each theoretical result, does the paper provide the full set of assumptions and a complete (and correct) proof?
    \item[] Answer: \answerTODO{} 
    \item[] Justification: \justificationTODO{}

    \item {\bf Experimental result reproducibility}
    \item[] Question: Does the paper fully disclose all the information needed to reproduce the main experimental results of the paper to the extent that it affects the main claims and/or conclusions of the paper (regardless of whether the code and data are provided or not)?
    \item[] Answer: \answerTODO{} 
    \item[] Justification: \justificationTODO{}

\item {\bf Open access to data and code}
    \item[] Question: Does the paper provide open access to the data and code, with sufficient instructions to faithfully reproduce the main experimental results, as described in supplemental material?
    \item[] Answer: \answerTODO{} 
    \item[] Justification: \justificationTODO{}

\item {\bf Experimental setting/details}
    \item[] Question: Does the paper specify all the training and test details (e.g., data splits, hyperparameters, how they were chosen, type of optimizer) necessary to understand the results?
    \item[] Answer: \answerTODO{} 
    \item[] Justification: \justificationTODO{}

\item {\bf Experiment statistical significance}
    \item[] Question: Does the paper report error bars suitably and correctly defined or other appropriate information about the statistical significance of the experiments?
    \item[] Answer: \answerTODO{} 
    \item[] Justification: \justificationTODO{}

\item {\bf Experiments compute resources}
    \item[] Question: For each experiment, does the paper provide sufficient information on the computer resources (type of compute workers, memory, time of execution) needed to reproduce the experiments?
    \item[] Answer: \answerTODO{} 
    \item[] Justification: \justificationTODO{}
    
\item {\bf Code of ethics}
    \item[] Question: Does the research conducted in the paper conform, in every respect, with the NeurIPS Code of Ethics \url{https://neurips.cc/public/EthicsGuidelines}?
    \item[] Answer: \answerTODO{} 
    \item[] Justification: \justificationTODO{}

\item {\bf Broader impacts}
    \item[] Question: Does the paper discuss both potential positive societal impacts and negative societal impacts of the work performed?
    \item[] Answer: \answerTODO{} 
    \item[] Justification: \justificationTODO{}
    
\item {\bf Safeguards}
    \item[] Question: Does the paper describe safeguards that have been put in place for responsible release of data or models that have a high risk for misuse (e.g., pre-trained language models, image generators, or scraped datasets)?
    \item[] Answer: \answerTODO{} 
    \item[] Justification: \justificationTODO{}

\item {\bf Licenses for existing assets}
    \item[] Question: Are the creators or original owners of assets (e.g., code, data, models), used in the paper, properly credited and are the license and terms of use explicitly mentioned and properly respected?
    \item[] Answer: \answerTODO{} 
    \item[] Justification: \justificationTODO{}

\item {\bf New assets}
    \item[] Question: Are new assets introduced in the paper well documented and is the documentation provided alongside the assets?
    \item[] Answer: \answerTODO{} 
    \item[] Justification: \justificationTODO{}

\item {\bf Crowdsourcing and research with human subjects}
    \item[] Question: For crowdsourcing experiments and research with human subjects, does the paper include the full text of instructions given to participants and screenshots, if applicable, as well as details about compensation (if any)? 
    \item[] Answer: \answerTODO{} 
    \item[] Justification: \justificationTODO{}

\item {\bf Institutional review board (IRB) approvals or equivalent for research with human subjects}
    \item[] Question: Does the paper describe potential risks incurred by study participants, whether such risks were disclosed to the subjects, and whether Institutional Review Board (IRB) approvals (or an equivalent approval/review based on the requirements of your country or institution) were obtained?
    \item[] Answer: \answerTODO{} 
    \item[] Justification: \justificationTODO{}

\item {\bf Declaration of LLM usage}
    \item[] Question: Does the paper describe the usage of LLMs if it is an important, original, or non-standard component of the core methods in this research? Note that if the LLM is used only for writing, editing, or formatting purposes and does \emph{not} impact the core methodology, scientific rigor, or originality of the research, declaration is not required.
    \item[] Answer: \answerTODO{} 
    \item[] Justification: \justificationTODO{}

\end{enumerate}

\end{document}